\documentclass[twoside]{article}

\usepackage[accepted]{aistats2024}
\usepackage[subtle, lists=tight]{savetrees}

\usepackage{parskip}        
\usepackage{mleftright}
\usepackage{multirow}
\usepackage[utf8]{inputenc} 
\usepackage[T1]{fontenc}    
\usepackage{hyperref}       
\usepackage{url}            
\usepackage{booktabs}       
\usepackage{amsfonts}       
\usepackage{nicefrac}       
\usepackage{microtype}      
\usepackage{xcolor}         
\usepackage{cuted}          

\usepackage{amsmath,amsfonts,bm}
\usepackage{xcolor}
\usepackage[normalem]{ulem}

\usepackage{graphicx}
\usepackage{caption}
\usepackage{subcaption}

\usepackage{algorithm,algorithmic}

\usepackage{amsthm}

\usepackage{footmisc}

\theoremstyle{plain}
\newtheorem{theorem}{Theorem}[section]

\newtheorem{lemma}[theorem]{Lemma}
\newtheorem{corollary}[theorem]{Corollary}
\theoremstyle{definition}

\newtheorem{assumption}[theorem]{Assumption}
\theoremstyle{remark}

\def\eqref#1{(\ref{#1})}

\def\1{\bm{1}}

\def\eps{{\epsilon}}

\def\vzero{{\bm{0}}}

\def\va{{\bm{a}}}
\def\vb{{\bm{b}}}
\def\vc{{\bm{c}}}
\def\vd{{\bm{d}}}
\def\ve{{\bm{e}}}
\def\vf{{\bm{f}}}
\def\vg{{\bm{g}}}
\def\vh{{\bm{h}}}

\def\vr{{\bm{r}}}

\def\vu{{\bm{u}}}
\def\vv{{\bm{v}}}
\def\vw{{\bm{w}}}
\def\vx{{\bm{x}}}
\def\vy{{\bm{y}}}
\def\vz{{\bm{z}}}

\def\mA{{\bm{A}}}
\def\mB{{\bm{B}}}

\def\mD{{\bm{D}}}

\def\mG{{\bm{G}}}

\def\mI{{\bm{I}}}

\def\mQ{{\bm{Q}}}
\def\mR{{\bm{R}}}

\def\mV{{\bm{V}}}
\def\mW{{\bm{W}}}
\def\mX{{\bm{X}}}

\DeclareMathAlphabet{\mathsfit}{\encodingdefault}{\sfdefault}{m}{sl}
\SetMathAlphabet{\mathsfit}{bold}{\encodingdefault}{\sfdefault}{bx}{n}

\def\gA{{\mathcal{A}}}
\def\gB{{\mathcal{B}}}

\def\gD{{\mathcal{D}}}
\def\gE{{\mathcal{E}}}
\def\gF{{\mathcal{F}}}
\def\gG{{\mathcal{G}}}
\def\gH{{\mathcal{H}}}

\def\gL{{\mathcal{L}}}

\def\gN{{\mathcal{N}}}

\def\gR{{\mathcal{R}}}
\def\gS{{\mathcal{S}}}

\def\gV{{\mathcal{V}}}
\def\gW{{\mathcal{W}}}

\def\sE{{\mathbb{E}}}

\def\sN{{\mathbb{N}}}

\def\sP{{\mathbb{P}}}

\def\sR{{\mathbb{R}}}

\newcommand{\R}{\mathbb{R}}

\usepackage[round]{natbib}

\begin{document}
\runningtitle{The effect of Leaky ReLUs on the training and generalization of overparameterized networks}
\twocolumn[

\aistatstitle{The effect of Leaky ReLUs on the training \\and generalization of overparameterized networks}

\aistatsauthor{Yinglong Guo \And Shaohan Li \And  Gilad Lerman}

\aistatsaddress{ School of Mathematics\\
  University of Minnesota\\
  Minneapolis, MN 55455 \And School of Mathematics\\
  University of Minnesota\\
  Minneapolis, MN 55455 \And School of Mathematics\\
  University of Minnesota\\
  Minneapolis, MN 55455} ]

\begin{abstract}
We investigate the training and generalization errors of overparameterized neural networks (NNs) with a wide class of leaky rectified linear unit (ReLU) functions. More specifically, we carefully upper bound both the convergence rate of the training error and the generalization error of such NNs and investigate the dependence of these bounds on the Leaky ReLU parameter, $\alpha$. We show that $\alpha =-1$, which corresponds to the absolute value activation function, is optimal for the training error bound. Furthermore, in special settings, it is also optimal for the generalization error bound. Numerical experiments empirically support the practical choices guided by the theory. 
\end{abstract}

\section{INTRODUCTION}\label{sec:intro}

Deep neural networks (DNNs) have demonstrated remarkable success in diverse fields, including image classification and text recognition. Despite their achievements, a comprehensive understanding of these networks remains elusive. Theoretical justifications for their performance have primarily centered around the overparameterized setting and mainly considered a rectified linear unit (ReLU). This paper aims to extend and generalize insights gained from recent theoretical works to any Leaky ReLU and provide practical guidance on selecting the most suitable Leaky ReLU for overparameterized networks. By doing so, we offer valuable insights for practitioners seeking optimal performance in real-world scenarios.

To address our aim, we begin by reviewing two recent theoretical trends. The first centers around a fundamental convergence theory for the training error of overparameterized neural networks (NNs). 
Its pioneering work by \citet{jacot2018neural} studied the training dynamics using the neural tangent kernel and showed that the training error goes to zero in the asymptotic regime where the width of the layers goes to infinity. 
A more reasonable regime assumes a sufficiently large lower bound on the width. In such overparameterized regime,
\citep{Goodfellow} empirically noticed that the corresponding NNs can avoid local minima and converge to their global optimal solutions. \citep{du2019gradient} proved the convergence of gradient descent (GD) for NNs with smooth and Lipschitz continuous activation functions whose width exponentially depends on the depth of the networks and polynomially depends on the number of samples. For 2-layer NNs with a ReLU,  \citet{li2018learning} proved the convergence of the training error, \citet{oymak2020toward} reduced the width requirement for training convergence, and \citet{song2021subquadratic}  established convergence whenever the width sub-quadratically depends on the number of samples and the activation functions are sufficiently smooth.

For DNNs, it has become common to consider the polynomial regime of overparameterization,  where the NN widths polynomially depend both on the numbers of samples and the depths. 
\citet{allen2019convergence} established the first convergence result for the training error in this polynomial regime, while assuming ReLU activation functions. 
They separately analyzed training by gradient descent and stochastic gradient descent (SGD). 
\citet{zou2019improved} improved the estimates of  \citet{allen2019convergence} by enhancing the lower bound of the gradient. 
\citet{Chen2019HowMO} further improved the polynomial dependence of the width on the number of samples that was established in  \citet{zou2019improved}, but on the other hand, their polynomial dependence on the depth is worse. 
\citet{banerjee2023neural} showed that for smooth activation functions a linear dependence of the width on the number of samples is sufficient to guarantee convergence.

Another recent progress involves bounding the generalization error of overparameterized NNs.
\citet{chizat2020implicit} 
established a generalization bound of 
infinitely wide two-layer NNs with homogeneous activation functions for classification and showed that the probability of the misclassification bound goes to $0$ as the size of the training samples increases. 
\citet{arora2019fine} bounded the generalization error of 2-layer overparameterized NNs for classification. They also analyzed 
the class of 
functions that are learnable by two-layer NNs.
\citet{allen2019learning} studied the generalization error of two-layer and three-layer NN with a non-negative, convex, and 1-Lipschitz smooth loss function using stochastic gradient descent. They showed that overparameterization improves generalization. 
\citet{cao2020generalization} further established the generalization error of deep NNs for classification using gradient descent. 
\citet{zhu2022generalization} extended the latter work for classification by using some other activation functions, including leaky ReLU with $\alpha\in(0, 1)$.

However, these foundational and important works  have not yet provided much practical guidance for designing NNs. 
Practitioners often use variants of ReLU for activation and this work aims to provide guidance on their choices.
Leaky ReLU is widely used in DNNs for supervised learning tasks \citep{redmon2016you, ridnik2021tresnet} and for generative tasks \citep{radford2015unsupervised, chen2016infogan, karras2019style, wang2021towards}. It is represented by the function $\sigma_\alpha(x)$, where $\sigma_\alpha(x) = x$ for $x > 0$ and $\sigma_\alpha(x) = \alpha x$ for $x \leq 0$, with $\alpha$ being a parameter. ReLU is a special case of Leaky ReLU when $\alpha=0$.
The Leaky ReLU function aims to prevent zero gradients for negative inputs, thus avoiding neurons from not activating. Empirical studies have demonstrated the advantage of using Leaky ReLU with small $\alpha > 0$ over ReLU \citep{xu2015empirical}. 
However, theoretical studies have primarily focused on ReLU and have not directly established the convergence theory and generalization for regression when using Leaky ReLU with any $\alpha < 1$. 
Moreover, the optimal choice of the Leaky ReLU parameter $\alpha$ to expedite the training process and enhance generalization remains unclear. Therefore, a theoretical study is needed to analyze the efficacy of leaky ReLU during training and to provide guidance on selecting the parameter $\alpha$ in practice.

This paper studies overparameterized DNNs with a wide class of leaky ReLU activation functions and develops theories for the convergence of the training error and the upper bound of the generalization error. It builds on  the proof framework and techniques introduced in previous studies, in particular, the ones of  \citet{allen2019convergence},  \citet{zou2019improved} and \citet{cao2020generalization}, but establishes the dependence of the convergence rate and the generalization error on the leaky ReLU parameter $\alpha$. It reveals that the optimal convergence rate bound is achieved at $\alpha=-1$ and the optimal bound of the generalization error is achieved at $\alpha=-1$ using small training epochs as long as the NN is sufficiently deep and the dataset is sufficiently large. This means that activation by the absolute value function may outperform activation by ReLU and the commonly used leaky ReLU (with small $\alpha>0$) in terms of faster training convergence and smaller generalization error. We are not aware of any prior use of the absolute value function for activating DNNs. We are only aware of using it for activating the scattering network    
\citep{scattering_Mallat_2012}  
due to its help with ``energy preservation'' \citep{Bruna_scattering2013}.

The main contributions are as follows:
\begin{enumerate}
\item We establish the convergence of the training errors in overparameterized NNs with any leaky ReLU using both GD and SGD. 
Our estimates clarify the effect of the Leaky ReLU parameter $\alpha$ on the network and its convergence rate bound. In particular, $\alpha=-1$, 
yields the optimal convergence rate bound. 

\item We upper bound the generalization error for overparameterized NNs for regression with leaky ReLUs. 
For sufficiently large datasets, deep NNs and small training epochs, the bound is optimal at $\alpha=-1$. 

\item
We improve previous results for ReLU  (see \S\ref{sec:technical_contribution}). In particular, we show that deep NNs achieve a similar convergence rate as a shallow NN.

\item 
Our predictions receive substantial support from a comprehensive set of numerical experiments
\end{enumerate}

The rest of the paper unfolds as follows:
\S\ref{sec:problem_setup} details the assumed setup of the NNs and the training algorithms; \S\ref{sec:main_res} presents the main theorems; \S\ref{sec:idea_proof} describes our technical contributions and sketches the proof of the main theorems; \S\ref{sec:numeric_exp} provides extensive numerical tests supporting our predictions from the theory on synthetic and real datasets; and \S\ref{sec:conclusion} concludes this work and discusses its limitations. 

\section{PROBLEM SETUP}\label{sec:problem_setup}

We follow the model of \citet{allen2019convergence}, while allowing a wide class of Leaky ReLU activation functions. We consider a dataset $\{\vx_i, \vy_i\}_{i=1}^n$, where $\vx_i\in \sR^p$, $\|\vx_i\| = 1$, $\vy_i\in \sR^d$, $\|\vy_i\|\leq O(1)$ and $d < O(1)$. We focus on a NN $\gN: \sR^p \to \sR^d$ with  $L$ hidden layers having $m$ neurons each and linear input and output layers. 
Its input layer produces
$\vh_0 = \mA \vx, \ \text{ where }, \mA\in \sR^{m\times p}$.   
For $l\in[L] := \{1,2,\dots L\}$, the output of the $l$th hidden layer, $\vh_l$,  
is inductively defined by 
\begin{equation}
\vh_l = \gH_l(\vh_{l-1}) = \sigma_\alpha(\mW_l \vh_{l-1}),
\label{eqn:hidden_layer}
\end{equation}
where $\mW_l\in \sR^{m\times m}$ and $\sigma_\alpha$ is the leaky ReLU activation function with $\alpha <1$: 
$$
\sigma_\alpha(x) = \left\{\begin{array}{cc}
x, &  \ \text{if}\  x \geq 0;\\
\alpha x, & \ \text{if}\   x < 0.
\end{array}\right. 
$$ 
The output layer produces 
$\hat{\vy} = \mB \vh_L, \text{ where } \mB\in \sR^{d\times m}$.  
Let $\mW := (\mW_1, \mW_2, \dots \mW_L)$ store all the trainable parameters
and we thus compactly denote $\hat{\vy} = \gN(\vx;\mW)$. For simplicity, we initialize $\mA$ and $\mB$ (see below), so they are fixed, and only train $\mW_{l},\ l \in [L]$.

We train the NN using the mean squared error (MSE):
    $\gL(\mW) =  \sum_{i=1}^n \|\vy_i - \gN(\vx_i;\mW)\|^2$. We denote its gradient by
$\nabla_\mW \gL(\mW) := (\nabla_{\mW_1} \gL(\mW), \dots \nabla_{\mW_L} \gL(\mW))$.  
Appendix \ref{subsec:append:exponential_loss} extends our theory to many other useful loss functions.   
We assume a specified upper bound $\eps>0$ on the training error and express our estimates in terms of this bound.

When discussing generalization, we assume that the set $\{\vx_i\}_{i=1}^n$ is i.i.d.~drawn from an arbitrary distribution $\gD_\mX$
and that for $1 \leq i \leq n$, $\vy_i=F(\vx_i)$ for an arbitrary measurable function $F$. The generalization error is thus 
$R(\mW) := \sE_{\vx\sim \gD_{\mX}} \|F(\vx) - \gN(\vx; \mW)\|^2$.

We assume the following data separation property: 
\begin{assumption}
\label{assump:separation}
There exists $0<\delta < c_0$, where  $c_0<1$, so that 
    $\min_{i, j \in [n]} \|\vx_i - \vx_j\| \geq \delta > 0$.
\end{assumption}
This assumption, suggested by \citet{allen2019convergence}, is reasonable. Indeed, if, on the other hand, there exists $i \neq j \in [n]$ such that $\vx_i=\vx_j$, then 
we can assume $\vy_i \neq \vy_i$ (otherwise we can combine these multiple instances into one single data point). It is then impossible to obtain a zero training error, which is needed for our convergence study.

\begin{algorithm}[h!]
   \caption{Rescaled initialization}
   \label{alg:rescale_initial}
\begin{algorithmic}
   \STATE {\bfseries Input:} Input dimension  $p$, width of hidden layer $m$, output dimension $d$, and leaky ReLU parameter $\alpha$.
   \STATE {\bfseries Initialize:}
      \vspace{-0.15in}
   \begin{align*}
&       \mA  \sim N\left(0, \frac{1}{m}\right), \  \mB \sim N\left(0, \frac{1}{d}\right),\\  
& \mW_l^{(0)}  \sim N\left(0, \frac{2}{m}\right)  ,\  l \in [L]   
    \end{align*}
   \vspace{-0.15in}
   \STATE {\bfseries Activation function:}
   \vspace{-0.1in}
\begin{equation}
   \tilde{\sigma}_\alpha(x) = \left\{\begin{array}{cc}
   \frac{1}{\sqrt{1+\alpha^2}} x, & \ \text{ if } x\geq 0\\
   \frac{\alpha}{\sqrt{1+\alpha^2}} x, & \ \text{ if } x < 0\\
\end{array}
   \right.\label{eqn:leaky_relu_rescaled}
\end{equation}
\end{algorithmic}
\end{algorithm}

Following \citet{he2015delving}, we initialize the network parameters as follows: $\mA\sim N(0, 1/m)$, $\mB\sim N(0, 1/d)$ and $\mW_l^{(0)}\sim N(0, 2/(m (1+\alpha^2)))$ for $l\in[L]$. 
Note that the factor $1/(1+\alpha^2)$
ensures a constant variance for any choice of $\alpha$. 
We can move the factor $1/(1+\alpha^2)$ from the weight initialization to the activation function, and equivalently initialize with Algorithm~\ref{alg:rescale_initial}. The theoretical study of the latter formulation with its rescaled Leaky ReLU function, $\tilde{\sigma}_\alpha(x)$ (see \eqref{eqn:leaky_relu_rescaled}), turns out to be more tractable.

Algorithms~\ref{alg:training_gd} and~\ref{alg:training_sgd} formulate the training procedures with simple GD and SGD, respectively. 

\begin{algorithm}[!ht]
   \caption{Training (gradient descent)}
   \label{alg:training_gd}
\begin{algorithmic}
   \STATE {\bfseries Input:} Learning rate $\eta$.
   \STATE {\bfseries Initialize:} Apply Algorithm~\ref{alg:rescale_initial} to obtain $\mA, \mB$ and $\mW^{(0)}$
   \vskip -0.1in
   \FOR{$t=0$ {\bfseries to} $T$}
   \STATE
   \vspace{-0.2in}
   \begin{align*}
        \mW^{(t+1)} = \mW^{(t)} - \eta \nabla_{\mW}\gL(\mW^{(t)}).
    \end{align*}
   \vspace{-0.2in}
    \ENDFOR
\end{algorithmic}
\end{algorithm}

\begin{algorithm}[!ht]
   \caption{Training (stochastic gradient descent)}
   \label{alg:training_sgd}
\begin{algorithmic}
   \STATE {\bfseries Input} Learning rate $\eta$.
   \STATE {\bfseries Initialize:} Apply Algorithm~\ref{alg:rescale_initial} to obtain $\mA, \mB$ and $\mW^{(0)}$
      \vskip -0.1in
   \FOR{$t=0$ {\bfseries to} $T$}
   \STATE Randomly select batch $B\subset [n]$ with $|B| = b.$
   \begin{align*}
        \mW^{(t+1)} = \mW^{(t)} - \eta \nabla_{\mW}\gL_B(\mW^{(t)}),
    \end{align*}
    where $\gL_B(\mW^{(t)}) := \sum\limits_{i\in B}\| \vy_i - \gN(\vx_i; \mW^{(t)})\|^2$.
    \ENDFOR
\end{algorithmic}
\end{algorithm}

\section{MAIN RESULTS}\label{sec:main_res}

The two theorems below establish the convergence of the training error for overparameterized NNs using a Leaky ReLU function with $\alpha <1$. The first theorem pertains to training with gradient descent (GD) (Algorithm~\ref{alg:training_gd}), while the second applies to training with stochastic gradient descent (SGD) (Algorithm~\ref{alg:training_sgd}). 
Both theorems are formulated within the context outlined in \S\ref{sec:problem_setup}. This setup includes  Assumption~\ref{assump:separation} with a parameter $\delta$,  Algorithm~\ref{alg:rescale_initial} for the initialization of the parameters of the NN, $n$ training points, $\{\vx_i, \vy_i\}_{i=1}^n$, where $\|\vx_i\| = 1$, and $\|\vy_i\|\leq O(1)$, output dimension $d$ ($\vy_i\in \sR^d$), NN depth $L$, NN width $m$, Leaky ReLU parameter $\alpha$, learning rate $\eta$,  batch size $b$ (for Algorithm \ref{alg:training_sgd}) and a desired upper bound $\eps>0$ on the training error.

\begin{theorem}\label{thm:gd_main_thm}
Assume the setup of \S\ref{sec:problem_setup}, where both
$m/\ln^4 m > \frac{1+\alpha^2}{(1-\alpha)^2}\Omega(\frac{n^5 L^{15} d}{\delta^4})$
and 
$m > \Omega\left(\ln\ln\epsilon^{-1}\right)$, 
and the training is according to Algorithm~\ref{alg:training_gd} with learning rate  $\eta \leq O(\frac{d}{nL^2 m})$. Then, with probability at least $1-e^{-\Omega(\ln m)}$,
\begin{equation}
\gL(\mW^{(T)}) < \epsilon
 \text{ and } \
    \gL(\mW^{(t)}) \leq \gamma^t \gL(\mW^{(0)})
        \text{,}\ \forall t  \leq T,
    \label{eqn:gd_conv} 
\end{equation}
where 
\begin{equation}
\label{eq:gamma_1st_theorem}
\gamma = 1 - \Omega\left(\frac{(1-\alpha)^2}{1+\alpha^2}\frac{\eta\delta m}{nd}\right), \ 
T = \frac{\ln \left(\epsilon / \gL(\mW^{(0)})\right)}{\ln\gamma}.
\end{equation}
\end{theorem}

\begin{theorem}\label{thm:sgd_main_thm}
Assume the setup of \S\ref{sec:problem_setup}, where 
both
$\frac{m}{\ln^4 m} > \frac{(1+\alpha^2)^4}{(1-\alpha)^8} \Omega(\frac{n^{8} L^{15}d}{b\delta^5})$ and $m\ln m > \Omega\left(\ln\ln \epsilon^{-1}\right)$
and the NN is trained according to Algorithm~\ref{alg:training_sgd} with 
$\eta \leq O(\frac{d\delta}{mn^3 L^3\ln^2 m})$ and 
$t > \frac{(1+\alpha^2)^2}{(1-\alpha)^4}\Omega(\frac{n^5L^2}{b\delta^2 }\ln^2 m)$
. There exists a constant $C_0 >1$ such that 
\begin{equation}
\begin{split}
    &\gL(\mW^{(T)}) < \epsilon
 \ \text{ and } \
\gL(\mW^{(t)}) \leq C_0 \gamma^t \gL(\mW^{(0)}) 
\\ &\ \text{ for all } t \leq T 
        \ \text{ with probability } \ 1-e^{-\Omega(\ln m)},
    \label{eqn:sgd_conv}
\end{split}
\end{equation}
where 
\begin{equation}
        \gamma = 1 - \Omega\left(\frac{(1-\alpha)^2}{1+\alpha^2}\frac{\eta b \delta m}{n^2d}\right), \ 
    T = \frac{\ln \left(\epsilon / C_0\gL(\mW^{(0)})\right)}{\ln\gamma}.
    \label{eqn:gamma_2nd_theorem}
\end{equation}
\end{theorem}

These theorems show that for any $\alpha < 1$ the training error linearly converges to zero when the NN width is sufficiently large and the learning rate $\eta$ is sufficiently small.

Moreover, these theorems reveal the dependence of the convergence rate bound on $\alpha$ and this information can guide one in selecting $\alpha$ for optimal training speed.  
We note that 
the typical choice of the leaky ReLU parameter $\alpha$ (e.g., $0.01$ or $0.05$) does not yield a better bound for the convergence speed than ReLU (i.e., $\alpha=0$); furthermore, the negative values of $\alpha$ yield better results than ReLU and the optimal choice of $\alpha$ is $-1$. 
We can prove that this observation is rather general as follows:
\begin{corollary}\label{coro:gd_main_coro}
Assume the setup of \S\ref{sec:problem_setup} with either Algorithms \ref{alg:training_gd} or \ref{alg:training_sgd} and that all parameters are chosen so that when $\alpha =0$, $\gamma <1$. Then $\alpha =-1$ minimizes the above convergence rate $\gamma$ among all $\alpha < 1$. Moreover, $\gamma$ is decreasing in $\alpha$ on  $(-\infty, -1)$ and increasing on $(-1, 1)$.
\end{corollary}

For $\alpha=0$, our result improves the previous analysis of both \citet{allen2019convergence} and \citet{zou2019improved}. 
We compare our bounds with the ones of \citet{zou2019improved}, since they improved the bounds of \citet{allen2019convergence}. 
For this purpose, we examine the difference in the setups. First, \citet{zou2019improved} divides the loss function $\gL(\mW)$ 
by $n$ and thus we need to convert their estimate by a factor of a power of $n$ accordingly. Second, our proof assumes that the hidden signals are separated by $\delta < O(1)$, whereas \citet{zou2019improved} assumes that $\delta < O(1/L)$. We establish this upper bound independently of $L$ with careful mathematical estimates; therefore, our setup eliminates implicit dependence on $L$ in the other formulas. At last, 
\citet{zou2019improved} enforces the initial scaled loss to be bounded by $O(1)$ (this amounts to a bound $O(n)$ on our loss) and their conclusion holds with probability at least $1-\Omega(1/n)$. On the other hand, we relax the initial unscaled loss to be bounded by $O(\sqrt{\ln m})$ and our conclusion holds with probability at least $1-e^{-\Omega(ln m)}$, which we find more natural for the overparameterized regime.

After converting to our setup, the convergence rate in \citet{zou2019improved} is $1- \Omega(\eta\delta m/(dnL))$ when using gradient descent, and our convergence rate improves to 
$1- \Omega(\eta\delta m/(dn))$; also, when using SGD the convergence rate in \citet{zou2019improved} is $1- \Omega(\eta\delta mb/(dn^2L))$ and we improve it to $1- \Omega(\eta\delta mb/(dn^2))$. The important finding is that in the overparameterized regime, a deeper NN does not lead to slower convergence, but rather achieves a similar convergence rate as a shallow NN. 
One can further note that we improved the bound of 
\citet{zou2019improved} on $m$ 
by the factor $n^{-3} L^{-1}$ for GD and 
$n^{-8}L^{-2} (n/b)^{-3}\delta^3$ for SGD.
Furthermore, our lower bound on the number of epochs $t$ in Theorem~\ref{thm:sgd_main_thm} improves the one of \citet{allen2019convergence} by a factor of order $n^{-2}L^{-2}$, where there is no explicit bound in \citet{zou2019improved}.


Appendix~\ref{subsec:append:exponential_loss} extends the above bounds to convex loss functions, which include the cross-entropy for classification and a special loss function proposed in \cite{kumar2023stochastic}. The convergence rate for these functions is different, but $\alpha=-1$ is still optimal for their bounds.

Next, we establish an upper bound of the generalization error of a NN trained using GD, where an analogous bound when using SGD is specified in Theorem~\ref{thm:generalization_in_training_sgd} in Appendix~\ref{appd:generalization_sgd}. We first follow the previous analysis of generalization in overparameterized NNs by \cite{cao2020generalization} and establish the corresponding bound for our setting with Leaky ReLU activation function.

\begin{theorem}
\label{thm:generalization_in_training_gd}
Assume the setup of \S\ref{sec:problem_setup} with GD, where $m = \Theta(\frac{n^{10+2\tau}L^{15 + 2\tau}d^{1+2\tau}}{ \delta^{4 - 2\tau}})$ for $\tau>0$ 
and $\eta = \Theta(\frac{d}{nL^2m})$. Assume further that $m$ is larger than its lower bound and $\eta$ is smaller than its upper bound in Theorem~\ref{thm:gd_main_thm} (by an appropriate choice of the hidden constants in $\Theta$ and compared to the constants hidden in the lower bound of $m$ and in the upper bound of $\eta$ in Theorem~\ref{thm:gd_main_thm}). 
Then at a given training epoch $t \leq T$ (see \eqref{eq:gamma_1st_theorem} for $T$), with probability at least $1-e^{-\Omega(\ln m)}$, the generalization error is bounded as follows 
\begin{multline}
         R(\mW^{(t)}) \leq \gamma^t \gL(\mW^{(0)}) 
+\min\left\{O\left(\frac{d^{3/2+\tau} \delta^\tau n^{1/2+\tau}}{L^{1/2 - \tau}\ln m}\right),\right. \\ 
\left.O\left(\frac{1-\alpha}{\sqrt{1+\alpha^2}}\frac{d^{1/3}t^{4/3}}{m^{1/6}n^{2/3}L^{2/3}}\right)\right\} +
\min\left\{O\left(\frac{\sqrt{d \ln m}~t }{nL}\right), \right. \\ \left.O\left(\frac{n^{1/2+\tau} L^{2+\tau} d^{1/2+\tau}}{\delta^{1/2 -\tau} \ln m}\right)\right\} + O\left(d\sqrt{\frac{\ln m}{n}}\right).
\label{eqn:generalization_error_gd}
\end{multline}
\end{theorem}
\vskip -0.1in

In Appendix~\ref{append:sec:discussion_generalization_bound}, we clarify the above estimates for different regimes for the number of training epochs, $t$. 
In particular, we indicate a tradeoff between the first training term and 
the other NN-complexity terms (excluding the last term of data complexity) and show that we cannot make both of these kinds of terms sufficiently small.
Stopping at a sufficiently small number of epochs results in a bound of the generalization error of order $O(\ln(m))$, which is also of order $O(\ln(n))$. This bound is composed of several terms. The term which contributes $O(\ln(m))$ is due to the training error and one cannot expect a better bound for it when having a small number of epochs. The rest of the terms do converge when $n$ and $L$ are sufficiently large and in this latter regime the overall bound is minimized when $\alpha =-1$. On the other hand, for larger numbers of epochs overfitting is observed, which results in divergent generalization error. 
Exploring the dependence of the generalization bound on $t$ is advantageous to an epoch-independent bound, like the one pursued by \citet{cao2020generalization} for classification instead of regression. Indeed, the bound of \citet{cao2020generalization} is  $\Theta(\text{poly}(n)\cdot n^{-1/2})$, which is significantly larger than $O(\log(n))$. 

For very special datasets (e.g., single-layer ReLU NN separability) \citet{cao2020generalization} reduced the term $\text{poly}(n)$ so their overall bound is sufficiently small. A natural, but more complicated, extension of this to regression is to consider datasets well-approximated by $L$-layer leaky ReLU NNs. In Appendix~\ref{subsec:appendix:special_func_class}, we improve the convergence rate, the lower bound of $m$ (so its dependence on $n$ is linear) and the generalization error bound for such datasets. However, for a large number of epochs we still notice overfitting with divergent generalization error (with a smaller rate of increase to infinity than for general datasets). 

At last, 
\citet{kumar2023stochastic}  
claimed that when using the loss function discussed in \eqref{eqn:exp_loss} of Appendix~\ref{subsec:append:exponential_loss}, minimizing a particular generalization error bound is equivalent to minimizing the latter loss function for training. 
Therefore, if $\alpha=-1$ is optimal for the training error, then it is also optimal for the generalization error bound. Since we verified the optimality of $\alpha=-1$ for our upper bound of the convergence rate in Appendix~\ref{subsec:append:exponential_loss} 
and experimentally demonstrated instances where this bound is comparable to the actual convergence rate in Figure~\ref{fig:alphas}, we get some numerical evidence that for the latter instances $\alpha=-1$ is optimal for bounding the generalization error.

\section{IDEAS OF PROOFS}\label{sec:idea_proof}
Our proofs follow ideas of \citet{allen2019convergence},  \citet{zou2019improved} and \citet{cao2020generalization} and adapt them to the general case of Leaky ReLU with $\alpha<1$. It also adapts \citet{cao2020generalization} to regression.  
We first sketch in \S\ref{sec:proof_sketch} the basic ideas of our proof, while we supplement all details in the appendix. We then highlight some of the innovative ideas in \S\ref{sec:technical_contribution}.

\subsection{Proof Sketch}
\label{sec:proof_sketch}
We describe here a quick roadmap to verifying the theory. The proofs of Theorems~\ref{thm:gd_main_thm} and \ref{thm:sgd_main_thm} follow the initial framework of \citet{allen2019convergence}, which was later followed by \citet{zou2019improved}, but consider the effect of using any leaky RELU with $\alpha <1$. 

These proofs use the following two lemmas, which are proved in \S\ref{appd:semismooth} and \S\ref{appd:gradient}. 
Let us first clarify their notation. We denote by $\|\mX\|_2$ and $\|\mX\|$ the spectral and Frobenius norms of 
a matrix $\mX$.
For $\mW = (\mW_1\dots \mW_L)$ and $\mV=(\mV_1\dots \mV_L)$, we define 
$\|(\mW_1\dots \mW_L)\|_F^2:=\sum_{l\in [L]}\|\mW_l\|_F^2$, 
$\|(\mW_1\dots \mW_L)\|_2:=\max_{l\in [L]}\|\mW_l\|_2$ and $\langle\mW, \mV\rangle := \sum_{l\in [L]} \langle\mW_l, \mV_l\rangle$. We denote by $\mW'$ a perturbation of $\mW$.
\begin{lemma}[Semi-smoothness]\label{thm:semi-smooth}
Assume the setup of \S\ref{sec:problem_setup}. If $\|\mW - \mW^{(0)}\|_2 < \omega < O\left(\frac{1}{L^{9/2}\ln^{3/2}m}\right)$ and $\| \mW' \|_2 <\omega$, then with a probability at least $1-e^{-\Omega(m)}$
\begin{align}
    &\gL(\mW + \mW')  \leq \gL(\mW) + \langle\nabla_\mW \gL(\mW), \mW'\rangle \notag\\
    & \ + \frac{nL^2m}{d} O(\|\mW'\|_2^2)\notag\\
    & \ + \frac{(1-\alpha)\omega^{1/3}L^2\sqrt{m n \gL(\mW) \ln m}}{\sqrt{d(1+\alpha^2)}}O(\|\mW'\|_2) .\label{eqn:semi-smooth}
\end{align}

\end{lemma}

\begin{lemma}[Gradient bounds]\label{thm:gradient_bound}
Assume the setup of \S\ref{sec:problem_setup}. If $\|\mW - \mW^{(0)}\|_2 < \omega <  O\left(\frac{\delta^{3/2}}{n^{3/2}L^{15/2}\ln^{3/2} m}\right)$, then with a probability at least $1-e^{-\Omega(m\delta^2/L^3)}$
\begin{align}
    \|\nabla_{\mW_l} \gL(\mW)\|_F^2 & \leq \gL(\mW)O\left(\frac{mn}{d}\right) , \quad \text{ for }\ l \in [L]\label{eqn:upper_bnd}\\
    \|\nabla_{\mW} \gL(\mW)\|_F^2 & \geq \gL(\mW)\Omega\left(\frac{(1-\alpha)^2}{(1+\alpha^2)}\frac{\delta m}{n d}\right).
    \label{eqn:lower_bnd}
\end{align}
\end{lemma}

We note that the factor $({1-\alpha})/{\sqrt{1+\alpha^2}}$ appears in the bounds of both lemmas, where it is squared in Lemma \ref{thm:gradient_bound}. This factor is the derivative gap in Leaky ReLU, i.e., $\tilde{\sigma}'_\alpha(0+) - \tilde{\sigma}'_\alpha(0-)$. Its value is larger for Leaky ReLU with $\alpha<0$ than for ReLU (with $\alpha=0$). 
We thus note the bound \eqref{eqn:semi-smooth} in Lemma \ref{thm:semi-smooth} is larger for Leaky ReLU with $\alpha<0$ than for ReLU. 
On the other hand, observing \eqref{eqn:lower_bnd} of Lemma \ref{thm:gradient_bound}, we note that the lower bound on the norm of the gradient is larger for Leaky ReLU with $\alpha <0$ than for ReLU.  
Our analysis below shows that when combining the two bounds, Leaky ReLU with $\alpha<0$ leads to better control of the decay of the loss function than ReLU.

Theorem~\ref{thm:gd_main_thm} can be proved as follows. Let $\mW : = \mW^{(t)}$ and $\mW':= -\eta\nabla_{\mW}\gL(\mW^{(t)})$ and note that by gradient descent, $\mW + \mW'= \mW^{(t+1)}$. Denoting $\gL^{(t)} := \gL(\mW^{(t)})$ and applying \eqref{eqn:semi-smooth} of Lemma \ref{thm:semi-smooth}, we can conclude that with a probability of at least $1-e^{-\Omega(m)}$, the following inequality holds
\begin{align}
    & \gL^{(t+1)} 
    \leq \gL^{(t)} - \eta \langle \nabla_{\mW} \gL^{(t)} , \nabla_{\mW} \gL^{(t)}\rangle \label{eqn:proof1} \\ 
     & \ + \frac{\eta(1-\alpha)\omega^{\frac{1}{3}}L^2\sqrt{m n \gL^{(t)} \ln m}}{\sqrt{d(1+\alpha^2)}}
          O\left(\|\nabla_\mW \gL^{(t)}\|_2\right)\label{eqn:proof2}\\
     & \ + \frac{\eta^2nL^2m}{d}O\left(\|\nabla_\mW \gL^{(t)}\|_2^2\right). \label{eqn:proof3}
\end{align}
Using \eqref{eqn:lower_bnd} we bound $\sqrt{\gL^{(t)}}$ as follows with probability at least $1-e^{-\Omega(m\delta^2/L^3)}$: 
\begin{equation}
\label{eq:bound_sqrt_L}
    \sqrt{\gL^{(t)}} \leq \frac{\sqrt{1+\alpha^2}}{1-\alpha} \, O\left(\sqrt{\frac{nd}{\delta m}}\right) \|\nabla_\mW \gL^{(t)}\|_F.
\end{equation}
Applying \eqref{eq:bound_sqrt_L},  
we control the term in \eqref{eqn:proof2}, with probability at least $1-e^{-\Omega(m\delta^2/L^3)}$, by
\begin{equation}
    \frac{\eta \omega^{1/3}nL^{2}\sqrt{\ln m}}{\sqrt{\delta}} O\left(\|\nabla_\mW \gL^{(t)}\|_F^2\right).\label{eqn:proof2_bd}
\end{equation}
Using 
$\omega < O(\frac{\delta^{3/2}}{n^{3/2}L^{15/2}\ln^{3/2} m})$, which is required by  Lemma~\ref{thm:gradient_bound},  we reduce \eqref{eqn:proof2_bd} to $\eta \|\nabla_\mW \gL^{(t)}\|_F^2/3$. Using $\eta < O({d}/{(nL^2 m)})$, which is required in Theorem \ref{thm:gd_main_thm}, we reduce the bound in \eqref{eqn:proof3} to $ \eta \|\nabla_\mW \gL^{(t)}\|_F^2/3$.

Next, we apply these bounds to the respective terms in \eqref{eqn:proof1} and use the identity
$ \langle \mX, \mX \rangle=\|\mX\|_F^2$ for a vector of matrices $\mX=(\mX_1,\ldots,\mX_L)$ to reduce \eqref{eqn:proof1} to 
\begin{equation}
    \gL^{(t+1)} \leq \gL^{(t)} - 1/3\eta \|\nabla_\mW \gL^{(t)}\|_F^2.\label{eqn:training_error_dynamic}
\end{equation}
Further application of the lower bound in \eqref{eqn:lower_bnd} to the above equation results in $\gL^{(t+1)} \leq \gamma \gL^{(t)}$
with $\gamma$ specified in \eqref{eq:gamma_1st_theorem} and we consequently conclude \eqref{eqn:gd_conv} of Theorem \ref{thm:gd_main_thm}.

The above argument holds for one training step with probability at least $1-e^{-\Omega(m)}$. This argument extends to $T$ steps with probability at least $1-Te^{-\Omega(m)}$. We note that the number of epochs $T$ can be bounded using the bound $\eps$ on the training error, the convergence rate in \eqref{eq:gamma_1st_theorem} and the estimate
$\gL(\mW^{(0)}) \leq O(n \sqrt{\ln m})$, which is shown in Appendix~\ref{appd:main_proof_gd}, as follows:
\begin{equation*}
    \begin{split}
    T &= \ln(\epsilon/\gL(\mW^{(0)}))/\ln\gamma\leq\Theta(\ln(\epsilon/n\sqrt{\ln m})/\ln\gamma) \\
    & \leq O\left(\frac{nd}{\eta\delta m} (\ln \epsilon^{-1} + \ln (n\sqrt{\ln m}))\right).
\end{split}
\end{equation*}
\vskip -0.1in

Thus the total probability to ensure $T$-steps training with training error lower than $\epsilon$ is at least $1-O(\frac{nd}{\eta\delta m}(\ln \epsilon^{-1} + \ln (n\sqrt{\ln m}))) e^{-\Omega(m)}$. Given that $m > \Omega(\text{poly}(n,L,d,\delta^{-1}))$ and $m > \Omega(\ln \ln \epsilon^{-1})$, this probability is of order $1-e^{-\Omega(m)}$.

In Appendix~\ref{appd:main_proof_gd}, we demonstrate that the inequality  $\|\mW^{(t)}-\mW^{(0)}\|_2 < \omega < O(\delta^{3/2}/(n^{3/2}L^{15/2}\ln^{3/2}m))$ holds
with probability at least $1-e^{-\Omega(\ln m)}$. Note that the latter bound implies the conditions for both Lemmas~\ref{thm:semi-smooth} and~\ref{thm:gradient_bound} and thus concludes the proof of Theorem \ref{thm:gd_main_thm}

The proof of Theorem~\ref{thm:sgd_main_thm} is detailed in \S\ref{appd:main_proof_sgd}. We briefly describe the proof idea as follows.  
First, we use a similar argument as in the proof of Theorem~\ref{thm:gd_main_thm} to bound the expectations of the loss functions at each step. Second, we use \eqref{eqn:upper_bnd} to find an absolute upper bound of the loss functions. By combining these two bounds and using Azuma's inequality, we derive the decay of the loss function in \eqref{eqn:sgd_conv} with the convergence rate in \eqref{eqn:gamma_2nd_theorem} in Theorem~\ref{thm:sgd_main_thm}. Finally, we verify that the conditions for Lemma~\ref{thm:semi-smooth} and Lemma~\ref{thm:gradient_bound} are satisfied when the NN width satisfies $m/\ln^4 m > (1+\alpha^2)^4/(1-\alpha)^8 \Omega(n^{8} L^{15}d/(b\delta^5))$  and thus conclude the theorem.

The proof of Theorem \ref{thm:generalization_in_training_gd}, which appears in \S\ref{appd:generalization_gd}, relies on the following lemma that bounds the generalization error for a class of NNs whose parameters are close to $\mW^{(0)}$.
\begin{lemma}[Generalization error with perturbation]\label{thm:generalization_overall}
Assume the setup of \S\ref{sec:problem_setup}, where $\alpha$ is the leaky ReLU parameter. 
If $ \|\mW - \mW^{(0)}\| < \omega < O\left(\frac{\delta^{3/2}}{n^{3/2}L^{15/2}\ln^{3/2} m}\right)$, then with probability at least $1-e^{-\Omega(\ln m)}$ 
\begin{equation*}
\begin{split}
    R(\mW) &\leq \frac{1}{n}\gL(\mW) + \frac{1-\alpha}{\sqrt{1+\alpha^2}}O(d (\ln m) \sqrt{m} L^2 \omega^{4/3}) +  \\
    & O(d  \sqrt{m(\ln m)/n} L \omega) + O\left(d \sqrt{\frac{\ln m}{n}}\right).\end{split}
\end{equation*}
\end{lemma}
\vskip -0.1in
The proof of Lemma~\ref{thm:generalization_overall}, which appears in Appendix~\ref{appd:generatlization_proof}, follows  similar ideas as those of \citet{cao2020generalization} 
but adapted to the different task of regression. 
Theorem~\ref{thm:generalization_in_training_gd} is a consequence of this lemma and two different estimates of the size of $\omega$ during training.  
The first estimate controls $\omega$ during the entire training with GD, regardless of how large the training epoch is, and is expressed in Lemma~\ref{lemma:training_in_perturbation}.  
The second estimate uses direct bounds of the learning steps and provides a better upper bound of $\omega$ when the training epoch is small. 

\subsection{Discussion of Innovation}\label{sec:technical_contribution}

While we followed, extended and improved an existing proof framework, we would like to emphasize some innovation in our proof techniques. 
To begin with, it is difficult to directly extend the previous methods to any leaky ReLU with $\alpha<1$. 
Our idea of rescaling the leaky ReLU activation function, along with the observation that, with rescaled initialization, it is equivalent to using the unscaled leaky ReLU, helped tremendously simplify our initial technical and complex effort. This allowed us to elegantly use the previous ideas
    and further improve them.
Nevertheless, we have made various notable improvements to previous estimates. 
In particular, we improved the lower bound for the gradient established by \citet{zou2019improved} by a factor of $L$. We also eliminated the previous dependence of the convergence rate on a negative power of $L$, which was undesirable as it implied that deeper networks might experience slower convergence. This demonstrates that the convergence rate of deep neural networks is at least comparable to that of shallow neural networks. 
Specifically, the later estimates can be found in the proof of  Lemma~\ref{thm:gradient_bound} in Appendix~\ref{appd:gradient}. 
They are motivated by a suggestion from \citet{allen2019convergence} to incorporate gradients from all layers' parameters, departing from previous estimates that solely relied on the gradients of parameters from the last layer. 
More specifically, improved lower bounds for the gradients from all layers' parameters can be found in Lemma~\ref{lemma:gradient_lower_bound_init} in Appendix~\ref{appd:gradient}. 
We also obtained a tighter bound for the spectral norm of $\mW^{(t)} - \mW^{(0)}$ when using SGD. This improved the lower bound on the width $m$ for training convergence by a factor of order $n^{-8} L^{-2} (n/b)^{-3} \delta^3$.

\begin{figure*}[ht]
\vskip -0.04in
\begin{center}
\includegraphics[width=0.64\columnwidth]{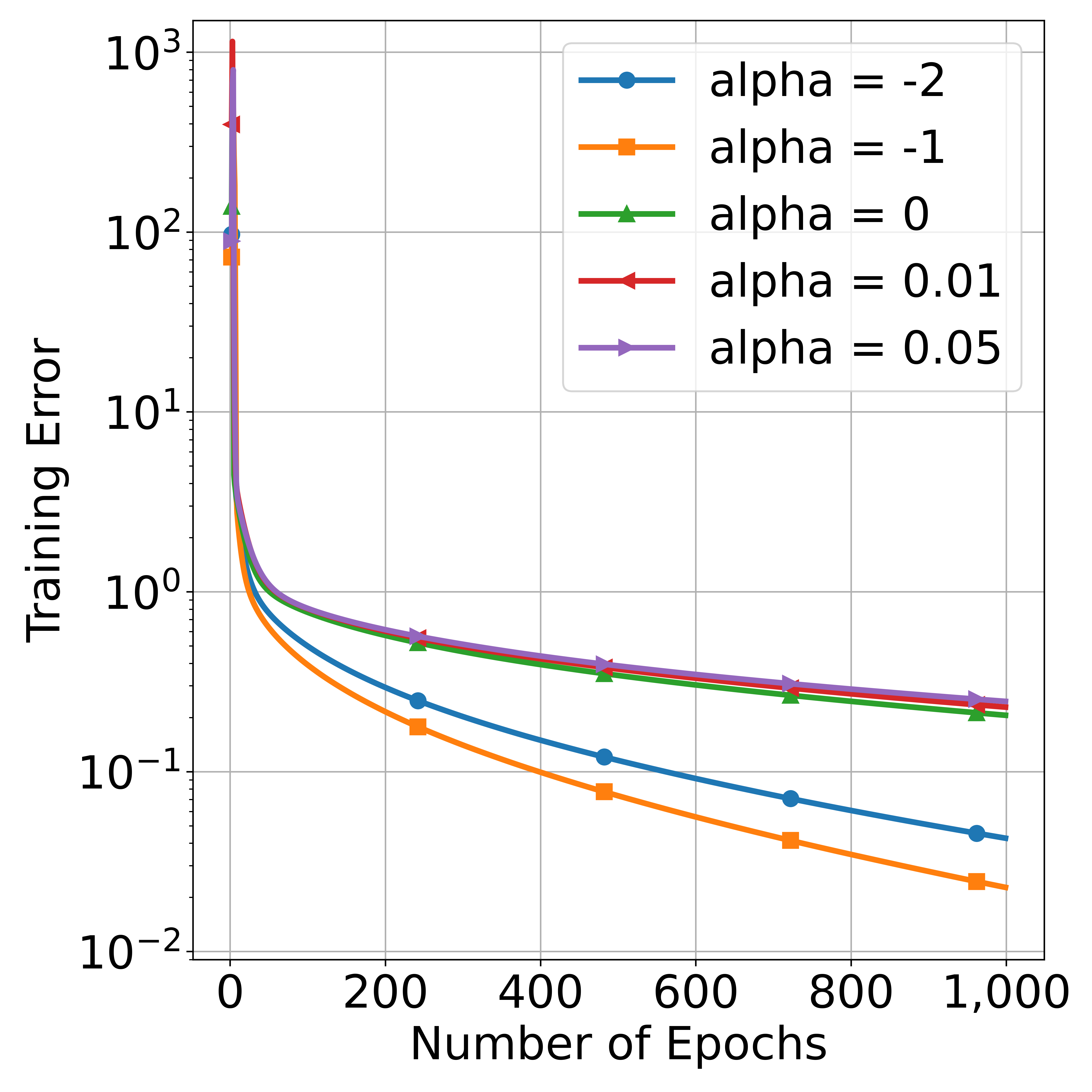}
\includegraphics[width=0.66\columnwidth]{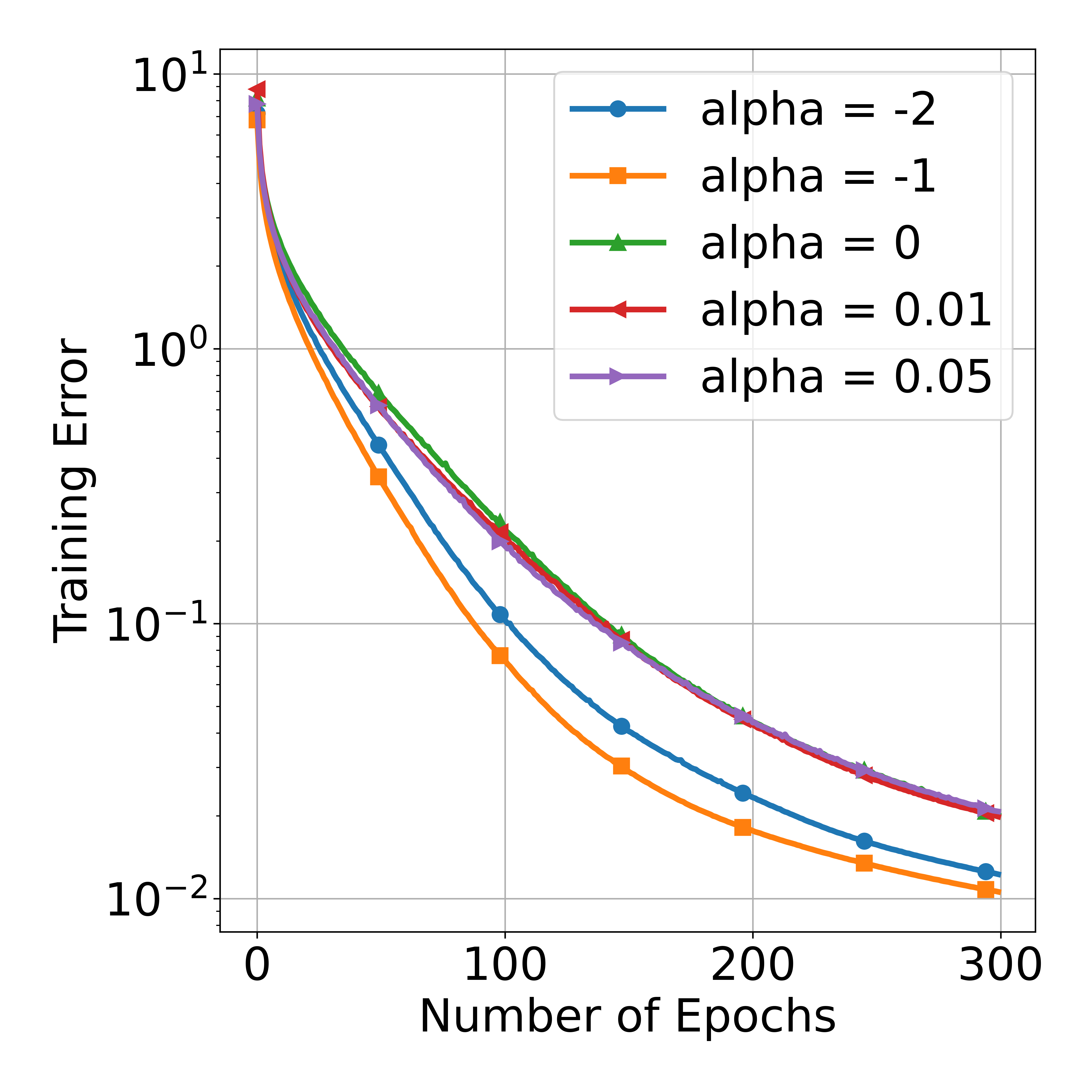}
\includegraphics[width=0.66\columnwidth]{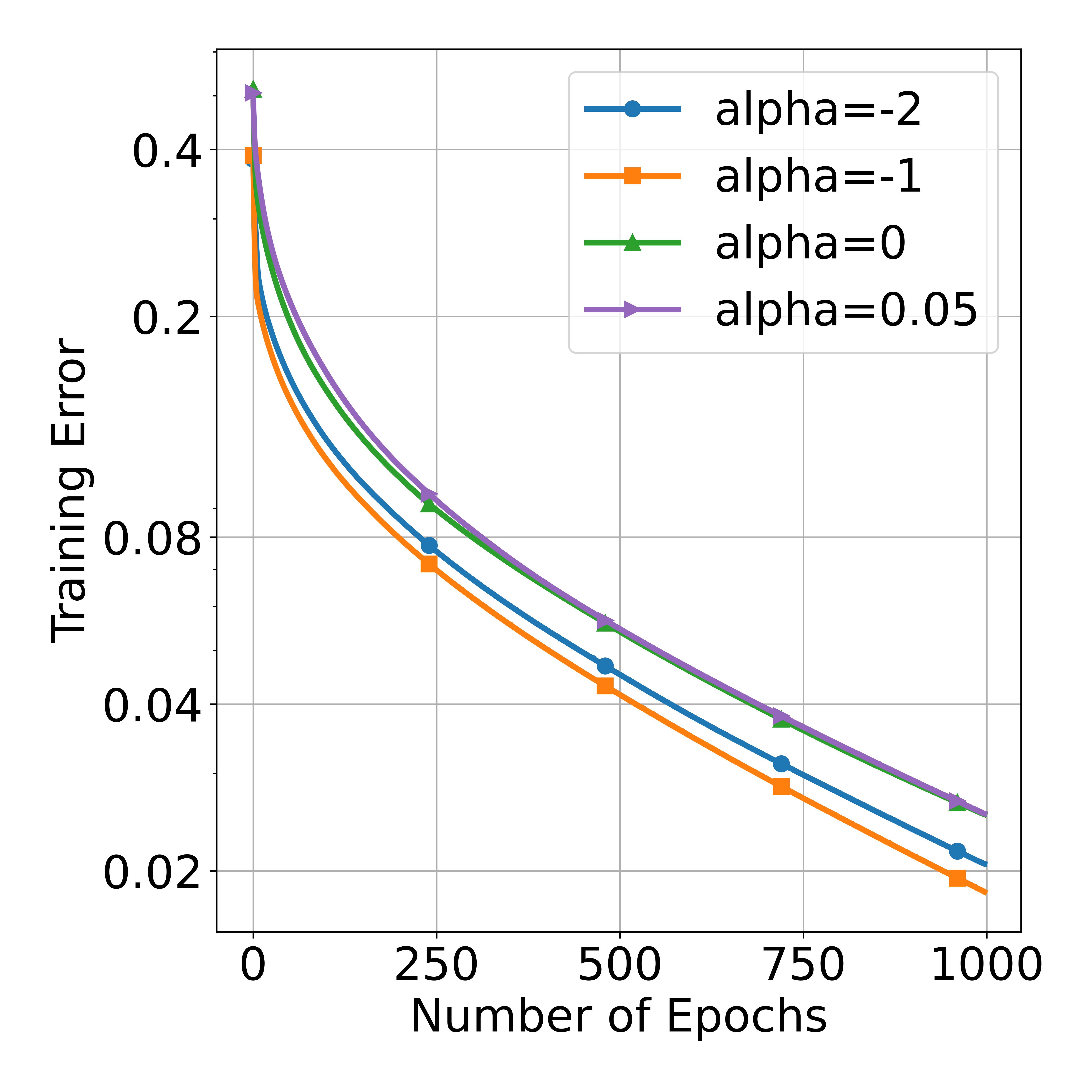}
\vskip -0.1in
\includegraphics[width=0.64\columnwidth]{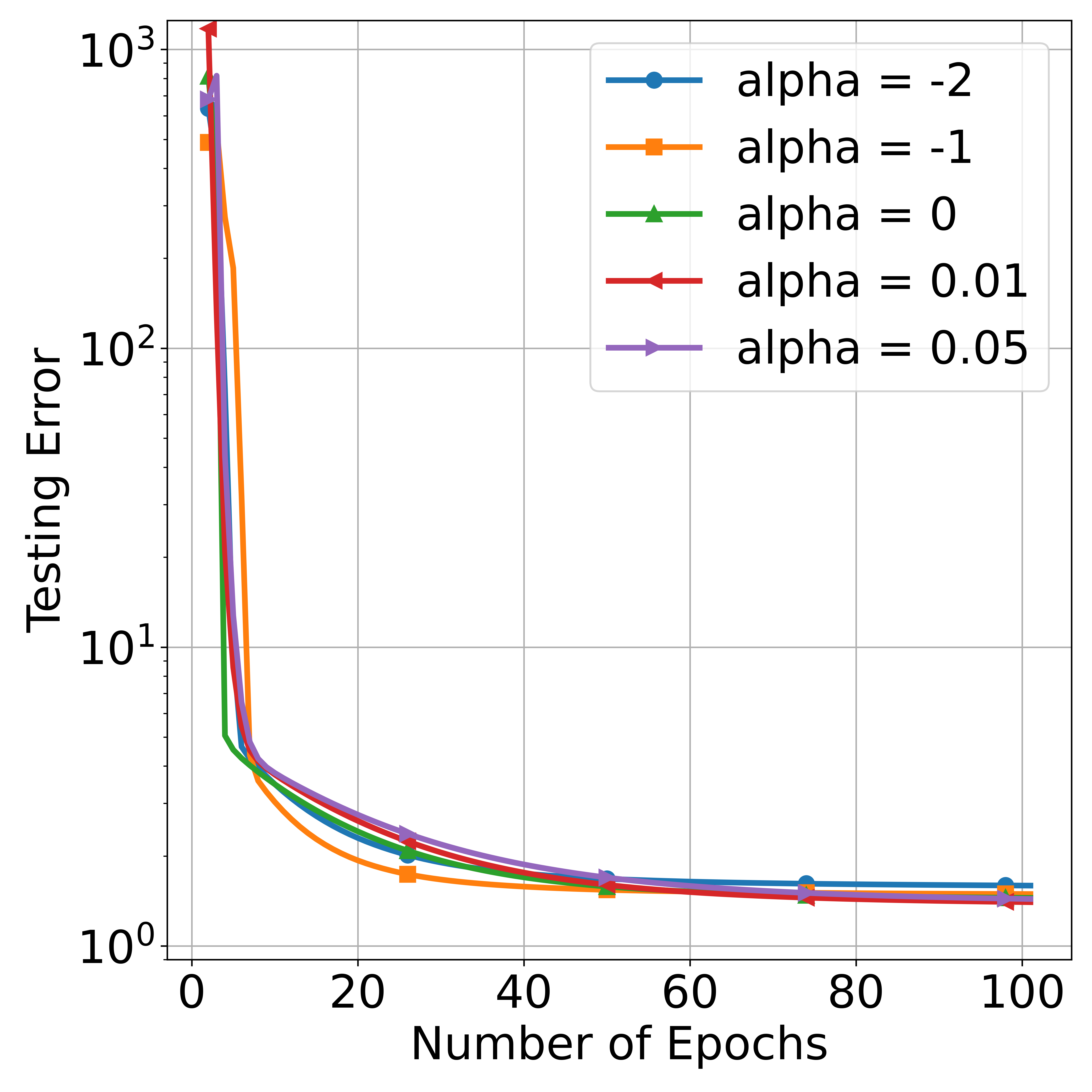}
\includegraphics[width=0.66\columnwidth]{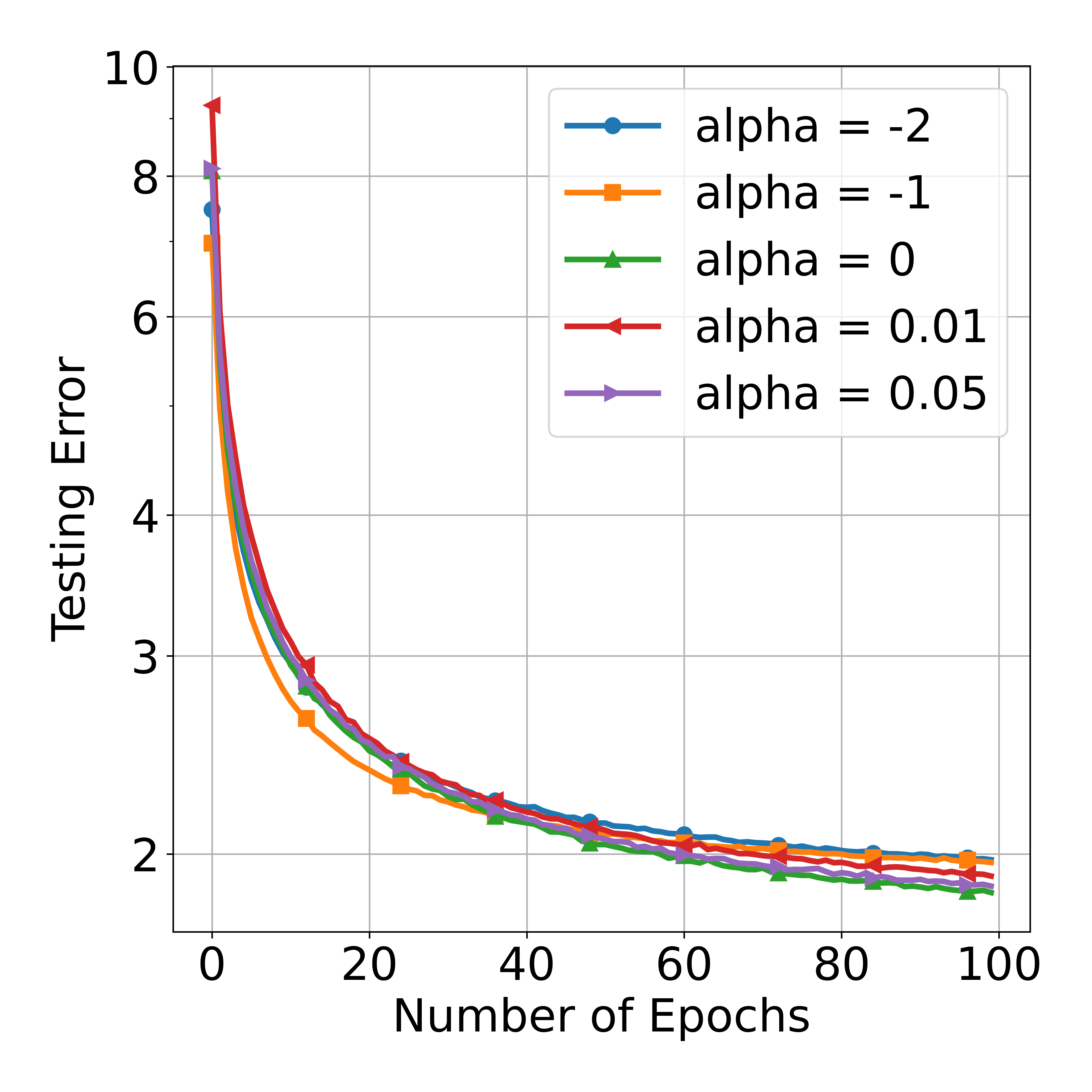}
\includegraphics[width=0.66\columnwidth]{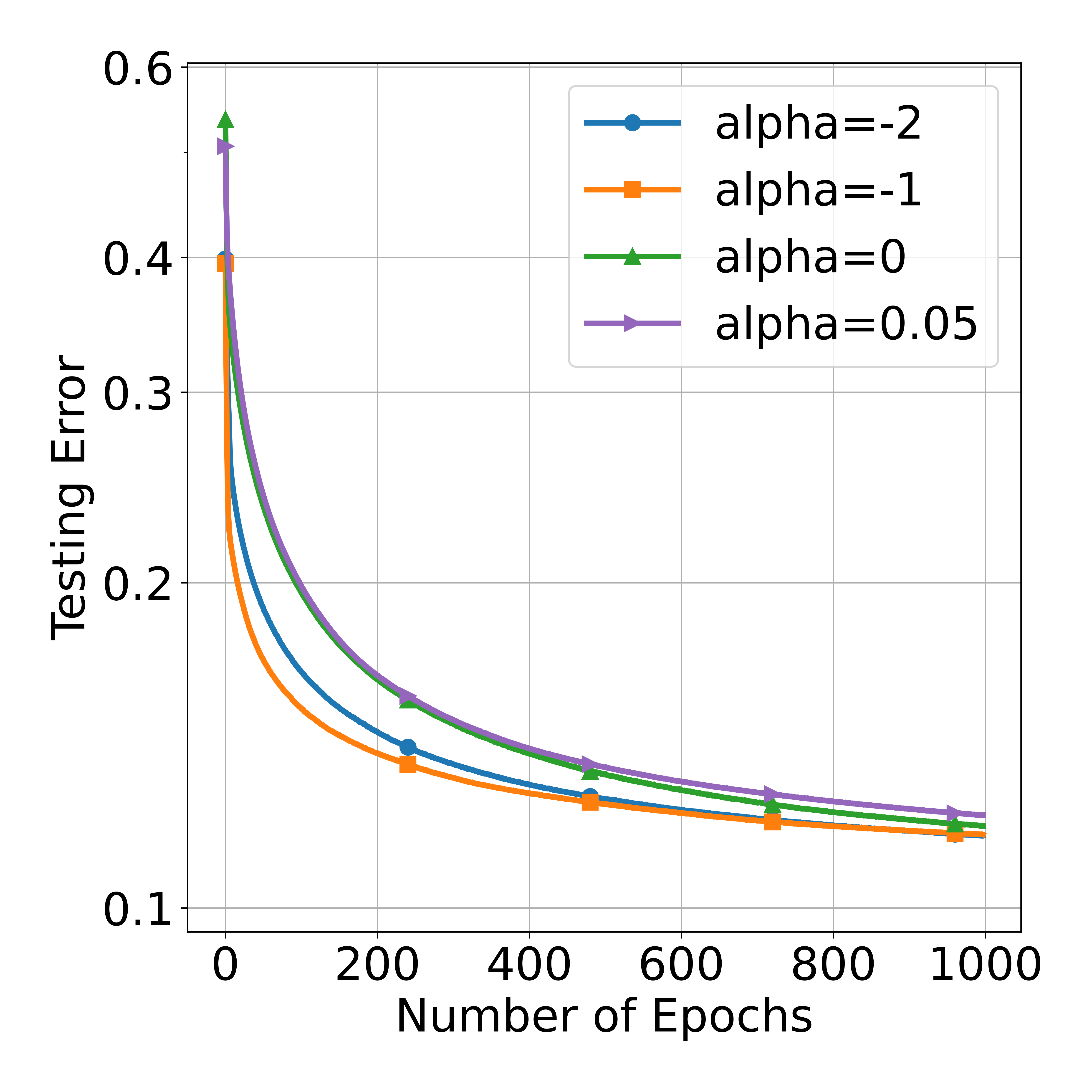}
\vskip -0.1in
\caption{Log-scale training and testing errors using different datasets and different $\alpha$'s. From left to right: synthetic dataset, F-MNIST and CIFAR-10. Top row: training errors. Bottom row: testing errors.}
\label{fig:train_err}
\end{center}
\vskip -0.2in
\end{figure*}

Additionally, a more careful and fresh look helped improve the interpretation of the results. In particular, noting the effect on the number of epochs $t$ on the generalization error, while developing tighter bounds when $t$ was sufficiently small, helped with a meaningful bound on the generalization error. 
Another example includes making all the probabilities dependent on $m$, a choice we deemed more suitable for the overparameterized regime. 
Furthermore, to avoid the hidden dependence of $\delta$ on $L$ in the previous works, we had to develop some careful mathematical estimates (see \eqref{eqn:append:upper_bound_inner_prod_hij} in the appendix), so we could explicitly identify the dependence on $L$ and relax the previous assumption $\delta < O({1}/{L})$ to $\delta < O(1)$.

\section{NUMERICAL EXPERIMENTS}\label{sec:numeric_exp}

As our theory deals only with upper bounds, we conduct numerical experiments to examine the dependence of the actual training convergence rate and generalization error, particularly at an early epoch, on the parameter $\alpha$. Our main goal is to determine whether $\alpha=-1$ is the optimal choice for convergence and generalization in overparameterized NNs with LeakyReLU activation functions. 
Appendix \ref{appd:numeric} provides additional experiments.

\subsection{Setup}\label{sec:numeric_setup}

We summarize our implementation for the following datasets. We provide additional details in \S\ref{appd:architecture_of_nns}.

\textbf{Synthetic dataset:} We simulate a dataset which contains 1,000 data points in $\R^5$ i.i.d.~sampled from a normalized Gaussian distribution, $N(0,\mI_5)$. 
We verified that Assumption~\ref{assump:separation} holds for the generated dataset with $\delta=0.21$. 
We generate real-valued labels, $y$, by the following noisy nonlinear function of $\vx$: 
\begin{align*}
y &= \sin (10x_1 + 20x_2^3) + \cos (3x_3 + 5x_4^2) \\ & \  + \frac{2}{(1+\textrm{ReLU}(0.05 + x_5))^{1/2}} + 2x_1x_5 + \varepsilon,
\end{align*}
where $\varepsilon \sim N(0, 0.01)$.
We construct NNs with five hidden layers, $m=5,000$ and leaky ReLUs with $\alpha\in\{-2,-1,0,0.01,0.05\}$. 
We initialize the NNs by Algorithm~\ref{alg:rescale_initial} and train them with GD using the MSE loss.  

\textbf{F-MNIST:}
This standard grayscale image classification benchmark consists of ten classes
\citep{fmnist}. We build NNs with two hidden layers and width $m=2,000$. We use leaky ReLUs with $\alpha \in \{-2,-1,0,0.01,0.05\}$. We initialize the NNs by  Algorithm~\ref{alg:rescale_initial} and train them using SGD with batch size $64$ and the cross entropy loss. 

\begin{table*}[ht!]
\caption{Training and testing errors for the three main datasets. The first three rows report the training error at the last epoch. The next ones report the testing error at an early epoch ($t=30$ for synthetic, $t=20$ for F-MNIST and $t=200$ for CIFAR-10). 
}
\label{tab:summary_result}
\begin{center}
\begin{small}
\begin{sc}
\begin{tabular}{llcccc}
\toprule
Metric & Dataset&  $\alpha=-2$ & $\alpha=-1$&$\alpha=0$&$\alpha=0.05$
\\
\midrule
\multirow{3}{*}{\begin{tabular}[c]{@{}l@{}}Final training \\error\end{tabular}} & Synthetic &  0.039$\pm 0.002$ & \textbf{0.022}$\pm0.002$ & 0.197$\pm0.013$&  0.245$\pm0.022$ \\
& F-MNIST& 0.096$\pm 0.009$& \textbf{0.076$\pm 0.008$}& 0.211$\pm 0.018$ & 0.229$\pm 0.032$\\
& CIFAR-10&  0.019$\pm 0.001$ & \textbf{0.018$\pm 0.001$}& 0.024$\pm 0.001$ & 0.027$\pm 0.001$ \\
\midrule
\multirow{3}{*}{
\begin{tabular}[c]{@{}l@{}}{Early Epoch} \\ {testing error} \end{tabular}
}& Synthetic & 1.914$\pm0.067$  & \textbf{1.775}$\pm0.065$ & 2.086$\pm0.173$ & 2.313$\pm0.218$ \\
& F-MNIST& 2.371$\pm 0.103$& \textbf{2.362$\pm 0.053$}& 2.442$\pm 0.067$ & 2.470$\pm 0.092$\\
& CIFAR-10& 0.146$\pm 0.004$& \textbf{0.143$\pm 0.005$}& 0.169$\pm 0.012$
& 0.173$\pm 0.007$ \\
\bottomrule\end{tabular}
\end{sc}
\end{small}
\end{center}
\vskip -0.1in
\end{table*}

\textbf{CIFAR-10:}
This is another standard dataset for image classification \citep{cifar10}. 
It consists of ten classes of RGB natural images. We modify the architecture of VGG19 \citep{vgg} with four convolutional layers (width 512) and two linear layers (width 512) using 
Leaky ReLUs with $\alpha \in \{-2,-1,0,0.05\}$. We use Algorithm~\ref{alg:rescale_initial} to initialize the NNs and train them using SGD with batch size $64$ and cross entropy loss.

\subsection{Results}\label{sec:numeric_res}

Figure~\ref{fig:train_err} demonstrates 
both training errors (top) and testing errors (bottom) for the synthetic dataset, F-MNIST and CIFAR-10 (from left to right) for different $\alpha$s. We remark that we use the testing error as an approximation of the generalization error. 
Observing the training errors in the top row we note that the convergence is fastest for the NN with $\alpha = -1$ and the ranking of $\alpha$ from fastest to slowest convergence corresponds to the one predicted by our theory; that is, if $\alpha$ obtains a lower estimate for $\gamma$ in  \eqref{eq:gamma_1st_theorem} than $\alpha'$, then it results in faster convergence in our experiments. 
Observing the testing errors, we note that around a small training epoch (e.g., 30 for the synthetic dataset, 20 for F-MNIST, and 200 for CIFAR-10), the testing error is smallest when $\alpha=-1$. However, at larger training epochs the gaps of the testing errors are small for most of the $\alpha$s.

To get a better quantitative idea, Table \ref{tab:summary_result} summarizes for the different data sets the training error at the last epoch and 
the testing error at an early epoch. 
We ran the experiments 10 times and reported the mean and standard deviations (std's). We note that the std's are small and for better visualization we did not include them in Figure \ref{fig:train_err}.
We observe that choosing $\alpha=-1$ gives the least final training error in all datasets. Compared to ordinary ReLU, our choice of $\alpha=-1$ reduces the final training error by at least $22\%$ (CIFAR-10) and at most $91\%$ (synthetic). At early training epoch, compared to ordinary ReLU, the choice of $\alpha=-1$ reduces the testing error by at least $4\%$ (F-MNIST) and at most $15\%$ (CIFAR-10). This correlates with the predictions we made by our theory that the optimal bounds of the convergence rate and generalization error (at a sufficiently small epoch) are achieved with $\alpha=-1$.

\begin{figure*}[ht]
\vskip -0.04in
\begin{center}
\includegraphics[width=0.99\columnwidth]{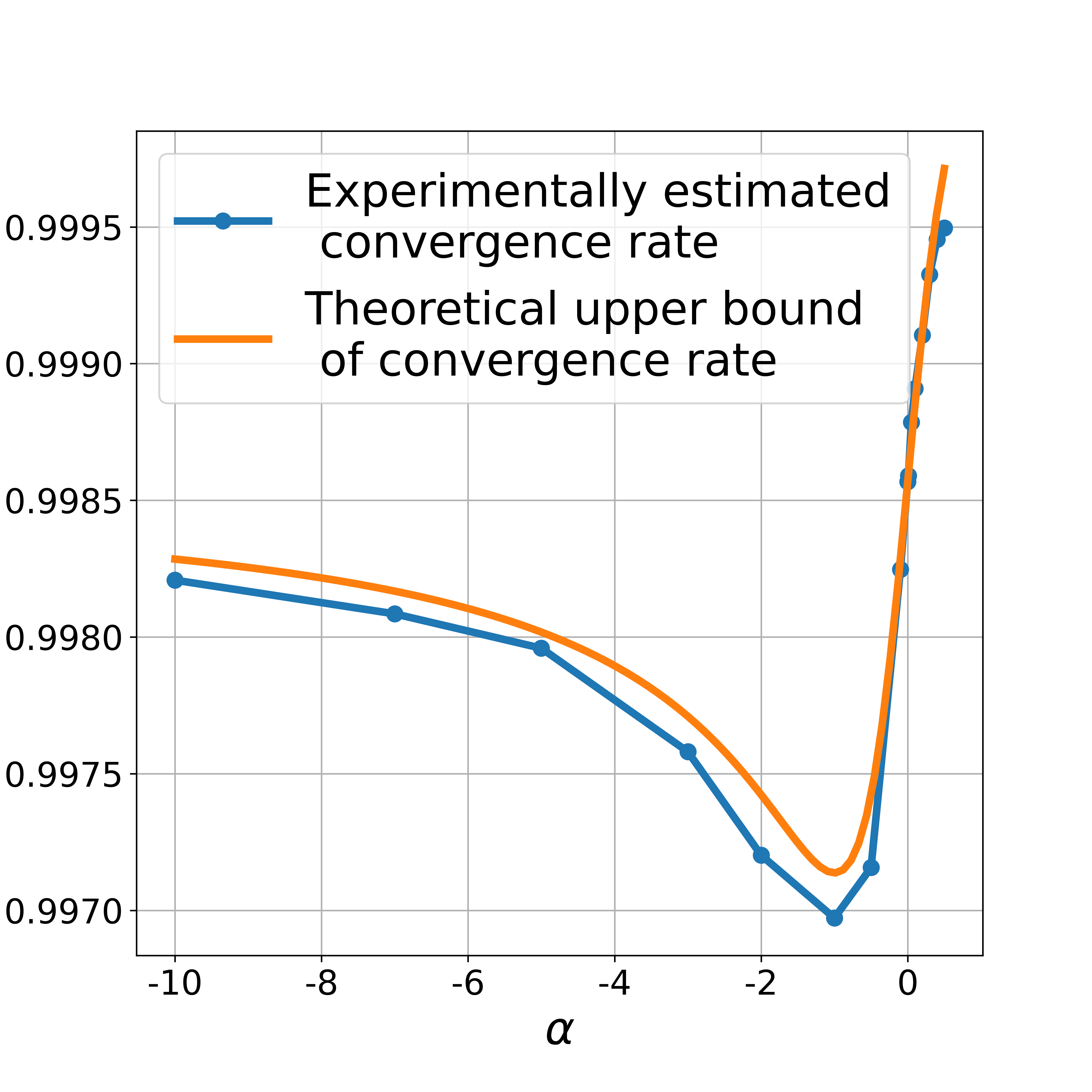}
\includegraphics[width=0.99\columnwidth]{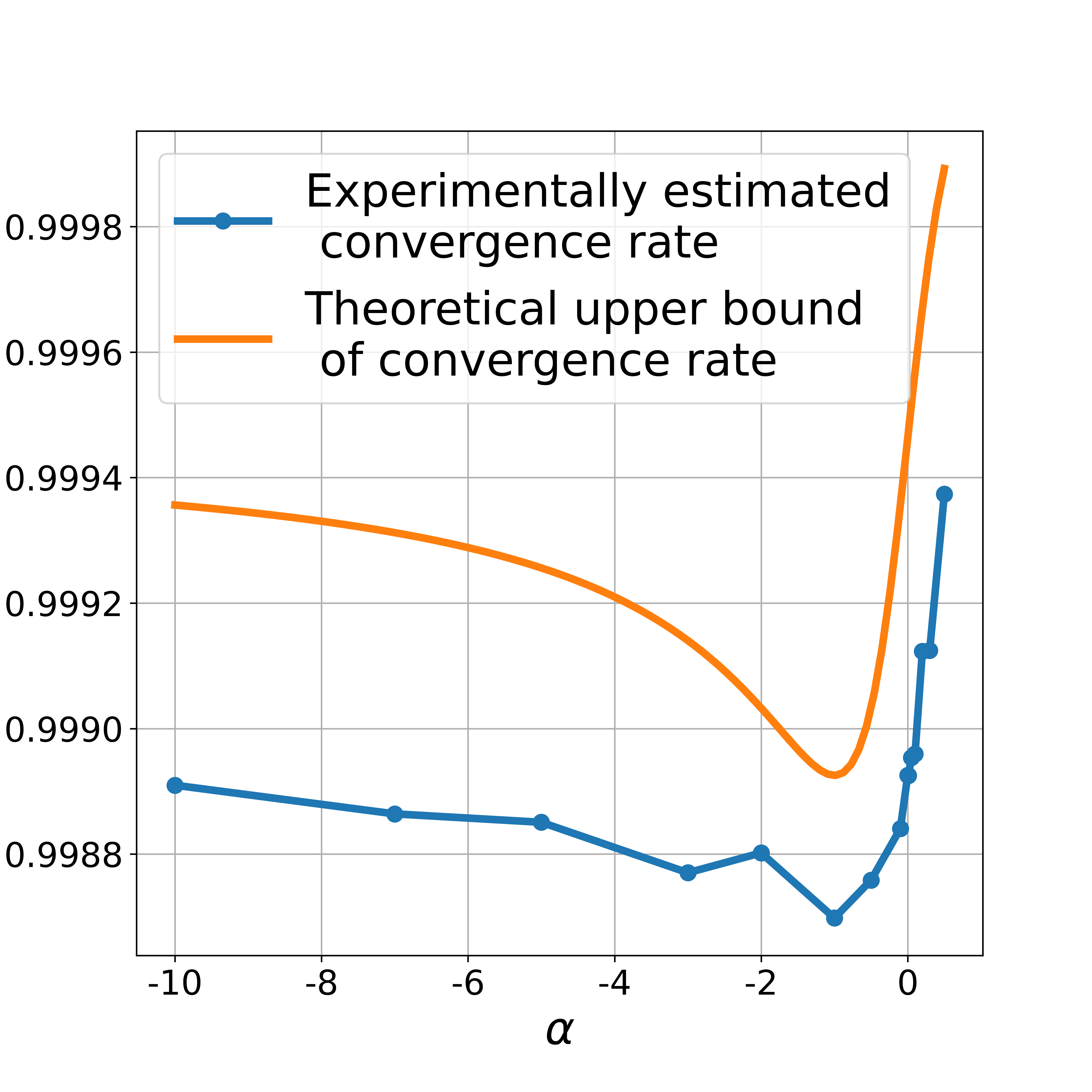}
\caption{Comparison of the ``shape'' of the theoretical upper bound of the training convergence rate  (orange line) with the calculated convergence rate  (blue dots). We used the synthetic dataset (left) and California housing dataset (right) with different values of $\alpha$'s.}
\label{fig:alphas}
\end{center}
\end{figure*}

Lastly, we compare the theoretically predicted upper bounds of the convergence rate and the empirical convergence rates with different $\alpha$s.
For this purpose, we ran experiments using the synthetic dataset and California housing (see its detailed description in Appendix~\ref{appd:architecture_of_nns}) with  choices of $\alpha$ from $[-10, 0.5]$. We approximate the convergence rate for each $\alpha$ using the training errors from the experiments at time steps $100$ (i.e., $\gL^{(100)}$) and $1,000$ (i.e., $\gL^{(1000)}$). The empirical convergence rate is calculated as $$\hat{\gamma}(\alpha) := (\gL^{(1000)}/\gL^{(100)})^{1/900}.$$
To simplify our upper bound, we denote the constant $\Omega\left(\frac{\eta\delta m}{nd} \right)$ in \eqref{eq:gamma_1st_theorem} by $C_\gamma$ and estimate its value based on the calculated convergence rate at $\alpha = 0$ as
\begin{equation}
    \label{eq:cgamma}
C_\gamma:=
C_0(1 - \hat{\gamma}(0)),
\end{equation}
where we choose $C_0=1$ for the synthetic dataset and $C_0=0.5$ for California housing.
Consequently, we obtain our theoretical upper bounds of the convergence rates
\begin{align*}
    \gamma(\alpha) &= 1 - 0.00143\frac{(1-\alpha)^2}{1+\alpha^2} \ \text{ for the synthetic dataset,} \\
    \gamma(\alpha) &= 1 - 0.000537\frac{(1-\alpha)^2}{1+\alpha^2} \ \text{ for California housing.} 
\end{align*}
Figure~\ref{fig:alphas} compares the theoretical upper bound of the convergence rate, $\gamma(\alpha)$, with the experimental convergence rate $\hat{\gamma}(\alpha)$ for the tested values of $\alpha$s.
It is interesting to note that the predicted upper bound dependence on $\alpha$ correlates very well with both numerical experiments.

Appendix~\ref{appd:additional_numerical_results}
includes additional details and numerical results. In particular, it performs experiments similar to the ones reported in Figure~\ref{fig:train_err}, while  incorporating the  datasets MNIST, California housing and 
IMDb movie reviews;  the architectures of  recurrent NNs and transformer NNs; and another loss function for regression.
It also demonstrates how the training and testing errors depend on the NN hyperparameters (e.g., depth and width). 

All codes are available at \url{https://github.com/sli743/leakyReLU}.

\section{DISCUSSION}\label{sec:conclusion}

We established a mathematical theory that clarifies the impact of the Leaky ReLU parameter on bounds of both the training error convergence rate  and the generalization error for overparameterized neural networks. We showed that the absolute value function yields the optimal convergence rate bound for the training error and also the optimal generalization error bound when the training epoch is sufficiently small, with a sufficiently large dataset and a deep NN. Our extensive empirical tests support using the absolute value function for effective training and for effective generalization with sufficiently small epochs and sufficiently large datasets and deep overparameterized NNs.

There are different possible extensions of our theory.  For example, it is useful to extend it to other structured NNs, such as convolutional NNs (CNNs), while allowing any Leaky ReLU. \citet{allen2019convergence} established convergence for overparameterized CNNs with ReLU and one can directly extend their analysis to any Leaky ReLU. Nevertheless, it still remains open to extend the generalization theory to other structured NNs.
Furthermore, it is useful to study the training convergence and generalization for larger classes of activation functions, such as the Gaussian error linear unit \citep{hendrycks2016gelu}.

Our work has three major limitations. First, our generalization error bound is not sufficiently small. Nevertheless, we believe it still indicates some interesting and relevant phenomena, in particular, the behavior when stopping at an early epoch. 
We further improved our estimates for a special class of datasets, although we observed that  it was not sufficiently small in general. This is likely due to the fact that the regression setting poses greater challenges than classification. We also highlighted the possible implications of \citet{kumar2023stochastic} to a generalization estimate given tight training error bounds. 

Second, the lower bound that we require on the width, $m$, is generally unrealistically large and we thus find it important to extend our  theory to lower values of $m$.  
Developing such a theory seems to require a careful analysis of nonlinear dynamical systems, given that current methods aim to linearize the underlying dynamical system. Nevertheless, for the special class of datasets discussed in Appendix~\ref{subsec:appendix:special_func_class}, we were able to provide a satisfying linear dependence of the lower bound of $m$ on $n$. 

Lastly, to theoretically guarantee the use of $\alpha = -1$, we need to develop respective lower bounds. We are not aware of useful and generic lower bounds and we find it rather difficult to develop them. Nevertheless, we still believe that making predictions based on the carefully developed upper bound and empirically testing them is valuable for practitioners.  Indeed, our numerical results indicate the optimality of $\alpha=-1$ in many scenarios of overparameterized networks. On the other hand, we are unaware of much practical guidance that stems from the many other important and fundamental estimates in the study of overparameterized NNs. 
Additionally, Figure~\ref{fig:alphas} shows cases where our upper bound for the convergence rate aligns with the observed convergence rate.

\subsubsection*{Acknowledgements}
This work was partially supported by NSF award DMS 2124913.

\bibliography{NNtheory}
\bibliographystyle{apalike}

\newpage
\appendix
\onecolumn

\begin{center}
{\bf \huge \center Appendix}
\end{center}

Section \ref{append:sec:discussion_generalization_bound} discusses the generalization error bound, established in Theorem~\ref{thm:generalization_in_training_gd}, under different regimes for the number of training epochs. Section \ref{sec:allproofs} completes the proofs of the theorems stated in the main text and establishes four additional theorems: Theorem \ref{thm:generalization_in_training_sgd}, which bounds the generalization error when applying SGD; Theorem \ref{thm:exp_loss:convergence}, which bounds the convergence rate when using another loss function for regression; and Theorems~\ref{append:thm:convergence_special_case} and ~\ref{append:thm:generalization_special_case}, which bound the convergence rate and generalization error, respectively, for a special class of datasets. 
Section \ref{appd:numeric} describes additional numerical experiments and
the full details of implementation for both the previous and the new experiments. 

\section{Discussion of the Generalization Error Bound}
\label{append:sec:discussion_generalization_bound}
In this section, we clarify the estimates for generalization error in \eqref{eqn:generalization_error_gd} for different regimes of the number of training epochs, $t$. 

We first note that the last term in \eqref{eqn:generalization_error_gd} can be sufficiently small for a sufficiently large sample size $n$, so we may ignore it. 
The first bounding term in \eqref{eqn:generalization_error_gd} reflects the training error and the middle two bounding terms represent the NN complexity. 
There is a tradeoff between the training and NN-complexity terms, as we explain below; in particular, we cannot make both of them sufficiently small.
We remark that the closest bound on the generalization error for overparameterized deep NNs was established in the context of classification using GD in \citet{cao2020generalization}. Their generalization bound is independent of the training epoch. Instead, their bound is of order $\Theta(\text{poly}(n)\cdot n^{-1/2})$ and is typically not small even for arbitrarily large $n$. For very special cases (e.g., linear separability) they reduced the term $\text{poly}(n)$ so their overall bound is sufficiently small. 
In this work, we investigate the dependence of the generalization bound on $t$ for regression without making assumptions about the data distribution. 
Nevertheless, one may consider similar special assumptions as in  \citet{cao2020generalization} and apply them to our theory in order to better control our generalization bound.

To better understand the bound in \eqref{eqn:generalization_error_gd}, we apply the bound on $\gamma$ from Theorem \ref{thm:gd_main_thm} and our choice of $m$. We first quickly show that $T$ is at order of $\Theta((nL)^2)$, from \S\ref{sec:idea_proof}, we know that 
    $$T = \ln(\epsilon/\gL(\mW^{(0)}))/\ln\gamma\leq\Theta(\ln(\epsilon/n\sqrt{\ln m})/\ln\gamma),
    $$
by using \eqref{eqn:gd_conv} and $\eta = \Theta(d/(nmL^2))$, this upper bound is $\Theta((nL)^2)$, and when $n$ is large, a lower bound with the same order can be achieved.
We observe two different regions of $t\leq\Theta((nL)^{2})$ (in \S\ref{sec:idea_proof}, we show that 
$\Theta((nL)^{2})$  approximates $T$). 
When  $t =  \Theta\left((nL)^{1-\kappa}\right)$, where $0<\kappa<1$, the first 3 terms of $R(\mW^{(t)})$ are bounded by
\begin{equation*}
 \exp\mleft(-\Omega\mleft(\frac{(1-\alpha)^2\delta}{(1+\alpha^2) (nL)^{1+\kappa}} \mright) \mright)O(\ln m) +  \frac{(1-\alpha)}{\sqrt{1+\alpha^2}} \, O\mleft(\frac{d^{1/6}\delta^{2/3}}{n L^{11/6} (nL)^{4\kappa/3}}\mright) + O\mleft(\frac{\sqrt{d\ln m}}{(nL)^{\kappa}}\mright).
\end{equation*}
The last two terms above are sufficiently small for sufficiently large $n$ or $L$ 
and the first training term is of the order $O(\ln m)$ and is thus the dominant one.
In practice, it can be reduced through careful initialization.
We note that this dominant term is minimized at $\alpha=-1$. When $n$ and $L$ are not sufficiently large and the second bounding term is comparable to the first term, then the bound is minimized at a certain $\alpha$ between $-1$ and $1$.
If, on the other hand, $t = \Omega((nL)^{(1+\kappa)})$, where $0 < \kappa \leq 1$, then the order of the NN-complexity terms of \eqref{eqn:generalization_error_gd} is $O(n^{\min\{\kappa, 1/2+\tau\}}L^{\min\{\kappa, 2+\tau\}})$, which becomes extremely large when $n$ and $L$ grow. 
This illustrates the overfitting phenomenon in neural network training, where the generalization error bound increases significantly as the training error approaches zero. 
Overall, we note that a smaller bound is obtained when $t=\Theta((nL)^{1-\kappa})$ and moreover overfitting occurs when $t=\Theta((nL)^{1+\kappa})$.  These observations support the benefit of early stopping. We remark that when $t=T$, which is roughly at $\Theta((nL)^{2})$, we can express the upper bound in \eqref{eqn:generalization_error_gd}, excluding its last term, in terms of $\epsilon$ as follows:
\begin{equation*}
\begin{split}
&\epsilon +\min\left\{\left(\frac{(1-\alpha)^{11/3}}{(1+\alpha^2)^{11/6}}\right) O\left(\frac{d^{1/3}\delta^{4/3}}{ m^{1/6}n^{10/3}L^{10/3}}  \ln^{4/3}(n\sqrt{\ln m}/\epsilon)\right), O\left(\frac{d^{3/2+\tau} \delta^\tau n^{1/2+\tau}}{L^{1/2 - \tau}\ln m}\right)\right\} \\ & \quad +
\min\left\{\left(\frac{(1-\alpha)^2}{1+\alpha^2}\right)O\left(\frac{d^{1/2} \delta\sqrt{\ln m} }{n^3L^3}\right)\ln(n\sqrt{\ln m}/\epsilon), O\left(\frac{n^{1/2+\tau} L^{2+\tau} d^{1/2+\tau}}{\delta^{1/2 -\tau} \ln m}\right)\right\}.
\end{split}
\end{equation*}

The examination of our above theoretical results on generalization error bounds reveals two weaknesses when compared to the convergence theorems, that is, Theorems~\ref{thm:gd_main_thm} and \ref{thm:sgd_main_thm}.
Firstly, unlike the convergence rate that guarantees the training error's convergence, the generalization error bound doesn't assure a convergence to zero. Consequently, this bound may not offer a precise guideline about the optimal choice of $\alpha$, especially when the number of epochs is large.
Secondly, $\alpha=-1$ is the optimal choice for the generalization error bound when training terminates early and both $n$ and $L$ are sufficiently large. In contrast, the convergence theorem asserts that $\alpha=-1$ consistently ensures the fastest convergence. Numerical results align with these observations.

\section{Proofs}
\label{sec:allproofs}

We detail the proofs of Lemmas~\ref{thm:semi-smooth}, ~\ref{thm:gradient_bound} and~\ref{thm:generalization_overall} and the conclusion of Theorems~\ref{thm:gd_main_thm}, ~\ref{thm:sgd_main_thm} and~\ref{thm:generalization_in_training_gd} from these lemmas. Moreover, we formulate and prove some the following additional theorems: a theorem that bounds the generalization error when using SGD, which is the analog of Theorem \ref{thm:generalization_in_training_gd} for SGD instead of GD; theorems that improve our estimates for for a special class of datasets; and a theorem for the convergence theory when using a different loss function.
Section~\ref{appd:notations} introduces notation needed for the proof, \S~\ref{appd:initialization} quantifies the bounds for the initial weights, \S~\ref{appd:perturbation} extends the latter bounds to weights within a small perturbation around the initialization, \S~\ref{appd:gradient} proves the lower and upper bounds for the gradient at initial weight and within a small perturbation (Lemma~\ref{thm:gradient_bound}), \S~\ref{appd:semismooth} shows the proof of semi-smoothness (Lemma~\ref{thm:semi-smooth}), \S~\ref{appd:main_proof_gd} and \S~\ref{appd:main_proof_sgd} conclude the main theorem for gradient descent and stochastic gradient descent (Theorem~\ref{thm:gd_main_thm} and~\ref{thm:sgd_main_thm}), 
\S\ref{appd:generatlization_proof} proves the upper bound of the generalization error for a class of NN functions (Lemma~\ref{thm:generalization_overall}), \S\ref{appd:generalization_gd} concludes the generalization error bound for GD (Theorem~\ref{thm:generalization_in_training_gd}), \S\ref{appd:generalization_sgd} formulates and clarifies an upper bound of the generalization error for SGD, \S\ref{subsec:appendix:special_func_class} introduces a special dataset and establishes theorems on the convergence rate bound and generalization error bound using this dataset, and \S\ref{subsec:append:exponential_loss} extends Theorem~\ref{thm:gd_main_thm} and provides bounds of the convergence rate for a special loss function.

For the study of training convergence, we follow the notation and proof framework of \citet{allen2019convergence}, while incorporating the improvements suggested by \citet{zou2019improved} and some additional ones. For the study of the generalization error, we follow the proof framework of \citet{cao2020generalization} while extending the latter work to the task of regression. 
Whenever previous ideas require adaptation to Leaky ReLUs or to some of our technical contributions (summarized in \S\ref{sec:technical_contribution}), we prefer to repeat and even add more details so the reader can fully follow the current text and will not need to switch between references. However, when we feel that the ideas of previous works directly extend to our setting we formulate the analogous lemmas without proving them.

\subsection{Notation}\label{appd:notations}
Throughout this appendix, we denote the entries of a vector $\vx\in\sR^{m}$ by $x_j$ or $(\vx)_j$, $j\in [m]$. We denote the entries of a matrix $\mA\in\sR^{m\times m}$ by $A_{ij}$ or $(\mA)_{ij}$, $i,j\in [m]$. For $i\in[m]$, the $i$th row vector of a matrix $\mA$ is denoted by $\mA_{i,\cdot}$ and its $i$th column vector is denoted by $\mA_{\cdot,i}$. The default norm $\|\cdot\|$ is the $\ell_2$ norm. We denote by $1_{E}$ the indicator function of the event $E$, which equals $1$ when $E$ occurs and $0$ otherwise. We denote by $\gB^m_1$ the unit ball in $\sR^m$. 

We use the rescaled leaky ReLU introduced in \eqref{eqn:leaky_relu_rescaled} as the activation function of the neural networks under consideration. 
When acting on each coordinate of a vector $\vx \in \sR^p$ we express its action using the following diagonal matrix $\mD_{\vx}$:
\begin{equation}
    \tilde{\sigma}_\alpha(\vx) = \mD_{\vx} \vx, \ \text{ where } (\mD_{\vx})_{jj} =  \frac{1_{x_j \geq 0}}{\sqrt{1+\alpha^2}} +  \frac{\alpha 1_{x_j < 0}}{\sqrt{1+\alpha^2}} \ \ \text{ and for } k \neq j \ (D_{\vx})_{kj}=0.
    \label{eqn:append:def_sigmaalpha}
\end{equation}
For $i \in [n]$ and a data point $\vx_i \in \sR^p$, We inductively define   
\begin{equation}
    \vg_{i, l}:= \mW_l \vh_{i, l-1},  \ \vh_{i, l} := \tilde{\sigma}_\alpha(\mW_l \vh_{i, l-1}) \equiv \tilde{\sigma}_\alpha (\vg_{i, l}), \ \vh_{i, 0} = \mA\vx_i \label{eqn:append:def_hg}
\end{equation}
and use the notation $h_{i,l,k}:=(\vh_{i,l})_k$ and $g_{i,l,k}:=(\vg_{i,l})_k$.
We denote 
$$\mD_{i,l} := \mD_{\vg_{i, l}} \ \text{ and } D_{i, l, jj} :=  (\mD_{i,l})_{jj} \equiv 
\frac{1_{g_{i, l,j} \geq 0}}{\sqrt{1+\alpha^2}} +  \frac{\alpha 1_{g_{i, l,j} < 0}}{\sqrt{1+\alpha^2}}.$$
We further denote $\mD_0 := \mI$ and use the new notation to express the outputs of all hidden layers via matrix products (where according to the notation of \S\ref{sec:problem_setup} $\mW_0 \equiv \mA$ and $\mW_{L+1} \equiv \mB$:
\begin{align*}
    &\vg_{i, 0}  = \vh_{i, 0} = \mA \vx_i,\\
    &\vg_{i, l}  = \mW_l \mD_{i, l-1} \mW_{l-1} \dots \mW_2 \mD_{i, 1} \mW_1\mA \vx_i, \\
    &\vh_{i, l}  = \mD_{i, l}\mW_l \mD_{i, l-1} \mW_{l-1} \dots \mW_2 \mD_{i, 1} \mW_1\mA \vx_i, \\
    &\vg_{i,L+1} :=\mB \vh_{i, L}  \equiv \mB \mD_{i, L}\mW_L \mD_{i, L-1} \mW_{L-1} \dots \mW_2 \mD_{i, 1} \mW_1\mA \vx_i.
\end{align*}

We denote the residual and its elements by 
$$\ve_i := \vg_{i, L+1}-\vy_i, \ e_{i,j}=(\ve_i)_j$$ and the loss function by 
$$\gL(\mW) := \sum_{i=1}^n \textrm{loss}(\vx_i, \vy_i; \mW) := \sum_{i=1}^n \frac{1}{2}\|\vy_i - \vg_{i, L+1}(\vx_i; \mW)\|^2 \equiv \frac{1}{2}\sum_{i=1}^n \|\ve_i\|^2.
$$ 

Section 5 in \citet{higham2019deep} presents a comprehensive derivation for the gradient of the loss function in a neural network. In our case, the activation function derivative can be written as 
$$
\frac{\partial h_{i, l, j}}{\partial g_{i, l, k}} = \delta_{jk}\cdot \left(\frac{1_{g_{i, l, k} \geq 0}}{\sqrt{1+\alpha^2}} +  \frac{\alpha 1_{g_{i, l, k} < 0}}{\sqrt{1+\alpha^2}}\right)\equiv D_{i, l, jk}, \ \text{for} \ l\in [L].
$$
Denoting $\textbf{Back}_{i, L+1}:= \mB$ and $\textbf{Back}_{i, l} := \mB \mD_{i, L} \mW_L \dots \mW_{l}$ (this is the backpropagation operator) we can express the derivative of the loss with respect to the $rt$ entry of $\mW_l$, where $r, t\in[m]$, as
$$
\nabla_{(\mW_{l})_{rt}} \textrm{loss}(\vx_i, \vy_i;\mW) =(\textbf{Back}^T_{i, l+1}\ve_i)_r D_{i, l, rr} \vh_{i, l-1, t}.
$$
Similarly, the gradient of the loss according to the matrix $\mW_l$ and according to its $k$th row vector,  $(\mW_{l})_{k,\cdot}$, can be expressed as 
\begin{align*}
    \nabla_{\mW_l} \text{loss}(\vx_i, \vy_i; \mW) & = \mD_{i, l}\textbf{Back}^T_{i, l+1}\ve_i  \vh_{l-1}^T(\vx_i),\\
    \nabla_{(\mW_{l})_{k,\cdot}} \text{loss}(\vx_i, \vy_i;\mW) & = D_{i, l, kk}\langle(\textbf{Back}_{i, l+1})_{\cdot, k}, \ve_i\rangle  \vh_{l-1}(\vx_i).    
\end{align*}

For a vector $\vv \in \sR^p$, we denote its $\ell_2$ norm by $\|\vv\|_2$ (where $\|\vv\|_2^2 = {\sum_{j \in [p]} v_j^2}$), $\ell_\infty$ norm by $\|\vv\|_\infty = \max_{j \in [p]} |v_j|$, and $\ell_0$ ``size" by $\|\vv\|_0 = |\{j \in [p]: v_j \neq 0\}|$. For a matrix $\mX \in \sR^{m \times m}$, we denote its spectral norm by $\|\mX\|_2=\max_{j \in [m]} |\lambda_j(\mX)|$, Frobenius norm by $\|\mW\|_F = \sqrt{\sum_{i,j \in [m]} W_{ij}^2}$, and $\ell_0$ ``size" by $\|\mD\|_0 = |\{(i, j) \in [m]^2: D_{ij}\neq 0\}|$.
For a vector of matrices $\mW =(\mW_1,\ldots,\mW_l)$, where $\mW_1$, $\ldots$, $\mW_l \in \sR^{m \times m}$, we define its $\ell_{2}$ norm by $\|\mW\|_{2} := \max_{l\in[L]} \|\mW_l\|_2$ and Frobenius norm by $\|\mW\|_F := \sqrt{\sum_{l=1}^L \|\mW_l\|_F^2}$. For simplicity of notation we use $\|\cdot\|$ instead of $\|\cdot\|_2$ for vectors, matrices and vectors of matrices.

Throughout this appendix, we apply Algorithm~\ref{alg:rescale_initial} to initialize the weights $\mW$, $\mA$, $\mB$ for the neural network.

We use the big $O$, $\Omega$ and $\Theta$ notation. That is,  $f=O(N)$ or $f=\Omega(N)$ if there exists $C>0$ and $N_0 \in \sN$ such that 
$f \leq CN$ or $f \geq CN$, respectively, for all $N>N_0$.
Also, $f=\Theta(N)$ if and only if $f=O(N)$ and $f=\Omega(N)$.

Throughout this appendix, we may neglect the  subscript $i$ or superscripts $(t)$ or $(0)$ when there is no confusion.

\subsection{Initialization}\label{appd:initialization}
In this section, we focus on properties of the weights initialized by  Algorithm~\ref{alg:rescale_initial} without training. We thus denote $\mW:= \mW^{(0)}$ and for any input vector $\vx \in \sR^p$ and $l \in [L]$ 
\begin{align*}
    \vg_0&=\vh_0 := \mA \vx, \\
    \vg_l & := \mW^{(0)}_l\mD_{\vg_{l-1}}\ldots \mW_2^{(0)} \mD_{\vg_1} \mW^{(0)}_1 \mA \vx,\\ 
    \vh_l & := \mD_{\vg_l} \vg_l.
\end{align*}
For simplicity, we denote $\mD_l:=\mD_{\vg_l}$.

We first establish Lemma~\ref{lemma:initial_vector_norm} which controls the norms of the outputs of the hidden layers with high probability.  
We then establish Lemma~\ref{lemma:separation_layers} that upper bounds $\max_{i\neq j\in [n]}\langle\vh_{i, l}/\|\vh_{i, l}\|, \vh_{j, l}/\|\vh_{j, l}\|\rangle$ for all $l\in [L]$. Lastly, Lemma~\ref{lemma:initial_forward_matrix_bound} summarizes useful bounds of the norms of some relevant matrices.

We remark that the proof of Lemma~\ref{lemma:initial_vector_norm} adapts ideas of \cite{allen2019convergence} to the setting of Leaky ReLUs. The proof of Lemma~\ref{lemma:separation_layers} follows ideas of \cite{zou2019improved}, while assuming that $\delta < O(1)$ instead of $\delta < O(1/L)$ and applying minor adaptation to Leaky ReLUs. At last, Lemma~\ref{lemma:initial_forward_matrix_bound} directly follows the same proof argument in \citet{allen2019convergence} (while using the conclusion of Lemma~\ref{lemma:initial_vector_norm}) and we thus omit its proof. 

\begin{lemma}
\label{lemma:initial_vector_norm}
Assume the setup of \S\ref{sec:problem_setup} and the above notation. If $\vx \in \sR^p$, $\|\vx\| = 1$ and $\epsilon$ is a fixed number in $(\Omega(\frac{L}{m}), 1)$, then  
\begin{align*}
    \|\vh_l\| \in [1-\epsilon, 1+\epsilon] \ \text{ for all } \ l \in \{0\}\cup[L]
    \ \text{ with probability at least } 1-e^{-\Omega(m\epsilon^2/L)}.
\end{align*}
\end{lemma}

\begin{proof}
We first prove the lemma for $l=0$. Due to the initialization of the input layer by Algorithm~\ref{alg:rescale_initial}, $\vh_0 = \mA \vx\sim N(0, \|\vx\|^2/m) = N\left(0, 1/m \right)$. Therefore, $m\|\vh_0\|^2 \sim \chi^2 (m)$, where $\chi^2(m)$ denotes the chi-square distribution with $m$ degrees of freedom. Using the tail bound for this sub-Gaussian distribution
\begin{equation}
    \sP\left(\left| \|\vh_0\|^2 - 1 \right| > \frac{\epsilon}{2}\right) \leq 2 e^{-m\epsilon^2/32} \leq e^{-\Omega(m\epsilon^2)}.\label{eqn:append:initial_layer}
\end{equation}

We next prove the lemma for $l \geq 1$. For each layer $l$, we analyze the distribution of each entry of $\vh_l$, and denote by $h_{l, j}:=(\vh_l)_j$, $j\in [m]$, conditioned on the output from the former layer $\vh_{l-1}$. We note that the randomness of $\vh_l$ comes from $\mW_l$ given the fixed $\vh_{l-1}$.

We note the following expression for $h_{l,j}$, which follows from \eqref{eqn:append:def_sigmaalpha} and  \eqref{eqn:append:def_hg}: 
\begin{equation}
\label{eq:hij}
h_{l, j} = \tilde{\sigma}_\alpha (g_{l,j}) = 1_{g_{l, j}>0} \frac{g_{l, j}}{\sqrt{1+\alpha^2}} + 1_{g_{l, j}\leq 0} \frac{\alpha g_{l, j}}{\sqrt{1+\alpha^2}} .
\end{equation}
We remark that unlike previous analyses \citep{allen2019convergence,zou2019improved}, we need to deal with two different terms in the sum in order to address Leaky ReLU and note just ReLU. We observe that due to the initialization of $\mW_l$ and \eqref{eqn:append:def_hg}, $g_{l, j}\sim N (0, {2\sum h_{l-1,k}^2}/{m}) = N(0, {2\|\vh_{l-1}\|^2}/{m})$. By the symmetry of the normal distribution, $g_{l, j}$ is positive with probability 0.5.  Thefore, the random variable
$$B_j := 1_{g_{l, j}>0}$$ is Bernoulli with probability 0.5, that is, $B_j\sim \text{B}(0.5)$. We further note that $B_j g_{l, j} = B_j g_{l, j}| g_{l, j} > 0$. We thus rewrite \eqref{eq:hij} as
\begin{equation}
\label{eq:hij_2}
h_{l, j} = \frac{1}{\sqrt{1+\alpha^2}} B_j g_{l, j}|\{g_{l, j} > 0\} - \frac{\alpha}{\sqrt{1+\alpha^2}} (1-B_j) (-g_{l, j})|\{g_{l, j} \leq 0\}.
\end{equation}

Conditioning on the event $g_{l, j} > 0$, $g_{l, j}{\buildrel d \over =}|X|$, where $X \sim N(0, {2\|\vh_{l-1}\|^2}/{m})$. Therefore, $$g_{l, j}\Big|\left(g_{l, j} > 0\right) \ \sim |N(0, {2\|\vh_{l-1}\|^2}/{m})|.$$ 
Similarly, $$-g_{l, j} \Big| \left(g_{l, j} \leq 0\right) \ \sim |N(0, {2\|\vh_{l-1}\|^2}/{m})|.$$ 
Therefore, \eqref{eq:hij_2} and the above two equations imply the following distribution law for $h_{ij}$:
$$
h_{l, j} {\buildrel d \over =} \frac{1}{\sqrt{1+\alpha^2}} B_j V_{j, 1} - \frac{\alpha}{\sqrt{1+\alpha^2}} (1-B_j) V_{j, 2},
$$
where $V_{j, 1}$, $V_{j, 2} \sim$ 
$|N(0, {2\|\vh_{l-1}\|^2}/{m})|$, $B_j\sim B(0, \frac{1}{2})$ and $V_{j, 1}$, $V_{j, 2}$ and $B_j$ are independent.
We further claim that  if the former layer $\vh_{l-1}$ is given, then $V_{j, 1}$ and $V_{j, 2}$ are independent for $j \in [m]$.
Indeed, We first observe that conditioned on $\vh_{l-1}$ the entries $h_{l, j}$, $j\in [m]$, are independent. Indeed, they depend on different rows in $\mW_l$ and due to Algorithm~\ref{alg:rescale_initial} for the initialization of the $l$th layer these rows are independent. We also note that $V_{j, 1}$ and $V_{j, 2}$ only rely on $h_{l, j}$, and thus conditioned on $\vh_{l-1}$ they are independent for  $j\in [m]$.

We next derive an expression that clarifies the distribution of $\|\vh_l\|^2$ conditioned on $\vh_{l-1}$. 
 We denote 
\begin{align*}
  &&P_l:= \{j\in[m]: \ g_{l, j} > 0\},   
  &&&K_l := |P_l|
  ,\\
  &&H_{l, 1}:= \frac{m}{2\|\vh_{l-1}\|^2}\sum_{j\in P_l} V_{j, 1}^2|\vh_{l-1},
  &&&H_{l, 2}:=\frac{m}{2\|\vh_{l-1}\|^2}\sum_{j\in [m], j\notin P_j} V_{j, 2}^2|\vh_{l-1}.
\end{align*} 
We note that $K_l$ is Bernoulli with $m$ trials and probability 0.5, i.e., $$K_l \sim B(m,0.5).$$  
The above observations imply that conditioning on $\vh_{l-1}$ and $P_l$, $H_{l,1} \sim \chi^2(K_l)$ and $H_{l,2}\sim \chi^2(m-K_l)$. Therefore, $\|\vh_l\|^2$ conditioned on $\vh_{l-1}$ is given by
\begin{equation}
    \|\vh_l\|^2 | \vh_{l-1} {\buildrel d \over =} \frac{2\|\vh_{l-1}\|^2}{(1+\alpha^2)m} H_{l, 1} + \frac{2\alpha^2\|\vh_{l-1}\|^2}{(1+\alpha^2)m} H_{l, 2}. 
    \label{eqn:append:vh_norm}
\end{equation}
Note that the indices used by $H_{l, 1}$ and indices used by $H_{l, 2}$ do not overlap and thus form a partition of $[m]$. This partition is determined by $P_l$ and $H_{l, 1}$ and $H_{l, 2}$ are conditionally independent given $P_l$.

We denote $\Delta_l := \frac{\|\vh_l\|^2}{\|\vh_{l-1}\|^2}$ and rewrite $\|\vh_b\|^2$ (fixing $l=b$) as follows
\begin{equation}
    \ln \|\vh_b\|^2 = \ln \|\vh_0\|^2 + \sum_{l=1}^b \ln\Delta_l.\label{eqn:append:h_norm_eqn}
\end{equation}
Using the distribution of $\|\vh_l\|^2$ conditioning on $\vh_{l-1}$, where $1 \leq l \leq b$, we first derive upper and lower bounds of the expectation $\sE (\ln\Delta_l|\vh_{l-1})$. We then show that given $\vh_{l-1}$ and other information, $\ln\Delta_l$ is an $O(m^{-1})$ sub-Gaussian random variable. With these two properties we conclude the lemma by applying a variant of Azuma's inequality for sub-Gaussian random variables on $\sum_{l=1}^b \ln \Delta_l$. 

\textbf{Bounds on the expectation of }$\textbf{ln} {\mathbf{\Delta}_l}\boldsymbol{|\vh_{l-1}}$\textbf{:}  
We note that 
$\sE(H_{l,1}|P_l)=K_l$,   $\sE(H_{l,1}|P_l)=m-K_l$ and thus $
    \sE(H_{l, 1}) = \sE(\sE(H_{l, 1}|P_l))  = \sE(K_l) = 0.5m$.
Similarly, $\sE(H_{l, 2}) = 0.5m$ and therefore $\sE(H_{l, 1})=\sE(H_{l, 2})$. Using the latter observation and \eqref{eqn:append:vh_norm} we obtain 
\begin{equation}
    \sE(\Delta_l|\vh_{l-1}) = \frac{2}{m (1+\alpha^2)}\left(\sE(H_{l, 1}) + \alpha^2\sE(H_{l, 1})\right) = 1.\label{eqn:append:E_Delta}
\end{equation}
Applying the concavity of the log function, Jensen's inequality 
and then \eqref{eqn:append:vh_norm} and \eqref{eqn:append:E_Delta} yields
\begin{equation}
\sE \left(\frac{1}{1+\alpha^2}\ln \frac{2}{m}H_{l, 1} + \frac{\alpha^2}{1+\alpha^2} \ln \frac{2}{m}H_{l, 2}\right) \leq \sE \ln (\Delta_l|\vh_{l-1}) \leq \ln \sE (\Delta_l|\vh_{l-1})=0.\label{eqn:append:bound_Delta}
\end{equation}

Using the Chernoff bound for the binomial distribution, we note that 
\begin{equation}
\label{eq:bound_K_l}
K_l \in [0.4m, 0.6m], \text{ or equivalently }  m-K_l \in [0.4m, 0.6m], \text{ with probability }1-e^{-\Omega(m)}\,.   
\end{equation}
We next use the property that if $H\sim\chi^2(K)$ and $K\in[0.4m, 0.6m]$, then $\sE \ln \frac{2}{m} H \geq -\frac{4}{m}$ (see page 13 in the proof of Lemma~7.1 in \citet{allen2019convergence}).
This property and \eqref{eqn:append:bound_Delta} imply
\begin{equation}
    \sE \left(\ln (\Delta_l)|\vh_{l-1}\right) \in \left[-\frac{4}{m}, 0\right].\label{eqn:append:e_ln_Delta_bound}
\end{equation}

\textbf{Conditional sub-Gaussianity of }$\textbf{ln}\mathbf{\Delta}_l$ \textbf{:}
We derive a tail bound for $\ln\Delta_l|\vh_{l-1}$ and consequently conclude its sub-Gaussianity.
We denote
$$E_l:= \{|P_l| \in [0.4m ,0.6m]\}.$$
The combination of \eqref{eqn:append:vh_norm}, basic probabilistic manipulations and the conditional independence of $H_{l, 1}$ and $H_{l, 2}$  yields
\begin{align*}
    &\sP\left(\left|\frac{m}{2}\Delta_l - \frac{m}{2}\right| < t\Big|\vh_{l-1}, E_l, P_l \right)\\
    & =\sP\left(\left|\frac{1}{1+\alpha^2} H_{l, 1} + \frac{\alpha^2}{1+\alpha^2}H_{l, 2} - \frac{m}{2}\right| < t\Big| E_l, P_l\right) \\
    & \geq \sP\left(\left|\frac{1}{1+\alpha^2} H_{l, 1}- \frac{1}{1+\alpha^2}\frac{m}{2}\right| < t/2 \ \text{ and } \ \left|\frac{\alpha^2}{1+\alpha^2}H_{l, 2} - \frac{\alpha^2}{1+\alpha^2}\frac{m}{2}\right| < t/2 \Big| E_l, P_l\right) \\
    & \geq \sP\left(\left|\frac{1}{1+\alpha^2} H_{l, 1}-  \frac{1}{1+\alpha^2}\frac{m}{2}\right| < t/2\Big| E_l, P_l\right) 
    \sP\left(\left|\frac{\alpha^2}{1+\alpha^2}H_{l, 2} -  \frac{\alpha^2}{1+\alpha^2}\frac{m}{2}\right| < t/2\Big| E_l, P_l\right)\\
    &\geq \sP\left(\left| H_{l, 1}-  \frac{m}{2}\right| < t/2\Big| E_l, P_l\right) 
    \sP\left(\left|H_{l, 2} -  \frac{m}{2}\right| < t/2\Big| E_l, P_l\right) .
\end{align*}

Recall that given $P_l$, $H_{l, 1}$ and $H_{l, 2}$ are $\chi^2(|P_l|)$ and $\chi^2(m-|P_l|)$, respectively. We thus apply the corresponding tail bounds of $H_{l, 1}$ and $H_{l, 2}$  and \eqref{eq:bound_K_l} to the bound above and obtain that
$$
\sP\left(\left|\frac{m}{2}\Delta_l - \frac{m}{2}\right| < t\Big|\vh_{l-1}\right)\geq \left(1-e^{-\Omega(t^2/m)}\right)^2 \geq 1 - \Omega\left(e^{-\Omega(t^2/m)}\right).
$$
Consequently, 
$$
\sP\left(\ln|\Delta_l| < \frac{t}{m}\Big|\vh_{l-1}\right) \geq 1- e^{-\Omega((\frac{t}{m})^2 m)} \ \text{ for } \ t\in (0, {m}/{4}].
$$
Therefore, $\ln\Delta_l$ conditioned on $\vh_{l-1}$ and $K_l\in[0.4m, 0.6m]$ 
is $O({m}^{-1})$-sub-Gaussian. 

\textbf{Conclusion of the proof of the lemma:}
We define a new variable $\tilde{\Delta}_l$, where $\tilde{\Delta}_l = \Delta_l$ if $K_l \in [0.4m , 0.6m]$ and $\tilde{\Delta}_l = 1$, otherwise. From the tail probability of $\ln\Delta_l$ and the definition, it is clear that $\ln\tilde{\Delta}_l|\vh_{l-1}$ is $O(m^{-1})$-sub-Gaussian. It follows from \eqref{eq:bound_K_l} that with overwhelming probability $\Delta = \tilde{\Delta}$. We consider the sequence of the following random variables $\{(\ln \tilde{\Delta}_l - \sE \ln\tilde{\Delta}_l)|\vh_{l-1}\}_{l=1}^b$. 
By Azuma's inequality for sub-Gaussian variables (see Theorem~2 with $c=m$ in \citet{shamir2011variant})
$$
   \sP\left(\left|\sum_{l=1}^b \ln\Delta_l - \sE (\ln\Delta_l|\vh_{l-1})\right|> b\epsilon\right) < e^{-\Omega(b \epsilon^2 m )}.
$$
Applying \eqref{eqn:append:e_ln_Delta_bound} to the above inequality yields 
$$
\sP\left(\left|\sum_{l=1}^b \ln \Delta_l\right| > \epsilon + O\left(\frac{b}{m}\right)\right) < e^{-\Omega(\epsilon^2 m /b)}.
$$
We can choose $\epsilon > \Omega(\frac{L}{m})$ such that  
$$
\sP\left(\left|\sum_{l=1}^b \ln \Delta_l\right| > \epsilon/2\right) < e^{-\Omega(\epsilon^2 m /b)}
$$
Combining \eqref{eqn:append:initial_layer}, \eqref{eqn:append:h_norm_eqn} and the above equation we obtain that
$$
\sP\left(\left|\|\vh_b\|^2 - 1\right| > \epsilon_0 \right) < e^{-\Omega(m\epsilon^2 / L)}, \ \text{for } \ b\in[L].
$$
\end{proof}

\begin{lemma}\label{lemma:separation_layers}
Assume the setup of \S\ref{sec:problem_setup} and the notation introduced in this section. If $\delta < O(1)$ and $m > \Omega(\ln n L^4)$, 
then 
\begin{equation}
    \max_{i\neq j\in[n]}\left\langle \frac{\vh_{i,l}}{\|\vh_{i,l}\|}, \frac{\vh_{j,l}}{\|\vh_{j,l}\|} \right\rangle^2 \leq 1 - \Omega\left(\frac{\delta^2}{L^2}\right)
    \ \text{ with probability at least } 
    1-e^{-\Omega(\delta^4 m/L^4)}. \label{eqn:append:upper_bound_inner_prod_hij}
\end{equation}
\end{lemma}
\begin{proof} 
We separate the proof of this lemma into three parts. The first one establishes a useful upper bound of the expectation of the multiplication of two leaky ReLUs of certain inner products (see \eqref{eqn:append:expectation_leakyReLU_uh} below). Given this upper bound, the second part shows that with high probability, 
$$\min_{i \neq j\in [n]} \|\vh_{i, l} - \vh_{j, l}\| \geq \Omega(\delta/L), \ \text{ for any } l\in [L].$$ The third part uses the result to conclude this Lemma.

\textbf{Part 1.} We verify the following probabilistic estimate:  
\begin{align}
    & \sE \tilde{\sigma}_\alpha (\vu^T \vh_i) \tilde{\sigma}_\alpha(\vu^T\vh_j) \leq \frac{1}{2}\left(1-\frac{1}{2}\theta^2\right) + \frac{(1-\alpha)^2}{(1+\alpha^2)} O(\theta^3), \label{eqn:append:expectation_leakyReLU_uh}\\ & \text{ where } \vh_i, \vh_j\in \sR^p, \ \text{ for } \theta > 0,  
    \langle \vh_i, \vh_j \rangle\leq 1 - \frac{1}{2}\theta^2, \  \text{ and } \vu\sim N(0, \mI)\in \sR^p. \notag
\end{align}

Since $\vu \sim N(0, \mI)$, $\sE u_k u_{k'} = 0$ whenever $k\neq k'$. We denote $\vu:=(u_1, u_2\cdots u_p)^T$, $\vh_i:=(h_{i, 1}, h_{i, 2}, \cdots h_{i, p})^T$ and $\vh_j:=(h_{j, 1}, h_{j, 2}, \cdots h_{j, p})^T$. We first note that
\begin{align*}
    \sE \left(\vu^T \vh_i\right)\left(\vu^T\vh_j\right) & = \sE \left(\sum_{k=1}^p u_k h_{i, k}\right) \left(\sum_{k'=1}^p u_{k'} h_{j, k'}\right) = \sE \sum_{k=1}^p u_k^2 h_{i, k}h_{j, k} = \vh_{i}^T\vh_j \sE \vu^T\vu \leq 1-\frac{1}{2}\theta^2. \end{align*}
 For simplicity, we denote $Z_i:= \vu^T\vh_i$ and $Z_j:= \vu^T\vh_j$ and thus express the above equation as
 \begin{equation}
  \label{eqn:append:ZiZj_bound}
  \sE (Z_i Z_j) \leq 1-\frac{1}{2}\theta^2.
 \end{equation}
 Using the symmetry of normal distribution, we obtain that $$\sE (Z_i Z_j|Z_i, Z_j\geq 0) = \sE (Z_i Z_j| Z_i, Z_j < 0)$$ 
 and $$\sE (Z_i Z_j|Z_i < 0, Z_j\geq 0) = \sE (Z_i Z_j| Z_i \geq 0, Z_j < 0).$$ Consequently, the expectation of $\tilde{\sigma}_\alpha(Z_i)\tilde{\sigma}_\alpha(Z_j)$ can be rewritten as
\begin{align}
    \sE \tilde{\sigma}_\alpha(Z_i)\tilde{\sigma}_\alpha(Z_j)& = \frac{1}{1+\alpha^2}\Big(\sE( Z_iZ_j |Z_i, Z_j\geq 0) + \alpha\sE(Z_i Z_j | Z_i \geq 0, Z_j < 0) \notag\\&\quad + \alpha\sE(Z_i Z_j | Z_i < 0, Z_j \geq 0) + \alpha^2\sE( Z_iZ_j |Z_i, Z_j< 0) \Big)\notag\\
    & =\sE( Z_iZ_j |Z_i, Z_j\geq 0) + \frac{2\alpha}{1+\alpha^2} \sE(Z_i Z_j | Z_i \geq 0, Z_j < 0)\label{eqn:append:cond_EZij_leaky}.
\end{align}
Similarly, we express $\sE Z_i Z_j$ as follows: \eqref{eqn:append:cond_EZij_leaky}
\begin{align*}
    \sE Z_i Z_j & = 2\sE (Z_iZ_j|Z_i,Z_j \geq 0) + 2\sE (Z_iZ_j|Z_i\geq 0, Z_j \leq 0) \\
    & = 2 \sE \tilde{\sigma}_\alpha(Z_1)\tilde{\sigma}_\alpha(Z_2) + \left(2-\frac{4\alpha}{1+\alpha^2}\right)\, \sE(Z_1Z_2|Z_1\geq 0, Z_2 \leq 0).
\end{align*}
Rearranging the above equation yields
\begin{equation}
 \sE \tilde{\sigma}_\alpha(Z_1)\tilde{\sigma}_\alpha(Z_2) = \frac{1}{2}\sE Z_1Z_2 - \frac{(1-\alpha)^2}{1+\alpha^2} \sE (Z_1Z_2 |Z_1\geq 0, Z_2<0).
\label{eq:almost_last_part_1}    
\end{equation}

Noting that $\sE Z_1Z_2\leq 1 - \frac{1}{2}\theta^2$ and using the proof of Lemma A.3 of \citet{zou2020gradient} result in 
$$
\left|\sE (Z_1Z_2 |Z_1\geq 0, Z_2<0)\right| \leq O(\theta^3).
$$
The application of both  \eqref{eqn:append:ZiZj_bound} and the above estimate to \eqref{eq:almost_last_part_1} results in\eqref{eqn:append:expectation_leakyReLU_uh} and thus concludes this part.

\textbf{Part 2.} For $l=0,\ldots,L$ and $\delta_l:= \frac{\delta}{2(l+1)}$ we prove by induction:
\begin{equation}
    \min_{i\neq j \in [n]}\| \vh_{i, l} - \vh_{j, l}\| \geq \delta_l \ \text{with probability at least}  \ 1-e^{-\Omega(\delta^4 m /L^4)}. \label{eqn:append:h_dist_lower_bound}
\end{equation}
We first prove \eqref{eqn:append:h_dist_lower_bound} when $l=0$. Recall that $\vh_{i, 0}=\mA \vx_i$ and note that for any $i,j\in[n]$,
\begin{align}
    \sE\left(\|\mA \vx_i - \mA \vx_j\|^2\right) &= \sE \left\langle \mA \vx_i - \mA \vx_j,\mA \vx_i - \mA \vx_j\right\rangle = \sE \|\mA \vx_i\|^2 + \sE \|\mA \vx_j\|^2  - 2\sE \langle \mA \vx_i ,\mA \vx_j\rangle\notag\\ 
    & = 2 - 2\sE \sum_{k=1}^m \sum_{s, t} A_{ks} x_{i, s} A_{kt} x_{j, t}  = 2 - 2\sE \sum_{k,s=1}^m x_{i, s}x_{j, s}A_{ks}^2\notag\\
    & = 2 - 2\sum_{s,k} x_{i, s}x_{j, s} \sE  A_{ks}^2 = 2- 2\sum_{s,k}  x_{i, s}x_{j, s} \frac{1}{m}\notag\\
    & = 2 - 2\vx_i^T\vx_j.\label{eqn:append:expectation_h0_bound}
\end{align}
Recall that Assumption~\ref{assump:separation} implies that $\|\vx_i - \vx_j\| \geq \delta$ and thus 
clearly $$\vx_i^T\vx_j \leq 1 - \delta^2/2 .$$ 
Applying this estimate in  \eqref{eqn:append:expectation_h0_bound} yields the that 
\begin{equation}
\sE (\|\mA \vx_i - \mA \vx_j\|^2) \geq \delta^2. 
    \label{eq:basic_low_bound_exp_h0}
\end{equation}
Due to the random initialization, $m\|\mA\vx\|^2 \sim \chi^2(m)$ and therefore  $\|\mA\vx\|^2$ is $(O(1/m, 4)$ sub-exponential. Since $$\sP(\|\mA\vx_i - \mA\vx_j\|^2 > s) \leq \sP(\|\mA\vx_i\|^2 >s/2) + \sP(\|\mA\vx_j\|^2 >s/2),$$  the tail probability of $\|\mA\vx_i - \mA\vx_j\|^2$ is of the same order as the tail probabilities of  $\|\mA\vx_i\|^2$ and $\|\mA\vx_j\|^2$. Therefore, we conclude that $\|\mA\vx_i - \mA\vx_j\|^2$ is also $(O(1/m), 4)$ sub-exponential. 
Using the assumption $\delta < c_0$, where $c_0$ can be appropriately chosen (here we assume that $c_0 \delta < 3/4$), \eqref{eq:basic_low_bound_exp_h0} 
and the fact that $\|\mA\vx_i - \mA\vx_j\|^2$ is $(O(1/m), 4)$ sub-exponential) we conclude that
\begin{align}
    \sP\left(\|\mA \vx_i - \mA \vx_j\|^2 < \frac{\delta^2}{4}\right) & \leq \sP\left(\|\mA \vx_i - \mA \vx_j\|^2 < \delta^2(1-\delta)\right) \notag\\
    & \leq  \sP\left(\|\mA \vx_i - \mA \vx_j\|^2 < (1-\delta)\sE \|\mA \vx_i - \mA \vx_j\|^2\right)\notag\\
    & \leq O(e^{-\delta^4 m}).\notag
\end{align}
Applying a union bound over all distinct $i$, $j \in [n]$, we conclude that with probability at least $1 - n^2 e^{-\Omega(\delta^4 m)}$,
$$
\min_{i \neq j\in[n]} \|\vh_{i, 0} - \vh_{j, 0}\| \geq \frac{\delta}{2} \equiv \delta_0.
$$

Next, we fix $l \in [L]$, 
assume that \eqref{eqn:append:h_dist_lower_bound} holds for all $k\in [0, l-1]$ and verify \eqref{eqn:append:h_dist_lower_bound} for $l$. Using the fact that $\sE(\|\vh_{i,l}\|^2|\vh_{i,l-1}) = \|\vh_{i,l-1}\|^2$ and the definition of $\vh_{i, l}$ we obtain
\begin{align}
    &\sE (\|\vh_{i, l} - \vh_{j, l}\|^2|\vh_{l-1})\notag\\
    & = \sE (\|\vh_{i, l}\|^2|\vh_{i, l-1}) + \sE (\|\vh_{j, l}\|^2|\vh_{j, l-1})  - 2\sE (\langle\vh_{i, l},  \vh_{j, l}\rangle|\vh_{l-1})\notag\\ 
    & = \|\vh_{i, l-1}\|^2 + \|\vh_{j, l-1}\|^2  - 2\sE (\langle\tilde{\sigma}_\alpha(\mW_l \vh_{i, l-1}), \tilde{\sigma}_\alpha(\mW_l \vh_{j, l-1})\rangle|\vh_{l-1}).\label{eqn:hidden_signal_l_distance}
\end{align}
Applying the induction assumption (i.e., \eqref{eqn:append:expectation_leakyReLU_uh} with $\theta = \delta_{l-1}$) and the fact that $(\mW_{l})_{k,\cdot}\sim N\left(0, \frac{2}{m}\mI\right)$ and denoting by 
$\vu$  
a random variable such that $\vu\sim N(0, \mI)$ 
so  
$(\mW_{l})_{k, \cdot}^T{\buildrel d \over =}2\vu/m$ result in
\begin{align*}
    \sE (\langle\tilde{\sigma}_\alpha(\mW_l \vh_{i, l-1}), \tilde{\sigma}_\alpha(\mW_l \vh_{j, l-1})\rangle|\vh_{l-1}) & = \sum_k \sE (\tilde{\sigma}_\alpha((\mW_l)_{k,\cdot}^T \vh_{i, l-1}) \tilde{\sigma}_\alpha((\mW_l)_{k,\cdot}^T \vh_{j, l-1})|\vh_{l-1}) \\ 
    & = \frac{2m}{m}\sE (\tilde{\sigma}_\alpha(\vu^T \vh_{i, l-1}) \tilde{\sigma}_\alpha(\vu^T \vh_{j, l-1})|\vh_{l-1})\\
    & \leq 1-\frac{1}{2}\delta_{l-1}^2 + \frac{(1-\alpha)^2}{1+\alpha^2}O(\delta_{l-1}^3).
\end{align*}
Using  Lemma~\ref{lemma:initial_vector_norm}, we note for any $i\in[n]$, $\|\vh_{i, l}\|^2\in (1 - O(\delta_{l-1}^3), 1+ O(\delta_{l-1}^3))$ with probability at least $1 - ne^{-\Omega(\delta_{l-1}^3 m)}$. Combining this observation with \eqref{eqn:hidden_signal_l_distance} yields for a constant $C>0$
$$
\sE (\|\vh_{i, l} - \vh_{j, l}\|^2|\vh_{l-1}) \geq \delta_{l-1}^2\left(1 - C \frac{(1-\alpha)^2}{1+\alpha^2} \delta_{l-1}\right) + O(\delta_{l-1}^3).
$$

It follows from \eqref{eqn:append:vh_norm} and the fact that $H_{l, 1}\sim \chi^2(K_l)$ and $H_{l, 2}\sim \chi^2(m-K_l)$  
that $\|\vh_{i, l}\|^2 | \vh_{l-1}$ is $(O(1/m), 4)$ sub-exponential and thus $\|\vh_{i, l} - \vh_{j, l}\|^2 | \vh_{l-1}$ is also $(O(1/m), 4)$ sub-exponential. Thus for $i \neq j \in [n]$$$
\sP\left(\|\vh_{i, l} - \vh_{j, l}\|^2 \leq \delta_{l-1}^2 \left(1 - \left(C\frac{(1-\alpha)^2}{1+\alpha^2}\right)\delta_{l-1}\right)(1-\delta_{l-1}) \bigg| \vh_{l-1}\right) \leq O(\exp(-\delta_{l-1}^4m)).
$$
Applying a union bound for all $n(n-1)/2$ pairs yields
\begin{multline}
\sP\left(\min_{i\neq j\in [n]}\|\vh_{i, l} - \vh_{j, l}\|^2 \leq \delta_{l-1}^2 \left(1 - \left(C\frac{(1-\alpha)^2}{1+\alpha^2}\right)\delta_{l-1}\right)(1-\delta_{l-1}) \bigg| \vh_{l-1}\right)\\ \leq n(n-1)/2O(\exp(-\delta_{l-1}^4m)).
\end{multline}
Consequently, 
\begin{align}
&1 - n^2 \Omega(\exp(-\Omega(\delta_{l-1}^4m)))\notag\\ &\leq  \sP\left(\min_{i\neq j\in [n]}\|\vh_{i, l} - \vh_{j, l}\|^2 \geq \delta_{l-1}^2 \left(1 - \left(C\frac{(1-\alpha)^2}{1+\alpha^2}\right)\delta_{l-1}\right)(1-\delta_{l-1}) \bigg|\vh_{l-1}\right) \notag\\
&= \sP\left(\min_{i\neq j\in [n]}\|\vh_{i, l} - \vh_{j, l}\|^2 \geq \delta_{l-1}^2 \left(1 - \left(C\frac{(1-\alpha)^2}{1+\alpha^2} + 1\right)\delta_{l-1}\right) + C \frac{(1-\alpha)^2}{1+\alpha^2}\delta_{l-1}^4 \bigg|\vh_{l-1}\right)\notag \\
&\leq \sP\left(\min_{i\neq j\in [n]}\|\vh_{i, l} - \vh_{j, l}\|^2 \geq \delta_{l-1}^2 \left(1 - \left(C\frac{(1-\alpha)^2}{1+\alpha^2} + 1\right)\delta_{l-1}\right) \bigg|\vh_{l-1}\right).
\label{eqn:separation}
\end{align}

Next, we verify  
that $1-(C{(1-\alpha)^2}/({1+\alpha^2}) + 1) \delta_{l-1} \geq {l^2}/{(l+1)^2}$ for a sufficiently small $c_0$ (recall that $\delta < c_0$).  
We first note that for $l\geq 1$, \begin{align*}
    \frac{l^2}{(l+1)^2} & = 1 - \frac{2l+1}{(l+1)^2} \leq 1 - \frac{l+1}{(l+1)^2} = 1 - \frac{1}{l+1}\leq 1 - \frac{1}{2l} .
\end{align*}
Therefore, if $\delta < 1/(C(1-\alpha)^2/(1+\alpha^2) + 1) < c_0$, then for any $l\in[L]$
\begin{align*}
    1 - \left(C\frac{(1-\alpha)^2}{1+\alpha^2} + 1\right)\delta_{l-1}   = 1 - \left(C\frac{(1-\alpha)^2}{1+\alpha^2} + 1\right)\frac{\delta}{2l} 
     \geq 1 - \frac{1}{2l} \geq \frac{l^2}{(l+1)^2}.
\end{align*}
Thus \eqref{eqn:separation} implies $\min_{i\neq j\in [n]}\|\vh_{i, l} - \vh_{j, l}\|^2 \geq \delta_{l-1}^2 \frac{l^2}{(l+1)^2} \equiv \delta_l^2$ with probability $1 - n^2e^{-\Omega(\delta^4 m/L^4)}$. When $m>\Omega(\ln n L^4)$, the latter probability can be written as $1 - e^{-\Omega(\delta^4 m/L^4)}$, which concludes \eqref{eqn:append:h_dist_lower_bound}.

\textbf{Part 3.} We conclude the lemma as follows. We recall that Lemma~\ref{lemma:initial_vector_norm} implies that with probability at least $1-e^{-\Omega(m\delta_l^3/L)}$: $\|\vh_{i, l}\|^2 \in [1 - O(\delta_l^3), 1 + O(\delta_l^3)]$. Applying this conclusion and  \eqref{eqn:append:h_dist_lower_bound} we conclude that for any $i\neq j\in [n]$ 
\begin{align*}
    \left\|\frac{1}{\|\vh_{j, l}\|}\vh_{j, l} - \frac{1}{\|\vh_{i, l}\|}\vh_{i, l}\right\| & =      \left\|\frac{1}{\|\vh_{j, l}\|}\vh_{j, l} - \frac{1}{\|\vh_{j, l}\|}\vh_{i, l} + \frac{1}{\|\vh_{j, l}\|}\vh_{i, l} -  \frac{1}{\|\vh_{i, l}\|}\vh_{i, l}\right\|\\
    & \geq \frac{1}{\|\vh_{j, l}\|}\|\vh_{j, l} - \vh_{i, l}\| - \left|\frac{1}{\|\vh_{j, l}\|} -  \frac{1}{\|\vh_{i, l}\|}\right| \|\vh_{i, l}\|\\
    &\geq \delta_l (1-\delta_l^{1/2}) \ \text{ with probability at least } 1-2e^{-\Omega(m\delta^4/L^4)}.
\end{align*}

We note that for $\delta<c_0<1/2$, $\delta_l < \frac{1}{4}$ and thus $\delta_l (1-\delta_l^{1/2})\geq \frac{1}{2}\delta_l$. Consequently,
\begin{align*}
    &\left\langle\frac{1}{\|\vh_{j, l}\|}\vh_{j, l}, \frac{1}{\|\vh_{i, l}\|}\vh_{i, l}\right\rangle\\ 
    & = \frac{1}{2} \left(\frac{\|\vh_{j, l}\|^2}{\|\vh_{j, l}\|^2} + \frac{\|\vh_{i, l}\|^2}{\|\vh_{i, l}\|^2} - \left\| \frac{1}{\|\vh_{j, l}\|}\vh_{j, l} - \frac{1}{\|\vh_{i, l}\|}\vh_{i, l}\right\|^2  \right)\\
    & = 1 - \frac{1}{2}\left\| \frac{1}{\|\vh_{j, l}\|}\vh_{j, l} - \frac{1}{\|\vh_{i, l}\|}\vh_{i, l}\right\|^2 \leq 1 - \frac{1}{8}\delta_l^2 \ \text{ with probability at least } 1-2e^{-\Omega(m\delta^4/L^4)}.
\end{align*}
Therefore, if $\delta_l^2 < 8$, then
$$
\left\langle\frac{1}{\|\vh_{j, l}\|}\vh_{j, l}, \frac{1}{\|\vh_{i, l}\|}\vh_{i, l}\right\rangle^2 \leq \left(1-\frac{1}{8}\delta_l^2\right)^2 \leq 1-\frac{1}{8}\delta_l^2 \ \text{ with probability at least } 1-2e^{-\Omega(m\delta^4/L^4)}.
$$
Finally, we apply a union bound on all the distinct $i$, $j$ pairs to obtain
$$
\max_{i, j\in [n]}\left\langle\frac{1}{\|\vh_{j, l}\|}\vh_{j, l}, \frac{1}{\|\vh_{i, l}\|}\vh_{i, l}\right\rangle^2 \leq 1-\frac{1}{8}\delta_l^2
 \ \text{ with probability at least } 1-n^2e^{-\Omega(m\delta^4/L^4)}.
$$
The proof of the lemma is concluded by the above bound and the following two immediate observations: $\delta_l\equiv \delta/2(l+1) \geq \Omega(\delta/L)$ and when $ m > \Omega(\ln n) L^4$ the above probability can be expressed as $1- e^{-\Omega(\delta^4 m /L^4)}$.

\end{proof}

\begin{lemma}\label{lemma:initial_forward_matrix_bound}
Assume the setup of \S\ref{sec:problem_setup} and the notation introduced in this section. If $\, 0\leq a<b\leq L$, then 
with probability at least $1-e^{-\Omega(m/L)}$ the following statements hold:
\begin{enumerate}
    \item $\|\mW_{b+1}\mD_b\mW_b\dots \mD_a\| \leq O(\sqrt{L})$.
    \item If $d<O(\frac{m}{L\ln m})$, then  $\|\textbf{Back}_a\| \equiv \|\mB\mD_L\mW_L\dots \mD_a\mW_a\| \leq O(\sqrt{\frac{m}{d}}).$
    \item If $\vv\in \sR^m$ and $\|\vv\|_0 \leq O\left(\frac{m}{L\ln m}\right)$, then $\|\mW_b\mD_{b-1}\ldots \mD_a\mW_a\vv\|\leq 2 \|\vv\|$.
\end{enumerate}
For $s < O(m/L \ln m)$ and $d < O(\frac{m}{L\ln m})$, with probability at least $1- \exp(-\Omega(s\log m))$, the following statement holds:
\begin{enumerate}
    \item[4.] For any vector $\vu\in \mR^d,\ \vv\in \mR^m$ such that $\|\vv\|_0 \leq s$, then $|\vu^T \mB \mD_L\mW_L \cdots \mD_a\mW_a\vv| \leq O(\sqrt{s\ln m/d}\|\vv\|\|\vu\|)$.
\end{enumerate}
\end{lemma}

The proof of the lemma follows the same argument of the proof of Lemma~7.3 (a), (b) and Lemma~7.4 (a), (b) in \cite{allen2019convergence} and is not directly affected by our use of Leaky ReLU. 
We remark though that it requires applying  Lemma~\ref{lemma:initial_vector_norm}, which was formulated for any Leaky ReLU function instead of Lemma~7.1 of \cite{allen2019convergence}.

\subsection{Perturbation}\label{appd:perturbation}

We establish Lemma~\ref{lemma:forward_perturbation} which quantifies the effect of a small perturbation of the randomly initialized parameters $\mW^{(0)}$
on the output of the hidden layers.
Lemma~\ref{lemma:intermediate_perturbation} uses the former lemma to bound the norms of the perturbed matrices and the perturbations themselves. 
The proof of Lemma~\ref{lemma:forward_perturbation} directly follows ideas of  Lemma~8.2 of \cite{allen2019convergence}, but adapts them to the setting of Leaky ReLUs. The final conclusion of this lemma is independent of $\alpha$ since the leading terms turn out to be independent of $\alpha$. For completeness, we find it useful to include all these details.   Lemma~\ref{lemma:intermediate_perturbation} directly follows arguments of \cite{allen2019convergence} and we thus omit its proof.

We denote the perturbation matrix by $\mW'$ and the perturbed matrix of parameters by $ \mW := \mW^{(0)} + \mW'$. Given an input vector $\vx$ such that $\|\vx\|=1$, we denote as follows the variables at the initialization (in first column), the variables after perturbation (in middle column) and the perturbation themselves (in last column):
\begin{align*}
& \vh_{0}^{(0)} = \mA\vx && \vh_{0} = \mA\vx && \vh'_{0} = \mathbf{0}\\
    &\vg_{l}^{(0)}= \mW^{(0)}_l \vh_{l-1}^{(0)},&&\vg_{l}= \mW_l \vh_{l-1} && \vg'_{l} = \vg_{l} - \vg_{l}^{(0)} \\
    &(\mD_{l})_{jj}^{(0)}=\frac{1_{(\vg^{(0)}_{l})_j\geq 0} + \alpha 1_{(\vg^{(0)}_{l})_j< 0}}{\sqrt{1+\alpha^2}},&&(\mD_{l})_{jj}=\frac{1_{(\vg_{l})_j\geq 0} + \alpha 1_{(\vg_{l})_j< 0}}{\sqrt{1+\alpha^2}}, && \mD'_{l} = \mD_{l} - \mD_{l}^{(0)} \\
    & \vh_{l}^{(0)}  = \tilde{\sigma}_\alpha(\mW^{(0)}_l\vh_{l-1}^{(0)})\equiv \tilde{\sigma}_\alpha(\vg_{l}^{(0)}),&& \vh_{l}  = \tilde{\sigma}_\alpha(\mW_l\vh_{l-1})\equiv \tilde{\sigma}_\alpha(\vg_{l}),&& \vh'_{l} = \vh_{l} - \vh^{(0)}_{l}.
\end{align*}
Since we fix $\mA$ and $\mW_{L+1}\equiv \mB$ in the training, $\mB^{(0)} :=\mB$ and $\mA^{(0)}:=\mA$.

\begin{lemma}\label{lemma:forward_perturbation}
If $\|\mW'\|_2 = \omega < O(\frac{1}{L^{9/2}\ln^{3/2} m})$ and $m \geq \Omega(L^2)$, then the following events hold with probability at least $1 - e^{-\Omega\left(\frac{m^{1/2}}{\ln m}\right)}$ 
\begin{enumerate}
    \item $\|\mD'_{l}\|_0 < O(m \omega^{2/3} L)$ and $\|\mD'_l\vg_l\| < \frac{1-\alpha}{\sqrt{1+\alpha^2}}O(\omega L^{3/2})$
    \item there exist vectors $\vg_{l, 1}'$ and $\vg_{l, 2}'$ such that $\vg_l' = \vg_{l, 1}' + \vg_{l, 2}'$, and $\|\vg_{l, 1}'\| = O(\omega L^{3/2})$ and $\|\vg_{l, 2}'\|_\infty = O\left(\frac{\omega L^{5/2}\sqrt{\ln m}}{\sqrt{m}}\right)$,
    \item $\|\vg'_l\|, \|\vh'_l\| < O\left(\omega L^{5/2}\sqrt{\ln m}\right)$.
\end{enumerate}

\end{lemma}

\begin{proof}
We divide the proof into two steps. First, we show that statements 2 and 3 of the lemma imply statement 1 . We then prove statements  2 and 3 of the lemma using an induction argument  for $l\in \{0, 1, \ldots L\}$.

\textbf{Statements 2 and 3 imply statement 1.}  We fix $l\in \{0, 1, \ldots L\}$.
In view of Lemma \ref{lemma:initial_vector_norm} and the focus on the $l$th layer, we assume that $\vh_{l-1}^{(0)}$ is a fixed vector such that $\|\vh_{l-1}^{(0)}\|\in [0.5, 1.5]$. More precisely, we can condition on $\vh_{l-1}^{(0)}$ and we know that with overwhelming probability  $\|\vh_{l-1}^{(0)}\|\in [0.5, 1.5]$.
We denote $g_{l, j}^{(0)}:=(\vg_l^{(0)})_j$ (note the difference between the vector notation $\vg_{i, l}$ and  the scalar notation $g_{l, j}^{(0)}$). 
We recall that
$$
\vg_l^{(0)} = \mW_l^{(0)}\vh_{l-1}^{(0)} \sim N\left(0, \frac{2\|\vh_{l-1}^{(0)}\|^2}{m}\mI\right) \text{ and thus } \ 
g_{l, j}^{(0)}\sim N\left(0, \frac{2\|\vh_{l-1}^{(0)}\|^2}{m}\right) \text{ for } j\in[m].
$$
We define the following vector $\vd$ and express it using the decomposition $\vg_l'=\vg_{l, 1}' + \vg_{l, 2}'$ in statement 2 of this lemma:
\begin{align*}
\vd &:= \mD_l'(\mW_l^{(0)}\vh_{l}^{(0)} + \vg'_l)
   = \mD_l'(\mW_l^{(0)}\vh_{l}^{(0)} + \vg'_{l,1} + \vg'_{l,2}).
\end{align*}
We denote $D'_{l, jj} := (\mD'_l)_{jj}$, $g_{l, 1, j}' := (\vg'_{l, 1})_j$ and  $g_{l, 2, j}' := (\vg'_{l, 2})_j$. 

To estimate $\|\vd\|$ and $\|\vd\|_0$ we define the following auxiliary sets that partition $\{j\in[m]: d_j\neq 0\}$, $S_1$ and $S_2$. To do this we arbitrarily choose a positive number $\xi >  2\|\vg'_{l, 2}\|_\infty$ and define 
$$S_1 := \{j \in [m]: | g_{l, j}^{(0)}|<\xi, \ d_{j}\neq 0\}$$
and
$$S_2 := \{j: j\in[m]/S_1, d_j\neq 0\}.$$
In the rest of the proof we bound $|S_1|$, $\sum_{j\in S_1} d_j^2$, $|S_2|$  and $\sum_{j\in S_2} d_j^2$. We then use these estimates to bound $\|\vd\|$ and $\|\vd\|_0$.
 
In order to bound $|S_1|$, we first note that 
$$
\sP(| g_{l, j}^{(0)}|<\xi, d_j\neq 0) 
\leq
\sP(| g_{l, j}^{(0)}|<\xi) 
\leq \Theta\left(\xi\sqrt{\frac{m}{\|\vh_{l-1}^{(0)}\|^2}}\right)= \Theta (\xi\sqrt{m}).
$$
Combining a Chernoff bound for the binomial distribution with the above estimate yields 
\begin{equation}
    |S_1| < O(\xi m^{3/2})
    \ \text{ with probability at least } \ 
1-e^{-\Omega(m^{3/2}\xi)}.    
    \label{eqn:appendix:S1bound}
\end{equation}

For $j\in S_1$, we upper bound the coordinate $d_j$ of $\vd$:
\begin{align*}
    |d_j|& \leq  \left|\frac{1-\alpha}{\sqrt{1+\alpha^2}}\right| |g_{l, j}^{(0)} + g_{l, 1, j}' + g_{l, 2, j}'| \leq \left|\frac{1-\alpha}{\sqrt{1+\alpha^2}}\right| (\xi + \|\vg'_2\|_\infty + |g_{l, 1, j}'|).
\end{align*}
For each index $j \in [m]$ such as $D'_{l, jj}\neq 0$ we note from the definition of $\mD'$ that $|D'_{l, jj}| = (1-\alpha)/\sqrt{1+\alpha^2}$.
By squaring both sides of the above inequality, summing over the indices in $S_1$ and applying \eqref{eqn:appendix:S1bound}, we conclude that with probability at least $1-e^{-\Omega(m^{3/2}\xi)}$
\begin{align}
\sum_{j\in S_1} |d_j|^2&\leq 3\sum_{j\in S_1}\frac{(1-\alpha)^2}{1+\alpha^2} (\xi^2 + \|\vg'_{l, 2}\|^2_\infty + |g_{l, 1, j}'|^2) \notag\\ 
& \leq \frac{3(1-\alpha)^2}{1+\alpha^2} |S_1|(\xi^2 + \|\vg'_{l, 2}\|^2_\infty) + \frac{3(1-\alpha)^2}{1+\alpha^2} \|\vg_{l, 1}'\|^2\notag \\
& \leq \frac{3(1-\alpha)^2}{1+\alpha^2} O\left(\xi m^{3/2}\right)(\xi^2 + \|\vg'_{l, 2}\|^2_\infty) + \frac{3(1-\alpha)^2}{1+\alpha^2} \|\vg_{l, 1}'\|^2.\label{eqn:append:d_l2_S1}
\end{align}

We next estimate $|S_2|$. 
The definitions of the diagonal matrices $\mD_l$, $\mD_l^{(0)}$ and $\mD_l'$ imply that if $D'_{jj}\neq 0$, then $g^{(0)}_{l, j}$ and $g_{l, j}$ have opposite signs, or equivalently, $g_{l, j}^{(0)} + g'_{l, j}$ and $g^{(0)}_{i, l}$ have opposite signs, which further implies that $|g'_{l, j}| \geq |g^{(0)}_{l, j}|$. We further note that by the triangle inequality $|g'_{l, j}| \leq |g'_{l, 1, j}| + |g'_{l, 2, j}|$. Combining these two observation and then applying additional basic estimates, we obtain 
\begin{align*}
    |g'_{l, 1, j}| & \geq |g^{(0)}_{l, j}| - |g'_{l, 2, j}| \geq \xi - \|\vg'_{l, 2}\|_\infty \ \text{ for } \ j \in S_2.
\end{align*}
This bound clearly implies 
\begin{align*}
    \|\vg_{l, 1}'\|^2 &\geq \sum_{j\in S_2}|g_{l, 1, j}'|^2\geq |S_2|(\xi - \|\vg'_{l, 2}\|_\infty)^2
\end{align*}
and consequently
\begin{equation}
    |S_2| \leq \frac{\|\vg_{l, 1}'\|^2}{(\xi - \|\vg_{l, 2}'\|_\infty)^2}.\label{eqn:append:S2_bound}
\end{equation}

For $j\in S_2$, we note as above that $g_{l, j}^{(0)}$ and $g_{l, j}'$ have opposite signs and $|g_{l, j}'| > |g_{l, j}^{(0)}|$. The combination of both of these observations imply  $|g^{(0)}_{l, j} + g'_{l, j}| \leq |g'_{l, j}|$. The later observation and the partition of $\vg_l$ according to the second statement of the lemma yield the following bound for $j \in S_2$:
\begin{align}
    |d_j|& = \frac{|1-\alpha|}{\sqrt{1+\alpha^2}} |g_{l, j}^{(0)} + g_{l, j}'|\leq \frac{|1-\alpha|}{\sqrt{1+\alpha^2}}|g_{l,j}'| \\
    &\leq \frac{|1-\alpha|}{\sqrt{1+\alpha^2}} (|g_{l, 1, j}'| + \|\vg_{l,2}'\|_\infty).\label{eqn:exclude_g0}
    \end{align}
Squaring both sides of \eqref{eqn:exclude_g0},  summing over $j \in S_2$  and applying \eqref{eqn:append:S2_bound} yield
    \begin{align}
    \sum_{j\in S_2}|d_j|^2& \leq 2\frac{(1-\alpha)^2}{ 1+\alpha^2}\sum_{j\in S_2}(|g_{l, 1, j}'|^2 + \|\vg_{l,2}'\|_\infty^2)\leq  2\frac{(1-\alpha)^2}{ 1+\alpha^2} (\|\vg_{l, 1}'\|^2 + |S_2|\|\vg_{l, 2}'\|_\infty^2)\notag\\
    & \leq 2\frac{(1-\alpha)^2}{ 1+\alpha^2} \left(\|\vg_{l, 1}'\|^2 + \frac{\|\vg_{l,2}'\|^2_\infty\|\vg_{l, 1}'\|^2}{(\xi - \|\vg_{l, 2}'\|_\infty)^2} \right).\label{eqn:append:d_l2_S2}   
\end{align}

Obtaining these four different estimates we conclude with bounds on $\|\vd\|_0$ and $\|\vd\|$. We first note that \eqref{eqn:appendix:S1bound} and \eqref{eqn:append:S2_bound} yield
\begin{align*}
    \|\vd\|_0 &\leq |S_1| + |S_2| \leq \Theta(\xi m^{3/2}) + \frac{\|\vg_{l, 1}'\|^2}{(\xi - \|\vg_{l, 2}'\|_\infty)^2} \ \text{ with probability at least } \ 1-e^{-\Omega(m^{3/2}\xi)}.
\end{align*}
Since $\xi > 2 \|\vg_{l, 2}'\|_\infty$,  we can obtain the following bound:
$$
\|\vd\|_0 \leq \Theta(\xi m^{3/2}) + \frac{4\|\vg_{l, 1}'\|^2}{\xi^2}.
$$
In order to tighten the above bound, we minimize the right hand side term with respect to $\xi$ and note that its minimal value is $m\|\vg_{l, 1}'\|^{2/3}$ and is obtained at $\xi_{\min} = \Theta\left({\|\vg_{l, 1}'\|^{2/3}}/{m^{1/2}}\right)$. 
We note that the assumed conditions: 
$\omega < O\left(L^{-9/2}(\ln m)^{-3/2}\right)$, $\|\vg_{l, 1}'\| = O(\omega L^{3/2})$ and $\|\vg_{l, 2}'\|_\infty < O(\omega L^{5/2}\sqrt{\ln m}/\sqrt{m})$ imply that $\xi_{\min} > 2\|\vg_{l, 2}'\|_\infty$ so that the minimum is achieved.  
Thus, an upper bound of $\|\vd\|_0$ is obtained as
$$
\|\vd\|_0 \leq O(m\|\vg_{l, 1}'\|^{2/3})\leq O(m\omega^{2/3} L).
$$

Combining \eqref{eqn:append:d_l2_S1} and \eqref{eqn:append:d_l2_S2} yields
\begin{align*}
    \|\vd\|^2 &= \sum_{j=1}^m d_j^2 = \sum_{j\in S_1} d_j^2 + \sum_{j\in S_2} d_j^2\\
    & \leq \frac{3(1-\alpha)^2}{1+\alpha^2}O\left(\xi m^{3/2}\right)(\xi^2 + \|\vg_{l, 2}\|^2_\infty) + \frac{5(1-\alpha)^2}{1+\alpha^2} \|\vg_{l, 1}'\|^2+ 2\frac{(1-\alpha)^2}{ 1+\alpha^2}\frac{\|\vg_{l, 1}'\|^2\|\vg_{l, 2}'\|_\infty^2}{(\xi - \|\vg_{l, 2}'\|_\infty)^2} \\
    & \leq C \frac{(1-\alpha)^2}{ 1+\alpha^2} (\xi^3m^{3/2} + \|\vg_{l, 1}'\|^2_2).
\end{align*}
Plugging in $\xi =\xi_{\min}$ to the above equation and applying the second statement of this lemma result in 
\begin{equation}
    \|\vd\|^2 \leq  O\left(\frac{(1-\alpha)^2}{1+\alpha^2}\|\vg_{l, 1}'\|^2\right) \leq \frac{1-\alpha}{\sqrt{1+\alpha^2}} O(\omega^2 L^3).\label{eqn:append:l2bound_d}
\end{equation}

Consequently, our bounds for $\|\mD_l\|_0$ and  $\|\vd\| = \|\mD'_l\vg_l\|$ are
\begin{align}
    & \|\mD\|_0 \leq \|\vd\|_0 \leq O(m(\omega L^{3/2})^{2/3}) = O(m\omega^{2/3}L)\label{eqn:append:bound_d0},\\
    & \|\mD'_l \vg_l\|=\|\vd\| \leq O(\omega L^{3/2})\label{eqn:append:bound_d2}.
\end{align}

\textbf{Proof of Statements 2 and 3.}
We prove statements 2 and 3 of Lemma~\ref{lemma:forward_perturbation} by induction on $l\in \{0, 1, \cdots L\}$. 
These statements clearly hold at $l=0$ because there is no perturbation at $l=0$ and $\vg'_0 = \vh'_0 = \vzero$.
In view of the previous part of the proof, 
we assume the lemma holds for layers $0 \leq j \leq l-1$ and 
prove that the second and third statements of the lemma hold at layer $l$. 

Following the given definitions, we expand $\vg_l'$ as follows
\begin{align}
    \vg_l' & = \mW_l \mD_{l-1} \vg_{l-1} - \mW_l^{(0)} \mD_{l-1}^{(0)} \vg_{l-1}^{(0)} \notag\\ 
    & = (\mW_l^{(0)} + \mW_l') (\mD_{l-1}^{(0)} + \mD_{l-1}') (\vg_{l-1}^{(0)} + \vg_{l-1}') - \mW_l^{(0)} \mD_{l-1}^{(0)} \vg_{l-1}^{(0)}\notag\\
    & = \mW_l' (\mD_{l-1}^{(0)} + \mD_{l-1}') (\vg_{l-1}^{(0)} + \vg_{l-1}')  + \mW_l^{(0)} \mD_{l-1}'(\vg_{l-1}^{(0)} + \vg_{l-1}') + \mW_l^{(0)}\mD_{l-1}^{(0)}\vg_{l-1}'.
\end{align}
We first expand $\vg_{l-1}'$ in the last term of the above equation. Similarly, we then iteratively expand $\vg_{l-2}'$, $\ldots$, $\vg_{1}'$ and obtain the following expression:
\begin{align*}
    \vg_l' & = \mW_l' (\mD_{l-1}^{(0)} + \mD_{l-1}') (\vg_{l-1}^{(0)} + \vg_{l-1}')  + \mW_l^{(0)} \mD_{l-1}'(\vg_{l-1}^{(0)} + \vg_{l-1}')\\
    & \quad +  \mW_l^{(0)}\mD_{l-1}^{(0)}\big(\mW_{l-1}' (\mD_{l-2}^{(0)} + \mD_{l-2}') (\vg_{l-2}^{(0)} + \vg_{l-2}')  + \mW_{l-1}^{(0)} \mD_{l-2}'(\vg_{l-2}^{(0)} + \vg_{l-2}')\big)\\
    & \quad + \mW_l^{(0)}\mD_{l-1}^{(0)}\mW_{l-1}^{(0)}\mD_{l-2}^{(0)}\vg_{l-2}'\\
    & = \dots\\
    & = \sum_{k=0}^{l-1} \left(\prod_{j=1}^{k} \mW_{l-j+1}^{(0)}\mD_{l-j}^{(0)}\right) \Big(\mW_{l-k}' (\mD_{l-k-1}^{(0)} + \mD_{l-k-1}') (\vg_{l-k-1}^{(0)} + \vg_{l-k-1}') \\ &
    \quad +\mW_{l-k}^{(0)} \mD_{l-k-1}'(\vg_{l-k-1}^{(0)} + \vg_{l-k-1}') \Big)  + \left(\prod_{j=1}^{l-1} \mW_{l-j+1}^{(0)}\mD_{l-j}^{(0)}\right) \vg_{0}'.
\end{align*}
Since $\vg_0' = \mathbf{0}$, the last term is $\vzero$. We consequently express $\vg_l'$ as a sum of the following two terms:
\begin{align}
    \vg_l' & = \sum_{k=0}^{l-1} \left(\prod_{j=1}^{k} \mW_{l-j+1}^{(0)}\mD_{l-j}^{(0)}\right) \left(\mW_{l-k}' (\mD_{l-k-1}^{(0)} + \mD_{l-k-1}') (\vg_{l-k-1}^{(0)} + \vg_{l-k-1}')\right)\label{eqn:perturbation:W_pime}\\
    &\quad  +\sum_{k=0}^{l-1} \left(\prod_{j=1}^{k} \mW_{l-j+1}^{(0)}\mD_{l-j}^{(0)}\right)\left(\mW_{l-k}^{(0)} \mD_{l-k-1}'(\vg_{l-k-1}^{(0)} + \vg_{l-k-1}') \right).\label{eqn:perturbation:W0Dprime}
\end{align}
We estimate with high probability the above first term (right hand side in \eqref{eqn:perturbation:W_pime}) by using the assumption $\|\mW'\| < \omega$ and the first statement in Lemma~\ref{lemma:initial_forward_matrix_bound} (to bound $\|\prod_{j=1}^k \mW_{l-j+1}^{(0)}\mD_{l-j}^{(0)}\|$, $k=0, 1, \ldots l-1$). We thus obtain with probability at least $1-Le^{\Omega(m/L)}$
\begin{align*}
&\left\|\sum_{k=0}^{l-1} \left(\prod_{j=1}^{k} \mW_{l-j+1}^{(0)}\mD_{l-j}^{(0)}\right) \big(\mW_{l-k}' (\mD_{l-k-1}^{(0)} + \mD_{l-k-1}') (\vg_{l-k-1}^{(0)} + \vg_{l-k-1}')\big) \right\|  \\ 
&\leq L  \max_{k}\left\|\prod_{j=1}^{k} \mW_{l-j+1}^{(0)}\mD_{l-j}^{(0)}\right\| \left\| \mW_{l-k}' (\mD_{l-k-1}^{(0)} + \mD_{l-k-1}') (\vg_{l-k-1}^{(0)} + \vg_{l-k-1}')\right\|\\
& \leq L \cdot O(\sqrt{L}) \cdot  \max_{k}\|\mW_{l-k}'\| \cdot \|\mD_{l-k-1}\| \cdot \|\vg_{l-k-1}^{(0)} + \vg_{l-k-1}'\|\\
& \leq L \cdot O(\sqrt{L}) \cdot  \omega \cdot \frac{\max(|\alpha|, 1)}{\sqrt{1+\alpha^2}}
\cdot \max_{k} \|\vg_{l-k-1}^{(0)} + \vg_{l-k-1}'\|\\
& \leq O(\omega L^{3/2}) \max_{k}\|\vg_{l-k-1}^{(0)} + \vg_{l-k-1}'\|.
\end{align*}
We further use Lemma~\ref{lemma:initial_vector_norm} to bound $\|\vg^{(0)}_{l-k-1}\|$, $k\in\{0, 1, \ldots l-1\}$,  by a constant and use the induction assumption to bound $\|\vg'_{l-k-1}\|$, $k\in\{0, 1, \ldots l - 1\}$, by $O(\omega L^{5/2}\sqrt{\ln m})$. With probability at least $1-O(L)e^{-\Omega(m/L)}$, the first term (right hand side in \eqref{eqn:perturbation:W_pime}) is thus bounded by
\begin{equation}
    O(\omega L^{3/2}) (O(1) + O(\omega L^{5/2}\sqrt{\ln m})) = O(\omega L^{3/2}).\label{eqn:append:bound_first_term_l2}
\end{equation}

In order to bound the second term, which appears in \eqref{eqn:perturbation:W0Dprime}, we denote 
$$\vd_k := \mD'_{l-k-1} (\vg_{l-k-1}^{(0)} + \vg_{l-k-1}'), \quad k=0, 1, \dots l-1 
$$
and
$$
\vy_k := \left(\prod_{j=1}^{k} \mW_{l-j+1}^{(0)}\mD_{l-j}^{(0)}\right)\mW_{l-k}^{(0)} \vd_k.
$$
We show it can be decomposed into $\vy_k = \vy_{k, 1} + \vy_{k, 2}$, where with probability at least $1-Le^{-\Omega(m/L)}$, $$\|\vy_{k,1}\|\leq O\left(\frac{(1-\alpha) \omega L^{3/2}}{(1+\alpha^2)^{1/2}\sqrt{m}}\right),\quad \|\vy_{k, 2}\|_\infty \leq O\left(\frac{(1-\alpha)
\omega L^{3/2} \sqrt{\ln m}}{(1+\alpha^2)^{1/2}\sqrt{m}}\right).$$ 
Denoting $\vu_k:= \mD_{l-1}^{(0)}\mW_{l-1}^{(0)},...\mD_{l-k}^{(0)}\mW_{l-k}^{(0)}\vd_k$ and applying the induction assumption we note that $\|\vd_k\|_0<O(m\omega^{2/3}L)$. Next, we apply the third statement of the Lemma~\ref{lemma:initial_forward_matrix_bound} for $\vu_k$ (instead of $\vv$) and obtain that with probability at least $1-e^{-\Omega(m/L)}$
\begin{equation}
    \|\vu_k\|\leq 4 \|\vd_k\|.\label{eqn:append:l2_u}
\end{equation}
We note that $\vy_k=\mW_{l}^{(0)}\vu_k$ and thus $
\vy_k| \vu_k \sim N\left(0, \frac{2\|\vu_k\|^2}{m}\mI\right)$.

We denote $y_{k, j}:=(\vy_k)_j$ and $\sigma^2 := 2\|\vu_k\|^2/m$ and we let $b = O(\|\vu_k\|\sqrt{\ln m/m})$. 
We investigate the tail probability of the Gaussian random variable $y_{k, j}$ conditioned on $\vu_k$. It is clear that 
\begin{equation}
    \sP(|y_{k,j}| \geq b t | \vu_k) \leq \frac{1}{\sqrt{2\pi}bt/\sigma}e^{-b^2t^2/2\sigma^2} \quad \forall t\in\sN.\label{eqn:append:prob_ykj_greater_bt}
\end{equation}

We denote $R_t := \{j: y_{k, j} \geq b t\}\subset[m]$ and $r_t:=\sqrt{m} / ((\ln m)^2 t^2)$. Using the independence of $\{y_{k, j}\}_{j \in [m]}$ given $\vu_k$ and applying a union bound for \eqref{eqn:append:prob_ykj_greater_bt} yield
\begin{align*}
 \sP(|R_t| \geq r_t|\vu_k) &\leq \left(\begin{matrix}
     m\\ r_t
 \end{matrix}\right) \times \left(\frac{1}{\sqrt{2\pi}bt/\sigma}e^{-b^2t^2/2\sigma^2}\right)^{r_t} \\
 &\leq \left(\frac{\|\vu_k\|}{\sqrt{\pi}bt\sqrt{m(1+\alpha^2)}}\right)^{r_t}\left(\frac{me}{r_t}\right)^{r_t}e^{-\Omega(b^2t^2m r_t)}\\
 & \leq O(1) \exp\left(-\Omega(b^2t^2m r_t) + \left(\frac{1}{2} \ln m - \ln b - \Omega(1)\right)r_t\right).
\end{align*} 
Denoting $q:=\sqrt{m}/\ln^2 m$, we simplify the above bound as follows
$$
 \sP(|R_t| \geq q/t^2) \leq e^{-\Omega(b^2 q m)}.
$$

We further denote $Q:=\{0,1, 2, 3, .. \lfloor\frac{1}{2}\log_2 q\rfloor\}$,  $N_Q:=\lfloor\frac{1}{2}\log_2 q\rfloor$ and $T:=\{2^p: p\in Q\}$. We designate the elements in $T$ by $t_p := 2^p$ for $p\in Q$.  
Let $t_{N_Q +1} := 2^{\lfloor\frac{1}{2}\log_2 q\rfloor + 1} \equiv 2^{N_Q + 1}$ and notice that $t_{N_Q + 1}^2 > q$. Thus, applying the above estimate and a union bound over $t\in T$ and $t_{N_Q+1}$ 
$$
    |R_t| < q/t^2, \ \forall t\in T, \ \text{ and } \  |R_{t_{N_Q + 1}}| < 1
    \ \text{ with probability at least } \ 1-(|T| + 1) e^{-\Omega(b^2q m)}.
$$

By definition, we note that when $|R_{t_{N_Q + 1}}| = 0$ and  $|y_{k, j}| < t_{N_Q+1}$ for $j\in R_{t_{N_Q}}$. We also note that for $j\in R_{t_p}\setminus R_{t_{p+1}}$, $|y_{k, j}| < t_{p+1}$. Thus, for $R:= R_1 \equiv \{j: |y_{k, j}| \geq b\}$, we bound $\sum_{j\in R} y_{k, j}^2$ with high probability  as follows
\begin{align*}
    \sum_{j\in R} y_{k, j}^2  & = \sum_{j\in R / R_{t_{N_Q}}} y_{k, j}^2 + \sum_{j\in R_{t_{N_Q}}} y_{k, j}^2 \leq \sum_{j\in R /R_{t_{N_Q}}} y_{k, j}^2  + |R_{t_{N_Q}}|(bt_{N_Q+1})^2\\ 
    & \leq \sum_{j\in R / R_{t_{N_Q}}/ R_{t_{N_Q-1}}} y_{k, j}^2+ 
    \sum_{j\in R_{t_{N_Q-1}}/R_{t_{N_Q}}} y_{k, j}^2  + |R_{t_{N_Q}}|(bt_{N_Q+1})^2 \\
    & \leq \sum_{j\in R / R_{t_{N_Q}}/ R_{t_{N_Q-1}}} y_{k, j}^2+ 
    |R_{t_{N_Q-1}}|(bt_{N_Q})^2 + |R_{t_{N_Q}}|(bt_{N_Q+1})^2 \\
    ...&...\\
    & \leq \sum_{p\in Q}|R_{t_p}| (b2^{p+1})^2\leq \sum_{p\in Q}  q/t_p^2 (b2^{p+1})^2\\
    & = \sum_{p\in Q} qb^2 2^2 = O(q b^2\ln q ) \ \text{ with probability at least } 1-\Omega(|T|) e^{-\Omega(b^2qm)}.
\end{align*}
Since $b = O\left(\|\vu_k\|\sqrt{\ln m/m}\right)$
 and $q = \sqrt{m}/\ln^2 m$, we express the above bound as 
\begin{equation}
    \sum_{j\in R}y_{k, j}^2 \leq O(\|\vu_k\|^2/m) \ \text{ with probability at least } \ 1- e^{-\Omega(\frac{m^{1/2}}{\ln m})}.\label{eqn:append:bound_sum_square_S}
\end{equation}

 We split vector $\vy_k$ into $\vy_k = \vy_{k,1} + \vy_{k, 2}$ using the indices set $R$ as
 \begin{align}
      \vy_{k, 1} = (y_{k, 1} 1_{1\in R}, y_{k, 2} 1_{2\in R}, \ldots, y_{k, m} 1_{m\in R})^T,\label{eqn:append:split_yk1} \\ 
 \vy_{k, 2} = (y_{k, 1} 1_{1\notin R}, y_{k, 2} 1_{2\notin R}, \ldots, y_{k, m} 1_{m\notin R})^T.\label{eqn:append:split_yk2}
 \end{align}
Using \eqref{eqn:append:bound_sum_square_S} and the definition of $R$, and then the induction assumption on the bound of $\|\vd_k\|$ and \eqref{eqn:append:l2_u} yield the following estimates with probability at least $1- e^{-\Omega\left(\frac{m^{1/2}}{\ln m}\right)}$:
\begin{align}
 &   \|\vy_{k, 1}\| \leq O\left((\frac{\|\vu\|}{m^{1/2}}\right) \leq O\left((\frac{(1-\alpha)\omega L^{3/2}}{(1+\alpha^2)^{1/2}m^{1/2}}\right),\label{eqn:append:bound_yk1}
\\
&\|\vy_{k, 2}\|_\infty \leq b = O\left(\frac{\|\vu\|\sqrt{\ln m}}{\sqrt{m}}\right) \leq O\left(\frac{(1-\alpha)\omega L^{3/2}\sqrt{\ln m}}{(1+\alpha^2)^{1/2}\sqrt{m}}\right).\label{eqn:append:bound_yk2}
\end{align}
Following the later decomposition of $\vy_k$ (with the components in \eqref{eqn:append:split_yk1} and \eqref{eqn:append:split_yk2}), we decompose the term in \eqref{eqn:perturbation:W0Dprime} into $\sum_{k=0}^{l-1} \vy_{k, 1}$ and $\sum_{k=0}^{l-1} \vy_{k, 2}$. We denote $\vg_{l, 2}' := \sum_{k=0}^{l-1}\vy_{k, 2}$ and $\vg_{l, 1}' := \vg_{l}' - \vg_{l, 2}'$. We note that $\vg_{l, 1}'$ is the sum of the term in \eqref{eqn:perturbation:W_pime} and $\sum_{k=0}^{l-1}\vy_{k,1}$. By using the bound of \eqref{eqn:perturbation:W_pime} given in \eqref{eqn:append:bound_first_term_l2} and \eqref{eqn:append:bound_yk1}, we bound $\vg_{l, 1}'$ as follows
\begin{align*}
    \|\vg_{l, 1}'\| &\leq \left\|\sum_{k=0}^{l-1} \left(\prod_{j=1}^{k} \mW_{l-j+1}^{(0)}\mD_{l-j}^{(0)}\right) \left(\mW_{l-k}' (\mD_{l-k-1}^{(0)} + \mD_{l-k-1}') (\vg_{l-k-1}^{(0)} + \vg_{l-k-1}')\right)\right\|  \\
    & \quad + \sum_{k=0}^{l-1}\|\vy_{k, 1}\|\\
    &\leq O(\omega L^{3/2}) + \sum_{k=0}^{l-1} \|\vy_{k, 1}\| \\
    &\leq O(\omega L^{3/2}) + L\max_{k\in\{0, 1,\ldots, l-1\}} \|\vy_{k, 1}\|\\
    & \leq  O(\omega L^{3/2}) + L O\left(\frac{(1-\alpha)\omega L^{3/2}}{(1+\alpha^2)^{1/2}m^{1/2}}\right).
\end{align*}
Using the fact that $m\geq \Omega(L^2)$, we show the $\ell_2$ norm for $\vg_{l, 1}'$ in the second statement of this lemma holds:
$$
\|\vg_{l, 1}'\| \leq  O(\omega L^{3/2}) + L O\left(\frac{(1-\alpha)\omega L^{3/2}}{(1+\alpha^2)^{1/2}m^{1/2}}\right)\leq O(\omega L^{3/2}).
$$

Applying the induction assumption, i.e., $\|\vg_{l-k, 1}'\|\leq O(\omega L^{3/2})$ for $k\in \{0, 1, \ldots,l-1 \}$, and \eqref{eqn:append:bound_yk2}, we conclude the second statement of the lemma for layer $l$ as follows
\begin{equation}
    \|\vg_{l, 2}'\|_\infty \leq \sum_{k=0}^{l-1} \|\vy_{k,2}\|_\infty \leq L O\left(\frac{1-\alpha}{(1+\alpha^2)^{1/2}}\frac{\sqrt{\ln m} \omega L^{3/2}}{\sqrt{m}}\right) =  O \left(\frac{\sqrt{\ln m} \omega L^{5/2}}{\sqrt{m}}\right).\label{eqn:append:bound_gl2_inf_norm}
\end{equation}

Finally, we note that  $\|\vg_l'\|\leq \|\vg_{l, 1}'\| + \|\vg_{l,2}'\|$, and thus the first part $\|\vg_{l, 1}'\|$ is bounded by $O(\omega L^{3/2})$. Furthermore, applying \eqref{eqn:append:bound_gl2_inf_norm}, we bound the second part, $\|\vg_{l,2}'\|$, as follows
$$
\|\vg_{l,2}'\| = \sqrt{\sum_{j\in S_2} g_{l, 2, j}^2} \leq \sqrt{m \frac{\ln m\omega^2 L^5}{m}} = \sqrt{\ln m} \omega L^{5/2}.
$$
By definition, $\vh_{l}' = \mD \vg_l' + \mD'\vg^{(0)}_l +\mD'\vg'_l = \mD \vg_l' + \mD'\vg_l$. Applying $\|\mD\| \leq 1$, $\|\vg_l'\|\leq O(\omega L^{5/2}\sqrt{\ln m})$ and $\|\mD'\vg_l\|\leq O(\omega L^{3/2})$, we bound the norm of $\vh_l'$ in the following way
$$\|\vh_l'\|\leq O(1)O(\omega L^{5/2}\sqrt{\ln m}) + O(\omega L^{3/2}) = O(\omega L^{5/2}\sqrt{\ln m}).$$
Thus the third statement of this lemma is concluded for layer $l$. 
\end{proof}

\begin{lemma}\label{lemma:intermediate_perturbation} 
For given integer $a$,$b$ as $1\leq a < b\leq L$, and if $d<O\left(\frac{m}{L\ln m}\right)$, $\|\mW'\|\leq \omega < O\left(\frac{1}{L^{9/2}\ln^{3/2}m}\right)$.
Then we obtain that with probability at least $1-e^{-\Omega(m/L)}$ 
\begin{enumerate}
    \item $\|\mW^{(0)}_b (\mD^{(0)}_{i, b-1} + \mD'_{i, b-1}) \mW^{(0)}_{b-1}... (\mD^{(0)}_{i, a} + \mD'_{i, a}) \mW^{(0)}_a\| \leq O(\sqrt{L})$.
    \item $\|(\mW^{(0)}_b + \mW'_b) (\mD^{(0)}_{i, b-1} + \mD'_{i, b-1}) (\mW^{(0)}_{b-1} + \mW'_{b-1})... (\mD^{(0)}_{i, a} + \mD'_{i, a}) (\mW^{(0)}_a + \mW'_a)\| \leq O(\sqrt{L})$.
    \item $
    \|\mW_{b+1}^{(0)} (\mD^{(0)}_{i, b} + \mD'_{i, b}) (\mW^{(0)}_{b} + \mW'_{b})... (\mD^{(0)}_{i, a} + \mD'_{i, a})  - \mW_{b+1}^{(0)}\mD^{(0)}_{i, b}\mW_b^{(0)}...\mW_{a+1}^{(0)}\mD_{i, a}^{(0)}\|  
    \leq O\left(\frac{1-\alpha}{\sqrt{1+\alpha^2}}L^{3/2}\right).$
\item $\|\mB(\mD^{(0)}_{L} + \mD'_{L}) (\mW^{(0)}_{L} + \mW'_{L})... (\mD^{(0)}_{i, a} + \mD'_{i, a})  - \mB\mD^{(0)}_{L}\mW_L^{(0)}...\mW_{a+1}^{(0)}\mD_{i, a}^{(0)}\| \\
\leq O\left(\frac{1-\alpha}{\sqrt{1+\alpha^2}}\frac{\omega^{1/3}L^2 \sqrt{m\ln m}}{\sqrt{d}}\right).$

\end{enumerate}
\end{lemma}
The proof of this lemma follows the same arguments of the proofs of Lemmas~8.6 and~8.7 in \citet{allen2019convergence}, but uses instead Lemma~\ref{lemma:forward_perturbation} and the fact that $\|\mD'\|=(1-\alpha)/\sqrt{1+\alpha^2}$. 

\subsection{Gradient Bounds and Proof of Lemma~\ref{thm:gradient_bound}}\label{appd:gradient}

We first introduce two lemmas (Lemmas~\ref{lemma:gradient_upper_bound_init} and~\ref{lemma:gradient_lower_bound_init}) that provide upper and lower bounds for the Frobenius norm of a certain matrix-valued function $\mG_{i, l}(\vv; \mW^{(0)})$ with randomly initialized parameters $\mW^{(0)}$. This function, which 
is defined below in \eqref{eqn:append:def_mG_func}  equals the gradient of the loss function when $\vv=\ve_i^{(0)}
\equiv \mB\vh_{L,i}^{(0)} - \vy_i$. At last, we conclude Lemma~\ref{thm:gradient_bound} by applying the perturbation bounds of  Lemmas~\ref{lemma:forward_perturbation} and~\ref{lemma:intermediate_perturbation} in order to show that the order of the bounds in Lemmas~\ref{lemma:gradient_upper_bound_init} and~\ref{lemma:gradient_lower_bound_init} are not affected by a small perturbation $\mW'$ as long as  $\|\mW'\| \leq \omega<O\left(\frac{\delta^{3/2}}{n^{3/2}L^{15/2}\ln^{3/2} m}\right)$.

We remark that the proof of Lemma~\ref{lemma:gradient_upper_bound_init} is straightforward and follows \citet{allen2019convergence}. The proof of Lemma~\ref{lemma:gradient_lower_bound_init} follows ideas of \citet{zou2019improved}, while adapting it to Leaky ReLUs and improving the lower bound of $\|\nabla_\mW\gL(\mW^{(0)})\|_F^2$ by quantifying lower bounds for layers before $L$ instead of only using $\|\nabla_{\mW_L}\gL(\mW^{(0)})\|_F^2$ as done in \citet{zou2019improved}. This improvement reduces a factor $L$ in the lower bound, which will eventually make the learning rate of the desired theory independent of $L$. The idea of concluding Lemma~\ref{thm:gradient_bound} by examining the effect of a small perturbation on the parameter follows \citet{allen2019convergence}.

We define the matrix-valued function, $\mG_{i,l}(\vv; \mW)$, for $l\in[L]$ and $i\in[n]$ and $\vv\in\sR^d$ as follows
\begin{equation}
    \mG_{i,l}(\vv; \mW) := \mD_{i, l} \textbf{Back}^T_{i,l}\vv \vh^T_{i,l-1}=(\textbf{Back}_{i,l}\mD_{i, l})^T\vv \vh^T_{i,l-1}.
    \label{eqn:append:def_mG_func}
\end{equation}
We note that $\mG_{i,l}(\vv; \mW)$ is related to the gradient of the loss function as follows: $$\mG_{i, l}(\ve_i;\mW)\equiv \nabla_{\mW_l} \text{loss}(\vx_i, \vy_i;\mW).$$

\begin{lemma}\label{lemma:gradient_upper_bound_init}
    Assume the setup of \S\ref{sec:problem_setup} with  randomly initialized   $\mW^{(0)}$. If $d \leq O(\frac{m}{L\ln m})$, then with probability at least $1- e^{-\Omega(m/L)}$ 
    \begin{equation}
    \|\mG_{i, l} (\vv;\mW^{(0)})\|_F^2 \leq O\left(\frac{m}{d}\right)\|\vv\|^2.
    \end{equation}
 
    \end{lemma}
\begin{proof}
The second statement in Lemma~\ref{lemma:initial_forward_matrix_bound} implies that $\|\textbf{Back}_{i, l}\| < O(\sqrt{\frac{m}{d}})$ with probability at least $1- e^{-\Omega(m/L)}$ and therefore 
\begin{align*}
\|\mG_{i, l}(\vv; \mW^{(0)})\|_F^2
&\leq \| \mD_{i, l}\textbf{Back}_{i, l}^{(0)T}\vv \vh_{i, l-1}^{(0)T}\|_F^2 \\
&\leq \| \mD_{i, l}\textbf{Back}_{i, l}^{(0)T}\vv \|^2 \|\vh_{i, l-1}^{(0)T}\|^2\\
& \leq O\left(\frac{m}{d}\|\vv_i\|^2_2\right).
\end{align*}
\end{proof}

\begin{lemma}\label{lemma:gradient_lower_bound_init}
    Assume the setup of \S\ref{sec:problem_setup} and with randomly initialized $\mW^{(0)}$. For any set of vector $\{\vv_i\}_{i=1}^n  \subset \sR^d$, 
    $$\|\sum_{i=1}^n\mG_{i, l}(\vv_i;\mW^{(0)})\|^2_F \geq \Omega\left(\frac{(1-\alpha)^2}{(1+\alpha^2)}\frac{\delta m}{ndL}\right)\sum_{i=1}^n\|\vv_i\|^2 \ \text{ with probability } \geq 1- e^{-\Omega(m\delta^2)}.
    $$
\end{lemma}
\begin{proof}
We separate the proof of this lemma into four parts. In the first part, we define a set in $\sR^m$ (see \eqref{eqn:append:def_gW_set} below) and show two important properties of this set (see \eqref{eqn:append:empty_set_gW} and \eqref{eqn:append:gW_probabbility1} below). In the second part, 
we establish a lower bound for a useful function (as defined in \eqref{eqn:append:def_b_func} below) with a probability at least $0.5$. In the third part, we use this lower bound to establish a lower bound of the loss function with a positive probability. In the fourth part, we conclude the lemma by using all the results proved in the former three parts.

Since we assume randomly initialized parameters without training, we simply denote $\vh_{i, l}:= \vh_{i, l}^{(0)}$ and $\mW := \mW^{(0)}$ across this proof.

\textbf{Part 1.} 
We arbitrarily fix $l\in [L]$ and recall that $\vh_{i, l}$ is the output of $l$th layer. We denote $$\hat{\vh}_{i, l} := \vh_{i, l}/\|\vh_{i, l}\|.$$ We form an orthogonal matrix $\mQ_{i, l}\in\sR^{m\times m}$ whose first column is $\hat{\vh}_{i, l}$.
We denote the matrix in $\sR^{m\times (m-1)}$ which completes this vector by $\tilde{\mQ}_{i, l}$, that is,
 $\mQ_{i, l} := [\hat{\vh}_{i, l}, \tilde{\mQ}_{i, l}]$.

For a small constant $c_1>0$ (the choice of $c_1$ will be determined during the proof), we let $\gamma=c_1\delta/(nL\sqrt{m})$. For $i\in[n]$ and the fixed $l\in [L]$, we define 
\begin{equation}
    \gW_{i, l} := \{\vw\in \sR^m: |\hat{\vh}_{i, l}^T \vw| < \gamma, |\langle\tilde{\mQ}_{i, l}\tilde{\mQ}_{i, l}^T\vw, \hat{\vh}_{j, l}\rangle| > 2\gamma, \forall j\in [n], j\neq i\}\subset\sR^m.\label{eqn:append:def_gW_set}
\end{equation}

We prove that for any choice of $\gamma$ the sets $\gW_{i, l}$, $i \in [n]$, have no intersection, that is, \begin{equation}
    \gW_{i, l}\cap \gW_{j, l} = \emptyset, \quad \forall i\neq j \in [n].\label{eqn:append:empty_set_gW}
\end{equation}

For any $\vw\in \gW_{i, l}$, we need to prove that $\vw\notin \gW_{j, l}$, where $j\neq i\in [n]$. We prove this by contradiction. Given $\vw\in\gW_{i, l}$, we assume that there exists $j\neq i \in [n]$ such that $\vw \in \gW_{j, l}$.
Since $\tilde{\mQ}_{j, l} \tilde{\mQ}_{j, l}^T = \mI - \hat{\vh}_{j, l}\hat{\vh}_{j, l}^T$, we rewrite $\tilde{\mQ}_{j, l}\tilde{\mQ}_{j, l}^T \vw$ as
\begin{equation}
    \tilde{\mQ}_{j, l} \tilde{\mQ}_{j, l}^T\vw  = (\mI - \hat{\vh}_{j, l}\hat{\vh}_{j, l}^T) \vw = \vw - \langle\vw , \hat{\vh}_{j, l}\rangle \hat{\vh}_{j, l}.\label{eqn:append:rewrite_tilde_Q}
\end{equation}
Applying \eqref{eqn:append:rewrite_tilde_Q} and the fact that   $\langle\vw , \hat{\vh}_{i, l}\rangle < \gamma$ and $\langle\vw , \hat{\vh}_{j, l}\rangle < \gamma$ for $\vw\in \gW_{i, l}\cap \gW_{j, l}$ results in
\begin{align*}
    |\langle\tilde{\mQ}_{j, l} \tilde{\mQ}_{j, l}^T\vw, \hat{\vh}_{i, l}\rangle| & = |\langle \vw - \langle\vw , \hat{\vh}_{j, l}\rangle  \hat{\vh}_{j, l},  \hat{\vh}_{i, l}\rangle| \\
    &\leq |\langle \vw , \hat{\vh}_{i, l}\rangle| + |\langle\vw , \hat{\vh}_{j, l}\rangle \langle \hat{\vh}_{j, l},  \hat{\vh}_{i, l}\rangle|\\
    & < \gamma + \gamma|\langle \hat{\vh}_{j, L},  \hat{\vh}_{i, L}\rangle|  \\
    & \leq 2\gamma.
\end{align*}
On the other hand, since  $\vw \in \gW_{j, l}$, $|\langle\tilde{\mQ}_{j, l}\tilde{\mQ}_{j, l}^T \vw, \hat{\vh}_{i, l}\rangle| > 2\gamma$ for $i\neq j$, which contradicts the above equation.
Therefore, we conclude \eqref{eqn:append:empty_set_gW}.

Next, we assume  $\vw\sim N(0, \frac{2}{m} \mI)$ and prove that
\begin{equation}
    \sP(\vw\in \gW_{i, l}) \geq \Omega\left(\frac{\delta}{nL}\right).\label{eqn:append:gW_probabbility1}
\end{equation}
The orthogonality of $\mQ_{i, l}$ 
implies that $\hat{\vh}_{i, l}^T\vw$ and $\tilde{\mQ}_{i, l}^T \vw$ are independent. We thus express the probability \eqref{eqn:append:gW_probabbility1} as follows
\begin{equation}
    \sP(\vw \in \gW_{i, l}) = \sP(|\hat{\vh}_{i, l}^T\vw| < \gamma) \sP(|\langle\tilde{\mQ}_{i, l}\tilde{\mQ}_{i, l}^T \vw, \hat{\vh}_{j, l}\rangle| > 2\gamma, \ \forall j\in[n], j\neq i). \label{eqn:append:probability_gW}
\end{equation}
We note that $\hat{\vh}_{i, l}^T\vw \sim N(0, \frac{2}{m})$ and thus express the first multiplicative term in \eqref{eqn:append:probability_gW} as
\begin{align}
    \sP(|\hat{\vh}_{i, l}^T\vw| < \gamma) &= \frac{\sqrt{m}}{\sqrt{4\pi}} \int_{-\gamma}^\gamma e^{-\frac{mx^2}{4}} dx  \geq \Omega\left(\gamma\sqrt{m} \right), \ \text{when } \gamma\sqrt{m} < 1.\label{eqn:append:prob_less_gamma}
\end{align} 
To express the second multiplicative term of \eqref{eqn:append:probability_gW}, we first derive the distribution of $\hat{\vh}_{j, l}^T \tilde{\mQ}_{i, l}\tilde{\mQ}_{i, l}^T \vw$. Since $\tilde{\mQ}_{i, l} \tilde{\mQ}_{i, l}^T = \mI_m - \hat{\vh}_{i, l}\hat{\vh}_{i, l}^T$ and $\tilde{\mQ}_{i, l}^T \tilde{\mQ}_{i, l} = \mI_{m-1}$, 
\begin{align*}
    \hat{\vh}_{j, l}^T \tilde{\mQ}_{i, l}\tilde{\mQ}_{i, l}^T \vw &\sim N\left(0,\frac{2}{m} \hat{\vh}_{j, l}^T \tilde{\mQ}_{i, l} \tilde{\mQ}_{i, l}^T\tilde{\mQ}_{i, l} \tilde{\mQ}_{i, l}^T \hat{\vh}_{j, l}\right)\\
& =N\left(0, \frac{2}{m}\hat{\vh}_{j, l}^T (\mI - \hat{\vh}_{i, l}\hat{\vh}_{i, l}^T) \hat{\vh}_{j, l}\right) \\ 
& = N\left(0, (1 - \langle\hat{\vh}_{j, l}, \hat{\vh}_{i, l}\rangle^2)\frac{2}{m}\right).
\end{align*} 
By Lemma~\ref{lemma:separation_layers}, we recall that with probability at least $1-e^{-\Omega(\delta^4 m /L^4)}$,
$$
\langle \hat{\vh}_{i, L}, \hat{\vh}_{j, L} \rangle^2 \leq 1-\Omega(\delta^2/L^2), \ \text{ for all } i\neq j \in [n].
$$
We thus note that $\hat{\vh}_{j, l}^T \tilde{\mQ}_{i, l}\tilde{\mQ}_{i, l}^T \vw \sim N(0, \tau^2)$, where $\tau^2$ is greater than $\Omega(\delta^2/mL^2)$.
Consequently, 
\begin{align*}
    \sP(|\hat{\vh}_{j, l}^T \tilde{\mQ}_{i, l}\tilde{\mQ}_{i, l}^T \vw |<2\gamma) & = \frac{1}{\sqrt{2\pi\tau^2}}\int_{-2\gamma}^{2\gamma}\exp\left(-\frac{x^2}{2\tau^2}\right) dx \leq O\left(\frac{\gamma}{\tau}\right) \leq O\left(\frac{\gamma L\sqrt{m}}{\delta}\right).
\end{align*}
Applying a union bound over all $j\in[n]$, $j \neq i$, yields
$$
\sP\left( 
\exists 
j\in[n], j\neq i \
\text{ such that }
|\hat{\vh}_{j, l}^T \tilde{\mQ}_{i, l}\tilde{\mQ}_{i, l}^T \vw|\leq 2\gamma
\right) \leq nO\left(\frac{\gamma L\sqrt{m}}{\delta}\right).
$$
Consequently,
\begin{equation}
    \sP\left(|\hat{\vh}_{j, l}^T \tilde{\mQ}_{i, l}\tilde{\mQ}_{i, l}^T \vw|>2\gamma \ \forall j\in[n], j\neq i\right) \geq 1 - O\left(\frac{\gamma n L\sqrt{m}}{\delta}\right).\label{eqn:append:prob_greater_2gamma}
\end{equation}
Plugging \eqref{eqn:append:prob_greater_2gamma} and \eqref{eqn:append:prob_less_gamma} into \eqref{eqn:append:probability_gW} yields
\begin{align*}
    \sP(\vw\in \gW_{i, l}) & = \sP(|u_{i, 1}| < \gamma) \sP(|\vv_{i, j}|>2\gamma, \ \forall j\in[n], j\neq i)  \\ &\geq \Omega\left(\gamma\sqrt{m}\right)\left(1-O\left(\frac{\gamma n L\sqrt{m}}{\delta}\right)\right).
\end{align*}
Recall that $\gamma = c_1\delta/(nL\sqrt{m})$, we select small $c_1$ such that both $O(\frac{\gamma n L\sqrt{m}}{\delta}) = O(1)\cdot c_1<1$ and $\gamma\sqrt{m} = c_1\delta/(nL) < 1$. We thus conclude this part as follows
$$
    \sP(\vw\in \gW_{i, l}) \geq \Omega\left(\frac{\delta}{nL}\right).
 $$

\textbf{Part 2. } Given integer $k\in [m]$ and $l\in [L]$, we define the following vector-valued function for $\va = (a_1, ...a_n)^T\in \sR^{n}$ and $\vw\in \sR^m$:
\begin{equation}
    \vb_{k, l}(\vw, \va) := \sum_{i=1}^n a_i \tilde{\sigma}_\alpha'(\langle\vw, \vh_{i, l}\rangle) \vh_{i, l}.\label{eqn:append:def_b_func}
\end{equation}
We prove that conditioning on the event $\vw\in\gW_{i, l}$, a certain lower bound of $\|\vb_{k, l}(\vw, \va)\|$ is achieved with a probability at least 0.5, that is,
$$
P\left(\|\vb_{k, l}(\vw, \va)\| \geq \frac{a_i}{2}\frac{(1-\alpha)}{\sqrt{1+\alpha^2}}\|\vh_{i, l}\|\;\big|\;\vw \in \gW_{i, l}\right) > \frac{1}{2}.
$$ 

We rewrite $\vw$ as $\vw= \mQ_{i, l}\mQ_{i, l}^T\vw=(\hat{\vh}_{i, l}^T \vw)\hat{\vh}_{i, l} + \tilde{\mQ}_{i, k}\tilde{\mQ}_{i, k}^T \vw$,
\begin{align}
& \langle\vw, \hat{\vh}_{j, l}\rangle = (\hat{\vh}_{i, l}^T \vw)\langle\hat{\vh}_{i, l}, \hat{\vh}_{j, l} \rangle + \langle \tilde{\mQ}_{i, k}\tilde{\mQ}_{i, k}^T \vw,  \hat{\vh}_{j, l}\rangle \ \text{ for }\ j \neq i \notag
\end{align}
Using the following two facts: $\vw\in \gW_{i, l}$ and both $\hat{\vh}_{i, l}$ and $\hat{\vh}_{j, l}$ are unit vectors, we bound the absolute value of the first term of the above expression as folows
$$|(\hat{\vh}_{i, l}^T \vw)||\langle\hat{\vh}_{i, l}, \hat{\vh}_{j, l} \rangle| < \gamma.$$ 
Since $\vw\in \gW_{i, l}$, the magnitude of the second term is greater than $2\gamma$. We note that the sign of $\langle\vw, \hat{\vh}_{j, l}\rangle$ is the same as that of $\langle \tilde{\mQ}_{i, k}\tilde{\mQ}_{i, k}^T \vw,  \hat{\vh}_{j, l}\rangle$. This and the piecewise linearity of the Leaky ReLU function imply that for $\vw\in\gW_{i, l}$
\begin{equation}
    \tilde{\sigma}_\alpha'(\langle\vw, \hat{\vh}_{j, l}\rangle) =     \tilde{\sigma}_\alpha'(\langle \tilde{\mQ}_{i, k}\tilde{\mQ}_{i, k}^T \vw,  \hat{\vh}_{j, l}\rangle\rangle) , \ \text{ for } \ j\neq i.\label{eqn:append:sigma_tilde_div_innerprod}
\end{equation}
We note \eqref{eqn:append:sigma_tilde_div_innerprod} implies the following expression for $\vb_{k, l}(\vw,\va)$ for $\vw \in \gW_{i, l}$:
by , \begin{align*}
    \vb_{k, l}(\vw,\va) & = a_i\tilde{\sigma}_\alpha'(\hat{\vh}_{i, l}^T \vw)\vh_{i, l} + \sum_{j\neq i} a_j\tilde{\sigma}_\alpha'(\hat{\vh}_{j, l}^T \vw)\vh_{j, l}\\
    & = a_i\tilde{\sigma}_\alpha'(\hat{\vh}_{i, l}^T \vw)\vh_{i, l} + \sum_{j\neq i} a_j\tilde{\sigma}_\alpha'(\langle \tilde{\mQ}_{i, k}\tilde{\mQ}_{i, k}^T \vw,  \hat{\vh}_{j, l}\rangle\rangle)\vh_{j, l} \\
    & = a_i\frac{(1-\alpha)}{\sqrt{1+\alpha^2}}1_{\hat{\vh}_{i, l}^T \vw > 0}\vh_{i, l} + a_i\frac{\alpha}{\sqrt{1+\alpha^2}}\vh_{i, l} + \sum_{j\neq i} a_j\tilde{\sigma}_\alpha'(\langle \tilde{\mQ}_{i, k}\tilde{\mQ}_{i, k}^T \vw,  \hat{\vh}_{j, l}\rangle\rangle)\vh_{j, l}. 
\end{align*}
We denote
\begin{align*}
    \vb_1 &:=  a_i\frac{(1-\alpha)}{\sqrt{1+\alpha^2}}\vh_{i, L-1}\\
    \vr &:=  a_i\frac{\alpha}{\sqrt{1+\alpha^2}}\vh_{i, L-1} + \sum_{j\neq i} a_j\phi_\alpha'(\langle\tilde{\mQ}_{i}\tilde{\vu}_i, \hat{\vh}_{j, L-1}\rangle)\vh_{j, L-1},
\end{align*}
and thus express $\vb_{k, l}(\vw, \va)$ as follows
\begin{equation}
    \vb_{k, l}(\vw, \va) = \vb_1 1_{\hat{\vh}_{i, l}^T \vw > 0}+\vr.\label{eqn:append:vb_relationship}
\end{equation}
By symmetry of normal distribution, we know that $\hat{\vh}_{i, l}^T \vw>0$ with  probability $0.5$. 
We also note that $\hat{\vh}_{i, l}^T\vw$ and $\tilde{\mQ}_{i, l}^T\vw$ are independent and thus $1_{\hat{\vh}_{i, l}^T \vw > 0}$ is independent with $\vr$. 

We consider two possibility for $\vr$: 
\begin{itemize}
    \item When $\|\vr\|\geq \frac{1}{2}\|\vb_1\|$, we know that with probability $0.5$, $\hat{\vh}_{i, l}^T \vw \leq 0$, which implies $\vb_{k, l}(\vw, \va) = \vr$, and thus $\|\vb_{k, l}(\vw, \va)\| \geq\frac{1}{2}\|\vb_1\|$. We thus note that at least with probability $0.5$ that $\|\vb_{k, l}(\vw, \va)\| \geq\frac{1}{2}\|\vb_1\|$.
    \item When $\|\vr\| < \frac{1}{2}\|\vb_1\|$, we note that $\hat{\vh}_{i, l}^T \vw > 0$ with probability $0.5$, then by triangle inequality, we imply $\|\vb_{k, l}(\vw, \va)\|\geq \|\vb_1\| - \|\vr\| \geq \frac{1}{2}\|\vb_1\|$.
\end{itemize}
We conclude that
\begin{equation}
    P\left(\|\vb_{k, l}(\vw, \va)\| \geq \frac{a_i}{2}\frac{(1-\alpha)}{\sqrt{1+\alpha^2}}\|\vh_{i, l}\||\;\big|\;\vw \in \gW_{i, l}\right) \geq \frac{1}{2}.\label{eqn:append:prob_b_one_half}
\end{equation}

\textbf{Part 3.} The proof of this part does not depend on a particular choice of $i\in[n]$. For simplicity, we thus drop the subscript $i$ in this part.

For $\vv\in\sR^{d}$, $k\in[m]$ and $l\in[L]$, we define $a_{k, l}:=\langle(\mathbf{Back}_{l})_{\cdot, k}, \vv\rangle$.
We want to show that for any integers $k\in[m]$ and $l\in[L]$,
\begin{equation}
    \sP\left((a_{k, l})^2 \geq O\left(\frac{\|\vv\|^2}{d}\right)\right) > 1-\exp(-O(1)).\label{eqn:append:lower_bound_akli}
\end{equation}
To prove the above statement, we also need an auxiliary statement for $l\in \{2, 3, \ldots L+1\}$,
\begin{equation}
\|\mD_{l-1}\textbf{Back}_{l}^T \vv\| 
\geq (1-\epsilon)\sqrt{\frac{m}{2d}}\|\vv\| \ \text{with probability at least }\  1-e^{-\Omega(m\epsilon^2/L^2)}.\label{eqn:append:lower_bound_backward_norm}
\end{equation}
In order to prove the above two statements \eqref{eqn:append:lower_bound_akli} and \eqref{eqn:append:lower_bound_backward_norm}, we first prove that $\mW_l \Big| \mD_l$ has the same distribution as $\mW_l$, i.e., $N(0, \frac{2}{m})$. Then we use a similar argument to that in the proof of Lemma~\ref{lemma:initial_vector_norm} in order to show  \eqref{eqn:append:lower_bound_backward_norm}. Finally, by using the distribution of $\mW_l$ given $\mD_l$, together with \eqref{eqn:append:lower_bound_backward_norm}, we prove \eqref{eqn:append:lower_bound_akli} and conclude this part.

We prove a more general statement for conditional distributions: given a normal random vector in $\sR^p$ as $\vw\sim N(0, \sigma^2\mI_p)$, and a random vector $\vh\in\sR^p$ that satisfies following three properties:
\begin{enumerate}
    \item $\vh$ is independent with $\vw$
    \item The norm $\|\vh\|$ is independent with the direction $\vh/\|\vh\|$
    \item The direction $\vh/\|\vh\|$ is uniform distribution in the unit sphere $\gS^{p-1}$
\end{enumerate}
We further define $B:=1_{\vh^T\vw>0}$ as a random variable. Then the conditional distribution of $\vw|B$ is the same as the unconditional distribution of $\vw$, that is 
\begin{equation}
    \vw|B {\buildrel d \over =} \vw \sim N(0, \sigma^2 \mI_p).\label{eqn:append:conditional_dist_equal_uncond}
\end{equation}  
\textit{Remark:} a normal random vector $ N(\vzero, \sigma^2\mI)$ satisfies the above three properties and thus $\vw$ also satisfies above three properties.

We denote the unit vectors $\hat{\vh}:=\vh/\|\vh\|$ and $\hat{\vw}:=\vw/\|\vw\|$. We first note that $B \equiv 1_{\hat{\vh}^T\hat{\vw} > 0}$ only depends on the directions of $\vh$ and $\vw$. By the former observation and the fact that $\|\vw\|$ is independent with $\hat{\vw}$, we thus note $\|\vw\| | B = \|\vw\|$. We denote the probability density function for a random variable $Y$ by $f_Y$. We next consider the probability density function $f_{\vw|B}(\vw)$, by independence of the norm and the direction for $\vw$, we obtain
\begin{equation}
    f_{\vw|B}(\vw) = f_{\vw|B}(\|\vw\|, \hat{\vw})=
f_{\|\vw\||B}(\|\vw\|) f_{\hat{\vw}|B}(\hat{\vw}) = f_{\|\vw\|}(\|\vw\|) f_{\hat{\vw}|B}(\hat{\vw}).\label{eqn:append:vwnorm_direction}
\end{equation}
Thus, in order to show \eqref{eqn:append:conditional_dist_equal_uncond}, it is sufficient suffices to show that $\hat{\vw}| B {\buildrel d \over =} \hat{\vw}$. We prove this by showing that for any set $\gA\subset \gS^{p-1}$ in unit sphere, $\sP (\hat{\vw}\in \gA |B=b) = \sP (\hat{\vw}\in \gA)$ for any $b=0$ or $1$. Given $\hat{\vh}$ is uniform in unit sphere, we know that for any fixed direction $\hat{\vw}$, $\sP(\hat{\vh}^T\hat{\vw} > 0) = 0.5$. By Bayes formula, former observation, and $\hat{\vh}$ is uniform in $\gS^{p-1}$
\begin{align*}
    \sP (\hat{\vw}\in \gA |B=1) & = \frac{\sP(\hat{\vw}\in \gA, \ B=1)}{\sP(B=1)}\\
    & = \frac{\sP(B=1|\hat{\vw}\in \gA) \sP(\hat{\vw}\in \gA)}{\int_{\hat{\vw}} \sP(\hat{\vh}^T \hat{\vw} > 0 | \hat{\vw})f_{\hat{\vw}}(\hat{\vw})}\\
    & = \frac{\int_{\gA}\sP(B=1|\hat{\vw}=\hat{\vw}) f_{\hat{\vw}}(\hat{\vw})d\hat{\vw}}{\int_{\gS^{p-1}} \sP(\hat{\vh}^T \hat{\vw} > 0 | \hat{\vw})f_{\hat{\vw}}(\hat{\vw})d\hat{\vw}}\\
    & = \frac{0.5\int_{\gA} f_{\hat{\vw}}(\hat{\vw})d\hat{\vw}}{0.5\int_{\gS^{p-1}} f_{\hat{\vw}}(\hat{\vw})d\hat{\vw}}\\
    & = \int_{\gA} f_{\hat{\vw}}(\hat{\vw})d\hat{\vw} = \sP(\hat{\vw}\in \gA).
\end{align*}
A similar argument leads to $\sP(\hat{\vw}\in\gA | B=0) = \sP(\hat{\vw}\in\gA)$. By \eqref{eqn:append:vwnorm_direction} and above argument, we conclude \eqref{eqn:append:conditional_dist_equal_uncond}.

Given the symmetry of normal distribution, we conclude that $\vg_{l}$ satisfies the three properties we required for $\vh$ above. Together with the fact that $(\mW_l)_{k,\cdot}$ is normal $N(0, 2/m\mI_m)$, we thus conclude that $(\mW_l)_{k,\cdot} | (\mD_{l})_{kk}$ is still normal $N(0, 2/m\mI_m)$.

Next, we estimate the norm of $\mD_{l-1}^T \textbf{Back}_{l}^T \vv$. We define vector $\vz_l:= \mD_{l} \textbf{Back}_{l+1}^T \vv$ for $l\in[L]$ and $\vz_{L+1}:=\vv$. We first note $\vz_{L} = \mD_L \mB^T \vv$ and $(\mB^T \vv)_j\sim N(0, \|\vv\|^2/d)$ for $j\in [m]$. By denoting Bournulli random variables $B_{L, j} := 1_{(\vg_{L})_j > 0}$, each index of $\vz_L$ can be expressed as 
$$
(\vz_L)_j^2 = \frac{B_{L, j}+\alpha^2(1-B_{L, j})}{1+\alpha^2} (\mB^T \vv)_j^2 \ \text{ for } j\in[m].
$$ 
We denote $Q_L:=\{j: B_{L, j} = 1\}$. Conditioning on $Q_L$, denote two independent random variables $H_{L, 1}\sim \chi^2(|Q_L|)$ and $H_{L, 2}\sim \chi^2(m-|Q_L|)$
, we note
$$
\|\vz_L\|^2 \Big| Q_L{\buildrel d \over =} \frac{\|\vv\|^2}{d(1+\alpha^2)} H_{L, 1} + \frac{\alpha^2\|\vv\|^2}{d(1+\alpha^2)} H_{L, 2}.
$$
By symmetry of random variables before $L$ layer, we know $B_{L, j}\sim \text{Bournulli}(0.5)$ and then by Chernoff bound on binomial distribution, we note that with probability at least $1-e^{-\Omega(m\epsilon^2)}$, $|Q_L|\in [(0.5-\epsilon/2)m, (0.5+\epsilon/2)m]$. Given this even happen, by using tail probability for chi-squared distribution, we note that
$$
\sP(H_{L, 1} < 0.5m(1-\epsilon)) < e^{-\Omega(m\epsilon^2)}.
$$
Similarly,
$$
\sP(H_{L, 2} < 0.5m(1-\epsilon)) < e^{-\Omega(m\epsilon^2)}.
$$
By taking event $|Q_L|\in [(0.5-\epsilon/2)m, (0.5+\epsilon/2)m]$ and using above probabilities, we conclude the lower bound for $\|\vz_L\|$
\begin{equation}
    \|\vz_L\|^2 \geq \frac{m\|\vv\|^2}{2d}(1-\epsilon) \ \text{ with probability at least } 1 - \Omega(e^{-\Omega(m\epsilon^2)}).\label{eqn:append:lower_bound_improve_bdL}
\end{equation}

We note that $\vz_{l-1} = \mD_{l-1}^T \mW_l^T \vz_{l}$.
Conditioning on $\vz_l$, we note that $\mW_l | \vz_l \equiv \mW_l | \mD_l$ is a random matrix whose entries are i.i.d $N(0, 2/m)$. We denote a random variable $B_{l, j}:= 1_{(\vg_{l})_j > 0}$, then  
$$
\|\vz_{l-1}\|^2|\vz_l = \sum_{j=1}^m \left(\frac{B_{l, j} + \alpha^2(1-B_{l, j})}{1+\alpha^2}\left(\sum_{i} (\mW_{l})_{i, j}(\vz_l)_i\right)^2\right) \Big| \vz_l.
$$
We note that $(\sum_i \mW_{l})_{i, j}(\vz_l)_i|\vz_l \sim N(0, 2\|\vz_l\|^2/m)$. We denote the indices set where $B_{l, j} =1$ by $Q_l := \{j: B_{l, j}=1\}$ and conditioning on $Q_l$, we further denote two independent random variables $H_{l, 1}\sim \chi^2 (|Q_l|)$ and $H_{l, 2}\sim \chi^2 (m-|Q_l|)$. We note that conditioning on $Q_l$, by similar argument we used above in proof of Lemma~\ref{lemma:initial_vector_norm}, we know that
\begin{equation}
    \|\vz_{l-1}\|^2 \Big|\vz_l, Q_l {\buildrel d \over =} \frac{2\|\vz_l\|^2}{m(1+\alpha^2)} H_{l, 1} + \frac{2\alpha^2\|\vz_l\|^2}{m(1+\alpha^2)} H_{l, 2}.
\end{equation}
By the same argument to derive \eqref{eqn:append:lower_bound_improve_bdL}, we know that by Chernoff bound for binomial distribution, with probability at least $1-e^{-\Omega(m\epsilon^2)}$, $|Q_l|\in [(0.5-\epsilon/2)m , (0.5+\epsilon/2)m]$, thus we note that 
$$
\sP(H_{l, 1} < 0.5m(1-\epsilon)) < e^{-\Omega(m\epsilon^2)}, \; \ \sP(H_{l, 2} < 0.5m(1-\epsilon)) < e^{-\Omega(m\epsilon^2)}.
$$
Consequently, 
\begin{equation}
    \|\vz_{l-1}\|^2 \geq \|\vz_l\|^2(1-\epsilon) \ \text{ with probability at least } 1 - \Omega(e^{-\Omega(m\epsilon^2)}).\label{eqn:append:lower_bound_improve_bd_l1}
\end{equation}
For any positive number $\epsilon_0$, when we choose $\epsilon = \epsilon_0/L$ in \eqref{eqn:append:lower_bound_improve_bdL} and \eqref{eqn:append:lower_bound_improve_bd_l1}, and then by $(1-\epsilon_0/L)^L > 1-\epsilon_0$, we conclude that 
\begin{equation}
    \|\vz_{l}\|^2 \geq \frac{m}{2d}\|\vv\|^2(1-\epsilon_0) \ \text{for all } l \in [L], \ \text{ with probability at least } 1 - \Omega(L)e^{-\Omega(m\epsilon_0^2/L^2)}.\label{eqn:append:lower_bound_improve_bds}
\end{equation}
Finally, recall that $a_{k, l} = \langle(\textbf{Back}_{l})_{\cdot, k}, \vv\rangle$ and by definition of $\vz_l$ in above proof, we note that $a_{k, l} \equiv \langle (\mW_l)_{\cdot, k}, \vz_l\rangle$. We note that $(\mW_l)_{\cdot, k} | \vz_l = (\mW_l)_{\cdot, k} | \mD_l$, by first statement we proved in this part, we further can derive that $(\mW_l)_{\cdot, k} | \vz_l\sim N(0, 2/m \mI)$. Thus, we know that conditioning on $\vz_l$,
$$
a_{k, l}|\vz_l \sim N(0, \frac{2\|\vz_l\|^2}{m}),
$$
By the tail probability of normal, we note that the with a constant probability that $a_{k, l}$ is lower bounded as
$$
\sP\left((a_{k, l})^2 \geq O\left(\frac{2\|\vz_l\|^2}{m}\right)\right) > 1-\exp(-\Omega(1)).
$$
Combining with \eqref{eqn:append:lower_bound_improve_bds}, which holds with an overwhelming probability, with a small constant choice of $\epsilon_0$, we conclude 
$$
\sP\left((a_{k, l})^2 \geq O\left(\frac{\|\vv\|^2}{d}\right)\right) > 1-\exp(-\Omega(1)) \ \text{for} l\in[L].
$$
Lastly, we also show this is also true for $l=L+1$.
Recall that $\textbf{Back}_{L+1}\equiv \mB$ and  that $a_{k, L+1}\equiv \langle\mB_{\cdot, k}, \vv\rangle\sim N\left(0, \frac{\|\vv\|^2}{d}\right)$. By using normal distribution property,
$$
\sP\left((a_{k, L+1})^2 \geq O\left(\frac{\|\vv\|^2}{d}\right)\right) > 1-\exp(-\Omega(1)).
$$
We conclude this part by the final statement that
\begin{equation}
    \sP\left((a_{k, l})^2 \geq O\left(\frac{\|\vv\|^2}{d}\right)\right) > 1-\exp(-\Omega(1)) \ \text{ for } l \in [L+1].\label{eqn:append:improvement_lower_bound_all_ls}
\end{equation}
\textbf{Part 4.}
We denote a vector $\va_{k, l}\in\sR^n$ by
denoting its entries as ${(\va_{k, l})}_i :=  \langle(\textbf{Back}_{i, l})_{\cdot, k}, \vv_i\rangle$ 
for $i\in [n]$. By definition \eqref{eqn:append:def_b_func}, we note that $\vb_{k, l-1}((\mW_{l})_{k, \cdot}, \va_{k, l+1})\equiv (\sum_{i=1}^n\mG_{i, l}(\vv_i;\mW))_{k, \cdot}$, by the definition of Frobenius norm of a vector of matrices, 
\begin{equation}
    \left\|\sum_{i=1}^n \mG_{i, l}(\vv_i;\mW)\right\|_F^2  = \sum_{k=1}^m\|\vb_{k, l-1}((\mW_{l})_{k, \cdot},  \va_{k, l+1})\|^2.\label{eqn:append:grad_loss_b_norm}
\end{equation}

Due to \eqref{eqn:append:empty_set_gW}, for any vector $\vw\in\sR^m$ and any integer $l\in [L]$, we note
\begin{equation}
     1 \geq \sum_{i=1}^n 1_{\vw\in \gW_{i, l-1}}. \label{eqn:append:sum_gWs}
\end{equation}
It follows from \eqref{eqn:append:grad_loss_b_norm} and \eqref{eqn:append:sum_gWs},
\begin{align*}
    \left\|\sum_{i=1}^n \mG_{i, l}(\vv_i;\mW)\right\|_F^2 
    & \geq \sum_{k=1}^m\|\vb_{k, l-1}((\mW_{l})_{k, \cdot},  \va_{k, l+1})\|^2 \sum_{i=1}^n 1_{(\mW_{l})_{k, \cdot}\in \gW_{i, l-1}}\\
    & = \sum_{k=1}^m \sum_{i=1}^n \|\vb_{k, l-1}((\mW_{l})_{k, \cdot},  \va_{k, l+1})\|^2 1_{(\mW_{l})_{k, \cdot}\in \gW_{i, l-1}}
\end{align*}
By \eqref{eqn:append:prob_b_one_half}, we know that with probability at least $0.5$, conditioning on $(\mW_{l})_{k, \cdot}\in \gW_{i, l-1}$, 
$$
\|\vb_{k, l-1}((\mW_{l})_{k, \cdot},  \va_{k, l+1})\|^2 \geq \frac{(\va_{k, l+1})_i^2}{4}\frac{(1-\alpha)^2}{1 + \alpha^2} \|\vh_{i, l-1}\|^2
$$

We introduce the following new event $\gV_{i, l}$ as follows 
\begin{align*}
    \gV_{i,l} := \Big\{&(\mW_{l})_{k, \cdot}\in \gW_{i, l-1}, \ (\va_{k, l+1})_i^2 \geq \frac{\|\vv_i\|^2}{2d}, \ \|\vh_{i, l-1}\| \geq \frac{1}{2}\Big\}.
\end{align*}
Using this event, the observation $\gV_{i, l}\subset \{(\mW_{l})_{k,\cdot}\in \gW_{i, l-1}\}$, the definition of $\gV_{i, l}$ and \eqref{eqn:append:prob_b_one_half}, 
we obtain the following lower bound on the squared norm in \eqref{eqn:append:grad_loss_b_norm}:
\begin{align*}
        \left\|\sum_{i=1}^n \mG_{i, l}(\vv_i;\mW)\right\|_F^2 & \geq  \sum_{k=1}^m \sum_{i=1}^n \|\vb_{k, l-1}((\mW_{l})_{k,\cdot}, \va_{k, l+1})\|^2 1_{(\mW_{l})_{k,\cdot}\in \gW_{i, l-1}}\\
        &\geq \sum_{k=1}^m \sum_{i=1}^n \|\vb_{k, l-1}((\mW_{l})_{k,\cdot},  \va_{k, l+1})\|^2 1_{\gV_{i, l}}\\
        & \geq \sum_{k=1}^m\sum_{i=1}^n \frac{\|\vv_i\|^2}{32d}\frac{(1-\alpha)^2}{1 + \alpha^2}  \sP(\gV_{i, l}).
\end{align*}
For simplicity, we denote $$Z_k := \sum_{i=1}^n \frac{\|\vv_i\|^2}{32d}\frac{(1-\alpha)^2}{1 + \alpha^2}  1_{\gV_{i, l}}.$$

To lower bound the probability $\sP(\gV_{i, l})$, we note that $\mW_l$, $\va_{k, l+1}$ and $\vh_{i, l-1}$ are independent because they depend on $\mW_{l+k}$ for $k\in[L-l+1]$, $\mW_l$ and $\mW_{l-k}$ for $k\in[l]$. We note that $(\va_{k, l+1})_i$ is corresponding to $a_{k, l+1}$ with selecting $\vv = \vv_i$ in the statement proven in the previous part \eqref{eqn:append:lower_bound_akli}. Then by using \eqref{eqn:append:probability_gW}, \eqref{eqn:append:lower_bound_akli} and applying Lemma~\ref{lemma:initial_vector_norm}
\begin{align*}
    \sP(\gV_{i, l}) & = \sP\left((\mW_{l})_{k, \cdot}\in \gW_{i, l-1}\right)\sP\left((\va_{k, l+1})_i^2 \geq \frac{\|\vv_i\|^2}{2d}\right)\sP\left(\|\vh_{i, l-1}\| \geq \frac{1}{2}\right)\\
    & \geq \Omega\left(\frac{\delta}{nL}\right)\times \left(1-\exp(-\Omega(1))\right) \times \left(1 - e^{-\Omega(m/L)}\right) = \Omega\left(\frac{\delta}{nL}\right).
\end{align*}
By property of indicator function, we note that
$$\sE Z_k = \sum_{i=1}^n \frac{\|\ve_i\|^2}{32d}\frac{(1-\alpha)^2}{1 + \alpha^2}  \sP(\gV_{i, l})$$ and 
$$\operatorname{Var}Z_k = \frac{\|\ve_i\|^2}{32d}\frac{(1-\alpha)^2}{1 + \alpha^2}  1_{\gV_{i, l}} \sP(\gV_{i, l}) ( 1- \sP (\gV_{i, l}).$$
Then, by using Hoeffding inequality, 
with probability at least $1- e^{-\Omega(m\delta^2/L^2)}$ that
\begin{align*}
    \sum_{k=1}^m Z_k &\geq \frac{m}{2}\sum_{i=1}^n \frac{\|\ve_i\|^2}{32d}\frac{(1-\alpha)^2}{1 + \alpha^2}  \sP(\gV_i) \\ & \geq \frac{Cm}{2}\sum_{i=1}^n \frac{\|\vv_i\|^2}{32d}\frac{(1-\alpha)^2}{1 + \alpha^2}  \Omega\left(\frac{\delta}{nL}\right).
\end{align*}
Thus we conclude the Lemma, for all $l\in [L]$, as follows:
$$
\left\|\sum_{i=1}^n\mG_{i, l}(\vv_i;\mW^{(0)})\right\|^2_F \geq \sum_{k=1}^m Z_k\geq \Omega\left(\frac{(1-\alpha)^2}{(1+\alpha^2)}\frac{\delta m}{ndL}\right)\sum_{i=1}^n\|\vv_i\|^2.
$$
\end{proof}

At last, we conclude the proof of Lemma~\ref{thm:gradient_bound}.

\begin{proof}[Proof of Lemma~\ref{thm:gradient_bound}] In order to prove the lower and upper bounds for the gradient for parameters $\mW$ close to $\mW^{(0)}$, we need leverage Lemma~\ref{lemma:intermediate_perturbation} to show that after perturbation from $\mW^{(0)}$, the change in gradient has a smaller order than the upper bound in Lemma~\ref{lemma:gradient_upper_bound_init} and the lower bound in Lemma~\ref{lemma:gradient_lower_bound_init}. Then the same upper and lower bounds hold for $\mW$ such that $\|\mW^{(0)}-\mW\| < \omega$  and thus conclude Lemma~\ref{thm:gradient_bound}.

We denote a perturbation of the function $\mG_{i, l}(\vv;\mW^{(0)})$ with respect to $\mW^{(0)}$,
\begin{align*}
& \mG_{i, l}(\vv;\mW) -\mG_{i, l}(\vv; \mW^{(0)}) \\
& =   (\vv^T\mB \mD_{i, L}\mW_L ... \mD_{i, l+1}\mW_{l+1}\mD_{i, l})^T \vh_{i, l-1}^T  -  (\vv^{T} \mB \mD_{i, L}^{(0)}\mW_L^{(0)} ... \mD_{i, l+1}^{(0)}\mW_{l+1}^{(0)}\mD_{i, l}^{(0)})^T  \vh_{i, l-1}^{(0)T}\\
 & =  (\vv^T\mB \mD_{i, L}\mW_L ... \mD_{i, l+1}\mW_{l+1} \mD_{i, l})^T \vh_{li, -1}^T  -  (\vv^{T} \mB \mD_{i, L}^{(0)}\mW_L^{(0)} ... \mD_{i, l+1}^{(0)}\mW_{l+1}^{(0)}\mD_{i, l}^{(0)})^T  \vh_{i, l-1}^{T}\\
 & +  (\vv^{T} \mB \mD_{i, L}^{(0)}\mW_L^{(0)} ... \mD_{i,l+1}^{(0)}\mW_{l+1}^{(0)}\mD_{i,l}^{(0)})^T  \vh_{i,l-1}^{T} -  (\vv^{T} \mB \mD_{i,L}^{(0)}\mW_L^{(0)} ... \mD_{i,l+1}^{(0)}\mW_{l+1}^{(0)}\mD_{i,l}^{(0)})^T  \vh_{i, l-1}^{(0)T}
 \end{align*}
Using $\|\vu \vv^T\|_F \leq \|\vu\|\|\vv\|$, and denoting vectors $\vv_i\in \sR^d$, we derive the bound for the change of the gradient by
\begin{align}
 &\|\sum_{i=1}^n \mG_{i, l}(\vv_i; \mW)-\mG_{i, l}(\vv_i; \mW^{(0)})\|_F \notag\\ &\leq  
    \sum_{i=1}^n \|\vv_i^T(\mB \mD_L\mW_L ... \mD_{l+1}\mW_{l+1}\mD_l - \mB \mD_L^{(0)}\mW_L^{(0)} ... \mD_{l+1}^{(0)}\mW_{l+1}^{(0)})\mD_l^{(0)}\|\|\vh_{l-1}\|  \notag\\ &
    \quad + \|\vv_i^T\mB \mD_L^{(0)}\mW_L^{(0)} ... \mD_{l+1}^{(0)}\mW_{l+1}^{(0)}\mD_l^{(0)}\|\|\vh_{l-1} - \vh_{l-1}^{(0)}\|\label{eqn:append:bound_G_WW0}.
\end{align}
By Lemma~\ref{lemma:intermediate_perturbation},
\begin{align}
         & \|(\vv^T\mB \mD_L\mW_L \cdots \mD_{l+1}\mW_{l+1}) - (\vv^T\mB \mD_L^{(0)}\mW_L^{(0)} \cdots \mD_{l+1}^{(0)}\mW_{l+1}^{(0)})\|\notag \\
         & \leq O\left(\omega^{1/3}L^2\frac{\sqrt{m\ln m}}{\sqrt{d}}\right)\|\vv\| \ \text{ with probability at least } 1-e^{-\Omega(m/L)}\label{eqn:append:perturbation_WG1}.
\end{align}
By Lemma~\ref{lemma:initial_vector_norm},
\begin{equation}
    \|\vh^{(0)}\| \leq 1.1 \ \text{ with probability at least } 1-e^{-\Omega(m/L)}.\label{eqn:append:perturbation_WG2}
\end{equation}
By Lemma~\ref{lemma:forward_perturbation}, 
\begin{equation}
    \|\vh_{l-1} - \vh_{l-1}^{(0)}\|\leq O(\omega L^{5/2}\sqrt{\ln m})\ \text{ with probability at least } 1-e^{-\Omega(m/L)}.\label{eqn:append:perturbation_WG3}
\end{equation} 
We note that the combination of \eqref{eqn:append:perturbation_WG2},  \eqref{eqn:append:perturbation_WG3} and the bound 
$\omega < O\left(\frac{1}{L^{5/2}\sqrt{\ln m}}\right)$ (which is a weaker bound than the one stated in the lemma) imply
\begin{equation}
    \|\vh\|\leq O(1)\ \text{ with probability at least } 1-e^{-\Omega(m/L)}.\label{eqn:append:perturbation_WG4}
\end{equation}
By applying \eqref{eqn:append:perturbation_WG1}, \eqref{eqn:append:perturbation_WG2}, \eqref{eqn:append:perturbation_WG3} and \eqref{eqn:append:perturbation_WG4}
 to \eqref{eqn:append:bound_G_WW0}, we conclude that with probability at least $1-e^{-\Omega(m/L)}$
\begin{equation}
    \left\|\sum_{i=1}^n \mG_{i, l}(\vv_i; \mW)-\mG_{i,l}(\vv_i; \mW^{(0)})\right\|_F^2 \leq O\left(\omega^{2/3}L^4\frac{m\ln m}{d}\right)\sum_{i=1}^n\|\vv_i\|^2.\label{eqn:append:perturbation_G_F}
\end{equation}
We note that for $l\in [L]$, $i\in[n]$,
\begin{align*}
    \nabla_{\mW_l}\text{loss}(\vx_i, \vy_i; \mW)  =  \mG_{i, l}(\ve_i; \mW).
\end{align*}
and thus
\begin{align}
    \nabla_{\mW_l}\gL(\mW)  =  \sum_{i=1}^n\mG_{i, l}(\ve_i; \mW).\label{eqn:relationship_gradient_G_sum}
\end{align}
Therefore, substituting $\vv_i = \ve_i$ in \eqref{eqn:append:perturbation_G_F}, the left-hand side of \eqref{eqn:append:perturbation_G_F} becomes the perturbation of the gradient of the loss function. Since $\omega < O\left(\frac{\delta^{3/2}}{n^{3/2}L^{15/2} \ln^{3/2} m}\right)$ and $\delta<c_0$, 
\begin{equation}
    \omega^{2/3}L^4\frac{m\ln m}{d} < O\left(\frac{\delta m}{ndL}\right) < O\left(\frac{mn}{d}\right).\label{eqn:append:small_omega_perturbation}
\end{equation}

For the upper bound, by Lemma~\ref{lemma:gradient_upper_bound_init}, \eqref{eqn:relationship_gradient_G_sum}, \eqref{eqn:append:perturbation_G_F} and then by \eqref{eqn:append:small_omega_perturbation}, with probability at least $1-e^{-\Omega(m/L)}$,
\begin{align*}
   \left\|\nabla_{\mW_l}\gL(\mW) \right\|_F^2 &=\left\|\sum_{i=1}^n\mG_{i, l}(\ve_i; \mW) \right\|_F^2\\
   &\leq 2\left\|\sum_{i=1}^n\mG_{i, l}(\ve_i; \mW^{(0)}) \right\|_F^2 + 2\left\|\sum_{i=1}^n\mG_{i, l}(\ve_i; \mW^{(0)}) - \mG_{i, l}(\ve_i; \mW) \right\|_F^2\\
   &\leq \left(O\left(\frac{mn}{d}\right) + O\left(\omega^{2/3}L^4 \frac{m\ln m}{d}\right)\right)\sum_{i=1}^n\|\ve_i\|^2\\
   & \leq  O\left(\frac{mn}{d}\right) \sum_{i=1}^n\|\ve_i\|^2\\
& = O\left(\frac{mn}{d}\right) \gL(\mW).
\end{align*}
By definition, we further conclude that
$$
\left\|\nabla_\mW\gL(\mW) \right\|^2 \leq \max_{l\in [L]}\left\|\nabla_{\mW_l}\gL(\mW) \right\|_F^2  \leq O\left(\frac{mn}{d}\right) \gL(\mW).
$$
For the lower bound, by Lemma~\ref{lemma:gradient_lower_bound_init}, \eqref{eqn:relationship_gradient_G_sum}, \eqref{eqn:append:perturbation_G_F} and then by \eqref{eqn:append:small_omega_perturbation}, with probability at least $1-e^{-\Omega(m\delta^2)}$,
\begin{align*}
    \left\|\nabla_{\mW_l}\gL(\mW) \right\|_F^2 &=\left\|\sum_{i=1}^n\mG_{i, l}(\ve_i; \mW) \right\|_F^2\\
   &\geq \frac{1}{2}\left\|\sum_{i=1}^n\mG_{i, l}(\ve_i; \mW^{(0)}) \right\|_F^2 - \left\|\sum_{i=1}^n\mG_{i, l}(\ve_i; \mW^{(0)}) - \mG_{i, l}(\ve_i; \mW) \right\|_F^2\\
   & \geq \Omega\left(\frac{(1-\alpha)^2}{1+\alpha^2}\frac{\delta m}{ndL} - O\left(\omega^{2/3}L^4 \frac{m\ln m}{d}\right)\right)\sum_{i=1}^n\|\ve_i\|^2\\
   &\geq \Omega\left(\frac{(1-\alpha)^2}{1+\alpha^2}\frac{\delta m}{ndL}\right) \gL(\mW).
\end{align*}
By definition, we conclude that
\begin{equation}
    \left\|\nabla_{\mW}\gL(\mW) \right\|_F^2 =\sum_{l\in[L]}\left\|\nabla_{\mW_l}\gL(\mW) \right\|_F^2 \geq \Omega\left(\frac{(1-\alpha)^2}{1+\alpha^2}\frac{\delta m}{nd}\right)\gL(\mW).\label{eqn:append:lower_bound_F_norm}
\end{equation}

\end{proof}

\subsection{Proof of Lemma \ref{thm:semi-smooth}}\label{appd:semismooth}

We prove Lemma \ref{thm:semi-smooth} by adapting the arguments of the proof of Theorem~4 in \cite{allen2019convergence} to Leaky ReLUs. 

Let us first introduce some notation. We let $\mW^*$ be a vector of matrices satisfying $\|\mW^\ast - \mW^{(0)}\| < \omega$, where we think of $\mW^*$ as a vector of matrices at an arbitrary training step (we will apply the lemma in this way). We denote a perturbation of $\mW^*$ by $\mW'$ and the perturbed matrix by $\mW := \mW^\ast + \mW'$. 
Additional notation corresponding to the original, perturbation and perturbed settings (of $\mW^*$, $\mW$ and $\mW'$, respectively)
is summarized as follows:
\begin{align*}
    &\vg_{i, l}^{\ast}= \mW^{\ast}_l \vh_{i, l-1}^{\ast},&&\vg_{i, l}= \mW_l \vh_{i, l-1} && \vg'_{i, l} = \vg_{i, l} - \vg_{i, l}^{\ast} \\
    &(\mD_{i, l})_{jj}^{\ast}=\frac{1_{(\vg^{\ast}_{i, l})_j\geq 0} + \alpha 1_{(\vg^{\ast}_{i, l})_j< 0}}{\sqrt{1+\alpha^2}},&&(\mD_{i, l})_{jj}=\frac{1_{(\vg_{i, l})_j\geq 0} + \alpha 1_{(\vg_{i, l})_j< 0}}{\sqrt{1+\alpha^2}}, && \mD'_{i, l} = \mD_{i, l} - \mD_{i, l}^{\ast} \\
    & \vh_{i, l}^{\ast}  = \tilde{\sigma}_\alpha(\mW^{\ast}_l\vh_{l-1}^{\ast})\equiv \tilde{\sigma}_\alpha(\vg_{i, l}^{\ast}),&& \vh_{i, l}  = \tilde{\sigma}_\alpha(\mW_l\vh_{l-1})\equiv \tilde{\sigma}_\alpha(\vg_{i, l}),&& \vh'_{i, l} = \vh_{i, l} - \vh^{\ast}_{i, l}\\
    & \ve_{i, l}^\ast = \vy_i - \mB\vh_{i, L}^{\ast},&&\ve_{i, l}= \vy_i - \mB\vh_{i, L}^{\ast} && \ve'_{i,l} = \ve_{i, l} - \ve_{i,l}^{\ast}.
\end{align*}
The loss functions at $\mW^\ast$ and $\mW$ are expressed as
\begin{equation}
    \gL(\mW^\ast) = \frac{1}{2}\sum_{i=1}^n \|\ve_i^\ast\|^2, \    \gL(\mW) = \frac{1}{2}\sum_{i=1}^n \|\ve_i\|^2.\label{eqn:append:def_e_lss}
\end{equation}

We introduce an auxiliary lemma before proving Lemma~\ref{thm:semi-smooth}.
\begin{lemma}\label{lemma:prime_byDprime}
There exists a set of diagonal matrices $\mD''_{i,l}\in [-\sqrt{2}, \sqrt{2}]^{m\times m}$ so that$$\vh'_{i, l} = \vh_{i, l} - \vh^\ast_{i, l} = \sum_{a=1}^l (\mD^\ast_{i, l} + \mD''_{i, l}) \mW^\ast_l(\mD^\ast_{i, l-1} + \mD''_{i, l-1}) ... \mW^\ast_{a+1} (\mD^\ast_{i, a} + \mD''_{i, a})\mW'_a \vh_{i, a-1}.
$$
Furthermore, the following bounds hold
$$
\|\vh'_{i, l}\| \leq O(L^{3/2})\|\mW'\|, \quad  \|\mB\vh'_{i, L}\|\leq O(L\sqrt{m/d})\|\mW'\| \
\text{ and } \ \|\mD''_{i, l}\|_0 \leq O(m\omega^{2/3}L).$$
\end{lemma}
The proof of this lemma is identical to the proof of Claim~11.2 in \citet{allen2019convergence}. It is obtained by replacing $|D''_{k, k}| \leq 1$ in the second statement of Proposition~11.3 in \citet{allen2019convergence} with $|D''_{k, k}| \leq \sqrt{2}$ in order to fit the setting of Leaky ReLUs.

The rest of this section  provides a detailed proof of Lemma~\ref{thm:semi-smooth}. 
\begin{proof}[Proof of Lemma~\ref{thm:semi-smooth}]
We first express the loss function  at $\mW$ as follows
\begin{align}
    &\textrm{loss}(\vx_i,\vy_i;\mW) \notag\\
    &= \frac{1}{2}\|\mB \vh_{i, L} - \vy_i\|^2 = \frac{1}{2}\|\mB (\vh_{i, L} - \vh^\ast_{i, L}) + \mB \vh^\ast_{i, L} - \vy_i\|^2 \notag\\ 
    & = \frac{1}{2}\|\ve^\ast_i + \mB (\vh_{i, L} - \vh^\ast_{i, L})\|^2 =  \frac{1}{2}\|\ve^\ast_i\|^2 + \frac{1}{2}\|\mB (\vh_{i, L} - \vh^\ast_{i, L})\|^2 + \langle \ve^\ast_i,  \mB (\vh_{i, L} - \vh^\ast_{i, L})\rangle\notag\\
    & = \textrm{loss}^\ast_i + \frac{1}{2}\|\mB (\vh_{i, L} - \vh^\ast_{i, L})\|^2 +  \ve^{\ast T}_i\mB (\vh_{i, L} - \vh^\ast_{i, L}) = \textrm{loss}^\ast_i + \frac{1}{2}\|\mB (\vh_{i, L} - \vh^\ast_{i, L})\|^2 +  \ve^{\ast T}_i\mB \vh_{i, L}'.\label{eqn:append:loss_expand_t1}
\end{align}
Then we expand $\langle\nabla \gL (\mW^\ast), \mW'\rangle$ as 
\begin{align*}
&\langle\nabla \gL(\mW^\ast), \mW'\rangle \\
& =   \sum_{l=1}^L \langle \nabla_{\mW_l} \gL (\mW^\ast), \mW'_l\rangle = \sum_{l=1}^L\sum_{i=1}^n \langle \mD^\ast_{i, l}\mW^{\ast T}_{l+1}\mD^\ast_{i, l+1}...\mD^\ast_{i, L}\mB^T \ve^\ast_i \vh^{\ast T}_{l-1}(\vx_i), \mW'_l  \rangle\\
& = \sum_{l=1}^L\sum_{i=1}^n \langle \mD^\ast_{i, l}\mW^{\ast T}_{l+1}\mD^\ast_{i, l+1}...\mD^\ast_{i, L}\mB^T \ve^\ast_i \vh^{\ast T}_{l-1}(\vx_i), \mW'_l  \rangle\\
& = \sum_{l=1}^L\sum_{i=1}^n
\ve_i^{\ast T} \mB\mD^\ast_{i, L} \mW^\ast_L...\mD^\ast_{i, l+1}\mW^\ast_{l+1} \mD^\ast_{i, l} \mW'_l \vh^\ast_{l-1}(\vx_i).
\end{align*}
The above two equations imply the following estimate 
\begin{align}
     &\gL(\mW^\ast + \mW') - \gL(\mW^\ast) - \langle\nabla \gL(\mW^\ast), \mW'\rangle \notag\\
    & =  - \langle\nabla \gL(\mW^\ast), \mW'\rangle +  \sum_{i=1}^n (\textrm{loss}_i - \textrm{loss}^\ast_i)
    \notag\\
    & = - \sum_{l=1}^L\sum_{i=1}^n
\ve_i^{\ast T} \mB\mD^\ast_{i, L} \mW^\ast_L...\mD^\ast_{i, l+1}\mW^\ast_{l+1} \mD^\ast_{i, l} \mW'_l \vh^\ast_{l-1}(\vx_i)  \notag\\
& \quad + \sum_{i=1}^n  \big(\frac{1}{2}\|\mB (\vh_{i, L} - \vh^\ast_{i, L})\|^2 +  \ve_i^{\ast T}\mB (\vh_{i, L} - \vh^\ast_{i, L})\big)\notag\\
& = \sum_{i=1}^n e_i^{\ast T} \mB \Big( (\vh_{i, L} - \vh^\ast_{i, L}) - \sum_{l=1}^L \mD^\ast_{i, L} \mW^\ast_L...\mD^\ast_{i, l+1}\mW^\ast_{l+1} \mD^\ast_{i, l} \mW'_l \vh^\ast_{l-1}(\vx_i)\Big)\label{eqn:semi:1}\\
& \quad + \frac{1}{2}\sum_{i=1}^n \|\mB (\vh_{i, L} - \vh^\ast_{i, L})\|^2. \label{eqn:semi:2}
\end{align}
Lemma \ref{lemma:prime_byDprime} provides the following upper bound for \eqref{eqn:semi:2}
\begin{equation}
    \frac{1}{2}\sum_{i=1}^n \|\mB (\vh_{i, L} - \vh^\ast_{i, L})\|^2 \leq O(nL^2 m/d)\|\mW'\|^2.
    \label{eqn:solut:semi:1}
\end{equation}
We note that \eqref{eqn:semi:1} can be differently expressed by using Lemma \ref{lemma:prime_byDprime} to replace $\vh - \vh^\ast$ with some diagonal matrices, $\mD''_{i, l}$, and by adding and subtracting the term $\sum_{l=1}^L \mD^\ast_{i, L} \mW^\ast_L\dots\mD^\ast_{i, l+1}\mW^\ast_{l+1} \mD^\ast_{i, l} \mW'_l \vh_{i, l}$ as follows
\begin{align}
    &\ve^{\ast T}_i \mB \Big( (\vh_{i, L} - \vh^\ast_{i, L}) - \sum_{l=1}^L \mD^\ast_{i, L} \mW^\ast_L...\mD^\ast_{i, l+1}\mW^\ast_{l+1} \mD^\ast_{i, l} \mW'_l \vh^\ast_{l-1}(\vx_i)\Big) \notag\\ 
    & = \ve^{\ast T}_i\mB \Big(\sum_{l=1}^L (\mD^\ast_{i, L} + \mD''_{i, L}) \mW^\ast_L ... \mW^\ast_{l+1} (\mD^\ast_{i, l} + \mD''_{i, l})\mW'_l \vh_{i, l-1} \notag\\
    & \quad - \sum_{l=1}^L \mD^\ast_{i, L} \mW^\ast_L...\mD^\ast_{i, l+1}\mW^\ast_{l+1} \mD^\ast_{i, l} \mW'_l \vh^\ast_{l-1}(\vx_i)\Big)\notag\\ 
    & = \ve^{\ast T}_i\mB \Big(\sum_{l=1}^L \big((\mD^\ast_{i, L} + \mD''_{i, L}) \mW^\ast_L ... \mW^\ast_{l+1} (\mD^\ast_{i, l} + \mD''_{i, l})\mW'_l  - \mD^\ast_{i, L} \mW^\ast_L...\mW^\ast_{l+1} \mD^\ast_{i, l} \mW'_l\big)\vh_{i, l-1} \label{eqn:semi:1:1}\\
    & \quad - \sum_{l=1}^L \mD^\ast_{i, L} \mW^\ast_L...\mD^\ast_{i, l+1}\mW^\ast_{l+1} \mD^\ast_{i, l} \mW'_l \big(\vh_{i, l-1} - \vh^\ast_{l-1}(\vx_i)\big)\Big)\label{eqn:semi:1:2}.
\end{align}

Next, we upper bound  \eqref{eqn:semi:1:1} and \eqref{eqn:semi:1:2}. 
In order to bound \eqref{eqn:semi:1:1}, we first use Lemma~\ref{lemma:intermediate_perturbation} to obtain the following bound
\begin{align}
    &\|\mB(\mD^\ast_{i, L} + \mD''_{i, L}) \mW^\ast_L ... \mW^\ast_{l+1} (\mD^\ast_{i, l} + \mD''_{i, l})\mW'_l  - \mB \mD^\ast_{i, L} \mW^\ast_L...\mW^\ast_{l+1} \mD^\ast_{i, l} \mW'_l\| \notag
    \\
    &\leq O\left(\frac{1-\alpha}{\sqrt{1+\alpha^2}}\frac{\omega^{1/3}L^2\sqrt{m\ln m}}{\sqrt{d}}\right) \|\mW_l'\|.\label{eqn:append:bound_BDast_Wast_minusprimeprime}
\end{align}
Using \eqref{eqn:append:def_e_lss}, we note that $(\sum_{i=1}^n \|\ve_i^\ast\|)^2 \leq n\sum_{i=1}^n \|\ve_i^\ast\|^2 = n\gL(\mW^{\ast})$. Combining this fact and  \eqref{eqn:append:bound_BDast_Wast_minusprimeprime} yields the following bound for \eqref{eqn:semi:1:1}: 
\begin{align}
   &\sum_{i=1}^n \ve^{\ast T}_i\mB \Big(\sum_{l=1}^L \big((\mD^\ast_{i, L} + \mD''_{i, L}) \mW^\ast_L ... \mW^\ast_{l+1} (\mD^\ast_{i, l} + \mD''_{i, l})\mW'_l  - \mD^\ast_{i, L} \mW^\ast_L...\mW^\ast_{l+1} \mD^\ast_{i, l} \mW'_l\big)\vh_{i, l-1}\Big)\notag
    \\& \leq \sqrt{n\gL (\mW^\ast)}O\left(\frac{1-\alpha}{\sqrt{1+\alpha^2}}\frac{\omega^{1/3}L^2\sqrt{m\ln m}}{\sqrt{d}}\right)\|\mW'\|.\label{eqn:solut:semi:1:1}
\end{align}

In order to bound \eqref{eqn:semi:1:2}, we apply  Lemma~\ref{lemma:initial_forward_matrix_bound} and Lemma~\ref{lemma:intermediate_perturbation} to obtain
\begin{align}
    & \|\mB{\mD^\ast}_{i, L} {\mW^\ast}_L\ldots{\mD^\ast}_{i, l+1}{\mW^\ast}_{l+1} {\mD^\ast}_{i, l}\| \notag\\
    &\leq 
        \|\mB{\mD^{0}}_{i, L} {\mW^{0}}_L\ldots{\mD^{0}}_{i, l+1}{\mW^{0}}_{l+1} {\mD^{0}}_{i, l}\| \notag \\
        & \quad + 
        \|\mB{\mD^{0}}_{i, L} {\mW^{0}}_L\ldots{\mD^{0}}_{i, l+1}{\mW^{0}}_{l+1} {\mD^{0}}_{i, l} - \mB{\mD^\ast}_{i, L} {\mW^\ast}_L\ldots{\mD^\ast}_{i, l+1}{\mW^\ast}_{l+1} {\mD^\ast}_{i, l}\|\notag
        \\
        &\leq 
        O(\sqrt{Lm/d}) + O\left(\frac{1-\alpha}{\sqrt{1+\alpha^2}}\frac{\omega^{1/3}L^2\sqrt{m\ln m}}{\sqrt{d}}\right).\label{eqn:append:bound2_semi:1:2}
\end{align}
Lemma~\ref{lemma:prime_byDprime} implies that $\|\vh - \vh^\ast\| = \|\vh'\| \leq O(L^{3/2}\|\mW'\|)$. Combining this observation and \eqref{eqn:append:bound2_semi:1:2} results in
\begin{align}
     \left\|\sum_{i=1}^n \ve^{\ast T}_i\sum_{l=1}^L\mB\mD^\ast_{i, L} \mW^\ast_L...\mD^\ast_{i, l+1}\mW^\ast_{l+1} \mD^\ast_{i, l} \mW'_l \vh'_{i, l-1}\right\|  \leq \sum_{i=1}^n \|\ve^\ast_i\| O(L^2\sqrt{m/d})\|\mW'\|^2.\label{eqn:append:bound_semi:1:2_longeterm}
\end{align}

In order to bound $\|\ve^\ast_i\|$, we first note that at initialization $$\|\ve^{(0)}_i\| = \|\mB\vh_{L,i}^{(0)} - \vy_i\| \leq \|\vy_i\| + \|\mB\vh_{L,i}^{(0)}\|,$$ where 
$$
\mB\vh_{L,i}^{(0)}\sim N\left(0, \frac{\|\vh\|^2}{d}\mI_d\right).
$$
For this Gaussian distribution and $d < O(1)$, 
$$
\sP\left(\|(\mB\vh_{L,i}^{(0)})\|^2 > \frac{\sqrt{m}}{\sqrt{d}}\right) \leq e^{-\Omega(\frac{m}{d})} = e^{-\Omega(m)}.
$$
Therefore, with probability at least $1-e^{-\Omega(m)}$, 
$$
\|\ve^{(0)}_i\| \leq O\left(\frac{\sqrt{m}}{\sqrt{d}}\right).
$$

For general $\ve_i^\ast$, $\|\ve_i^\ast\| = \left\|\mB\big(\vh_{i, L}^{(0)} + (\vh^\ast_{i, L} - \vh_{i, L}^{(0)})\big) - \vy_i\right\|$. Lemma~\ref{lemma:prime_byDprime} implies that if $\omega \leq O(1/L)$
\begin{equation}
\|\ve_i^\ast\| \leq \|\ve^{(0)}_i\| + \|\mB(\vh^\ast_{i, L} - \vh_{i, L}^{(0)})\| \leq 
O\left(\frac{\sqrt{m}}{\sqrt{d}}\right).
\label{eq:e_star}    
\end{equation}
The combination of \eqref{eq:e_star} and \eqref{eqn:append:bound_semi:1:2_longeterm}  results in the following bound on the term specified in \eqref{eqn:semi:1:2}
\begin{align}
    \left\|\sum_{i=1}^n \ve_i^{\ast T}\sum_{l=1}^L\mB\mD^\ast_{i, L} \mW^\ast_L... \mD^\ast_{i, l+1}\mW^\ast_{l+1} \mD^\ast_{i, l} \mW'_l \vh'_{i, l-1}\right\| \leq O\left(\frac{nL^2m}{d}\right) \|\mW'\|^2
    \label{eqn:solu:semi:1:2}
    \end{align}
Combining the bounds in \eqref{eqn:solut:semi:1}, \eqref{eqn:solut:semi:1:1} and \eqref{eqn:solu:semi:1:2} we bound the terms in \eqref{eqn:semi:1} and \eqref{eqn:semi:2} with the above specified probability. We thus conclude the desired result, that is, if $\mW^\ast$ is such that $\|\mW^\ast - \mW^{(0)}\|<\omega$, then with probability at least $1-e^{-\Omega(m)}$
\begin{align}
    & \gL(\mW^\ast + \mW') - \gL(\mW^\ast) - \langle\nabla \gL(\mW^\ast), \mW'\rangle \notag \\ 
    & \leq  \sqrt{n\gL(\mW^\ast)} O\left(\frac{1-\alpha}{\sqrt{1+\alpha^2}}\frac{\omega^{1/3}L^2\sqrt{m\ln m}}{\sqrt{d}}\right)\|\mW'\| +  O(n L^2 m/d) \|\mW'\|^2. \notag
\end{align}
\end{proof}

\subsection{Conclusion of the Proof of Theorem~\ref{thm:gd_main_thm}}\label{appd:main_proof_gd}
Most of the proof of Theorem~\ref{thm:gd_main_thm} was given in \S\ref{sec:idea_proof}. 
The only part that remains unverified is to show that during training
$$\|\mW^{t}-\mW^{(0)}\| < \omega< O\left(\frac{\delta^{3/2}}{n^{3/2} L^{15/2}\ln^{3/2}m}\right).$$  For this purpose, we establish Lemma~\ref{lemma:training_in_perturbation} below. 

\begin{lemma}\label{lemma:training_in_perturbation}
    Assume the setup of \S\ref{sec:problem_setup}, where the learning rate satisfies $\eta < \frac{\delta^{3/2} d^{1/2}}{n^{3} L^{15/2} m^{1/2} \ln^{2} m}$   
    and the width $m$ of the neural network satisfies $\frac{m}{\ln^4 m} > \Omega\left(\frac{1+\alpha^2}{(1-\alpha)^2} \frac{n^5 L^{15} d}{\delta^4}\right)$. Then in the training stage described by Algorithm~\ref{alg:training_gd}
    \begin{equation}        
    \|\mW^{(t)} - \mW^{(0)}\| < O\left(\frac{\delta^{3/2}}{n^{3/2} L^{15/2} \ln^{3/2} m}\right) \ \text{ with probability at least } \  1-e^{-\Omega(\ln m)}. \label{eqn:W_bound_gd_theorem}
    \end{equation}
\end{lemma}

\begin{proof}
We first establish the bound
\begin{equation}
\label{eqn:append:init_loss_upper_bound}
\gL(\mW^{(0)}) < O(n\ln^{1/2} m) \ \text{ with probability at least }  1-e^{-\Omega(\ln m)}.   
\end{equation}   
We note that $\mB\vh_{i, L}\sim N\left(0, \frac{\|\vh_{i, L}\|^2}{d}\right)$ and thus $\frac{d}{\|\vh_{i, L}\|^2}\|\mB \vh_{i, L}\|^2 \, | \, \vh_{i, L} \sim \chi^2(d)$. Applying this observation and Lemma~\ref{lemma:initial_vector_norm} (i.e., $\|\vh_i\| \in [0.5, 1.5]$,  with probability at least $1-e^{-\Omega(m/L)}$) yields 
$$
\sP \left(\frac{d}{\|\vh_{i, L}\|^2}\|\mB \vh_{i, L}\|^2 > (1+\epsilon) d\right) < e^{-\Omega(d\epsilon^2)}. 
$$
Choosing $\epsilon = \sqrt{\ln m}$ and applying 
a union bound over $i \in [n]$ (but noting that since $m>\Omega(n)$ the probability $1 - n e^{-\Omega(d \ln m)}$ is of the same order as $1 - e^{-\Omega(d \ln m)}$), we obtain the bound
\begin{equation}
    \|\mB \vh_{i, L}\|^2 \leq O(\sqrt{\ln m}) \ \text{ with probability at least } 
1 - e^{-\Omega(d \ln m)}.\label{eqn:initial_bound_f}
\end{equation}
Therefore, we conclude \eqref{eqn:append:init_loss_upper_bound} as follows:
\begin{align*}
    \gL(\mW^{(0)}) & = \sum_{i=1}^n \|\vy_i - \mB \vh_{i, L}\|^2  \leq n (O(1) + O(\sqrt{\ln m}))  = O(n\sqrt{\ln m}).
\end{align*}

Next, we prove \eqref{eqn:W_bound_gd_theorem} by induction on $t=1,\cdots$. It is trivial that the statement holds for $t=0$.

To prove the induction step we follow ideas that were introduced in the proof of Lemma~4.1 in \cite{zou2019improved}. 
Using the induction assumption, we can apply \eqref{eqn:training_error_dynamic} and then \eqref{eq:bound_sqrt_L} (indeed, the conditions for these bounds are guaranteed by the induction assumption) and consequently obtain
\begin{align*}
    \sqrt{\gL(\mW^{(s)})} - \sqrt{\gL(\mW^{(s+1)})} & = 
    \frac{\gL(\mW^{(s)}) - \gL(\mW^{(s+1)})}{    \sqrt{\gL(\mW^{(s)})} + \sqrt{\gL(\mW^{(s+1)})}}\geq \Omega(1)\frac{\eta \|\nabla\gL(\mW^{(s)})\|_F^2}{\sqrt{\gL(\mW^{(s)})}} \\
    & \geq \frac{(1-\alpha)}{\sqrt{1+\alpha^2}}\Omega\left(\sqrt{\frac{\delta m}{nd}}\right)
     \eta\|\nabla\gL(\mW^{(s)})\|_F,
\end{align*}
or equivalently,
\begin{align}
\eta\|\nabla\gL(\mW^{(s)})\|_F\leq\frac{\sqrt{1+\alpha^2}}{(1-\alpha)}\Omega\left(\sqrt{\frac{nd}{\delta m}}\right) \left(\sqrt{\gL(\mW^{(s)})} - \sqrt{\gL(\mW^{(s+1)})}\right).\label{eqn:append:nabla}
\end{align}
Combining the training procedure with \eqref{eqn:append:nabla} yields
\begin{align}
\|\mW^{(t)} - \mW^{(0)}\| &\leq \eta\sum_{s=0}^{t-1} \|\nabla_{\mW}\gL(\mW^{(s)})\| \notag\\ 
& \leq \frac{\sqrt{1+\alpha^2}}{(1-\alpha)}\Omega\left(\sqrt{\frac{nd}{\delta m}}\right) \left(\sqrt{\gL(\mW^{(0)})} - \sqrt{\gL(\mW^{(t)})}\right)\notag\\
&\leq \frac{\sqrt{1+\alpha^2}}{(1-\alpha)}\Omega\left(\sqrt{\frac{nd}{\delta m}}\right)\sqrt{\gL(\mW^{(0)})}.\label{eqn:omega_bound_gd}
\end{align}
Applying \eqref{eqn:append:init_loss_upper_bound} to the  bound above we conclude that when $\frac{m}{\ln^4 m} > \frac{1+\alpha^2}{(1-\alpha)^2} \Omega(n^5 L^{15} d /\delta^4)$, 
$$\|\mW^{(t)} - \mW^{(0)}\| \leq O\left(\frac{\delta^{3/2}}{n^{3/2}L^{15/2}\ln^{3/2}m}\right) \ \text{with probability at least } \  1- e^{-\Omega(\ln m)}.$$
\end{proof}

\subsection{Proof of Theorem~\ref{thm:sgd_main_thm}}\label{appd:main_proof_sgd} 
Throughout this proof we assume that $\|\mW^{(t)} - \mW^{(0)}\| \leq O(\frac{\delta^{3/2 }}{n^{3/2} L^{15/2} \ln^{3/2} m})$ during training, which is a sufficient condition for some of the propositions used, such as for Lemma~\ref{thm:semi-smooth}. 
After finalizing the proof under this assumption, we establish Lemma~\ref{lemma:training_in_perturbation_sgd} that guarantees this assumption.

Applying Lemma~\ref{thm:semi-smooth} and taking expectations yield
\begin{align}
\label{eq:first_ineq_thm_3}
    \sE \gL(\mW^{(t+1)}) 
    & = \sE\gL(\mW^{(t)} - \eta \nabla_{\mW} \gL_B(\mW^{(t)})) \notag \\
    & \leq \sE\gL(\mW^{(t)}) - \sE \eta \langle \nabla_{\mW} \gL(\mW^{(t)}) , \nabla_{\mW} \gL_B(\mW^{(t)})\rangle \notag \\ 
     & \ + \frac{\eta(1-\alpha)\omega^{\frac{1}{3}}L^2\sqrt{m n \gL(\mW^{(t)}) \ln m}}{\sqrt{d(1+\alpha^2)}}
          \sE O\left(\|\nabla_\mW \gL_B(\mW^{(t)})\|\right)\\
     & \ + \frac{\eta^2nL^2m}{d}\sE O\left(\|\nabla_\mW \gL_B(\mW^{(t)})\|^2\right). \notag
\end{align}
Applying the following basic observations:
$$\sE\langle \nabla_{\mW} \gL(\mW^{(t)}) , \nabla_{\mW} \gL_B(\mW^{(t)})\rangle = \frac{b}{n}\|\nabla_{\mW} \gL(\mW^{(t)})\|_F^2,$$ 
$$\|\nabla_{\mW}\gL(\mW^{(t)})\| = \max_{l\in [L]}\|\nabla_{\mW_{l}} \gL(\mW^{(t)})\| \leq \max_{l\in [L]}\|\nabla_{\mW_{l}} \gL(\mW^{(t)})\|_F
\leq \|\nabla_\mW \gL (\mW^{(t)})\|_F,$$ while 
selecting $\omega < \frac{\delta^{3/2}}{n^3 L^{6} \ln^{3/2}m}$ and $\eta < \frac{d}{bL^2 m}$, to 
\eqref{eq:first_ineq_thm_3}
results in
\begin{equation}
    \sE\gL(\mW^{(t+1)}) \leq \gL(\mW^{(t)}) - \frac{\eta b}{n} \|\nabla_{\mW} \gL(\mW^{(t)})\|_F^2 \leq \left(1- \Omega\left(\frac{(1-\alpha)^2}{1+\alpha^2} \frac{\eta\delta m b}{n^2 d}\right) \right)\gL(\mW^{(t)}).\label{eqn:expectation_Lt}
\end{equation}
For simplicity, we define 
$$\gamma := \left(1- \Omega\left(\frac{(1-\alpha)^2}{1+\alpha^2} \frac{\eta\delta m b}{n^2 d}\right) \right),
$$
and \eqref{eqn:expectation_Lt} becomes
\begin{equation}
\sE\gL(\mW^{(t+1)})\leq \gamma\gL(\mW^{(t)}).\label{eqn:expectation_Lt_gamma}
\end{equation}

Next, we establish a bound for $\gL(\mW^{(t+1)})$ without expectation. We note that \eqref{eqn:upper_bnd} implies $$\|\nabla_{\mW_l} \gL_B(\mW^{(t)})\|_F^2 \leq (bm/d)\gL(\mW^{(t)}) ,$$ and consequently 
\begin{equation}
    \|\nabla_{\mW} \gL_B(\mW^{(t)})\|_F^2 \leq \frac{bmL}{d}\gL(\mW^{(t)}) \ \text{ and } \ \|\nabla_\mW \gL_B(\mW^{(t)})\|^2 \leq \frac{bm}{d}\gL(\mW^{(t)}).\label{eqn:upper_bound_SGD_LB}
\end{equation} 
The application of Lemma~\ref{thm:semi-smooth}, \eqref{eqn:upper_bound_SGD_LB} and our choice of $\eta$ results in
\begin{align}
\gL(\mW^{(t+1)}) & \leq \gL(\mW^{(t)}) + \eta\| \nabla_{\mW} \gL_B(\mW^{(t)})\|_F\|\nabla_{\mW}\gL(\mW^{(t)})\|_F
 + \eta \frac{b^2mn}{d}\gL(\mW^{(t)})\notag\\
    & \leq \left(1+O\left(\frac{\eta m L\sqrt{nb}}{d}\right)\right)\gL(\mW^{(t)}).\label{eqn:upperbound_before_beta}
\end{align}
For simplicity, we define $\beta:=1+O(\eta m L \sqrt{nb}/d)$, and \eqref{eqn:upperbound_before_beta} becomes
\begin{equation}
    \gL(\mW^{(t+1)})\leq \beta \gL(\mW^{(t)}).\label{eqn:upperbound_beta}
\end{equation}

We denote $$\gL^t := \gL(\mW^{(t)})$$ and define the filtration $$\gF_t := \sigma(\mW^{(0)}, ..\mW^{(t)}).$$ We further define 
$$Y_t := \ln \gL^{t} - \ln\gL^{t-1} - \sE(\ln \gL^{t} - \ln\gL^{t-1} | \gF_{t-1})$$ 
and   
$$X_t := \sum_{s=1}^t Y_s.$$ 
We note that $\{X_t\}$ is a martingale. 

We will use Azuma's inequality to bound $X_t$. 
We thus need to show that $\{X_t\}$ is $\vc-$Lipschitz ( i.e.,
$|Y_t| \leq c_t,$). 
We verify the $\vc-$Lipschitz property by applying the definition of $Y_t$, \eqref{eqn:upperbound_beta} and \eqref{eqn:expectation_Lt_gamma} as follows:
\begin{align*}
|Y_{t+1}| &= |\ln \gL^{(t+1)} - \ln\gL^{(t)} - \sE \ln \gL^{(t+1)} - \ln\gL^{(t)}|\gF_{t}| \\
&\leq \ln \beta - \ln\gamma = \ln\frac{\beta}{\gamma}.
\end{align*}
Then by Azuma's inequality,
\begin{equation}
    \sP \left(|X_t - E X_t| \geq \lambda \right) \leq 2\exp\left(-\frac{\lambda^2}{2t\ln^2 \beta/\gamma}\right).\label{eqn:azuma}
\end{equation}
Choosing $\lambda = \sqrt{t}\ln(\beta/\gamma)\ln m$ in \eqref{eqn:azuma} yields
\begin{equation}
    |X_t| \leq \sqrt{t}\ln(\beta/\gamma)\ln m \ \text{ with probability at least } 1-e^{-\Omega(\ln^2 m)}.\label{eqn:Azuma_bound_X_t}.
\end{equation}

Applying the definition of $Y_t$ and  \eqref{eqn:expectation_Lt_gamma} results in
$$
\ln\gL^{t} = X_t + \ln\gL^{(0)} + \sum_{s=1}^t \sE(Y_s - Y_{s-1}|\gF_{s-1}) \leq X_t + \ln\gL^{(0)} + t\ln\gamma.
$$
We further apply the above observation and \eqref{eqn:Azuma_bound_X_t} to conclude that with probability at least $1-e^{-\Omega(\ln^2 m)}$
\begin{align*}
\ln \gL^{(t)} &\leq \ln \gL^{(0)} + t\ln \gamma + \sqrt{t}\ln\left(\frac{\beta}{\gamma}\right)\ln m \\
& \leq \ln \gL^{(0)} + \frac{\ln^2\left(\frac{\beta}{\gamma}\right) \ln^2 m}{4|\ln\gamma|} - \left(\sqrt{|\ln\gamma|}\sqrt{t} - \frac{\ln\frac{\beta}{\gamma} \ln m}{2\sqrt{|\ln\gamma|}}\right)^2.
\end{align*}
We note that for $f(x) = (ax + b)^2$ and $x > 4b/a$, 
$f(x) \geq \frac{1}{2}a^2x^2$. Using this fact, we conclude that when $\sqrt{t} > \frac{2\ln\frac{\beta}{\gamma} \ln m}{|\ln\gamma|}$, or equivalently, when $t > \frac{4\ln^2\frac{\beta}{\gamma} \ln^2 m}{\ln^2\gamma}$, 
\begin{align}
\ln \gL^{(t)} &\leq \ln \gL^{(0)}  + \frac{\ln^2\frac{\beta}{\gamma} \ln^2 m}{4|\ln\gamma|} + t \times 1_{\left\{t > \frac{4\ln^2\frac{\beta}{\gamma} \ln^2 m}{\ln^2\gamma}\right\}}\ln\gamma \ \text{ with probability } \geq 1-e^{-\Omega(\ln^2 m)}.\label{eqn:append:sgd:convergence}
\end{align}
This implies that when $t > \frac{4\ln^2\frac{\beta}{\gamma} \ln^2 m}{\ln^2\gamma}$ we achieve linear convergence with a convergence rate of $\gamma$. 
By our choice of $\eta$, the additional term in \eqref{eqn:append:sgd:convergence} is bounded as follows\begin{align*}:
    \frac{\ln^2(\beta/\gamma)\ln^2 m}{|\ln\gamma|} & \leq O\left(\frac{(\beta-1)^2}{1-\gamma}\ln^2m\right)  = O\left(\eta \frac{mn^3 L^2\ln^2 m}{d\delta}\right) < O(1).
\end{align*}

The above lower bound of $t$ that guarantees linear convergence can be further simplified.  
Since $
\frac{x}{1+x}\leq \ln(1+x) \leq x
$, we note $t \geq \frac{((\beta - 1)^2 + \frac{(\gamma-1)^2}{\gamma^2})\gamma^2\ln^2 m }{(\gamma -1)^2}$ and thus
\begin{align*}
  t & \geq \ln^2 m \times\left(1 + \frac{(\beta-1)^2}{(\gamma-1)^2}\right).
\end{align*}
Recalling the expressions for $\beta$ and $\gamma$, we conclude that linear convergence is achieved when
\begin{equation}
    t > \Omega\left(\frac{(1+\alpha^2)^2}{(1-\alpha)^4}\frac{n^5L^2}{\delta^2 b}\ln^2 m \right).\label{eqn:t_requirement}
\end{equation}

The above argument holds for one training step with probability at least $1-e^{-\Omega(m)}$. It extends to $T$ steps with probability at least $1-Te^{-\Omega(m)}$. We note that the number of epochs $T$ can be bounded using the bound $\eps$ on the training error, the convergence rate in \eqref{eqn:gamma_2nd_theorem}, and \eqref{eqn:append:init_loss_upper_bound}, as follows:
$$T = \ln(\epsilon/C_0\gL(\mW^{(0)}))/\ln\gamma<\Theta(\ln(\epsilon/C_0n\sqrt{\ln m})/\ln\gamma) \leq O\left(\frac{\eta b\delta m}{n^2d} (\ln \epsilon^{-1} + \ln (C_0 n\sqrt{\ln m}))\right).
$$ 
Therefore, the probability that ensures $T$-steps training with training error lower than $\epsilon$ is at least $$1-O\left(\frac{\eta b\delta m}{n^2d}(\ln \epsilon^{-1} + \ln (C_0 n\sqrt{\ln m}))\right) e^{-\Omega(m)}.$$ Because $m > \Omega(\text{poly}(n,L,d,\delta^{-1}, b))$ and also $m > \Omega(\ln \ln \epsilon^{-1})$, this probability is of order $1-e^{-\Omega(m)}$.

\begin{lemma}\label{lemma:training_in_perturbation_sgd}
    Assume the setup of \S\ref{sec:problem_setup} with learning rate $\eta < \frac{\delta^{3/2}d^{1/2} }{b^{1/2} n^{3} L^{15/2} m^{1/2} \ln^{2} m}$ and neural network width $m$ satisfying $\frac{m}{\ln^4 m} > \frac{(1+\alpha^2)^4}{(1-\alpha)^8} \Omega\left(\frac{n^{8} L^{15}d}{\delta^5 b}\right)$. Then during training according to  Algorithm~\ref{alg:training_sgd},
    $$
    \|\mW^{(t)} - \mW^{(0)}\| < O\left(\frac{\delta^{3/2}}{n^{3/2} L^{15/2} \ln^{3/2} m}\right) \ \text{ with probability at least } 1-e^{-\Omega(\ln m)}.
    $$
\end{lemma}
\begin{proof} The proof is similar to Lemma~\ref{lemma:training_in_perturbation}, but with different bounds in this SGD setting. We first show bound the perturbation at initialization as follows:
$$
\|\mW^{(1)} - \mW^{(0)}\| = \eta \; \|\nabla_{\mW} \gL_B(\mW^{(0)})\| \leq \eta \; O\left(\sqrt{\frac{mnb}{d}}\right)\gL(\mW^{(0)}) < O\left(\frac{\delta^{3/2}}{n^{3/2}L^{15/2}\ln^{3/2}m}\right).
$$

We denote $T_0:=\Omega\left(\frac{(1+\alpha^2)^2}{(1-\alpha)^4}\frac{n^5L^2}{\delta^2 b}\ln^2 m \right)$. Combining the SGD update step, \eqref{eqn:upper_bound_SGD_LB}, \eqref{eqn:append:sgd:convergence}, \eqref{eqn:t_requirement} and our choice of $\eta$ yields
\begin{align*}
\|\mW^{(t)} - \mW^{(0)}\| &\leq \eta\sum_{s=0}^{t-1} \|\nabla_{\mW}\gL_B(\mW^{(s)})\|\leq \eta\sum_{s=0}^{t-1} \sqrt{\frac{mb}{d}}\sqrt{\gL(\mW^{(s)})}\\ 
& \leq \sqrt{\frac{mb}{d}}\eta \left(\frac{1}{1-\sqrt{\gamma}} + T_0\right)\sqrt{\gL(\mW^{(0)})}\\
& \leq O(1)\sqrt{\frac{mb}{d}}\eta\left(\frac{1+\alpha^2}{(1-\alpha)^2}\frac{n^2 L\eta }{\delta m b}  + \Omega\left(\frac{(1+\alpha^2)^2}{(1-\alpha)^4}\frac{n^5L^2}{\delta^2 b}\ln^2 m \right)\right)\sqrt{\gL(\mW^{(0)})} \\
& \leq  O(1)\sqrt{\frac{mb}{d}}\eta\Omega\left(\frac{(1+\alpha^2)^2}{(1-\alpha)^4}\frac{n^5L^2}{\delta^2 b}\ln^2 m \right)\sqrt{\gL(\mW^{(0)})}\\
& \leq 
O(1)\sqrt{\frac{mb}{d}}
\frac{d\delta}{mn^3 L^2\ln^2 m}\Omega\left(\frac{(1+\alpha^2)^2}{(1-\alpha)^4}\frac{n^5L^2}{\delta^2 b}\ln^2 m \right)\sqrt{\gL(\mW^{(0)})}\\
& = O(1)\frac{(1+\alpha^2)^2}{(1-\alpha)^4} \frac{\sqrt{d}n^2}{\sqrt{mb}\delta}\sqrt{\gL(\mW^{(0)})}.
\end{align*}
It is thus clear that when $\frac{m}{\ln^4 m} > \left(\frac{1+\alpha^2}{(1-\alpha)^2}\right)^4 \Omega\left(\frac{n^{8} L^{15}d}{\delta^5 b}\right)$, $\|\mW^{(t)} - \mW^{(0)}\|\leq O\left(\frac{\delta^{3/2}}{n^{3/2}L^{15/2}\ln^{3/2}m}\right)$.
\end{proof}

\subsection{Proof of Lemma~\ref{thm:generalization_overall}}
\label{appd:generatlization_proof}
To simplify the proof, we study the generalization error of each output coordinate
separately. We denote the $k-$th row of the matrix $\mB$ by $\mB_{k,\cdot}$ and treat it as a column vector. For $k\in [d]$, we define function $$f_k(\vx; \mW) := \mB_{k, \cdot}^T \vh_{L}(\vx),$$ that is, $f_k(\vx; \mW)$ is the $k-$th coordinate of the NN output vector. The loss function can be written as $\text{loss}_k(\vx, \vy; \mW) := (f_k(\vx; \mW) - \vy_k)^2$. Recall that the underlying measurable function $F(\vx)$ (i.e., $\vy_i = F(\vx_i)$) is a $d-$dimensional vector-valued function and we denote by $F_k(\vx)$ the $k-$th coordinate of $F(\vx)$. The generalization error is similarly defined as $R_k(\mW) := \sE_{\vx \sim \gD_{\mX}} (f_k(\vx; \mW) - F_k(\vx))^2$. We also denote $$B_\omega(\mW^{(0)}) := \{\mW : \|\mW - \mW^{(0)}\| < \omega\}.$$

From Lemma~\ref{lemma:training_in_perturbation} and Lemma~\ref{lemma:training_in_perturbation_sgd}, with high probability, $\mW^{(t)}$ is close to $\mW^{(0)}$ during the training. Therefore, we just need to only consider NN functions whose parameters $\mW$ fall in a small ball around $\mW^{(0)}$, i.e. $\|\mW - \mW^{(0)}\| < \omega$, where $\omega < O\left(\frac{\delta^{3/2}}{n^{3/2}L^{15/2}\ln^{3/2} m}\right)$. For a given $k\in[d]$, we denote the corresponding function class as 
$$
\gG_{k,\omega} := \{g: (x, y)\mapsto f_k(\vx; \mW) : \|\mW - \mW^{(0)}\| < \omega\}
$$

We introduce the empirical Rademacher complexity on the dataset $\{x_i, y_i\}_{i=1}^n$ as follows:
$$\hat{\gR}(\gG_{k, \omega}) := 
\sE_\sigma \sup_{g\in \gG_{k, \omega}} \sum_{i=1}^n\sigma_i g(x_i, y_i)
$$


For $k\in[d]$, we first bound the generalization error on the $k-$th coordinate of the output vector.

We first note that by \eqref{eqn:initial_bound_f} and Lemma~\ref{lemma:forward_perturbation}, with high probability that $f_k(\vx_i;\mW) < O(\ln^{1/4} m)$ for all $i\in[n]$ and $\|\mW-\mW^{(0)}\| < \omega$.
We apply Theorem~11.3 in \citet{mohri2018foundations} with function class $\gG_{k, \omega}$, and thus bound $R_k(\mW)$ with probability at least $1-\Omega(1/m)$  by
\begin{equation}
    \sE_{\vx\sim \gD_{\mX}} \text{loss}_k(\vx, F(\vx); \mW) \leq  \frac{1}{n}\sum_{i=1}^n \text{loss}_k(\vx_i, \vy_i; \mW) + 2O\left(\ln^{1/4} m\right) \hat{\gR}(\gG_{k, \omega}) + O\left(\ln^{1/4} m\right) O\left(\sqrt{\frac{\ln 2m}{2n}}\right).\label{eqn:generalization_error}
\end{equation}

We note that the first term in \eqref{eqn:generalization_error} is bounded by the previously discussed training error, and the third term in \eqref{eqn:generalization_error} is very small when we collect a sufficiently large dataset since $m$ is polynomially dependent on $n$. Next, we estimate the bound for the second term, the empirical Rademacher complexity.
\begin{align}
   \hat{\gR}(\gG_{k, \omega}) 
   & = \sE_\sigma \sup_{\mW\in B_\omega (\mW^{(0)})} \frac{1}{n}\sum_{i=1}^n\sigma_i \Big(f_k(x_i, y_i;\mW) - f_k(x_i, y_i;\mW^{(0)}) - \langle \nabla_{\mW}f_k(x_i, y_i;\mW^{(0)}), \mW - \mW^{(0)}\rangle  \notag\\
    & \quad + f_k(x_i, y_i;\mW^{(0)}) + \langle \nabla_{\mW}f_k(x_i, y_i;\mW^{(0)}), \mW - \mW^{(0)}\rangle\Big)\notag\\
    & \leq \sup_{\mW\in B_\omega (\mW^{(0)})} \sup_{i}\left|f_k(x_i, y_i;\mW) - f_k(x_i, y_i;\mW^{(0)}) - \langle \nabla_{\mW}f_k(x_i, y_i;\mW^{(0)}), \mW - \mW^{(0)}\rangle\right| \label{eqn:general_1}\\
    & \quad  + \sE_\sigma \sup_{\mW\in B_\omega (\mW^{(0)})} \frac{1}{n}\sum_{i=1}^n\sigma_if_k(x_i, y_i;\mW^{(0)})  \label{eqn:general_2}\\
    & \quad + \sE_\sigma \sup_{\mW\in B_\omega (\mW^{(0)})} \frac{1}{n}\sum_{i=1}^n \sigma_i\langle \nabla_{\mW}f_k(x_i, y_i;\mW^{(0)}), \mW - \mW^{(0)}\rangle. \label{eqn:general_3}
\end{align}
We first consider \eqref{eqn:general_2}, since there is no dependence on $\mW$, it is clear that
\begin{equation}
    \sE_\sigma \sup_{\mW\in B_\omega (\mW^{(0)})} \frac{1}{n}\sum_{i=1}^n \sigma_i f_k(x_i, y_i;\mW^{(0)}) = \frac{1}{n}\sum_{i=1}^n f_k(x_i, y_i;\mW^{(0)}) \sE_\sigma \sigma_i = 0.
\end{equation}

In order to help bound \eqref{eqn:general_1}, which is the first term in our bound of the empirical Rademacher complexity for the NN function class, we introduce the following lemma.
\begin{lemma} \label{lemma:bound_f} 
If $\omega < O\left(\frac{\delta^{3/2}}{n^{3/2}L^{15/2}\ln^{3/2} m}\right)$ and $\|\mW - \mW^{(0)}\| < \omega$, then for any $\vx\in\mR^p$, with probability at least $1 - \exp(-\Omega(\sqrt{m}/\ln m)$,
\begin{equation}
    \left| f_k(\vx; \mW) - f_k(\vx; \mW^{(0)}) - \langle\nabla_{\mW}f_k(\vx; \mW^{(0)}), \mW - \mW^{(0)} \rangle \right|  < \frac{1-\alpha}{\sqrt{1+\alpha^2}}O(\omega^{4/3} L^2 \sqrt{m\ln m}).
\end{equation}
\end{lemma}

We prove Lemma~\ref{lemma:bound_f} at the end of this section after we finalize the proof of Lemma~\ref{thm:generalization_overall} (while applying Lemma~\ref{lemma:bound_f}).

Using Lemma~\ref{lemma:bound_f}, the term \eqref{eqn:general_1} can be bounded as
\begin{align*}
    & \sup_{i}\left|f_k(x_i, y_i;\mW) - f_k(x_i, y_i;\mW^{(0)}) - \langle \nabla_{\mW}f_k(x_i, y_i;\mW^{(0)}), \mW - \mW^{(0)}\rangle\right| \\
    & \leq \frac{1-\alpha}{\sqrt{1+\alpha^2}}\omega^{4/3}L^2 \sqrt{m\ln m}.
\end{align*}

Applying Cauchy-Schwarz inequality and Jessen's inequality to \eqref{eqn:general_3} and using Lemma~\ref{thm:gradient_bound}, we conclude 
\begin{align*}
    & \left|\sE_\sigma \sup_{\mW\in B_\omega (\mW^{(0)})} \frac{1}{n}\sum_{i=1}^n \sigma_i\langle \nabla_{\mW}f_k(\vx_i, \vy_i;\mW^{(0)}), \mW - \mW^{(0)}\rangle\right|  \\
    & \leq \omega  \sE_\sigma \sup_{\mW\in B_\omega (\mW^{(0)})} \frac{1}{n}\sum_{i=1}^n \sum_{l=1}^L\| \nabla_{\mW_l}f_k(\vx_i, \vy_i;\mW^{(0)})\|_F\\
    & \leq \frac{\omega }{n} \sum_{l=1}^L\sqrt{\sum_{i=1}^n\|\nabla_{\mW_l} f_k(\vx_i, \vy_i; \mW^{(0)})\|^2_F}\\
    & \leq \frac{\omega}{n} L\sqrt{mn} 
    \leq \frac{\omega L \sqrt{m}}{\sqrt{n}}
\end{align*}
Using the bound for \eqref{eqn:general_1}, \eqref{eqn:general_2} and \eqref{eqn:general_3} in \eqref{eqn:generalization_error}, it follows that with probability at least $1-e^{-\Omega(\ln m)}$ 
\begin{equation*}
R_k(\mW) \leq \frac{1}{n}\text{loss}_k (\vx_i, \vy_i; \mW) + \frac{1-\alpha}{\sqrt{1+\alpha^2}} 
\, O\left(\ln m \sqrt{m} L^2 \omega^{4/3}\right) + \omega \, O\left(L \sqrt{m\ln m/n}\right) + O\left(\sqrt{\frac{\ln m}{n}}\right).
\end{equation*}
Summing over $k\in[d]$, we conclude Lemma~\ref{thm:generalization_overall} as follows
\begin{equation*}
R(\mW) \leq 
\frac{1}{n}\text{loss} (\vx_i, \vy_i; \mW) + \frac{1-\alpha}{\sqrt{1+\alpha^2}} \, O\left(d \ln m \sqrt{m} L^2 \omega^{4/3}\right) +  O\left(d \sqrt{m\ln m/n}L\omega\right) + O\left(d \sqrt{\frac{\ln m}{n}}\right).    
\end{equation*}


Finally, we complete this section by presenting the proof of Lemma~\ref{lemma:bound_f}.
\begin{proof}[Proof of Lemma~\ref{lemma:bound_f}] 
Using the notation of \S\ref{appd:perturbation} (in particular, $\vh_l$ and $\vh^{(0)}_l$) and the definitions of $f_k(\vx;\mW^{(0)})$ and  $f_k(\vx;\mW)$ and recalling that $\vh_0 = \vh^{(0)}_0$ and $\vh_{l} = \mD_l\mW_l\vh_{l-1}$ we derive the following expression:
\begin{align}
    & f_k(\vx; \mW) - f_k(\vx; \mW^{(0)}) = \mB_{k, \cdot}^T \left(\mD_L\mW_L\cdots \mD_1\mW_1 - \mD_L^{(0)}\mW_L^{(0)}\cdots \mD_1^{(0)}\mW_1^{(0)} \right) \mA\vx \nonumber \\
    & = \mB_{k, \cdot}^T \Big( \mD_L\mW_L \vh_{L-1} - \mD_L^{(0)}\mW_L^{(0)}\vh_{L-1}  \nonumber \\
    & \quad + \mD_L^{(0)}\mW_L^{(0)}\vh_{L-1} -  \mD_L^{(0)}\mW_L^{(0)}\mD_{L-1}^{(0)}\mW_{L-1}^{(0)}\vh_{L-2} \nonumber \\ 
    & \quad + \mD_L^{(0)}\mW_L^{(0)}\mD_{L-1}^{(0)}\mW_{L-1}^{(0)}\vh_{L-2} -  \mD_L^{(0)}\mW_L^{(0)}\mD_{L-1}^{(0)}\mW_{L-1}^{(0)}\mD_{L-2}^{(0)}\mW_{L-2}^{(0)}\vh_{L-3} \nonumber\\
    & \cdots \ \cdots 
    \label{eq:diff_f} \\
    & + \mD_L^{(0)}\mW_L^{(0)}\cdots\mD_2^{(0)}\mW_2^{(0)}\vh_{1} - \mD_L^{(0)}\mW_L^{(0)}\cdots\mD_2^{(0)}\mW_2^{(0)}\mD_1^{(0)}\mW_1^{(0)}\vh_{0}
    \Big). 
    \nonumber
\end{align}
Let $a$ and $b$ be two integers in $[1, L]$. If $b \leq a$, we denote $$(\mD^{(0)}\mW^{(0)})_{a\mapsto b} := \mD^{(0)}_a\mW^{(0)}_a\mD^{(0)}_{a-1}\mW^{(0)}_{a-1}\cdots \mD^{(0)}_b\mW^{(0)}_b.$$ 
If $a<b$, we denote $$(\mD^{(0)}\mW^{(0)})_{a\mapsto b}:= \mI.$$   
Applying $\vh_{l} = \mD_l\mW_l\vh_{l-1}$ for the first term in each line of \eqref{eq:diff_f}, \eqref{eq:diff_f} can be written as
\begin{align}
    &f_k(\vx; \mW) - f_k(\vx; \mW^{(0)}) \notag\\& = \mB_{k, \cdot}^T \Big((\mD_L\mW_L - \mD_L^{(0)}\mW_L^{(0)})\vh_{L-1}  \notag\\ 
    & + (\mD^{(0)}\mW^{(0)})_{L\mapsto L}(\mD_{L-1}\mW_{L-1} - \mD_{L-1}^{(0)}\mW_{L-1}^{(0)})\vh_{L-2}  \notag\\ 
    & + (\mD^{(0)}\mW^{(0)})_{L\mapsto L-1}(\mD_{L-2}\mW_{L-2} - \mD_{L-2}^{(0)}\mW_{L-2}^{(0)})\vh_{L-3}  \notag\\ 
    & \cdots \ \cdots\notag\\
    & + (\mD^{(0)}\mW^{(0)})_{L\mapsto 2}(\mD_{1}\mW_{1} - \mD_{1}^{(0)}\mW_{1}^{(0)})\vh_{0} 
    \Big) \notag\\ 
    & = \mB_{k, \cdot}^T \sum_{l=1}^{L}  (\mD^{(0)}\mW^{(0)})_{L\mapsto l+1}(\mD_{l}\mW_{l} - \mD_{l}^{(0)}\mW_{l}^{(0)})\vh_{l-1}\notag\\
    & = \mB_{k, \cdot}^T \sum_{l=1}^{L}  (\mD^{(0)}\mW^{(0)})_{L\mapsto l+1}
    (\mD_l - \mD_l^{(0)})\mW_l \vh_{l-1} \label{eqn:rademacher_1}\\ 
    & \quad + \mB_{k, \cdot}^T \sum_{l=1}^{L}  (\mD^{(0)}\mW^{(0)})_{L\mapsto l+1}\mD_l^{(0)}(\mW_l - \mW_l^{(0)})(\vh_{l-1} - \vh_{l-1}^{(0)})\label{eqn:rademacher_2} \\
    & \quad  + \mB_{k, \cdot}^T \sum_{l=1}^{L}  (\mD^{(0)}\mW^{(0)})_{L\mapsto l+1}\mD_l^{(0)}(\mW_l-\mW^{(0)}_l)\vh_{l-1}^{(0)}
    \label{eqn:rademacher_3}.
\end{align}

According statement 1 in Lemma~\ref{lemma:forward_perturbation}, with probability at least $1-e^{-\Omega(\sqrt{m}/\ln m)}$,$\|\mD'_l\vg_l\| < (1-\alpha)/\sqrt{1+\alpha^2} O(L^{3/2}
\omega
)$ and $\|\mD'_l\|_0 \leq O(m\omega^{2/3}L)$. Combining this with statement 4 in Lemma~\ref{lemma:initial_forward_matrix_bound} with $\vv = (1, 1,\cdots, 1)^T \in \sR^d$, we bound the norm $\|(\mB_k^T\mD^{(0)}\mW^{(0)})_{L\mapsto l+1}\|$ by $O(\omega^{1/3}\sqrt{mL \ln m})$ with probability at least $1 - \exp(-\Omega(m\omega^{3/2}L \ln m)$. Then \eqref{eqn:rademacher_1} can be bounded (with the same probability) by 
\begin{equation}
    \frac{1-\alpha}{\sqrt{1+\alpha^2}}O(\omega^{4/3}L^2 \sqrt{m\ln m} ).\label{eqn:rademacher_bound}
\end{equation}
By using statement 3 in Lemma~\ref{lemma:forward_perturbation}, i.e., $\|\vh_l - \vh_l^{(0)}\| < O(\omega L^{5/2}\ln m)$ with probability at least $1-e^{-\Omega(\sqrt{m}/\ln m)}$, we note the norm of the summation in \eqref{eqn:rademacher_2} is bounded by $O(\omega^2 L^{5/2}\ln m)$ (with the latter probability), which is much smaller than \eqref{eqn:rademacher_bound} when $\omega$ is small as given.

By noting that the gradient of $f_k(\vx; \mW)$ with respect to $\mW_l$ can be written as
    \begin{equation*}
        \nabla_{\mW_l} f_k(\vx; \mW^{(0)})  = (\mB_{k,\cdot}^T \mD^{(0)}_L\mW^{(0)}_L\cdots \mW^{(0)}_{l+1}\mD^{(0)}_l)^T \vh_{l-1}^{(0)T}, 
    \end{equation*}
we express the summands in \eqref{eqn:rademacher_3} as follows
\begin{equation}
    \langle \nabla_{\mW_l} f_k(\vx; \mW^{(0)}), \mW_l - \mW_l^{(0)}\rangle \equiv \mB_{k,\cdot}^T \mD^{(0)}_L\mW^{(0)}_L\cdots \mW^{(0)}_{l+1}\mD^{(0)}_l (\mW_l - \mW_l^{(0)}) \vh_{l-1}^{(0)}\label{eqn:gradient_fk}.
\end{equation}
Using \eqref{eqn:gradient_fk} and bounding \eqref{eqn:rademacher_1} and \eqref{eqn:rademacher_2} by \eqref{eqn:rademacher_bound}, we conclude that with probability at least $1-e^{-\Omega(\sqrt{m}\ln m)}$
$$
\left| f_k(\vx; \mW) - f_k(\vx; \mW^{(0)}) - \langle\nabla_{\mW}f_k(\vx; \mW^{(0)}), \mW - \mW^{(0)} \rangle \right|  < \frac{1-\alpha}{\sqrt{1+\alpha^2}}O(\omega^{4/3} L^2 \sqrt{m\ln m}).
$$
\end{proof}

\subsection{Proof of Theorem~\ref{thm:generalization_in_training_gd}}
\label{appd:generalization_gd}
The key idea of the proof of this theorem is to establish a bound for $\omega$, such that $\|\mW^{(t)}-\mW\| < \omega$ during training. 
Considering the learning rate $\eta$ and training steps $t$, we first establish a simple bound for $\omega$ as
\begin{align*}
    \|\mW^{(t)} - \mW^{(0)}\|  \leq \sum_{t=0}^{t-1} \eta\|\nabla_\mW \gL^{(t)}\| \leq \eta \sqrt{\frac{mn}{d}} \sum_{t=0}^t \sqrt{\gL^{(t)}}  \leq \eta t \sqrt{\frac{n m\ln m}{d}}.
\end{align*}

Furthermore, in the proof of Lemma~\ref{lemma:training_in_perturbation}, the following universal bound of $\omega$ was introduced:
$$
    \|\mW^{(t)} - \mW^{(0)}\| \leq O\left(\sqrt{\frac{nd}{\delta m}}\right)\sqrt{\gL^{(0)}}.
$$
By combining these two bounds with the universal bound $\omega < O(\frac{\delta^{3/2}}{n^{3/2}L^{15/2}\ln^{3/2}m})$ and using Lemma~\ref{thm:generalization_overall}, we conclude the theorem. 


\subsection{Generalization Error Bound for SGD}
\label{appd:generalization_sgd}
We present a theorem similar to Theorem~\ref{thm:generalization_in_training_gd} that establishes the upper bound of the generalization error for SGD.

\begin{theorem}
\label{thm:generalization_in_training_sgd}
Assume the setup of \S\ref{sec:problem_setup} with SGD, where $m = \Theta\left(\frac{n^{13+2\epsilon}L^{15 + 2\epsilon}d^{1+2\epsilon}}{ b\delta^{5 - 2\epsilon}}\right)$ for $\epsilon>0$ 
and $\eta = \Theta(\frac{d\delta}{n^3L^3m\ln^2 m})$. Assume further that $m$ is larger than its lower bound and $\eta$ is smaller than its upper bound in Theorem~\ref{thm:sgd_main_thm} (by an appropriate choice of the hidden constants in $\Theta$ and in comparison to the constants hidden in the lower bound of $m$ and the upper bound of $\eta$ in Theorem~\ref{thm:sgd_main_thm}). 
Then at a given training epoch $t$, with probability at least $1-e^{-\Omega(\ln m)}$, the generalization error is bounded as follows 
\begin{equation}
\begin{split}
        R(\mW^{(t)}) & \leq\gamma^t  O(\ln m) + 
\min\left\{\left(\frac{1-\alpha}{\sqrt{1+\alpha^2}}\right) O\left(\frac{d^{1/3}t^{4/3}}{m^{1/6}n^{10/3}L^{2}\ln^{8/3}m}\right), O\left(\frac{d^{3/2+\epsilon}  n^{2+\epsilon}}{b^{1/2} \delta^{1/2 - \epsilon} L^{1/2 - \epsilon}\ln m}\right)\right\}  +\\
&
\min\left\{O\left(\frac{\sqrt{d}~t }{n^3L^2\ln^{3/2} m}\right), O\left(\frac{n^{2+\epsilon} L^{2+\epsilon} d^{1/2+\epsilon}}{b^{1/2}\delta^{1 -\epsilon} \ln m}\right)\right\} + O\left(d\sqrt{\frac{\ln m}{n}}\right).
\end{split}
\label{eqn:generalization_error_sgd}
\end{equation}
\end{theorem}

The proof is similar to the proof of Theorem~\ref{thm:generalization_in_training_gd}. We estimate the bound of $\omega$ when $t$ is small as 
$$
\|\mW^{(t)} - \mW^{(0)}\| < O\left(\frac{mn\eta t }{d}\right)\sqrt{\gL^{(0)}}.
$$
Also, the bound of $\omega$ in the entire training for SGD can be obtained in the proof of Lemma~\ref{lemma:training_in_perturbation_sgd} as 
$$
\|\mW^{(t)} - \mW^{(0)}\| < O\left(\frac{d\sqrt{n}}{\delta\sqrt{mb}}\sqrt{\gL^{(0)}}\right).
$$
Then combining these two bounds of $\omega$ and using Lemma~\ref{thm:generalization_overall}, we could conclude the theorem.

\subsection{Special dataset}
\label{subsec:appendix:special_func_class}
In this section, we consider a special class of datasets and improve our theory for datasets from this class. We first introduce the special dataset and establish the assumption, then present theorems to bound the convergence rate and generalization error under this assumption. The proof will be given in Appendix~\ref{appendix:proof_special_dataset}.

First, with the parameters $\mW^{(0)}$ before the $l-$th layer, we denote the output at the $l-$th layer as 
$$
\gN_l (\vx; \mA, \mW_1^{(0)}, \mW_2^{(0)},\cdots \mW_{l-1}^{(0)}, \vu) = \tilde{\sigma}_\alpha(\vu^T \tilde{\sigma}_\alpha(\mW^{(0)}_{l-1}\tilde{\sigma}_{\alpha}(\mW^{(0)}_{l-2}\cdots \tilde{\sigma}_\alpha (\mW_1^{(0)} \mA \vx)))),
$$
and define the following class of functions:
    \begin{equation}
        \begin{split}            
    \gF_l :=  \Big\{&\vf(\vx) = (f_1(\vx), \cdots, f_d(\vx))^T : \sR^p \mapsto \sR^d, \ \text{where }\\&  f_j(\vx) = \sE_{\vu} c_j(\vu) \, \gN_l(\vx; \mA,\mW^{(0)}_1\cdots\mW^{(0)}_{l-1}, \vu)  \text{ for } \vu \in \sR^{m} \sim N\left(0, \frac{2}{m}\right)   \text{ and } \\ 
     & c_j:\sR^m \mapsto \sR  \text{ such that }  |c_j(\cdot)| \leq 1  \ \text{for} \  j \in [d]  \Big\}.
        \end{split}\label{eqn:def:Fl}
    \end{equation}
We note that this function class $\gF_l$ includes functions defined by an $l-$layer leaky ReLU neural network, where the first $l-1$ layers use the initialized parameters $\mW^{(0)}$ and only the parameters of the $l-$th layer are tuned with a certain regularization condition ($\|c_j\|_\infty < 1$). Given this function class, we restrict our discussion to datasets satisfying the assumption below. For such datasets,  we can improve the upper bound for the convergence rate, the lower bound for the width $m$, and the upper bound for the generalization error.
\begin{assumption}\label{assump:special_dataset}
For any small constant  $0<\lambda < \frac{1}{2\sqrt{nd}}$, there exists $f\in \gF_{L-1}$ such that 
$$
\|\vf(\vx_i) - (\vy_i - \hat{\vy}_i)\| < \lambda , \ \text{for all } \ i\in[n].
$$
\end{assumption}

\begin{theorem}\label{append:thm:convergence_special_case}
    Assume the setup in \S\ref{sec:problem_setup} with a dataset satisfying Assumption~\ref{assump:special_dataset}, where both $m/ln^4 m>\Omega(d^5 n L^{12})$ and $m>\Omega(\ln\ln\epsilon^{-1})$ , and the NN is trained according to Algorithm~\ref{alg:training_gd}, with learning rate $\eta\leq O(\frac{d}{nL^2m})$. Then with probability at least $1-e^{-\Omega(\ln m)}$
\begin{equation*}
\gL(\mW^{(T)}) < \epsilon
 \text{ and } \
    \gL(\mW^{(t)}) \leq \gamma^t \gL(\mW^{(0)})
        \text{,}\ \forall t  \leq T,
\end{equation*}
where 
\begin{equation*}
\gamma = 1 - \Omega\left(\frac{(1-\alpha)^2}{1+\alpha^2}\frac{\eta m}{d^2}\right) \text{ and } 
T = \frac{\ln \left(\epsilon / \gL(\mW^{(0)})\right)}{\ln\gamma}.
\end{equation*}
\end{theorem}

\begin{theorem}
\label{append:thm:generalization_special_case}
Assume the setup of \S\ref{sec:problem_setup} with  GD, a dataset satisfying Assumption~\ref{assump:special_dataset}, $m = \Theta(n^{1+2\tau}L^{12 + 2\tau}d^{5+2\tau})$ for $\tau>0$ 
and $\eta = \Theta(\frac{d}{nL^2m})$. Assume further that $m$ is larger than its lower bound and $\eta$ is smaller than its upper bound in Theorem~\ref{append:thm:convergence_special_case} (by an appropriate choice of the hidden constants in $\Theta$ and compared to the constants hidden in the lower bound of $m$ and in the upper bound of $\eta$ in Theorem~\ref{append:thm:convergence_special_case}). 
Then at a given training epoch $t \leq T$ (see \eqref{eq:gamma_1st_theorem} for $T$), with probability at least $1-e^{-\Omega(\ln m)}$, the generalization error is bounded as follows 
\begin{equation*}
\begin{split}
             R(\mW^{(t)}) &\leq \gamma^t \gL(\mW^{(0)}) 
+\min\left\{O\left(\frac{d^{3/2+\tau} n^{1/2+\tau} L^\tau}{\ln m}\right), 
O\left(\frac{1-\alpha}{\sqrt{1+\alpha^2}}\frac{d^{1/3}t^{4/3}}{m^{1/6}n^{2/3}L^{2/3}}\right)\right\} 
\\
& \ + \min\left\{O\left(\frac{\sqrt{d \ln m}~t }{nL}\right), O\left(\frac{n^{\tau} L^{1+\tau} d^{2+\tau}}{\ln m}\right)\right\} + O\left(d\sqrt{\frac{\ln m}{n}}\right).
\end{split}
\end{equation*}
\end{theorem}

We notice that for datasets  satisfying Assumption~\ref{assump:special_dataset} several significant improvements from the previous estimates are obtained. 
Firstly, the lower bound for $m$ is improved to linear dependence on $n$, whereas in the general scenario the lower bound grows as $n^5$. Secondly, 
the bound of $1-\gamma$ is improved in Theorem~\ref{append:thm:convergence_special_case} by a factor of $\frac{n}{\delta d}$. 
Thirdly, several terms in the generalization error bound in Theorem~
\ref{append:thm:generalization_special_case} are improved from Theorem~\ref{thm:generalization_in_training_gd}, including the first term in the first minimum is improved by a factor of $\frac{1}{L}$ and the second term in the second minimum is improved by a factor of $\frac{\sqrt{d^3\delta}}{L\sqrt{n}}$.
On the other hand, the optimal choice of $\alpha$ remains the same as the dependence on $\alpha$ is the same as that in Theorems~\ref{thm:gd_main_thm} and \ref{thm:generalization_in_training_gd}.
\subsection{Convergence Theorem for General Convex Loss Functions}
\label{subsec:append:exponential_loss}
We extend our convergence theory, in particular Theorem~\ref{thm:gd_main_thm}, to convex loss functions, i.e., loss functions of the form 
\begin{equation}
    \gL_\text{convex}(\mW) = \sum_{i}^n l(\vy_i, \hat{\vy}_i), \ \text{ where } l(\vy_i, \cdot) \text{ is convex}.\label{eqn:convex}
\end{equation}
These include common loss functions for classification, such as the binary cross entropy and categorical cross entropy. Furthermore, it also includes the following loss function suggested in \citep{kumar2023stochastic}: 
\begin{equation}
    \gL_{\exp}(\mW) := \frac{1}{2}\sum_{i}^ne^{\lambda \|\vy_i - \hat{\vy}_i\|^2}.\label{eqn:exp_loss},
\end{equation}
\citet{kumar2023stochastic} obtained a special bound for the generalization error when using this loss function. We later use the following theorem and the proposition of \cite{kumar2023stochastic} to infer that $\alpha=-1$ is also optimal for generalization when using re-weighted gradient descent and overparameterized neural networks.  

We next formulate the main theorem using the following definition. Let $\mW^\ast$ denote the matrix of parameters minimizing the loss function and define 
$$\gE^{(t)} := \gL_{\text{convex}}(\mW^{(t)}) - \gL_{\text{convex}}(\mW^\ast).$$

\begin{theorem}\label{thm:exp_loss:convergence}
    Assume the setup of \S\ref{sec:problem_setup} with the convex loss function defined in \eqref{eqn:convex}, where the width $m$ satisfies both $m/\ln^4 m \geq \frac{1+\alpha^2}{(1-\alpha)^2}\Omega(\frac{n^6 L^{16} d}{\delta^4})$  and   $ m > \Omega(\ln (\epsilon^{-1}\ln \epsilon^{-1}))$ and the training is according to Algorithm~\ref{alg:training_gd} with learning rate $\eta\leq O(\frac{d}{nL^2 m})$.  Then with probability at least $1-\exp^{-\Omega(\ln m)}$,
    \begin{equation}
        \gE^{(T)} < \eps  \text{ and } \ 
        \gE^{(t+1)} \leq \gamma^{(t)}\gE^{(t)},
       \ \ \forall t \leq T, \ \text{ where}
        \label{eqn:exp_loss:conv}
    \end{equation}
   where
    \begin{equation}
\label{eqn:exp_loss:rate_T}
       \gamma^{(t)} = 1 - \gE^{(t)}\Omega\left(\frac{(1-\alpha)^2}{(1+\alpha^2)}\frac{\delta \eta m}{n^2 d L}\right) \ \text{ and } \ 
       T \leq O\left( \frac{1+\alpha^2}{(1-\alpha)^2} \frac{n^2 d L}{\eta \delta m \epsilon}(\ln\epsilon^{-1} + \ln(n\sqrt{\ln m}))\right).
    \end{equation}
\end{theorem}

Combining \eqref{eqn:exp_loss:conv} and the expression for $\gamma^{(t)}$ in \eqref{eqn:exp_loss:rate_T}, we note that the rate of convergence is slower than the one in \eqref{eqn:gd_conv} and that $\alpha=-1$ corresponds to the smallest upper bound for the number of epochs needed for the training error to be smaller than $\epsilon$.

Proposition~3.1 in \citet{kumar2023stochastic} implies that minimizing the generalization error bound is equivalent to minimizing the training error when using the modified loss function given in \eqref{eqn:exp_loss}. Applying \eqref{eqn:exp_loss:rate_T} of  Theorem~\ref{thm:exp_loss:convergence}, we conclude that when using the modified loss function in \eqref{eqn:exp_loss} for training, the choice of $\alpha=-1$ yields the smallest bound for the required number of training epochs to achieve a bound $\epsilon$ on the training error. Using this observation with arbitrarily small $\epsilon$ and the above discussed theory of \citet{kumar2023stochastic}
we can conclude that $\alpha=-1$ minimizes the generalization error bound with the smallest upper bound on the number of training epochs for which the training error is guaranteed to be less than $\epsilon$. Nevertheless, this discussion involves an upper bound we obtained for the number of epochs and does not apply to the actual number of epochs. Consequently, the above stated prediction may not be precise, that is, it is possible that at a smaller number of epochs than the bound, one may obtain an error less than $\epsilon$ by $\alpha \neq -1$. 
For a synthetic dataset, we empirically verified the predicted optimal choice of $\alpha=-1$ (see Figure~\ref{fig:exp_loss}).

\begin{proof}[Proof of Theorem~\ref{thm:exp_loss:convergence}]

We consider the modified loss function defined in \eqref{eqn:convex}, and, for simplicity, we let $\gamma=1$ in this section, while the proof can be easily extended for any $\gamma>0$. By introducing the convex function $l(\vy, \vz)$,  the loss function for each data point $i\in[n]$ can be written as
$$
\text{loss}_{\text{convex}}(\vx_i, \vy_i; \mW) := l(\vy_i, \vg_{i, L+1}(\vx_i; \mW)) 
$$
We denote the following notation in this subsection,
\begin{equation}
    \ve_i := \nabla_{\vz} l(\vy_i, \vg_{i, L+1}(\vx_i; \mW)) 
    .     \label{eqn:append:exponential_ei}
\end{equation}

We will establish similar gradient bounds as in Lemma~\ref{thm:gradient_bound} and a similar semi-smoothness inequality as in Lemma~\ref{thm:semi-smooth} for the loss function defined in \eqref{eqn:convex}, and then we will prove the convergence theory with this loss function using some of the above established results.

For the first part, in order to achieve the gradient bound, we follow the proof of Lemma~\ref{thm:gradient_bound} with the modified loss function defined in \eqref{eqn:convex}.  The proof is mostly the same as the one in Appendix~\ref{appd:gradient}. The only difference is the use of the previous definition in \eqref{eqn:append:def_mG_func}. We remark that  it is straightforward to verify that $\mG_{i, l}(\ve_i; \mW) \equiv \nabla_{\mW_l} \text{loss}_{\text{convex}}(\vx_i, \vy_i; \mW)$ by using the definition of $\ve_i$ in \eqref{eqn:append:exponential_ei}. By Lemma~\ref{lemma:gradient_upper_bound_init} and Lemma~\ref{lemma:gradient_lower_bound_init}
\begin{align}
    \|\nabla_{\mW_l} \gL_{\text{convex}}(\mW)\|_F^2 & \leq \sum_{i=1}^n\|\ve_i\|^2 O\left(\frac{mn}{d}\right) , \quad \text{ for }\ l \in [L]\label{eqn:append:exp_loss:upper_bnd}\\
    \|\nabla_{\mW} \gL_{\text{convex}}(\mW)\|_F^2 & \geq \sum_{i=1}^n\|\ve_i\|^2\Omega\left(\frac{(1-\alpha)^2}{(1+\alpha^2)}\frac{\delta m}{n d}\right).
    \label{eqn:append:exp_loss:lower_bnd}
\end{align}
\textit{Remark: In the above bounds, when we use the MSE loss function, $\sum_{i=1}^n\|\ve_i\|^2 \equiv \gL$, which yields the original bounds provided in Lemma~\ref{thm:gradient_bound}.}

In order to show the semi-smoothness, we follow the proof of Lemma~\ref{thm:semi-smooth} in Appendix~\ref{appd:semismooth}, where most parts are exactly the same. By using \eqref{eqn:append:exp_loss:lower_bnd} and \eqref{eqn:append:exp_loss:upper_bnd}, and plugging the notation of $\ve_i$ (defined in \eqref{eqn:append:exponential_ei}) into \eqref{eqn:append:loss_expand_t1}, the semi-smoothness inequality becomes
\begin{equation}
\begin{split}
    \gL_{\text{convex}}(\mW + \mW') & \leq \gL_{\text{convex}}(\mW) + \langle\nabla_\mW \gL_{\text{convex}}(\mW), \mW'\rangle + \frac{nL^2m}{d} O(\|\mW'\|_2^2) \\ & \ + \frac{(1-\alpha)\omega^{1/3}L^2\sqrt{m n \sum_{i=1}^n \|\ve_i^{(t)}\|^2 \ln m}}{\sqrt{d(1+\alpha^2)}}O(\|\mW'\|_2). 
\end{split}
    \label{eqn:append:dd}
\end{equation}

Lastly, we prove the convergence for this loss function. By using \eqref{eqn:append:dd} and the same argument discussed in \S\ref{sec:proof_sketch}, the inequality \eqref{eqn:training_error_dynamic} still holds. Then using the lower bound of the gradient in \eqref{eqn:append:exp_loss:lower_bnd}, \eqref{eqn:training_error_dynamic} becomes
\begin{equation}
    \gL_{\text{convex}}(\mW^{(t+1)}) \leq \gL_{\text{convex}}(\mW^{(t)}) -  \Omega\left(\frac{(1-\alpha)^2}{(1+\alpha^2)}\frac{\delta \eta m}{n d}\right) \sum_{i=1}^n\|\ve_i^{(t)}\|^2.\label{eqn:append:exp_loss:conv}
\end{equation}

By convexity of $l(\vy, \vz)$, we first establish that for any $i\in[n]$ and any $\vy, \vz \in \sR^d$,  
\begin{equation*}
    l(\vy_i, \vy) - \vy_{\vy_i, \vz} \leq \langle \nabla_\vy l(\vy_i, \vy), \vy - \vz \rangle \leq \|\nabla_\vy l(\vy_i, \vy) \| \|\vy - \vz\|.
\end{equation*}
We denote by $\mW^\ast$ the optimal parameter that minimizes $\gL_{\exp}(\mW)$. Letting $\vy := \vg_{i, L+1} (\vx_i ; \mW^{(t)})$ and $\vz := \vg_{L+1}(\vx_i; \mW^\ast)$ and using the above inequality result in
\begin{equation}
   \|\ve_i^{(t)} \| \geq \frac{l(\vy_i, \vg_{i, L+1} (\vx_i ; \mW^{(t)})) - l(\vy_i, \vg_{L+1}(\vx_i; \mW^\ast))}{\|\vg_{i, L+1} (\vx_i ; \mW^{(t)}) - \vg_{L+1}(\vx_i; \mW^\ast)\|}. \label{eqn:append:ve_convexity}
\end{equation}
Using Lemma~\ref{lemma:intermediate_perturbation}, yields that, with probability at least $1-e^{-\Omega(\ln m)}$, 
\begin{align}
    & \|\vg_{i, L+1}(\vx_i;\mW^\ast) - \vg_{i, L+1}(\vx_i; \mW^{(t)})\| \leq \|\vg_{i, L+1}(\vx_i;\mW^\ast) - \vg_{i, L+1}(\vx_i; \mW^{(0)})\|\notag\\
    & \quad + \|\vg_{i, L+1}(\vx_i;\mW^{(t)}) - \vg_{i, L+1}(\vx_i; \mW^{(0)})\|
    \notag\\
    & \leq (\|\mW^{(0)} - \mW^{\ast}\|+ \omega) \|\nabla_{\mW} \vg_{i, L+1}(\vx_i;\mW^{(0)})\| \leq O(\sqrt{L}) \label{eqn:append:M_gap_g}.    
\end{align}
Applying \eqref{eqn:append:M_gap_g} to \eqref{eqn:append:ve_convexity} results in \begin{equation}\|\ve_i^{(t)}\| \geq \left(l(\vy_i, \vg_{i, L+1} (\vx_i ; \mW^{(t)})) - l(\vy_i, \vg_{L+1}(\vx_i; \mW^\ast))\right)/O(\sqrt{L}).\label{eqn:appd:bound_ve}
\end{equation}
Applying \eqref{eqn:appd:bound_ve}  to \eqref{eqn:append:exp_loss:conv}, we derive that
\begin{equation}
\begin{split}
    & \gL_{\text{convex}}(\mW^{(t+1)}) - \gL_{\text{convex}}(\mW^\ast) \\ 
    & \leq \gL_{\text{convex}}(\mW^{(t)}) - \gL_{\text{convex}}(\mW^\ast) - \Omega\left(\frac{(1-\alpha)^2}{(1+\alpha^2)}\frac{\delta \eta m}{n^2 d L}\right)  \left(\gL_{\text{convex}}(\mW^{(t)}) - \gL_{\text{convex}}(\mW^\ast)\right)^2.
    \end{split}
\label{eqn:append:exp_loss:convergence}
\end{equation}

In order to derive a bound for the number of training epoch $T$ that is required for $\gL^{(T)} -\gL^\ast < \epsilon$, for $t<T$, by assuming $\gL^{(t)}-\gL^\ast > \epsilon$, the above equation is bounded by
$$
    \gL_{\text{convex}}(\mW^{(t+1)}) - \gL_{\text{convex}}(\mW^\ast) \leq \left(1 - \epsilon\Omega\left(\frac{(1-\alpha)^2}{(1+\alpha^2)}\frac{\delta \eta m}{n^2 d L}\right)\right)\left(\gL_{\text{convex}}(\mW^{(t)}) - \gL_{\text{convex}}(\mW^\ast)\right). 
$$
The lower bound for $m$ becomes $m\ln^4m> \frac{1+\alpha^2}{(1-\alpha)^2} \Omega(n^6L^{16} d /\delta^4)$ to ensure the same perturbation bound in Lemma~\ref{lemma:training_in_perturbation}. 
Denoting $\gamma:= \left(1 - \epsilon\Omega\left(\frac{(1-\alpha)^2}{(1+\alpha^2)}\frac{\delta \eta m}{n^2 d L^2}\right)\right)$, it follows that
$$
T = \frac{\ln \epsilon^{-1} + \ln (\gL_{\text{convex}}(\mW^{(0)}) - \gL_{\text{convex}}(\mW^\ast))}{\ln \gamma^{-1}}\leq O\left(\frac{n^2 d L}{\eta \delta m \epsilon}(\ln\epsilon^{-1} + \ln(n\sqrt{\ln m}))\right).
$$
We follow the exact same steps in the proof of Lemma~\ref{lemma:training_in_perturbation} in Appendix~\ref{appd:main_proof_gd} and verify that when $m>\ln(\epsilon^{-1}\ln\epsilon^{-1})$, the probability that \eqref{eqn:append:exp_loss:convergence} holds for $T-$steps is at least $1 - e^{-\Omega(\ln m)}$.
\end{proof}

\subsection{Proofs for a special class of datasets}
\label{appendix:proof_special_dataset}
This section includes the proof of Theorems \ref{append:thm:convergence_special_case} and \ref{append:thm:generalization_special_case} in Appendix~\ref{subsec:appendix:special_func_class}. We first present several lemmas and their proofs, then use these lemmas to prove those theorems.

\begin{lemma}
\label{lemma:F_l:inner_prod}
Consider a dataset $\{\vx_i, \vy_i\}_{i\in[n]}$  satisfying  Assumption~\ref{assump:special_dataset}, where $m\geq \Omega(nd)$, then there exists a vector $\vu_{l, j}\in\gB^{m}_1\subset\sR^m$, such that
$$
|\langle \vu_{l, j}, \vh_{i, L-1}^{(0)}\rangle - (y_{i, j} - \hat{y}_{i, j})| \leq O(\lambda) , \ \text{ for all } \ i \in [n], j\in [d], \ \text{ with probability at least } \ 1 - e^{-\Omega(\lambda^2 m)}.
$$
\end{lemma}

\begin{proof}
 We complete the proof by constructing the following unit vector $\vu_{j}$: 
$$\vu_{j} = \frac{1}{\sqrt{2m}} \left(c_j(\sqrt{m/2}(\mW^{(0)}_{L-1})_{1, \cdot}), c_j(\sqrt{m/2}(\mW^{(0)}_{L-1})_{2, \cdot}), \cdots ,c_j(\sqrt{m/2}(\mW^{(0)}_{L-1})_{m, \cdot} )\right)^T.
 $$

One can easily verify that $\vu_{j}\in \gB_1^m$ by using the fact that $|c_j| < 1$, which is guaranteed by the definition of the function class in \eqref{eqn:def:Fl}.

The inner product of $\vu_{j}$ and $\vh^{(0)}_{i, L-1}$ is given by
\begin{align}
    \langle \vu_{j},\vh_{i, L-1}^{(0)}\rangle &= \frac{1}{\sqrt{2m}}\sum_{k=1}^m c_j(\sqrt{m/2}(\mW_{L-1}^{(0)})_{k, \cdot}) \gN_{L-1}(\vx; \mA, \mW_{1}^{(0)}, \cdots\mW_{L-2}^{(0)}, \mW_{L-1}^{(0)})\notag \\
    & = \frac{1}{\sqrt{2m}}\sum_{k=1}^m c_j(\sqrt{m/2}(\mW_{L-1}^{(0)})_{k, \cdot}) \sqrt{\frac{2}{m}}\gN_{L-1}(\vx; \mA, \mW_{1}^{(0)}, \cdots\mW_{L-2}^{(0)}, \sqrt{m/2}\mW_{L-1}^{(0)})\notag\\
    & = \frac{1}{m}\sum_{k=1}^m c_j(\sqrt{m/2}(\mW_{L-1}^{(0)})_{k, \cdot}) \gN_{L-1}(\vx; \mA, \mW_{1}^{(0)}, \cdots\mW_{L-2}^{(0)}, \sqrt{m/2}\mW_{L-1}^{(0)})
    \notag
\end{align}
For simplicity, we denote that $Z_{k, i, j}:=c_j(\sqrt{m/2}(\mW_{L-1}^{(0)})_{k, \cdot}) \gN_{L-1}(\vx; \mA, \mW_{1}^{(0)}, \cdots\mW_{L-2}^{(0)}, \sqrt{m/2}\mW_{L-1}^{(0)})$, and above equation becomes
\begin{equation}
    \langle \vu_{j},\vh_{i, L-1}^{(0)}\rangle =  \frac{1}{m}\sum_{k=1}^m Z_{k, i, j}.\label{eqn:appd:inner_prod}
\end{equation}
Noting that $\sqrt{m/2}(\mW_{L-1}^{(0)})_{k, \cdot}\sim N(0, 1)$, by using \eqref{eqn:def:Fl}, it implies that $\sE_{\mW^{(0)}} Z_{k, i, j} ((\mW_{L-1})_{k, \cdot}^{(0)}) = f_j(\vx_i)$. 

Since $|c_j(\cdot)| < 1$ is bounded, $Z_{k, i, j}$ is a sub-Gaussian random variable, therefore we can apply Hoeffding's inequality and conclude  that
\begin{equation}
    \left| \frac{1}{m}\sum_{k=1}^m Z_{k, i, j} - f_j(\vx_i) \right| \leq \lambda, \ \text{ with probability at least } \  1-e^{-\Omega(\lambda^2 m)}. \label{eqn:appd:mean_Z_f}
\end{equation}
Using \eqref{eqn:appd:inner_prod}, \eqref{eqn:appd:mean_Z_f} and  Assumption~\ref{assump:special_dataset}, it follows that with probability at least $1- (nd)e^{-\Omega(\lambda ^2 m)}$
\begin{align*}
    |\langle \vu_{l, j}, \vh_{i, L-1}^{(0)}\rangle - (y_{i, j} - \hat{y}_{i, j})|  & \leq       |\langle \vu_{l, j}, \vh_{i, L-1}^{(0)}\rangle - f_j(\vx_i)| + |f_j(\vx_i) - (y_{i, j} - \hat{y}_{i, j})| \leq 2\lambda. \ \text{ for all }\ i\in[n],\ j\in[d].
\end{align*}
We conclude the Lemma by noting that the probability is at least $1-e^{-\Omega(\lambda^2 m)}$ when $m \geq \Omega(nd)$. 
\end{proof}

\begin{lemma}
\label{lemma:F_k:gradient_lower_bound}
Under Assumption~\ref{assump:special_dataset}, when $m\geq \Omega(nd)$, the lower bound for the gradient of the loss function becomes
$$
\|\nabla_\mW \gL (\mW^{(0)}) \|^2_F \geq \Omega\left(\frac{(1-\alpha)^2}{1+\alpha^2}\frac{m}{d^2}\right)\gL(\mW^{(0)}), \ \text{ with probability at least } \ 1-e^{-\Omega(m)}.
$$
\end{lemma}
\begin{proof}
    We first note that by definition, $\|\nabla_\mW \gL\|_F\geq \|\nabla_{\mW_L} \gL\|$. Then by definition of matrix $F-$norm and \eqref{eqn:append:def_mG_func}, it follows that $\|\nabla_{\mW_L} \gL (\mW^{(0)}) \|^2_F = \sum_{k=1}^m \left\| \sum_{i=1}^n (\mG_{i,L}(\ve_i;\mW^{(0)}))_{k,\cdot} \right\|^2_2$. We write that the $k-$th row of the matrix $\sum_{i=1}^n\mG_{i, L}(\ve_i, \mW^{(0)})$ by
    \begin{equation}
        \left\| \sum_{i=1}^n (G_{i,L}(\ve_i;\mW^{(0)}))_{k,\cdot} \right\|^2_2 = \left\|\sum_{i=1}^n \mB_{k, \cdot}^T\ve_{i} D_{i, L, kk} \vh_{i, L-1}\right\|^2_2.\label{eqn:appd:G_to_BeDh}
        \end{equation}
    Using the vector $\vu_j\in\gB^m_1$ chosen in Lemma~\ref{lemma:F_l:inner_prod}, and denoting $\vu:= \frac{1}{d}\sum_{j=1}^d \vu_j$, we note that $\|\vu\|\leq 1$. Thus we conclude that 
    \begin{align}
   & \left\|\sum_{i=1}^n \mB_{k, \cdot}^T\ve_{i} D_{i, L, kk} \vh_{i, L-1}\right\|^2_2 \geq  
   \left|\left\langle \sum_{i=1}^n \mB_{k, \cdot}^T\ve_{i} D_{i, L, kk} \vh_{i, L-1}, \vu \right\rangle\right|^2 \notag\\
   & = \left|\sum_{i=1}^n \mB_{k, \cdot}^T\ve_{i} D_{i, L, kk}\left\langle \vh_{i, L-1}, \vu \right\rangle\right|^2 \notag\\ 
   & = \frac{1}{1+\alpha^2}\left|\sum_{i=1}^n \mB_{k, \cdot}^T\ve_{i}  (\alpha + (1-\alpha) 1_{\vh_{i, L, k}>0})\left\langle \vh_{i, L-1}, \vu \right\rangle\right|^2\notag\\
   & = \frac{1}{1+\alpha^2}\left|(1-\alpha)\sum_{i=1}^n \mB_{k, \cdot}^T\ve_{i}  1_{\vh_{i, L, k}>0}\left\langle \vh_{i, L-1}, \vu \right\rangle + \alpha\sum_{i=1}^n \mB_{k, \cdot}^T\ve_{i} \left\langle \vh_{i, L-1}, \vu \right\rangle\right|^2
   \label{eqn:appd:lowerbd}
    \end{align}
By Jensen's inequality, we note that the expectation of \eqref{eqn:appd:lowerbd} becomes
\begin{align}
   & \sE_{\vh_{L-1}} \left|(1-\alpha)\sum_{i=1}^n \mB_{k, \cdot}^T\ve_{i}  1_{\vh_{i, L, k}>0}\left\langle \vh_{i, L-1}, \vu \right\rangle + \alpha\sum_{i=1}^n \mB_{k, \cdot}^T\ve_{i} \left\langle \vh_{i, L-1}, \vu \right\rangle\right|^2  \notag\\ 
   &\geq \left|\sE_{\vh_{L-1}}\left((1-\alpha)\sum_{i=1}^n \mB_{k, \cdot}^T\ve_{i}  1_{\vh_{i, L, k}>0}\left\langle \vh_{i, L-1}, \vu \right\rangle + \alpha\sum_{i=1}^n \mB_{k, \cdot}^T\ve_{i} \left\langle \vh_{i, L-1}, \vu \right\rangle\right)\right|^2.\label{eqn:appd:exp_lowerbd}
\end{align}

Since $1_{\vh_{i, L, k}>0}$ is independent with $\vh_{i, L-1}$, for any integer $N$, the conditional expectation can be given as 
\begin{align}
    &\sE_{\vh_{L-1}} \left(\sum_{i=1}^n \mB_{k, \cdot}^T\ve_{i}  1_{\vh_{i, L, k}>0}\left\langle \vh_{i, L-1}, \vu \right\rangle\Big| \sum 1_{\vh_{i, L, k}>0} = N\right)
    = \frac{N}{n} \sE_{\vh_{L-1}}\left(\sum_{i=1}^n\mB_{k, \cdot}^T\ve_{i}  \left\langle \vh_{i, L-1}, \vu \right\rangle\right)\label{eqn:appd:cond_on_sumB}
\end{align}
Moreover, since $1_{\vh_{i, L, k}>0}$ a Bernoulli random variable $B(0.5)$, when $n > 100$, using an approximation of the probability by the central limit theorem, we know that 
\begin{align}
    \sum_i 1_{\vh_{i, L, k}>0} & > n/2 + \sqrt{n}, \ \text{ with probability at least } \ 0.1, \label{eqn:appd:sum_B_larger}\\
        \sum_i 1_{\vh_{i, L, k}>0} & < n/2 - \sqrt{n}, \ \text{ with probability at least } \ 0.1.\label{eqn:appd:sum_B_smaller}
\end{align}

To find a lower bound for \eqref{eqn:appd:exp_lowerbd}, we consider two cases for the second term, when $|\sE_{\vh_{L-1}}\alpha\sum_{i=1}^n \mB_{k,\cdot}^T \ve_{i}\langle\vh_{i, L-1, \vu}\rangle| > \frac{n}{2}|\sE_{\vh_{L-1}}(\mB_{k, \cdot}^T\ve_{i}  \left\langle \vh_{i, L-1}, \vu \right\rangle)|$, then by \eqref{eqn:appd:sum_B_smaller} and \eqref{eqn:appd:cond_on_sumB}, we know that with probability at least $0.1$ that
\begin{align}
  &  \left|\sE_{\vh_{L-1}}\left((1-\alpha)\sum_{i=1}^n \mB_{k, \cdot}^T\ve_{i}  1_{\vh_{i, L, k}>0}\left\langle \vh_{i, L-1}, \vu \right\rangle + \alpha\sum_{i=1}^n \mB_{k, \cdot}^T\ve_{i} \left\langle \vh_{i, L-1}, \vu \right\rangle\right)\right|\notag \\
   & \geq \frac{1}{\sqrt{n}} \left|\sE_{\vh_{L-1}}\left(\sum_{i=1}^n (1-\alpha) \mB_{k, \cdot}^T\ve_{i}  \left\langle \vh_{i, L-1}, \vu \right\rangle\right)\right|.\label{eqn:appd:lower_bd_cases_p}
\end{align}
Using similar argument, when $|\sE_{\vh_{L-1}}\alpha\sum_{i=1}^n \mB_{k,\cdot}^T \ve_{i}\langle\vh_{i, L-1, \vu}\rangle| \leq \frac{n}{2}|\sE_{\vh_{L-1}}(\mB_{k, \cdot}^T\ve_{i}  \left\langle \vh_{i, L-1}, \vu \right\rangle)|$, by using \eqref{eqn:appd:sum_B_larger}, it also follows that with probability at least $0.1$ that \eqref{eqn:appd:lower_bd_cases_p} holds. Thus we conclude that
\begin{align}
  &  \left|\sE_{\vh_{L-1}}\left((1-\alpha)\sum_{i=1}^n \mB_{k, \cdot}^T\ve_{i}  1_{\vh_{i, L, k}>0}\left\langle \vh_{i, L-1}, \vu \right\rangle + \alpha\sum_{i=1}^n \mB_{k, \cdot}^T\ve_{i} \left\langle \vh_{i, L-1}, \vu \right\rangle\right)\right|\notag \\
   & \geq \frac{1}{\sqrt{n}} \left|\sE_{\vh_{L-1}}\left(\sum_{i=1}^n (1-\alpha) \mB_{k, \cdot}^T\ve_{i}\left\langle \vh_{i, L-1}, \vu \right\rangle\right)\right|, \ \text{ with probability at least } \ 
 0.1.\label{eqn:appd:lower_bd_cases}
\end{align}
Combining \eqref{eqn:appd:lowerbd}, \eqref{eqn:appd:exp_lowerbd}, and \eqref{eqn:appd:lower_bd_cases}, it follows that
\begin{equation}
    \sE_{\vh_{L-1}} \left\|\sum_{i=1}^n \mB_{k, \cdot}^T \ve_i D_{i, L, kk}\vh_{i, L-1}\right\| \geq \frac{(1-\alpha)^2}{n(1+\alpha^2)}\left|\sE_{\vh_{L-1}} \sum_{i=1}^n \mB_{k, \cdot}^T \ve_i \langle\vh_{i, L-1}, \vu\rangle\right|^2 , \ \text{ with probability at least } \ 
 0.1.\label{eqn:appd:lower_bd_exp}
\end{equation}

By using the lower bound of $|\mB_{\cdot, k}^T\va|$ derived in \eqref{eqn:append:lower_bound_akli}, we obtain that with at least a constant probability $p_0 := 1-\exp(-\Omega(1))$
\begin{equation}
    \left|\sE_{\vh_{L-1}} \sum_{i=1}^n  \mB_{k, \cdot}^T \ve_i \langle\vh_{i, L-1}, \vu\rangle\right|^2 = \left|\mB_{k, \cdot}^T \left(\sE_{\vh_{L-1}}  \sum_{i=1}^n   \ve_i \langle\vh_{i, L-1}, \vu\rangle\right)\right|^2
\geq 
\sum_{j=1}^d\left|\sE_{\vh_{L-1}} \sum_{i=1}^n   \ve_{i,j} \langle\vh_{i, L-1}, \vu\rangle\right|^2.
\label{eqn:appd:conclude_lower_B}
\end{equation}

We thus conclude that by \eqref{eqn:appd:G_to_BeDh}, Hoeffding inequality, \eqref{eqn:appd:lower_bd_exp} and \eqref{eqn:appd:conclude_lower_B}, with probability at least $1-e^{-\Omega(m)}$
\begin{equation}
    \sum_{k=1}^m\left\|\sum_{i=1}^n (\mG_{i, L}(\ve_i ;\mW^{(0)}))_{k, \cdot}\right\| = \sum_{k=1}^m\sE_{\vh_{L-1}} \left\|\sum_{i=1}^n \mB_{k, \cdot}^T \ve_i D_{i, L, kk}\vh_{i, L-1}\right\| \geq \frac{0.1 p_0 m}{2}\frac{(1-\alpha)^2}{n(1+\alpha^2)}\sum_{j=1}^d\left|\sE_{\vh_{L-1}} \sum_{i=1}^n \ve_{i,j} \langle\vh_{i, L-1}, \vu\rangle\right|^2.\label{eqn:appd:lower_final}
\end{equation}

Using \eqref{eqn:appd:lower_final}, Assumption~\ref{assump:special_dataset} and the fact that $\sE \ve_{i, j} \ve_{i, j'} = 0$ if $j\neq j'$, when $\gL > 1$ imply
\begin{equation}
    \begin{split}
    & \sum_{j=1}^d\left|\sE_{\vh_{L-1}} \sum_{i=1}^n \ve_{i,j} \langle\vh_{i, L-1}, \vu\rangle\right|  = \frac{1}{d}\left|\sE\sum_{i=1}^n \ve_{i,j} \left(\langle \vh_{i, L-1}, \vu_j \rangle - y_{i, j} + \hat{y}_{i, j} + y_{i, j} - \hat{y}_{i, j}\right)\right|  \notag\\
    & \geq   \frac{1}{d}\sum_{j=1}^d\left|\sum_{i=1}^n  \ve_{i,j} \left( y_{i, j} - \hat{y}_{i, j}\right)\right| - 
    \frac{1}{d}\sum_{j=1}^d\left|\sum_{i=1}^n  \ve_{i,j} \left(\langle \vh_{i, L-1}, \vu_j \rangle - (y_{i, j}-\hat{y}_{i, j})\right)\right|\notag\\
  &\geq \frac{1}{d}\sum_{j=1}^d\left|\sum_{i=1}^n  \ve_{i,j}^2 \right| - \frac{\lambda}{d} \sum_{j=1}^d\left|\sum_{i=1}^n  \ve_{i,j}\right| \geq \frac{1}{d}\gL - \lambda \frac{\sqrt{n}}{\sqrt{d}}\gL \geq\frac{1}{2d}\gL. 
\end{split}
\label{eqn:appd:lower_bound:inner_prod}
\end{equation}
Applying  \eqref{eqn:appd:lower_final} and \eqref{eqn:appd:lower_bound:inner_prod}, we conclude that 
\begin{equation}
    \sum_{k=1}^m  \left\|\sum_{i=1}^n \mB_{k, \cdot}^T\ve_{i} D_{i, L, kk} \vh_{i, L-1}\right\|^2_2 \geq \frac{m}{nd^2} \frac{(1-\alpha)^2}{1+\alpha^2}\gL^2,
     \ \text{ with probability at least } \ 1-e^{-\Omega(m)}.\label{eqn:appd:conclude_lower_bound}
\end{equation}

Considering at initial parameter $\mW^{(0)}$, 
$$
\gL(\mW^{(0)}) = \sum_{i=1}^n \|\vy_i - \mB \vh_{i, L}\|^2 = \sum_{i=1}^n \|\vy_i\|^2 + \|\mB \vh_{i, L}\|^2 - 2\langle\vy_i, \mB\vh_{i, L}\rangle.
$$
Using the fact that $\sE \langle\vy_i, \mB\vh_{i, L}\rangle=0$ and Lemma~\ref{lemma:initial_vector_norm}, we establish a lower bound for $\gL(\mW^{(0)})$, 
\begin{equation}
    \gL(\mW^{(0)}) \geq \Omega(n).
    \label{eqn:appd:initial_gL_lower_bd}
\end{equation}
Applying the definition of the gradient norms, and \eqref{eqn:appd:conclude_lower_bound} with the lower bound of $\gL(\mW^{(0)})$ in \eqref{eqn:appd:initial_gL_lower_bd}, we conclude that
$$
\|\nabla_\mW\gL(\mW^{(0)})\|_F^2 \geq \sum_{k=1}^m  \left\|\sum_{i=1}^n \mB_{k, \cdot}^T\ve_{i} D_{i, L, kk} \vh_{i, L-1}\right\|^2_2 \geq \Omega\left(\frac{(1-\alpha)^2}{1+\alpha^2} \frac{m}{d^2}\right)\gL(\mW^{(0)}),
     \ \text{ with probability at least } \ 1-e^{-\Omega(m)}.
$$    
\end{proof}

\begin{lemma}\label{lemma:gradient_bound_special_dataset}
    Assume the setup of \S\ref{sec:problem_setup} and the dataset satisfy Assumption~\ref{assump:special_dataset}, when $\|\mW-\mW^{(0)}\| < \omega < O(\frac{1}{d^{3/2} L^6 \ln^{3/2} m })$, with probability at least $1-e^{-\Omega(m)}$, 
    $$
    \|\nabla_\mW \gL(\mW)\| \geq \Omega\left(\frac{(1-\alpha)^2}{1+\alpha^2} \frac{m}{d^2}\right)\gL(\mW).    
    $$
\end{lemma}
\begin{proof}
For $\mW$ such that $\|\mW^{(0)} - \mW\| < \omega$, we use the same argument in the proof of Lemma~\ref{thm:gradient_bound}. It is straightforward to verify that $\omega^{2/3}L^4 m\ln m/d$ (this is the right hand side of \eqref{eqn:append:perturbation_G_F}) is smaller than $O(m/d^2)$ by plugging in $\omega < O(\frac{1}{d^{3/2} L^6 \ln^{3/2} m })$. We conclude that for $\mW$ such that $\|\mW - \mW^{(0)}\| < \omega$, with probability at least $1-e^{-\Omega(m)}$
\begin{equation}
    \|\nabla_\mW \gL(\mW)\| \geq \Omega\left(\frac{(1-\alpha)^2}{1+\alpha^2} \frac{m}{d^2}\right)\gL(\mW).
    \label{eqn:appd:gradient_specialcase_final_lower_bd}
\end{equation}
\end{proof}

Compared to the conclusion in Lemma~\ref{thm:gradient_bound} with Lemma~\ref{lemma:gradient_bound_special_dataset}, the lower bound is improved by a factor of $\frac{n}{\delta d}$ when the dataset satisfies Assumption~\ref{assump:special_dataset}.

\begin{proof}[Proof of Theorem~\ref{append:thm:convergence_special_case}]
    We follow the exact same proof of Theorem~\ref{thm:gd_main_thm}, with the lower bound of the gradient given in Lemma~\ref{lemma:gradient_bound_special_dataset}. It is straight-forward to derive the same inequality \eqref{eqn:training_error_dynamic}, and by the lower bound in \eqref{eqn:appd:gradient_specialcase_final_lower_bd}, we obtain that with probability $1-e^{-\Omega(m)}$
    $$
    \gL^{(t+1)} \leq \left(1 - \Omega(\frac{(1-\alpha)^2}{1+\alpha^2} \frac{\eta m}{d^2})\right)\gL^{(t)}.
    $$

    We conclude the theorem by verifying that during the training process, we always have $|\mW^{(0)} - \mW^{(0)}| < \omega < O(\frac{1}{d^{3/2} L^6 \ln^{3/2} m })$, which satisfies the condition for both Lemmas~\ref{lemma:gradient_bound_special_dataset} and~\ref{thm:semi-smooth}. Following the same argument that derives \eqref{eqn:omega_bound_gd} and using the lower bound in \eqref{eqn:appd:gradient_specialcase_final_lower_bd}, we achieve that
\begin{equation}
    \|\mW^{(t)} - \mW^{(0)}\| \leq \frac{\sqrt{1+\alpha^2}}{1-\alpha}\Omega\left(\frac{d}{\sqrt{m}}\right)\sqrt{\gL^{\mW^{(0)}}}.\label{eqn:special_case:bound_omega1}
\end{equation}
By further applying \eqref{eqn:append:init_loss_upper_bound}, we verify that when $m/\ln^4 m > \frac{1+\alpha^2}{(1-\alpha)^2}\Omega(d^5 n L^{12})$, we can derive that 
\begin{equation}
\|\mW^{(t)} - \mW^{(0)}\| \leq O\left(\frac{1}{d^{3/2} L^6 \ln^{3/2} m }\right).\label{eqn:special_case:bound_omega2}
\end{equation}
\end{proof}

\begin{proof}[Proof of Theorem~\ref{append:thm:generalization_special_case}]
    The universal bound of $\|\mW^{(t)}-\mW^{(0)}\|$ is improved as shown in \eqref{eqn:special_case:bound_omega1} and \eqref{eqn:special_case:bound_omega2}. Then, we can conclude the theorem by following exactly the same as the proof of Theorem~\ref{thm:generalization_in_training_gd} which is shown in Appendix~\ref{appd:generalization_gd}. 
\end{proof}

\section{Supplemental numerical experiments and details for the previous experiments}\label{appd:numeric}

Section \ref{appd:architecture_of_nns} provides the full details of implementation for both the previous and the new experiments. Section 
\ref{appd:additional_numerical_results} describes new numerical experiments.

\subsection{Details of Implementation}\label{appd:architecture_of_nns}

We provide some general implementation details and also details specific to the different datasets. Two datasets are new to this section. For completeness, we repeat some information that was provided in Section \ref{sec:numeric_setup}.

\textbf{General implementation details:} 
Throughout the numerical experiments, we applied Algorithm~\ref{alg:rescale_initial} to initialize the parameters of the neural networks. In order to implement the rescaled leaky ReLU as given in \eqref{eqn:leaky_relu_rescaled}, we introduce a $\textrm{MULTIPLIER}(c)$ layer, which simply does element-wise multiplication with a given constant $c$. 
By combining $\textrm{Leaky ReLU}(\alpha)$ and $\textrm{MULTIPLIER}\left(1/\sqrt{1+\alpha^2}\right)$, we replicate the rescaled Leaky ReLU with parameter $\alpha$.

In the experiments, we train the NN on the training set and report the error on a reserved testing set (we view it as an approximation for the generalization error). For the synthetic dataset, we generated additional $500$ synthetic data points for the testing set. For the real dataset, we performed a standard training-testing split for each dataset.

\textbf{Synthetic dataset:}
The architecture of the NNs that we used for the synthetic dataset is shown in Table~\ref{tab:synthetic_nn}. We generate 1,000 data points as the training dataset and 500 data points as the testing dataset (following the model and sampling procedure described in the main text). 
We train the NN with GD and a learning rate of $10^{-4}$.

\textbf{California housing:}
We use an updated version of the California housing dataset, which can be downloaded from Kaggle  (\url{https://www.kaggle.com/datasets/camnugent/california-housing-prices}) and is licensed by CC0. 
This dataset was drawn from the 1990 U.S. Census and contains 20,640 observations with 10 different characteristics. Nine of them are numerical ones (e.g., the median income for households and the median value of the houses within a block) and are given in the original dataset \cite{KelleyPace1997}. An additional categorical characteristic is the ocean proximity. 
\citet{borisov2021deep} used this dataset as a benchmark for regression, where one needs to predict the value of the house given the other numerical characteristics. The last characteristic, which is the median house value for households within a block, provides labels for the dependent variable. We follow a similar setting of regression, but we also use the categorical feature of the updated dataset.
We standardize the 9 numerical characteristics using the respective means and standard deviations of the training data. We generated a one-hot coding vector for the feature "ocean proximity", including 5 categories, $<$1H OCEAN, INLAND, NEAR OCEAN, NEAR BAY, and ISLAND. In total, the input $\vx$ is a $13-$dimensional vector. The training data contains 15,480 data points and the testing data contains 5,160 data points.

We built NNs to predict the median housing value in the dataset. The architecture of the NNs is given in Table~\ref{tab:CH_nn_ar}. We applied Algorithm~\ref{alg:training_sgd} with a batch size of $512$ and a learning rate of $10^{-5}$ to train the NNs.

\begin{table}[ht]
\caption{Architecture of the NNs with Leaky ReLU parameter $\alpha$ used for the synthetic dataset.}
\label{tab:synthetic_nn}
\vskip 0.15in
\begin{center}
\begin{small}
\begin{sc}
\begin{tabular}{llc}
\toprule
\multicolumn{2}{c}{Layer}                                                                                        & Parameter             \\ \midrule
\multicolumn{1}{l}{}                                                                                & Linear     & (5, 5000)             \\ \midrule
\multicolumn{1}{c}{\multirow{3}{*}{\begin{tabular}[c]{@{}l@{}}Repeat \\ $\ \ 5$ \\ times\end{tabular}}} & Linear     & (5000, 5000)          \\ \cline{2-3} 
\multicolumn{1}{l}{}                                                                                & Leaky Relu & $\alpha$              \\ \cline{2-3} 
\multicolumn{1}{l}{}                                                                                & Multipler  & $1/(1+\alpha^2)^{1/2}$ \\ \midrule
\multicolumn{1}{l}{}                                                                                & Linear     & (5000, 1)                \\ \bottomrule
\end{tabular}
\end{sc}
\end{small}
\end{center}
\vskip -0.1in
\end{table}

\begin{table}[!t]
\caption{Architecture of the NNs with Leaky ReLU parameter $\alpha$ used for California housing.}
\label{tab:CH_nn_ar}
\vskip 0.15in
\begin{center}
\begin{small}
\begin{sc}
\begin{tabular}{llc}
\toprule
\multicolumn{2}{c}{Layer}                                                                                        & Parameter             \\ \midrule
\multicolumn{1}{l}{}                                                                                & Linear     & $(13, 5000)$             \\ \midrule
\multicolumn{1}{c}{\multirow{3}{*}{\begin{tabular}[c]{@{}l@{}}Repeat \\ $\ \ 7$ \\ times\end{tabular}}} & Linear     & $(5000, 5000) $         \\ \cline{2-3} 
\multicolumn{1}{l}{}                                                                                & Leaky Relu & $\alpha$              \\ \cline{2-3} 
\multicolumn{1}{l}{}                                                                                & Multipler  & $1/(1+\alpha^2)^{1/2}$ \\ \midrule
\multicolumn{1}{l}{}                                                                                & Linear     & $(5000, 1)$                \\ \bottomrule
\end{tabular}
\end{sc}
\end{small}
\end{center}
\vskip -0.1in
\end{table}

\textbf{MNIST:} 
We used the MNIST dataset of 28 $\times$ 28 images of handwritten digits in order to classify  handwritten digits. This dataset is licensed by CC BY-SA 3.0. We flattened each image to a vector of length $784$. We randomly sample 2,100 data points from the training set of MNIST as our training set, and use the rest as our testing set. We normalized the training data with 0.5 mean and 0.5 standard deviation. We applied SGD with batch size 64 and a learning rate of $10^{-3}$. The architecture of the NN is presented in Table~\ref{tab:M_nn_ar} and we use leaky ReLUs with $\alpha\in\{-2,-1,0,0.01,0.05\}$. 
MNIST was also used to test the Transformer networks. In this case, we normalized the training set with 0.1307 mean and 0.3081 standard deviation. Furthermore, we used  the Vision Transformer (ViT) \citep{dosovitskiy2020image} architecture and applied SGD with batch size 100 and learning rate $10^{-4}$. The details of this architecture are shown in Table \ref{tab:M_tf_nn_ar}.

\begin{table}[!t]
\caption{Architecture of the neural networks with $\alpha$ Leaky ReLU parameter used for MNIST.}
\label{tab:M_nn_ar}
\vskip 0.15in
\begin{center}
\begin{small}
\begin{sc}
\begin{tabular}{lc}
\toprule
Layer & Parameter \\
\midrule
Linear & (784, 2000) \\
\hline
Linear & (2000, 2000) \\
\hline
Leaky ReLU & $\alpha$ \\
\hline
Multiplier & $\frac{1}{\sqrt{1+\alpha^2}}$ \\
\hline
Linear & (2000, 2000) \\
\hline
Leaky ReLU & $\alpha$ \\
\hline
Multiplier & $\frac{1}{\sqrt{1+\alpha^2}}$ \\
\hline
Linear & (2000, 10) \\
\bottomrule
\end{tabular}
\end{sc}
\end{small}
\end{center}
\vskip -0.1in
\end{table}

\begin{table}[!t]
\caption{Architecture of the Transformer neural networks with $\alpha$ Leaky ReLU parameter used for IMDB movie reviews.}
\label{tab:M_tf_nn_ar}
\vskip 0.15in
\begin{center}
\begin{small}
\begin{sc}
\begin{tabular}{lc}
\toprule
Layer & Parameter \\
\midrule
Positional Embedding & (49, 64) \\
\hline
& head dim = 64, \\
Transformer & output dim = 64,\\
Encoder & number of heads = 8 \\
 & number of layers = 6 \\
 & mlp dim = 8192 \\
\hline
Linear & (64,8192)\\
\hline
Leaky ReLU & $\alpha$ \\
\hline
Multiplier & $\frac{1}{\sqrt{1+\alpha^2}}$ \\
\hline
Linear & (8192, 10) \\
\bottomrule
\end{tabular}
\end{sc}
\end{small}
\end{center}
\vskip -0.1in
\end{table}

\textbf{F-MNIST:} 
We used the F-MNIST dataset of 28 $\times$ 28 images of clothing or accessory items for classification. This dataset is licensed by MIT. We flattened each image to vectors of length $784$. We randomly sample 3000 data points from the training set of F-MNIST as the training set, and use the rest for testing. We normalized the training data with 0.5 mean and 0.5 standard deviation. We applied SGD with batch size 64 and a learning rate of $10^{-5}$. The architecture of the NN is presented in Table \ref{tab:FM_nn_ar} and we use leaky ReLUs with $\alpha\in\{-2,-1,0,0.01,0.05\}$.

\begin{table}[!t]
\caption{Architecture of the neural networks with $\alpha$ Leaky ReLU parameter with width of $m$ and depth of $L$ used for Fashion MNIST.}
\label{tab:FM_nn_ar}
\vskip 0.15in
\vskip 0.15in
\begin{center}
\begin{small}
\begin{sc}
\begin{tabular}{llc}
\toprule
\multicolumn{2}{c}{Layer}                                                                                        & Parameter             \\ \midrule
\multicolumn{1}{l}{}                                                                                & Linear     & $(784, m)$             \\ \midrule
\multicolumn{1}{c}{\multirow{3}{*}{\begin{tabular}[c]{@{}l@{}}Repeat \\ $\ \ L$ \\ times\end{tabular}}} & Linear     & $(m, m) $         \\ \cline{2-3} 
\multicolumn{1}{l}{}                                                                                & Leaky Relu & $\alpha$              \\ \cline{2-3} 
\multicolumn{1}{l}{}                                                                                & Multipler  & $1/(1+\alpha^2)^{1/2}$ \\ \midrule
\multicolumn{1}{l}{}                                                                                & Linear     & $(m, 10)$                \\ \bottomrule
\end{tabular}
\end{sc}
\end{small}
\end{center}
\vskip -0.1in
\vskip -0.1in
\end{table}

\textbf{CIFAR-10:} 
We used the CIFAR-10 dataset of $32 \times 32$ RGB images with 10 categories for classification. This dataset is licensed by MIT. We randomly sample 2560 data points from the training set of CIFAR-10 as our training set, and randomly sample 2560 data points from the testing set of CIFAR-10 as our testing set. We normalized the training data with 0.5 mean and 0.5 standard deviation. Then we applied SGD with batch size 64 and a learning rate of $10^{-6}$. The architecture of the NN is presented in Table \ref{tab:CF10_nn_ar}.

\begin{table}[!t]
\caption{Architecture of the neural networks with $\alpha$ Leaky ReLU parameter used for CIFAR-10.}
\label{tab:CF10_nn_ar}
\vskip 0.15in
\begin{center}
\begin{small}
\begin{sc}
\begin{tabular}{lc}
\toprule
Layer & Parameter \\
\midrule
CNN & conv3-512 \\
\hline
Leaky ReLU & $\alpha$ \\
\hline
Multiplier & $\frac{1}{\sqrt{1+\alpha^2}}$ \\
\hline
CNN & conv3-512 \\
\hline
Leaky ReLU & $\alpha$ \\
\hline
Multiplier & $\frac{1}{\sqrt{1+\alpha^2}}$ \\
\hline
Max Pooling & $2 \times 2$ \\
\hline
CNN & conv3-512 \\
\hline
Leaky ReLU & $\alpha$ \\
\hline
Multiplier & $\frac{1}{\sqrt{1+\alpha^2}}$ \\
\hline
CNN & conv3-512 \\
\hline
Leaky ReLU & $\alpha$ \\
\hline
Multiplier & $\frac{1}{\sqrt{1+\alpha^2}}$ \\
\hline
Max Pooling & $2\times 2$ \\
\hline
Flatten & \\
\hline
Linear & (32,768, 512) \\
\hline
Leaky ReLU & $\alpha$ \\
\hline
Multiplier & $\frac{1}{\sqrt{1+\alpha^2}}$ \\
\hline
Linear & (512, 10) \\
\bottomrule
\end{tabular}
\end{sc}
\end{small}
\end{center}
\vskip -0.1in
\end{table}

\textbf{IMDB movie reviews dataset:} This is a dataset of highly popular movie review paragraphs and it is used for positive or negative sentiment classification \citep{maas-EtAl:2011:ACL-HLT2011(imdb_sentiment_analysis)}. We downloaded it from the following URL: 
\url{http://ai.stanford.edu/~amaas/data/sentiment/}. 
We randomly sample 5,000 data points from the IMDB movie reviews as the training dataset and use the rest as the testing dataset. 
We processed the data as follows. We first recorded the words that appeared at least once in the training dataset. 
For each word, a unique integer was assigned to index it. Then we mapped each paragraph to a vector, whose $i$-th entry is the assigned index of the $i$-th word of the paragraph. Finally, we padded each vector with zeros, so that each vector was of the same length. We used the zero-padded vectors as the input of our neural network.
After preprocessing, we applied SGD with batch size 50 and learning rate $10^{-5}$ with an LSTM network, whose architecture is presented in Table \ref{tab:IMDB_nn_ar}.

\begin{table}[!t]
\caption{Architecture of the neural networks with $\alpha$ Leaky ReLU parameter used for IMDB movie reviews dataset}
\label{tab:IMDB_nn_ar}
\vskip 0.15in
\begin{center}
\begin{small}
\begin{sc}
\begin{tabular}{lc}
\toprule
Layer & Parameter \\
\midrule
Embedding & (1000, 64) \\
\hline
& input dim = 64, \\
LSTM & hidden dim = 256,\\
 & number of layers = 2 \\
\hline
Dropout & 0.3 \\
\hline
Linear & (256,4096)\\
\hline
Leaky ReLU & $\alpha$ \\
\hline
Multiplier & $\frac{1}{\sqrt{1+\alpha^2}}$ \\
\hline
Linear & (4096, 1) \\
\bottomrule
\end{tabular}
\end{sc}
\end{small}
\end{center}
\vskip -0.1in
\end{table}
\subsection{Additional numerical results}
\label{appd:additional_numerical_results}

We describe here two additional experiments.

\textbf{Additional experiments using MNIST and California housing:}
We ran the same experiments done in Section \ref{sec:numeric_setup}, but with MNIST and California housing.
Figure~\ref{fig:errors_ch_mnist} demonstrates the training errors (top) and testing errors (bottom) for the two datasets. For MNIST (left) we used the cross entropy loss and for California housing (right) the MSE loss. For the training errors, we note that the convergence rate is the fastest at $\alpha=-1$ for both datasets and this observation aligns with our theoretical prediction and the previous experiments. 
The testing errors are rather similar for different choices of $\alpha$. 
For the California housing dataset, we note that $\alpha=-1$ achieves the smallest testing error at an early epoch (about $t=20$), but the advantage is marginal compared to other $\alpha$'s. For MNIST, we note that $\alpha=-1$ gets to the same level of testing error as the other $\alpha$s from a much larger initial testing error. We note that $\alpha=-1$ is not optimal for the testing error. This might happen because the number of samples and the depth are not sufficiently large enough.

\textbf{Experiments with Long Short-Term Memory (LSTM) and Transformer networks:} We ran the same experiment done in Section \ref{sec:numeric_exp} on MNIST and IMDB with Transformer and LSTM networks, respectively. These architectures are described in Section \ref{appd:architecture_of_nns}. Figure \ref{fig:errors_lstm_tf} demonstrates the training errors (top) and testing errors (bottom). For MNIST (right) we used the negative log likelihood loss and for IMDB (left) we used the binary cross entropy loss. For both algorithms, the training errors converge fastest with $\alpha = -1$. This observation agrees with our theoretical findings and previous experiments. The testing error for IMDB decreases during the first 100 epochs and then increases for the rest of the training. 
This is because of severe overfitting that is due to the following property of IMDB: the training dataset is small (we used randomly sampled 5,000 data points to be able to deal with sufficiently large widths for overparameterization) compared to the input data dimension (we have 1,000 unique words).
The testing error on MNIST, on the other hand, is also the lowest for $\alpha = -1$, but there's no overfitting phenomenon since MNIST is a simple dataset.

\begin{figure*}[ht]
\vskip -0.04in
\begin{center}
\includegraphics[width=0.4\columnwidth]{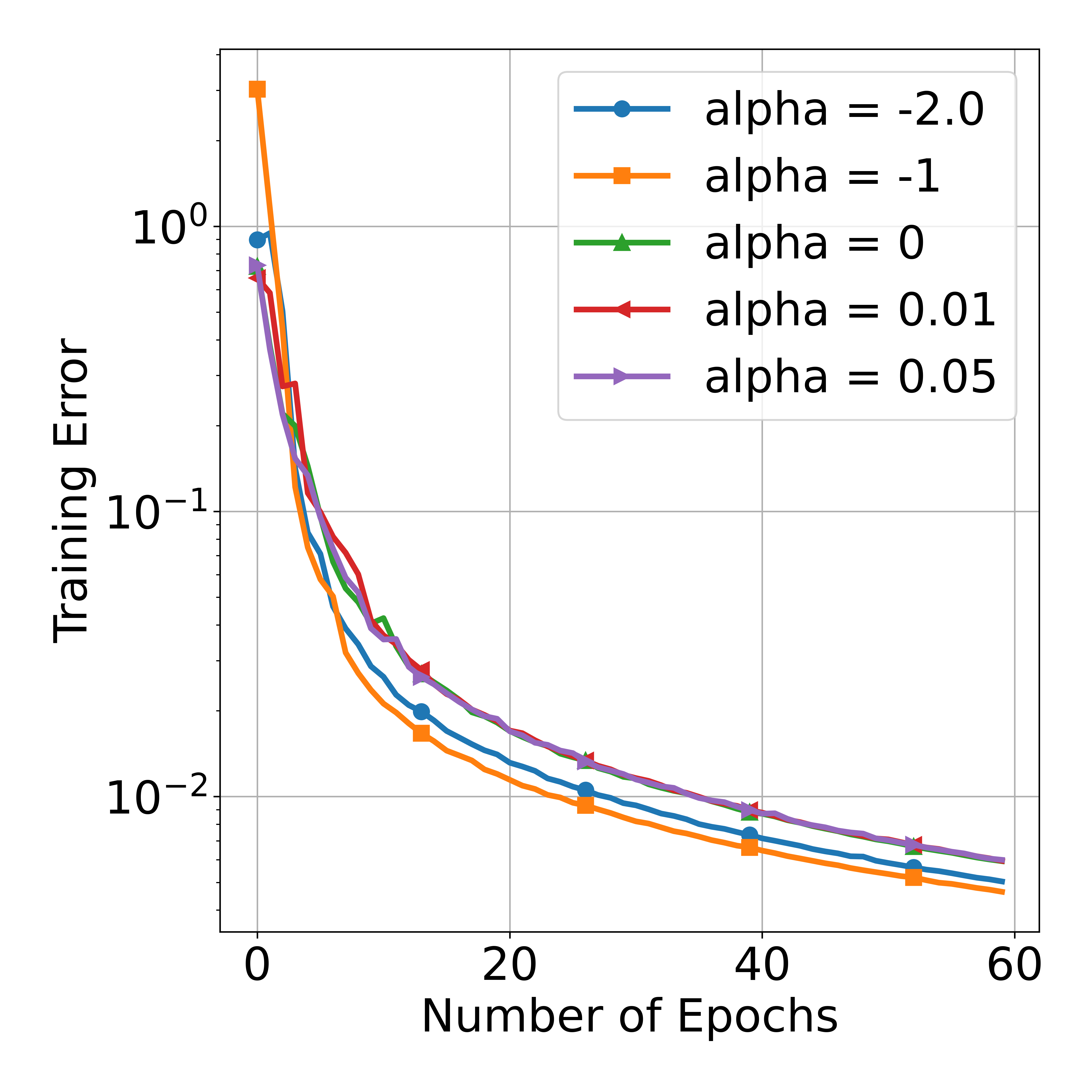}
\includegraphics[width=0.39\columnwidth]{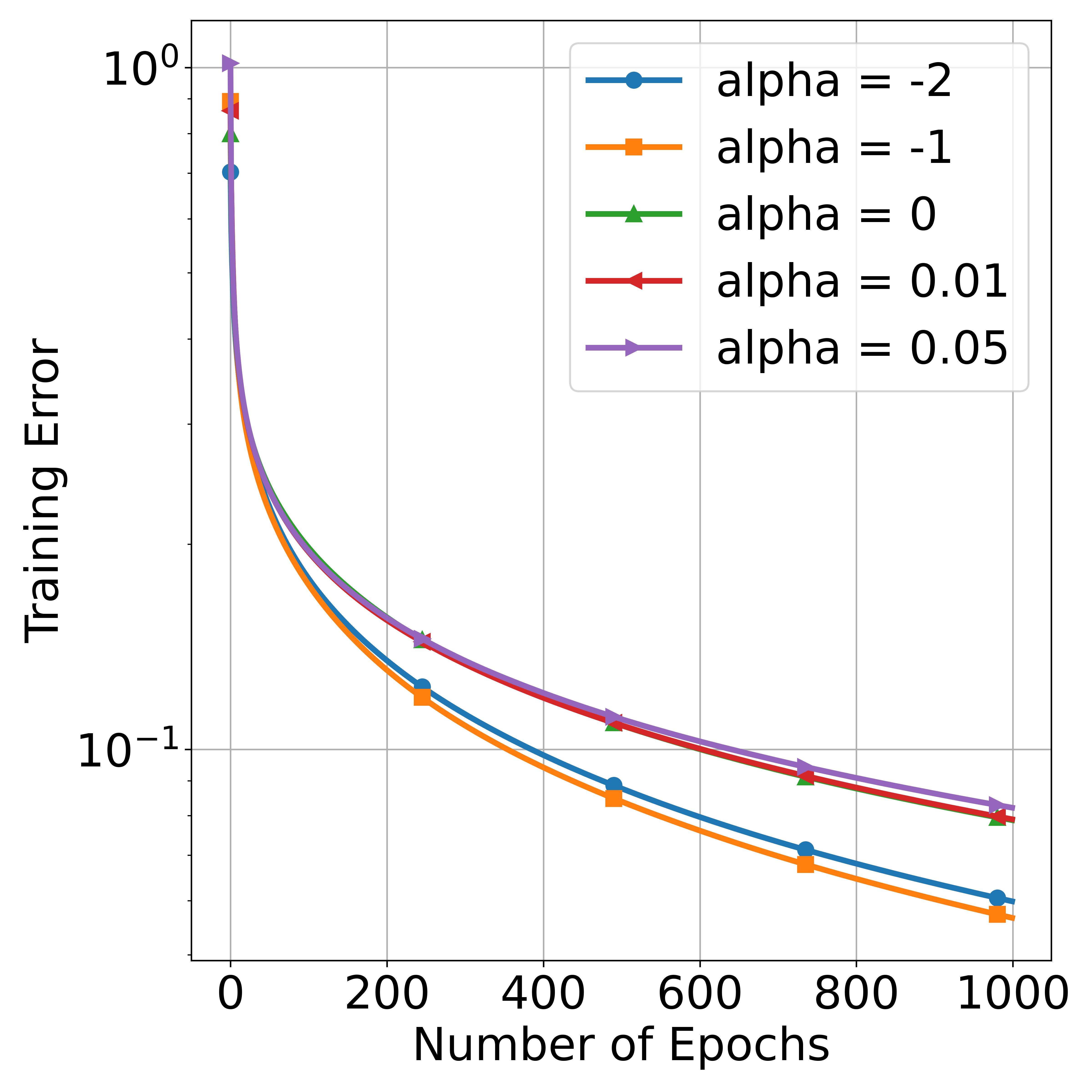}
\includegraphics[width=0.4\columnwidth]{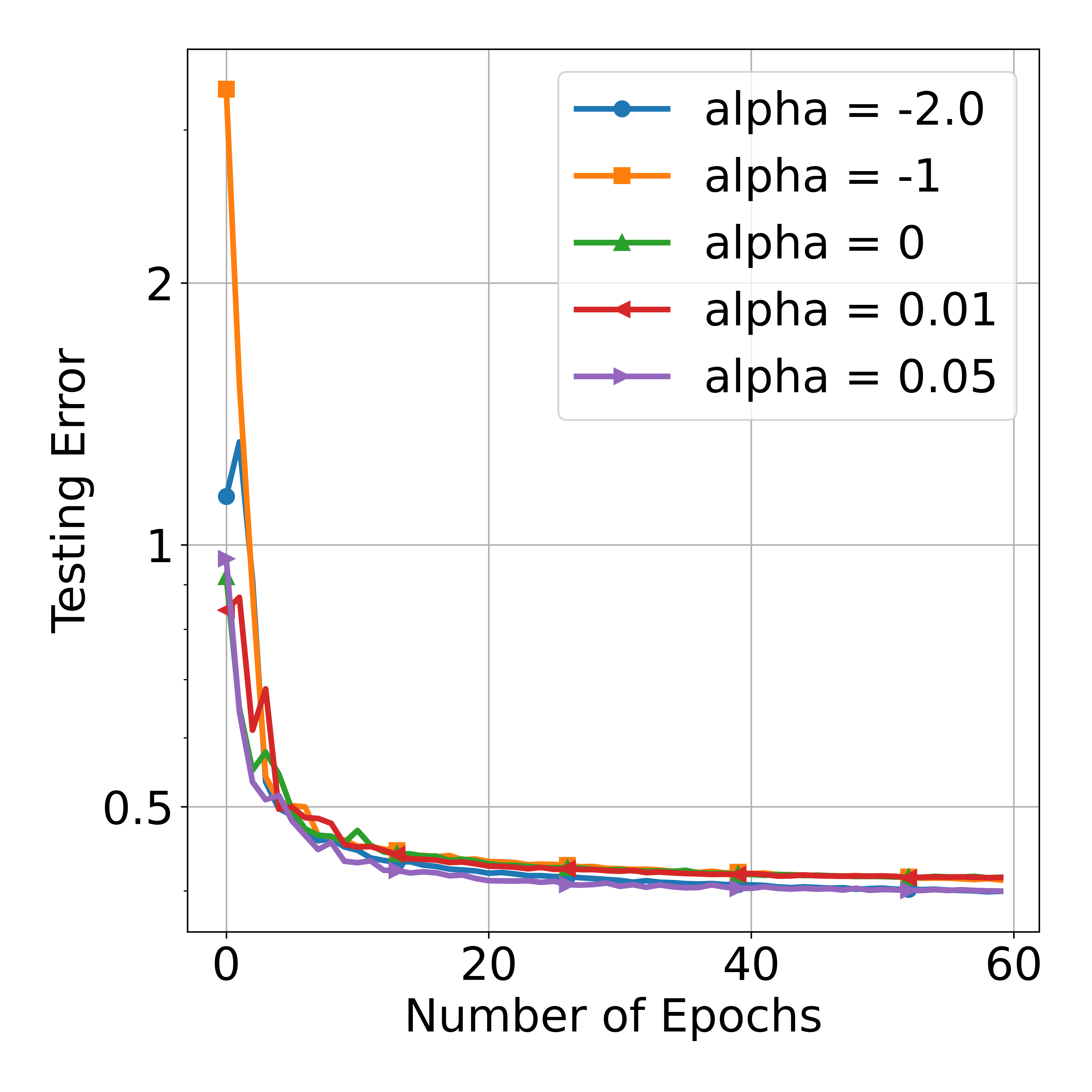}
\includegraphics[width=0.39\columnwidth]{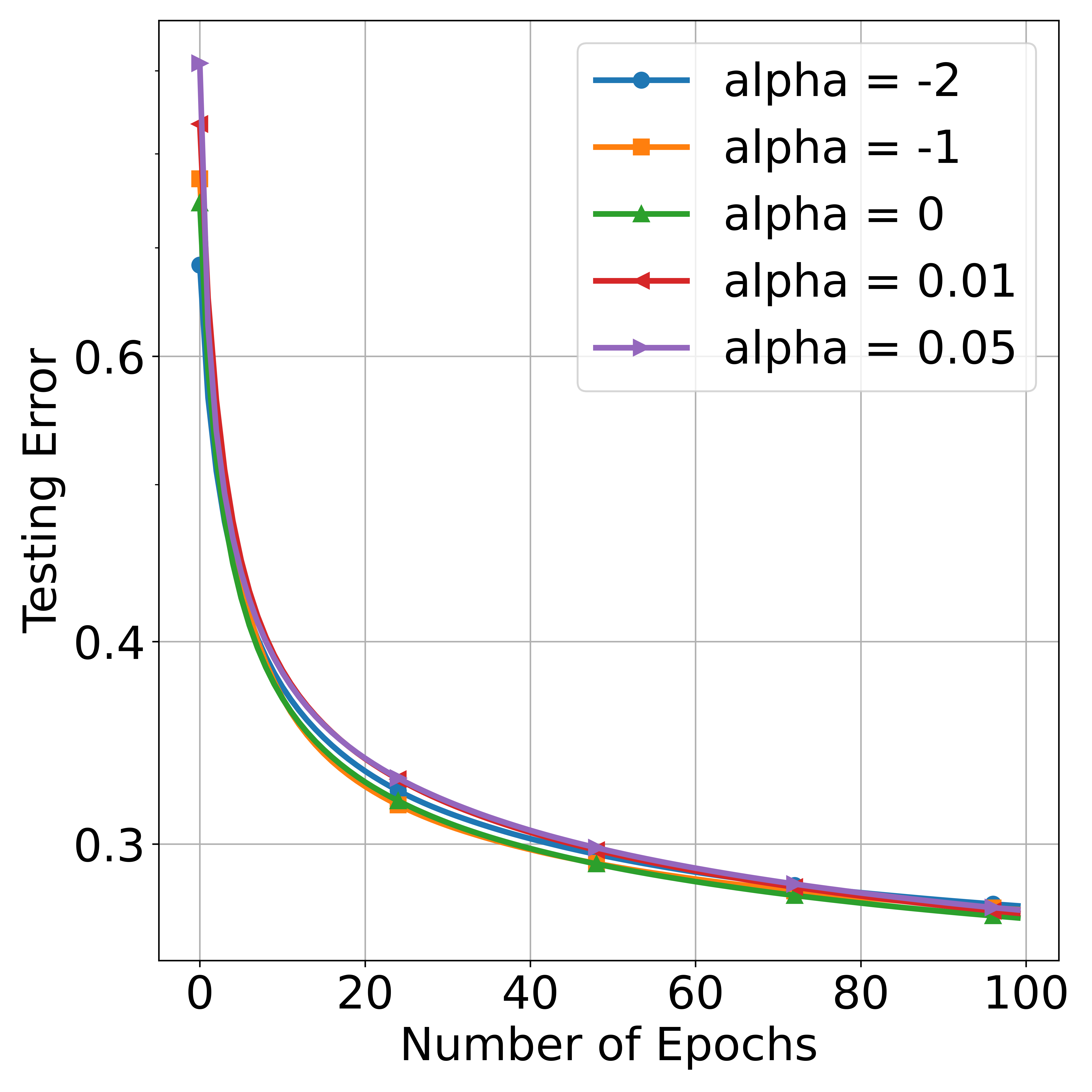}
\caption{Log-scale training and testing errors using different datasets and different $\alpha$'s. Left: cross entropy errors for MNIST; Right: MSE for California housing. Top row: training errors. Bottom row: testing errors.}
\label{fig:errors_ch_mnist}
\end{center}
\vskip -0.2in
\end{figure*}

\begin{figure*}[ht]
\vskip -0.04in
\begin{center}
\includegraphics[width=0.4\columnwidth]{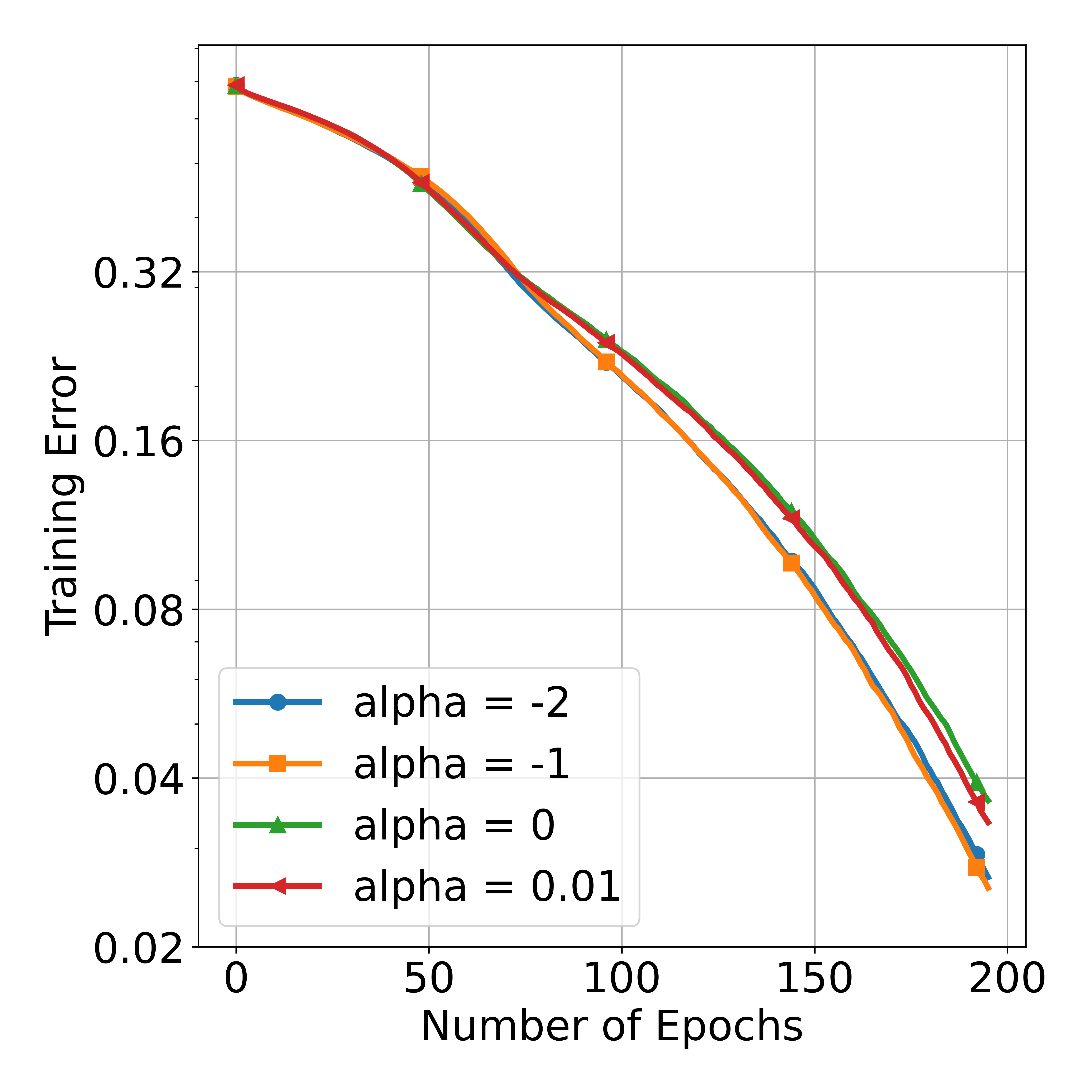}
\includegraphics[width=0.4\columnwidth]{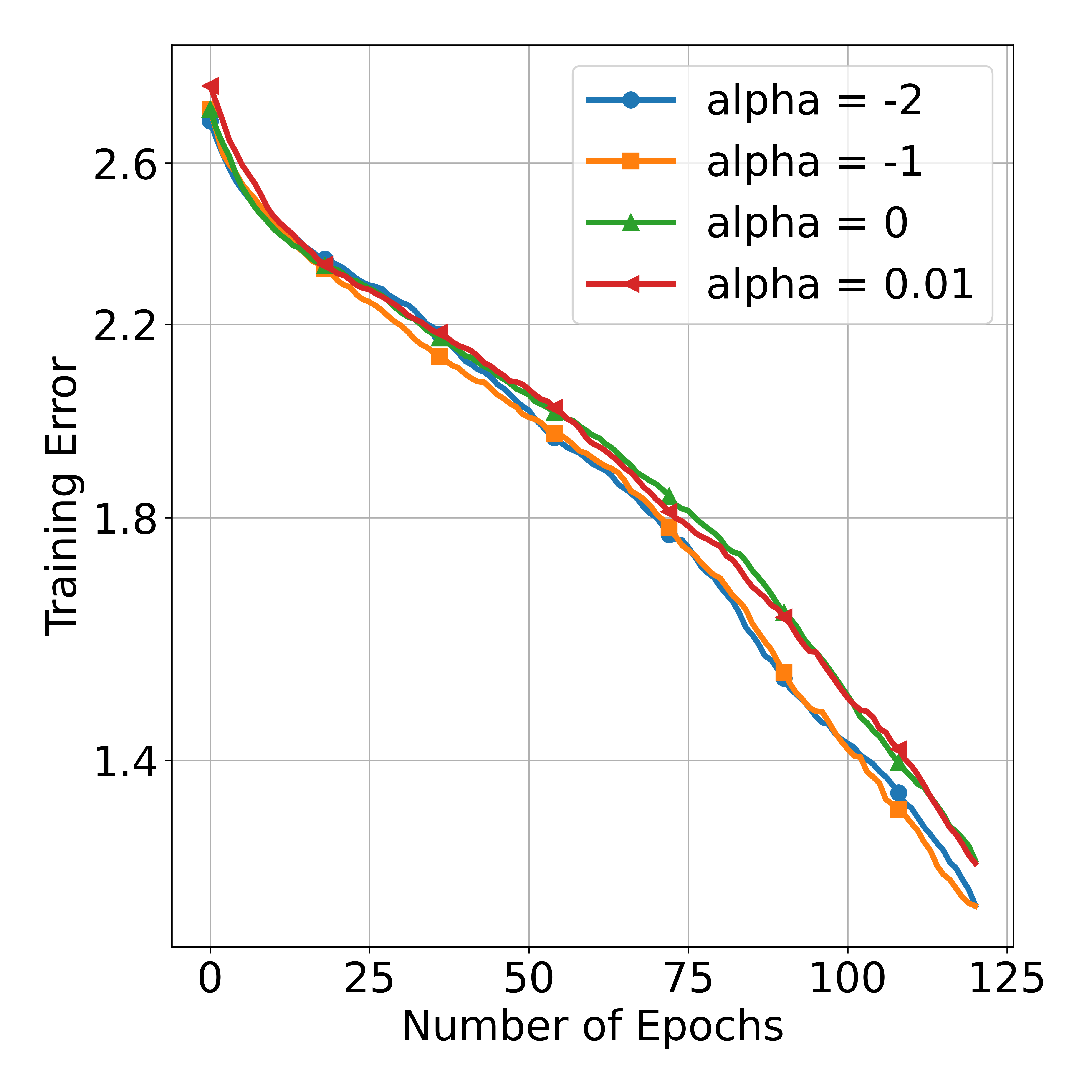}
\includegraphics[width=0.4\columnwidth]{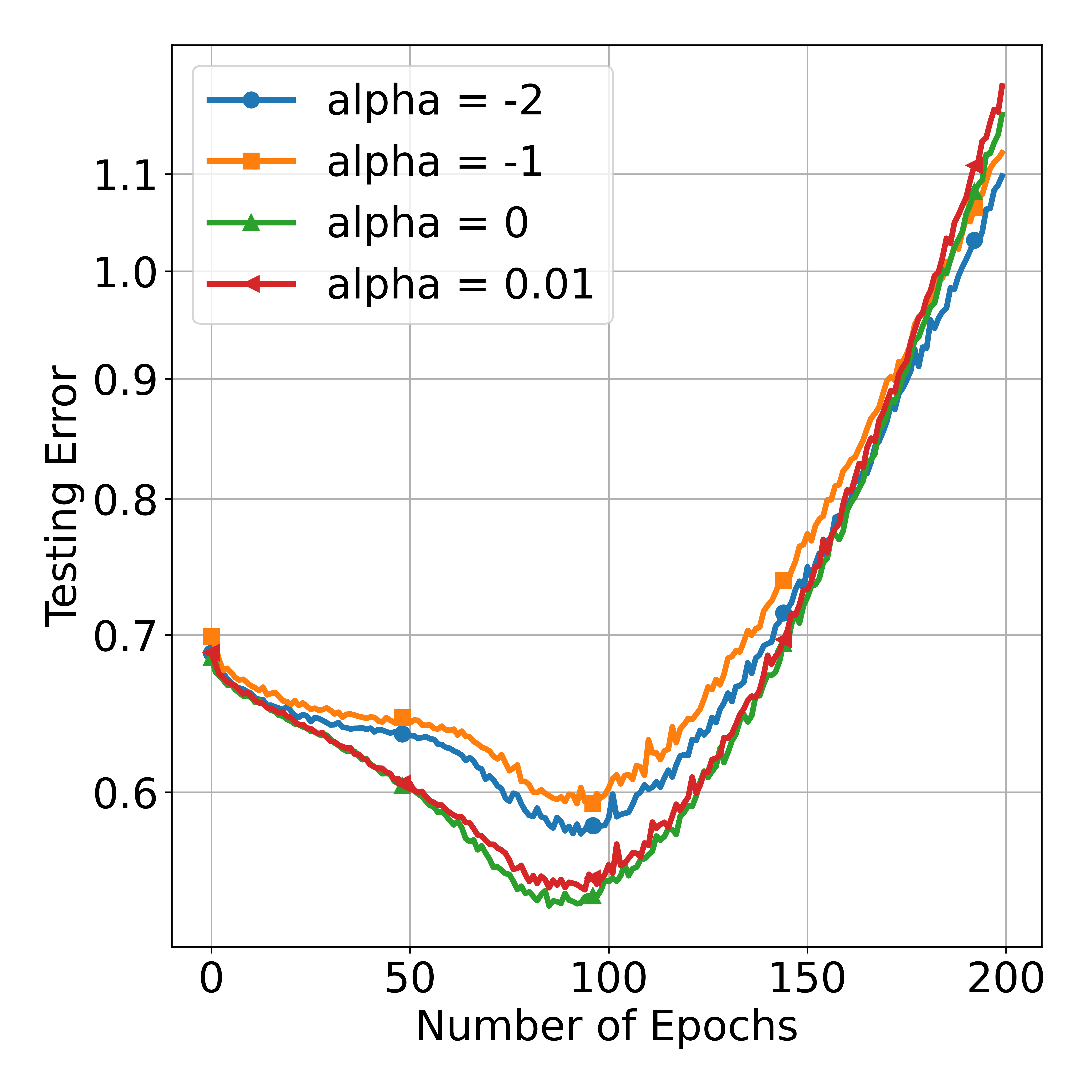}
\includegraphics[width=0.4\columnwidth]{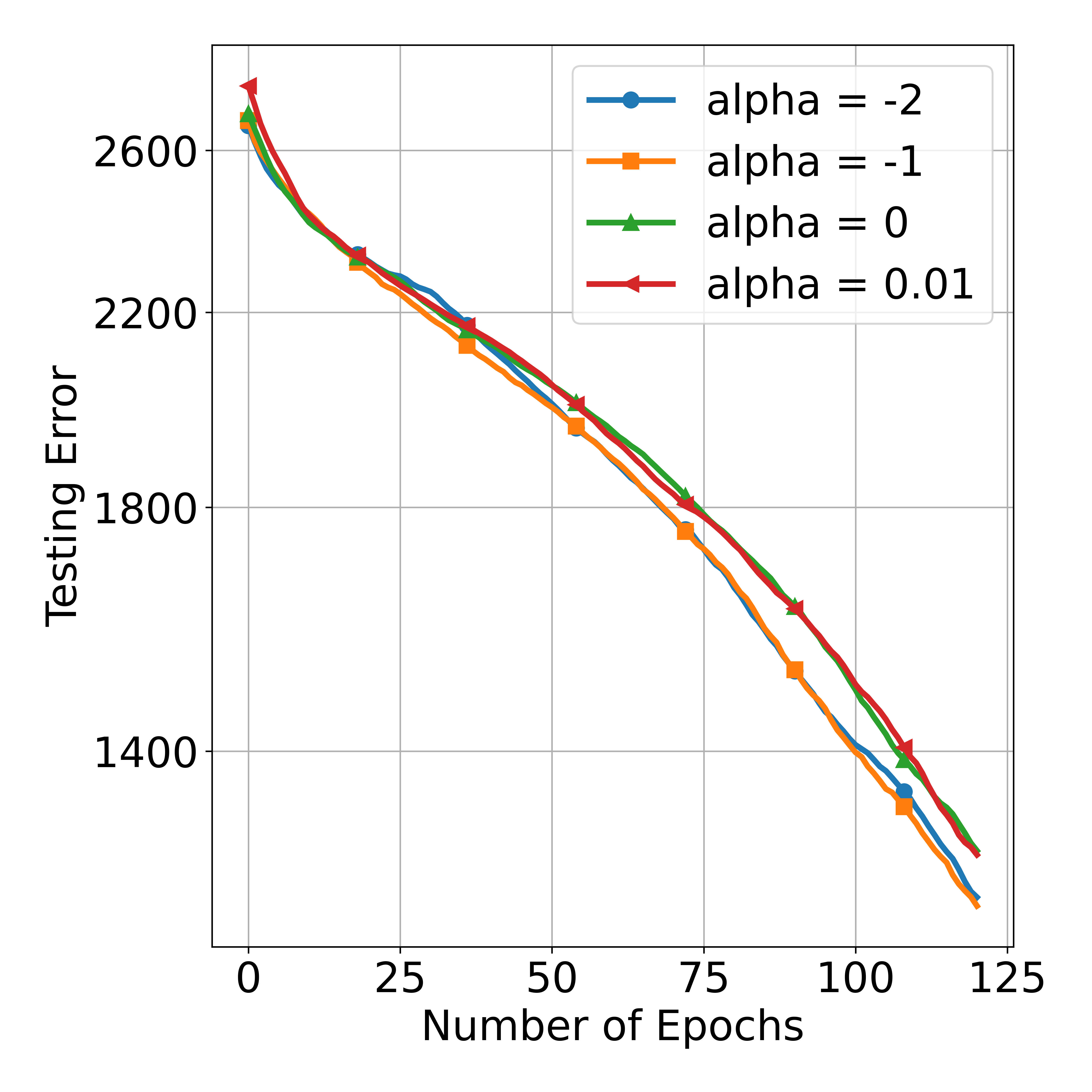}
\caption{Log-scale training and testing errors using different datasets and different $\alpha$'s. Left: binary entropy errors for IMDB; Right: negative log likelihood errors for Transformer on MNIST. Top row: training errors. Bottom row: testing errors.}
\label{fig:errors_lstm_tf}
\end{center}
\vskip -0.2in
\end{figure*}

\textbf{Dependence on $\boldsymbol{m}$ and $\boldsymbol{L}$:}
We demonstrate the dependence of the training error on $m$ and the testing error on $L$.
We thus ran additional experiments on F-MNIST with different choices of $L$ and $m$ and $\alpha\in\{-2,-1,0,0.01,0.05\}$. First, we fixed $L=2$  and tested $m\in\{1,000, 2,000, 5,000, 10,000\}$. Next,  we fixed $m=5,000$ and tested $L\in \{2, 3, 4\}$. The architectures of thes NNs were given in Table~\ref{tab:FM_nn_ar}, but with the latter choices of $m$ and $L$.

Figure~\ref{fig:L_m_dependence} shows the dependence of training errors on $m$ and the dependence of testing errors on $L$. We note that the training error is monotonically decreasing w.r.t.~the width $m$. We also note that the minimal training error is always achieved at $\alpha=-1$ for different choices of $m$. This matches our theoretical predictions. 
We note that the testing error is decreasing for $L \le 4$. This aligns with equation \eqref{eqn:generalization_error_sgd} when $t$ is small. Moreover, when $L=4$, we observe that the minimal testing error is achieved at $\alpha=-1$.

\begin{figure*}[ht]
\vskip -0.04in
\begin{center}
\includegraphics[width=0.4\columnwidth]{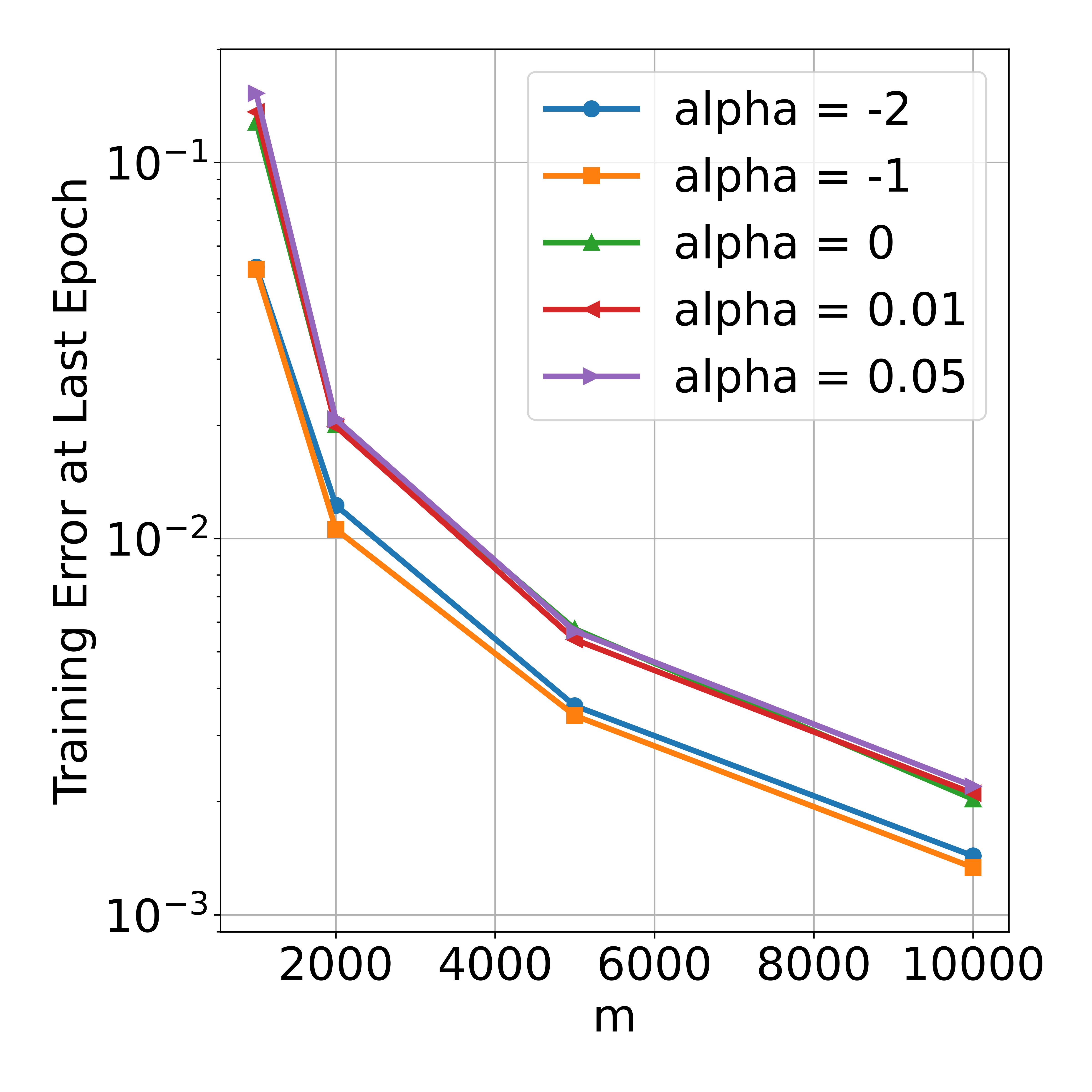}
\includegraphics[width=0.4\columnwidth]{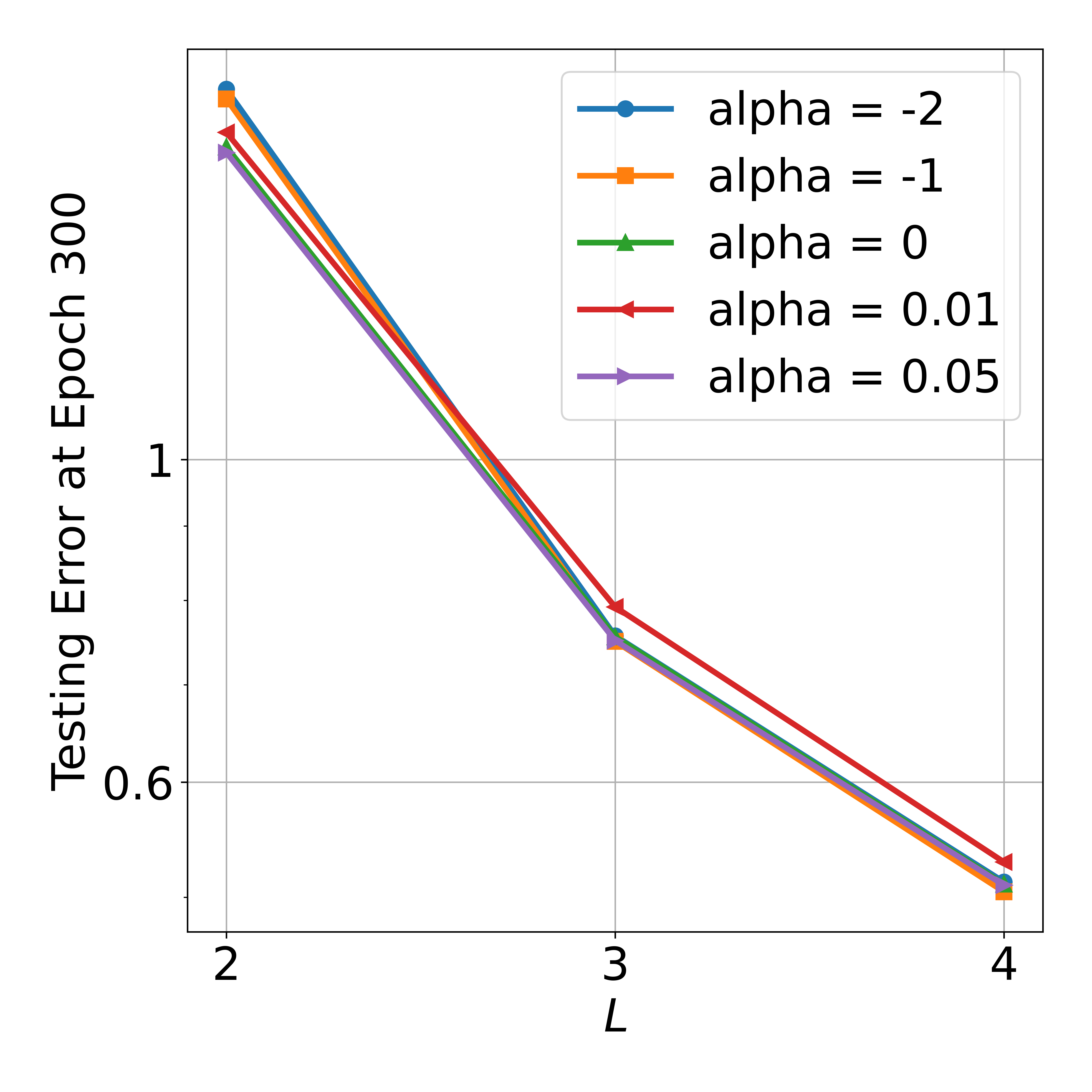}
\caption{Log-scale errors on F-MNIST with different $\alpha$'s. Left: training errors at the last epoch with $L=2$ and different widths ($m$s); Right: testing errors at the epoch $t=300$ with $m=5000$ and different depths ($L$s).}
\label{fig:L_m_dependence}
\end{center}
\end{figure*}

\textbf{Results using the loss function given in \eqref{eqn:exp_loss}:} We perform numerical experiments using the exponential loss function in the two datasets for the regression task, the synthetic dataset and the California housing dataset. Figure~\ref{fig:exp_loss} reports the results.  We observe that $\alpha=-1$ achieves both the optimal training error and the optimal generalization error in the synthetic dataset. 
Furthermore, $\alpha=-1$ achieves optimal training error and  second optimal generalization error for the California housing dataset, though the difference in the testing error is small.

\begin{figure*}[ht]
\vskip -0.04in
\begin{center}
\includegraphics[width=0.4\columnwidth]{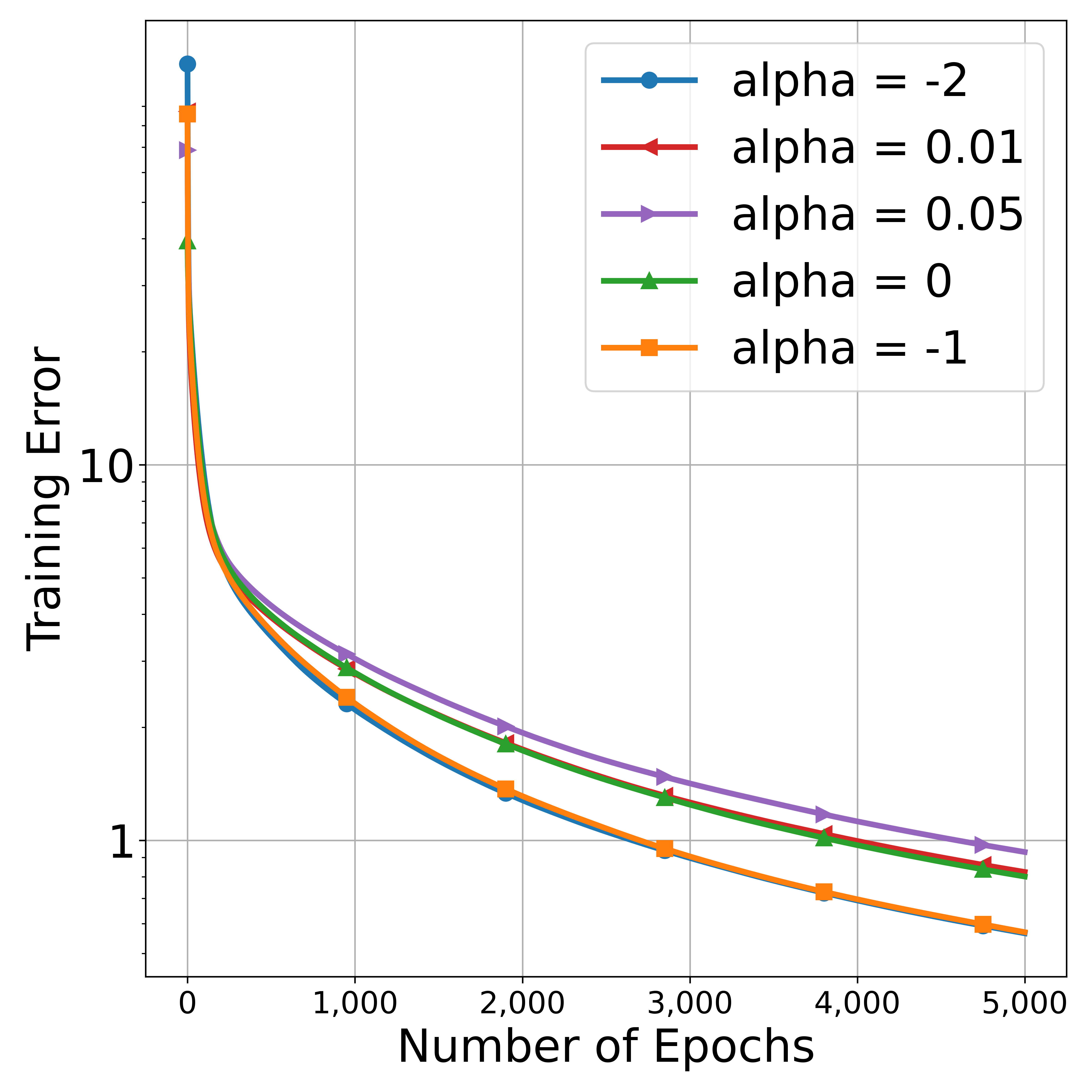}\includegraphics[width=0.4\columnwidth]{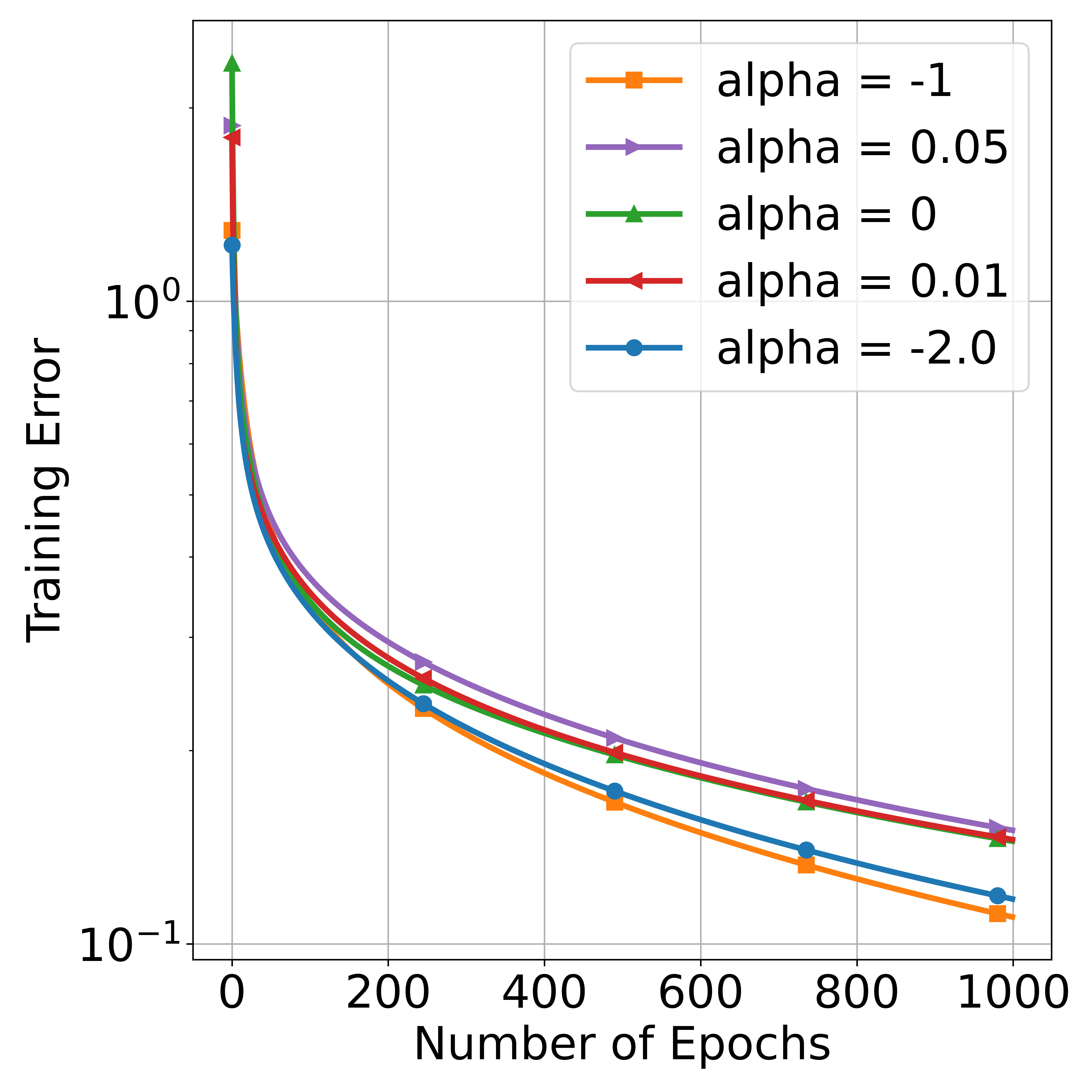}
\includegraphics[width=0.4\columnwidth]{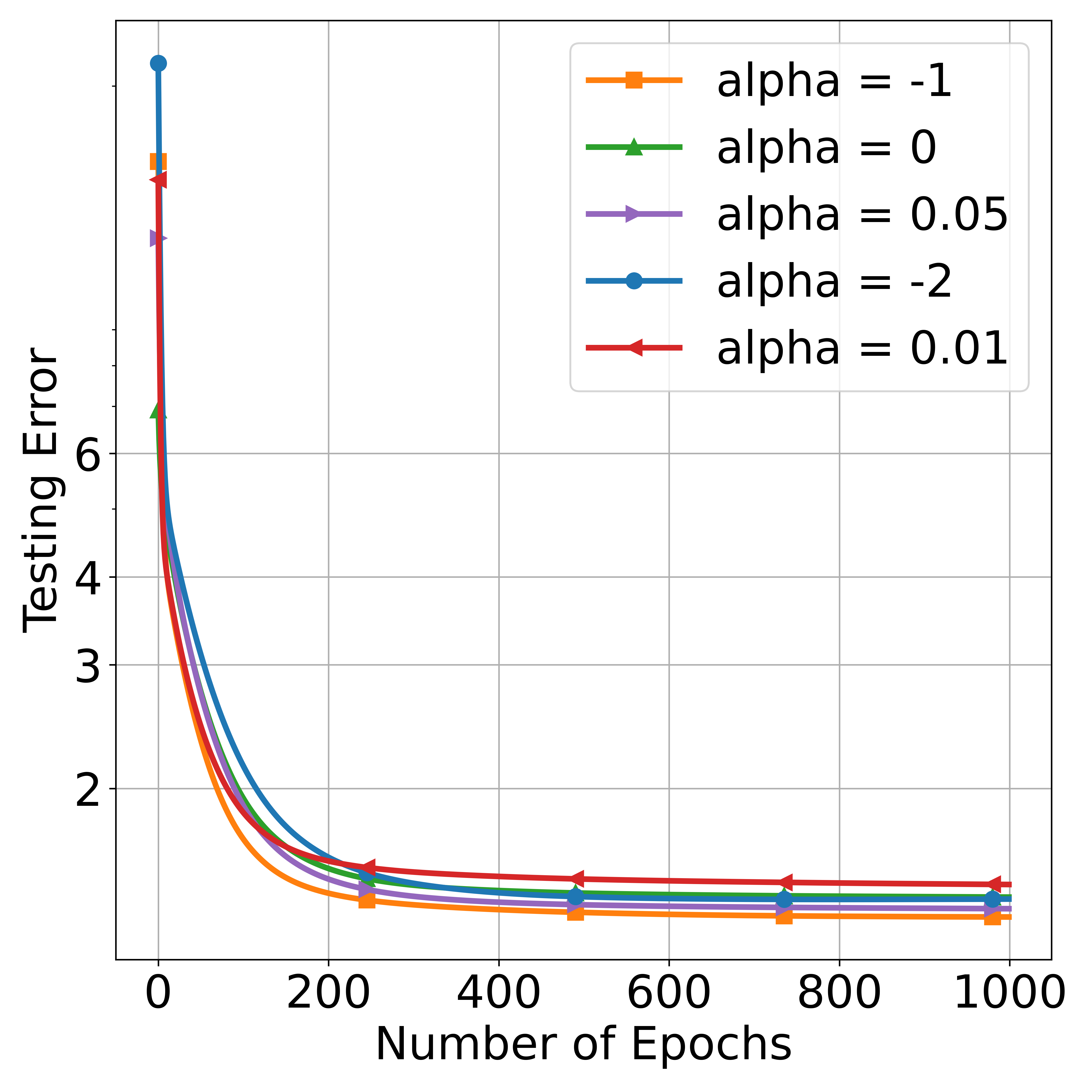}
\includegraphics[width=0.4\columnwidth]{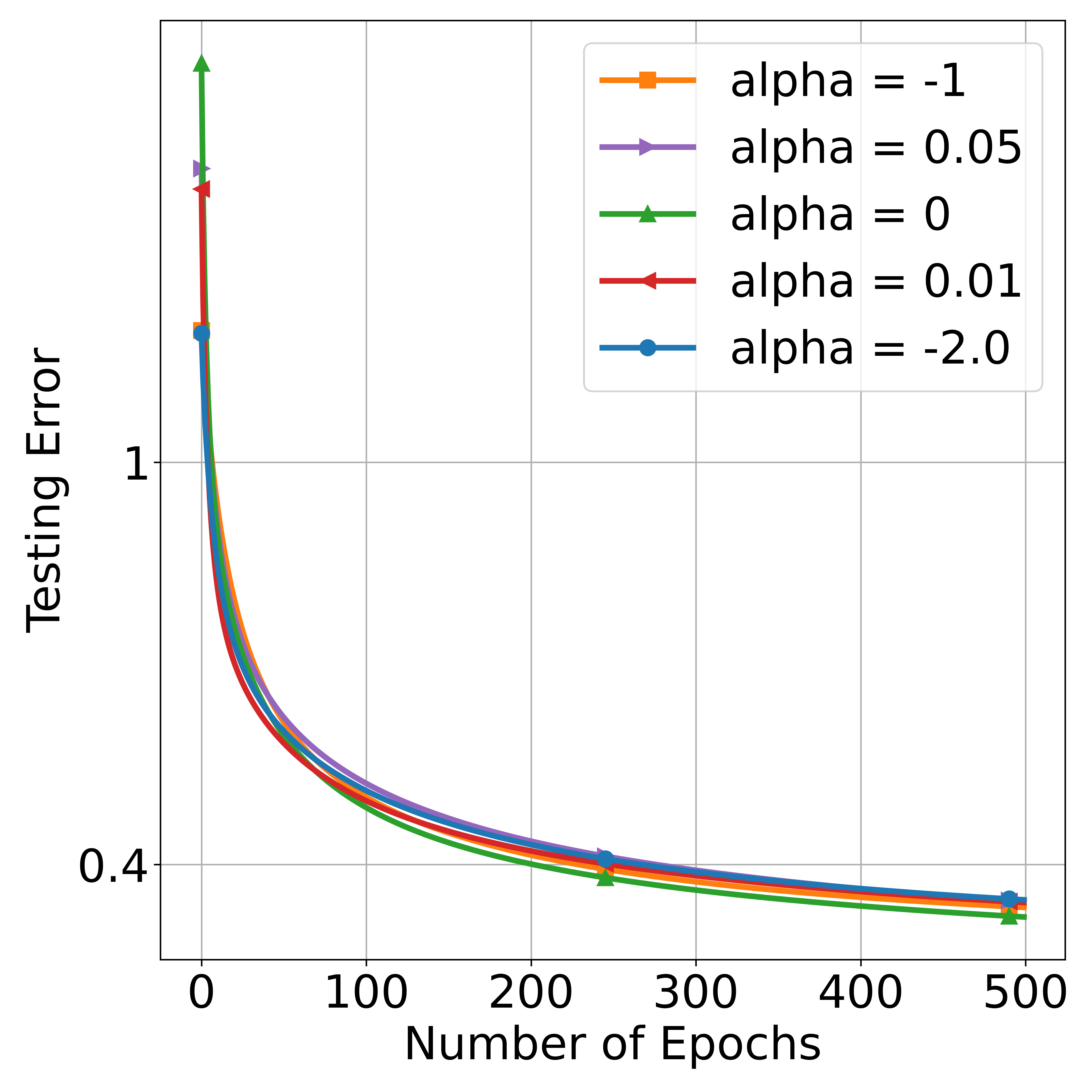}
\caption{Log-scaled mean square errors for the synthetic dataset with different $\alpha$'s using the loss function of \eqref{eqn:exp_loss}. Left: MSE for the synthetic dataset. Right: MSE for California housing dataset. Top row: training errors. Bottom row: testing errors.}
\label{fig:exp_loss}
\end{center}
\end{figure*}

\end{document}